%% file: main.tex
\documentclass[twoside,11pt]{article}

\usepackage[nohyperref, preprint]{jmlr2e}
\usepackage{colortbl}
\usepackage{multirow}
\RequirePackage{etoolbox}
\let\proof\relax

\RequirePackage{amsthm}
\RequirePackage{amsmath}
\RequirePackage{derivative}
\RequirePackage{mdframed}

\RequirePackage[pdftex,bookmarksnumbered,bookmarksopen,colorlinks,citecolor={blue!70!black},linkcolor={blue!70!black},urlcolor={blue!60!black}]{hyperref}

\RequirePackage{listings}
\RequirePackage{cleveref}
\RequirePackage{bm}
\RequirePackage{mleftright}
\RequirePackage{subcaption}
\RequirePackage{booktabs}
\RequirePackage[shortlabels]{enumitem}
\RequirePackage[linesnumbered, ruled]{algorithm2e}
\RequirePackage{xcolor}
\crefname{equation}{}{}
\Crefname{equation}{}{}

\usepackage{tikz}
\usetikzlibrary{calc, scopes, hobby}

\definecolor{commentcolor}{RGB}{110,154,155}   %

\NewDocumentCommand{\bp}{m}{\left(#1\right)}
\NewDocumentCommand{\rp}{m}{\mleft(#1\mright)}
\newcommand{\base}{o}
\newcommand{\bs}[1]{\left[#1\right]}
\newcommand{\bc}[1]{\left\{#1\right\}}

\newcommand{\R}{\mathbb{R}}
\newcommand{\N}{\mathbb{N}}
\newcommand{\Z}{\mathbb{Z}}

\newcommand{\op}{}

\newcommand{\casework}[1]{\begin{cases}#1\end{cases}}

\newcommand{\dNorm}{\mathcal{N}}
\newcommand{\dBin}{\operatorname{Bin}}

\theoremstyle{plain}
\newtheorem{theorem}{Theorem}
\newtheorem{lemma}[theorem]{Lemma}
\newtheorem{proposition}[theorem]{Proposition}
\newtheorem{approximation}[theorem]{Approximation}
\newtheorem{assumption}[theorem]{Assumption}

\theoremstyle{remark}

\newtheorem{remark}[theorem]{Remark}

\newcommand{\e}{\bm{e}}
\newcommand{\w}{\bm{w}}
\newcommand{\x}{\bm{x}}
\newcommand{\xh}{\widehat{\bm{x}}}
\newcommand{\y}{\bm{y}}
\newcommand{\z}{\bm{z}}

\newcommand{\eps}{\varepsilon}
\newcommand{\bR}{\mathbb{R}}

\newcommand{\M}{\bm{M}}

\newcommand{\I}{\bm{I}}
\newcommand{\W}{\bm{W}}
\newcommand{\vX}{\bm{X}}
\newcommand{\vXh}{\widehat{\bm{X}}}
\newcommand{\D}{\bm{D}}
\newcommand{\Q}{\bm{Q}}

\newcommand{\A}{\bm{A}}
\renewcommand{\P}{\bm{P}}

\renewcommand{\th}{\ensuremath{^\mathrm{th}}}

\NewDocumentCommand{\logdet}{m}{\operatorname{logdet}\rp{#1}}
\NewDocumentCommand{\mat}{m}{\begin{bmatrix}#1\end{bmatrix}}
\NewDocumentCommand{\softmax}{m}{\operatorname{softmax}\rp{#1}}
\NewDocumentCommand{\diaglr}{m}{\operatorname{diag}\rp{#1}}
\NewDocumentCommand{\explr}{m}{\operatorname{exp}\rp{#1}}

\NewDocumentCommand{\given}{}{\;\ifnumequal{\currentgrouptype}{16}{\middle|}{\mid}\;}

\usepackage{algpseudocode}
\algnewcommand\Hyperparameter{\item[\textbf{Hyperparameter:}]}%
\algnewcommand\Parameter{\item[\textbf{Parameter:}]}%

\usepackage{listings}
\definecolor{keywordcolor}{rgb}{0.57,0.15,0.79}
\definecolor{commentcolor}{rgb}{0.12,0.54,0.13}
\definecolor{stringcolor}{rgb}{0.63,0.12,0.11}
\definecolor{bgcolor}{rgb}{1,1,1}
\definecolor{numbercolor}{rgb}{0.5,0.5,0.5}

\newcommand{\head}{\mathrm{head}}
\newcommand{\pre}{\mathrm{pre}}
\newcommand{\post}{\mathrm{post}}
\newcommand{\pos}{\mathrm{pos}}
\newcommand{\cls}{\texttt{[CLS]}}

\newcommand{\ISTA}[1]{\operatorname{\texttt{ISTA}}(#1)}

\newcommand{\SSA}[1]{\operatorname{\texttt{SSA}}(#1)}
\newcommand{\MSSA}[1]{\operatorname{\texttt{MSSA}}(#1)}
\newcommand{\ours}{\textsc{crate}}
\newcommand{\ourscaps}{CRATE}

\newcommand\blfootnote[1]{%
	\begingroup
	\renewcommand\thefootnote{}\footnote{#1}%
	\addtocounter{footnote}{-1}%
	\endgroup
}

\input{macros.def}

\usepackage{pifont}
\newcommand{\xmark}{\ding{55}}

\colorlet{revision}{black}

\allowdisplaybreaks

\ShortHeadings{White-Box Transformers via Sparse Rate Reduction}{Yu, Buchanan, Pai, Chu, Wu, Tong, Bai, Zhai, Haeffele, Ma}
\firstpageno{1}

\begin{document}

\lstdefinestyle{mintedstyle}{
    morekeywords={self, class, def, return, import, from, lambda, with, as, None, True, False, pass, break, continue, in, is, not, and, or},
    backgroundcolor=\color{bgcolor},
    basicstyle=\small\ttfamily,
    keywordstyle=\color{blue},
    commentstyle=\color{green!50!black},
    stringstyle=\color{red!50!black},
    numberstyle=\tiny\color{gray},
    numbers=left,
    numbersep=10pt,
    showstringspaces=false,
    breaklines=true,
    frame=none,
    tabsize=4,
    captionpos=b,
}
\lstset{style=mintedstyle}

\title{White-Box Transformers via Sparse Rate Reduction: \\ Compression Is All There Is?}

\author{\name Yaodong Yu$^{\dagger, \star}$ \email yyu@eecs.berkeley.edu 
\AND
\name Sam Buchanan$^{\ddagger, \star}$ \email sam@ttic.edu 
\AND
\name Druv Pai$^{\dagger, \star}$ \email druvpai@berkeley.edu 
\AND
\name Tianzhe Chu$^{\dagger, \natural}$ \email chutzh@berkeley.edu 
\AND 
\name Ziyang Wu$^{\dagger}$ \email zywu@berkeley.edu 
\AND 
\name Shengbang Tong$^{\dagger}$ \email tsb@berkeley.edu 
\AND 
\name Hao Bai$^{\sharp}$ \email haob2@illinois.edu 
\AND 
\name Yuexiang Zhai$^{\dagger}$ \email simonzhai@berkeley.edu 
\AND 
\name Benjamin D. Haeffele$^{\flat}$ \email bhaeffele@jhu.edu 
\AND 
\name Yi Ma$^{\dagger, \diamondsuit}$ \email mayi@hku.hk, yima@eecs.berkeley.edu\,\,\\
\vspace{-0.1in}\\
\addr $^\dagger$ University of California, Berkeley\\
\addr $^\ddagger$ Toyota Technological Institute at Chicago\\
\addr$^\natural$ ShanghaiTech University\\
\addr$^\sharp$ University of Illinois, Urbana-Champaign\\
\addr$^\flat$ Johns Hopkins University\\
\hspace*{-0.7mm}\addr$^{\diamondsuit}$ 
 \hspace*{-0.7mm}University of Hong Kong
}

\maketitle

\vspace{-1.75em}

\blfootnote{\hspace{-3.5mm}$^ \star$\,\,\,Equal contribution.}

\begin{abstract}%
    In this paper, we contend that a natural objective of representation learning is to compress and transform the distribution of the data, say sets of tokens, towards a low-dimensional Gaussian mixture supported on incoherent subspaces. The goodness of such a representation can be evaluated by a principled measure, called \textit{sparse rate reduction}, that simultaneously maximizes the intrinsic information gain and extrinsic sparsity of the learned representation. From this perspective, popular deep network architectures, including transformers, can be viewed as realizing iterative schemes to optimize this measure. Particularly, we derive a transformer block from alternating optimization on parts of this objective: the multi-head self-attention operator compresses the representation by implementing an approximate gradient descent step on the coding rate of the features, and the subsequent multi-layer perceptron sparsifies the features. This leads to a family of \textit{white-box} transformer-like deep network architectures, named \ours{}, which are mathematically fully interpretable. We show, by way of a novel connection between denoising and compression, that the inverse to the aforementioned compressive encoding can be realized by the same class of \ours{} architectures. Thus, the so-derived white-box architectures are universal to both encoders and decoders. Experiments show that these networks, despite their simplicity, indeed learn to compress and sparsify representations of large-scale real-world image and text datasets,
    {\color{revision} and achieve strong performance across different settings: ViT, MAE, DINO, BERT, and GPT2.}
    We believe the proposed computational framework demonstrates great potential in bridging the gap between theory and practice of deep learning, from a unified perspective of data compression. Code is available at:  \url{https://ma-lab-berkeley.github.io/CRATE}.
\end{abstract}

\newpage
\tableofcontents

\newpage

\input{sec_intro.tex}

\input{sec_encoding.tex}

\input{sec_autoencoding.tex}

\input{sec_exp.tex}

\input{sec_conclusion.tex}

\acks{Yaodong Yu would like to thank Kwan Ho Ryan Chan for the valuable discussions they  had regarding visualizing tokens in vision transformers. Yaodong Yu and Yi Ma acknowledge support from the joint Simons Foundation-NSF DMS grant \#2031899, the ONR grant N00014-22-1-2102, and Yi Ma also acknowledges partial support from TBSI, InnoHK, and the University of Hong Kong. Other members of the team were partially supported by NSF 1704458, the Northrop Grumman Mission Systems Research in Applications for Learning Machines (REALM) initiative, NIH NIA 1R01AG067396, and ARO MURI W911NF-17-1-0304.}

\newpage

\appendix

\begin{center}
    \LARGE \textbf{Appendix}
\end{center}

\section{Technical Details for  \Cref{sec:encoding}}
\label{app:section-2}

\input{app_technical_compression.tex}

\input{app_technical_mlp.tex}

\section{Technical Details for \Cref{sec:autoencoding}}

\input{app_technical_diffusion.tex}

\input{app_technical_tweedie.tex}

\input{app_technical_compression_denoising.tex}

\input{app_technical_RR_score.tex}

\newpage
\section{Additional Implementation Details and Experimental Results}\label{sec:appendix-exp}

\input{app_experiment.tex}

\input{app_segmentation.tex}

\input{app_torchcode}

\clearpage

\bibliography{reference.bib}

\end{document}

%% file: sec_intro.tex
\section{Introduction}\label{sec:intro}

\subsection{The Representation Learning Problem}
In recent years, deep learning has seen tremendous empirical success in processing
and modeling massive amounts of high-dimensional and multi-modal data \citep{krizhevsky2009learning,He2016-lc,Radford2021-ir,Chen2020-ha,He2021-lb}. As argued by  \cite{ma2022principles}, much of this success is owed to deep networks' ability in effectively learning compressible low-dimensional structures in the data distribution and then transforming
the distribution to a parsimonious, i.e. \textit{compact and structured},
representation. 
Such a representation then facilitates many downstream tasks, e.g., in vision, classification
\citep{He2016-lc,dosovitskiy2020image}, recognition and segmentation
\citep{Carion2020-fm,He2017-be,Kirillov2023-pm}, and generation
\citep{Karras2018-si,rombach2022high,Saharia2022-na}. 

\paragraph{Representation learning via compressive encoding and decoding.} To state the common problem behind all these practices more formally, one may view a given dataset as samples of a random vector $\x$ in a high-dimensional space, say $\mathbb{R}^D$. Typically, the distribution of $\x$ has much lower intrinsic dimension than the ambient space. Generally speaking, by {\em learning a representation}, we typically mean to learn a continuous mapping, say $f(\cdot)$, that transforms $\x$ to a so-called {\em feature vector} $\z$ in another (typically lower-dimensional) space, say $\mathbb{R}^d$. It is hopeful that through such a mapping:
\begin{equation}
     \x \in \mathbb{R}^D \xrightarrow{\hspace{2mm} f(\x)\hspace{2mm}} \z  \in \mathbb{R}^d, 
     \label{eqn:encoding}
\end{equation}
the low-dimensional intrinsic structures of $\x$ are identified and represented by $\z$ in a more compact and structured way so as to facilitate subsequent tasks such as classification or generation. The feature $\z$ can be viewed as a (learned) compact code for the original data $\x$, so the mapping $f$ is also called an \textit{encoder}. 
The fundamental question of representation learning, then, and a central problem that we will address in this work, is:
\begin{quotation}
\noindent{\em What is a principled and effective measure for the goodness of representations?}
\end{quotation}

Conceptually, the quality of a representation $\z$ depends on how well it identifies the most relevant and sufficient information of $\x$ for subsequent tasks, and how efficiently it represents this information.
For long it was believed and argued that ``sufficiency'' or ``goodness'' of a learned feature should be defined in terms of a specific task. For example, $\z$ just needs to be sufficient for predicting a class label $\y$ in a classification problem. To understand the role of deep learning or deep networks in this type of representation learning, \cite{Tishby-ITW2015} proposed the \textit{information bottleneck} framework, which suggests that a measure of feature goodness is to maximize the mutual information between $\z$ and $\y$ while minimizing the mutual information between $\z$ and $\x$. 

Nevertheless, in recent years the predominant practice has been to learn first a \textit{task-agnostic} representation by pre-training a large deep neural network, in some cases known as a \textit{foundation model} \citep{Bommasani2021-vm}. The so-learned representation can subsequently be fine-tuned for multiple specific tasks. 
This has been shown to be more effective and efficient for many practical tasks across diverse data modalities, including speech \citep{Radford2022-en}, language \citep{brown2020language}, and natural images \citep{Oquab2023-ff}. 
Notice that representation learning in this context is very different from that for a specific task, where $\z$ only needs to be good enough for predicting a specific $\y$. In a task-agnostic setting, the learned representation $\z$ needs to encode \textit{almost all essential information about the distribution of the data $\x$}. That is, the learned representation $\z$ not only is a more compact and structured representation for the intrinsic structures of $\x$, but can also recover $\x$ to a certain degree of faithfulness. 
Hence, it is natural to ask, in the task-agnostic context, what a principled measure of goodness  for a learned (feature) representation should be.\footnote{%
As we know, in recent practice of learning task-agnostic representations, one type of deep architectures, known as transformers \citep{vaswani2017attention}, have emerged as an almost universal choice for the backbone of deep networks, for either discriminative or generative tasks, from language to vision. We will review the details of this architecture momentarily. As we will see in this work, clarifying the principled measure for feature goodness is also the key to fully understand why a transformer-like architecture is suitable for task-agnostic pretraining, as well as to reveal the precise role and function of each layer in transformer-like deep networks.}

Conceptually, we argue that one effective way, perhaps the only way, to verify whether a representation $\z$ has encoded sufficient information about $\x$ is to see how well we can recover $\x$ from $\z$ through an (inverse) mapping, say $g$, known as a {\em decoder} (or a  generator): 
\begin{equation}
     \x \in \mathbb{R}^D \xrightarrow{\hspace{2mm} f(\x)\hspace{2mm}} \z  \in \mathbb{R}^d\xrightarrow{\hspace{2mm} g(\z) \hspace{2mm}} \widehat{\x} \in \mathbb{R}^D.
     \label{eqn:generative}
\end{equation}
As the encoder $f$ is typically compressive and lossy, we should not expect the inverse mapping to recover $\x$ exactly, but an approximate $\widehat{\x} = g\circ f(\x) \approx \x$. We normally seek optimal encoding and decoding mappings such that the decoded $\widehat{\x}$ is the closest to $\x$, either sample-wise---say, by minimizing the expected mean squared error---or in a relaxed distributional sense. We refer to the above process as {\em compressive encoding and decoding} or {\em compressive autoencoding}.  
This idea is highly compatible with the original goals laid out for autoencoders by \cite{Kramer1991NonlinearPC,Hinton-NIPS1993}, which can be viewed as a generalization of the classic principal component analysis \citep{Jolliffe2002} for the case where the low-dimensional structure of $\x$ is linear. 

Through tremendous empirical efforts over the last eleven years, it has become clear that deep networks %
are very effective in modeling nonlinear encoding and decoding mappings. 
{\color{revision}Many} applications of deep learning, including those mentioned above, rely on realizing such an encoding or decoding scheme partially or entirely by learning $f$ or $g$ separately or together. 
Although,  conceptually, the decoder $g$ should be the ``inverse'' to the encoder $f$, in practice it has never been clear how the architectures of encoder and decoder should be related to each other. 
In many cases, the architectural design of the decoder has little to do with that of the encoder, often chosen via empirical tests and ablations {\color{revision}(for instance, in masked autoencoders \citep{he2022masked} and latent diffusion models \citep{Esser2020-do,rombach2022high})}.
\textit{We believe a good theoretical framework for representation learning should clearly reveal relationships between architectures for the encoder and the decoder.} We strive to achieve this level of clarity in this work. 

\subsection{Review of Existing Approaches}
\paragraph{Opening the black-box of modern deep networks through compression.} Along the development of deep learning, many deep network architectures have been proposed and practiced for $f$ or $g$, from the classic LeNet~\citep{lecun1998gradient} to AlexNet~\citep{Krizhevsky2012-um}, to ResNet~\citep{He2016-lc} and then to the more recent transformer~\citep{vaswani2017attention}. Despite their popularity, these networks have largely been designed empirically and trained and used as ``black-box'' function approximators. As a result, desired properties of the learned feature representation $\z$ are not clearly specified or justified, and 
many heuristic measures or loss functions have been proposed and practiced for training task-agnostic representations with these models.

\begin{figure}
    \centering
    \includegraphics[width=\textwidth]{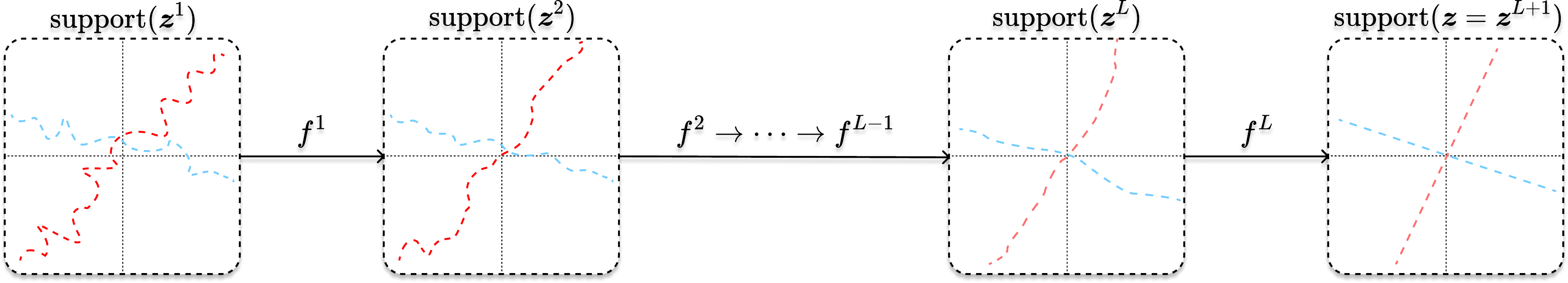}
    \caption{\textbf{Deep network layers \(f^{\ell}\) which optimize the rate reduction.} The separate components of the data distribution are transformed by the network operators to a configuration which maximizes the information gain. Here, \(f\) may be realized by a ReduNet \citep{chan2021redunet}, in which each layer implements a gradient descent iteration for optimizing the rate reduction.}
    \label{fig:mcr2_maximization}
\end{figure}

The recent work of~\citet{OriginalMCR2,chan2021redunet} has attempted to provide a principled framework that interprets the deep architectures of the ResNet and CNNs from the perspective of optimizing a measure of ``information gain'' for the learned representation. 
When the structured representation sought is a mixture of low-dimensional Gaussians, the information gain can be precisely measured by the so-called coding \textit{rate reduction}, denoted as $\Delta R(\z)$, and defined as the difference between the coding rates for \textit{the feature set as a whole} and \textit{the coding rate for its structured components}. 
It was shown that one can derive from this objective a deep network architecture, known as the ReduNet \citep{OriginalMCR2,chan2021redunet}, that shares a striking resemblance to ResNets and CNNs. 
The layers of a ReduNet are fully interpretable as realizing an iterative gradient descent method for optimizing the coding rate reduction objective $\Delta R(\z)$, as in \Cref{fig:mcr2_maximization}:
\begin{equation}
f\colon \vx
\xrightarrow{\hspace{1mm} f^{\pre} \hspace{1mm}} \vz^{1} \rightarrow \cdots \rightarrow \vz^\ell \xrightarrow{\hspace{1mm} f^\ell \hspace{1mm}} \vz^{\ell+1} \rightarrow  \cdots \xrightarrow{\hspace{1mm} f^{L}\hspace{1mm}} \vz^{L + 1} = \vz,
\label{eq:incremental-encoding}
\end{equation}
where \(f^{\pre}\) is some data pre-processing map, and
\begin{equation}
\vz^{\ell+1} = f^\ell(\vz^\ell) \approx \vz^\ell 
+ \eta\nabla\big[\Delta R(\z^{\ell})\big] %
\end{equation}
i.e., each layer $\ell$ is constructed to incrementally optimize the $\Delta R(\z^{\ell})$ by taking an approximate gradient ascent step with step size $\eta$. We will refer to such a mathematically interpretable network as a ``white-box'' deep network {in the sense that the motivation and structure of each network layer is well understood (i.e., as approximating an incremental improvement of some desired objective function)}. Although rate reduction offers a good theoretical framework for understanding architectures of existing deep networks such as ResNets and CNNs, direct implementations of ReduNet have not yet generated competitive practical performance on real-world datasets and tasks at scale. \textit{In this work, we will see how this outstanding gap between theory and practice\footnote{The gap between theory and practice is not just characteristic of the rate reduction framework. The situation is as dire for all theoretical frameworks ever proposed for understanding deep networks.} can be bridged through a generalization and improvement to the rate reduction objective such that its gradient descent operator resembles the structure of a transformer layer, in such a way that the resulting trasnformer-like architecture achieves competitive empirical performance.}

\paragraph{Transformer models and compression.} In recent years, transformers \citep{vaswani2017attention} have emerged as the most popular, nearly universal, model of choice for the encoder $f$ and decoder $g$ in learning representations for high-dimensional structured data, such as text \citep{vaswani2017attention,devlin2018bert,brown2020language}, images \citep{dosovitskiy2020image,dehghani2023scaling}, and other types of signals \citep{gong2022contrastive,arnab2021vivit}. 
In a nutshell, a transformer first converts each data point (such as a text corpus or image) into a set or sequence of \textit{tokens}, and then performs further processing on the token sets, in a medium-agnostic manner \citep{vaswani2017attention,dosovitskiy2020image}. 
A cornerstone of the transformer model is the so-called \textit{(self-)attention layer}, which exploits the statistical correlations among the sequence of tokens to refine the token representation. Yet the transformer network architecture is empirically designed and lacks a rigorous mathematical interpretation. In fact, the output of the attention layer itself has several competing interpretations \citep{vidal2022attention,li2023theoretical, DBLP:conf/aistats/SanderABP22, DBLP:journals/corr/abs-2312-10794}. 
As a result, the statistical and geometric relationship between the data $\x$ and the final representation $\z$ learned by a transformer largely remains a mysterious black box. 

Nevertheless, in practice, transformers have been highly successful in learning compact representations that perform well on many downstream tasks. 
In particular, it serves as the backbone architecture for the celebrated large language models (LLMs) such as OpenAI's GPT-4~\citep{openai2023gpt4}. Although the precise reason why it works well remains unclear, it has been hypothesized by OpenAI's researchers from a heuristic standpoint that the transformer architecture in LLMs implicitly minimizes the Kolmogorov complexity of the representations \citep{Sutskever2023Observation}, a quantitative notion of compression measured by the length of the code that can generate the data in consideration. 
However, we know that Kolmogorov complexity is largely a theoretical  concept and in general not computationally tractable for high-dimensional distributions. 
Hence, if transformers in LLMs indeed conduct compression, they should be based on a measure of complexity that is amenable to tractable and efficient computation. The design of Helmholtz machines (and Boltzman machines) based on the \textit{minimum description length principle} can be viewed as early attempts to make compression computable \citep{Hinton-NIPS1993}. \textit{In this work, we argue that a natural choice of this computable measure of compression behind transformers is precisely a combination of rate reduction and sparsity of the learned representations.} As we will see, revealing such a measure could be the key to understand the transformer architecture.

\paragraph{Denoising-diffusion models and compression.} Diffusion models
\citep{Sohl-Dickstein2015-kz,ho2020denoising,Song2019-ww,Song2020-xo,Song2020-hb}
have recently become a popular method for learning high-dimensional data distributions, particularly of natural images, which are known to be highly
structured in a manner that is notoriously difficult to model mathematically 
\citep{Ruderman1994-he,wakin2005multiscale,Donoho2005-ag}. The core concept of
diffusion models is to start with features $\z$ sampled from a Gaussian noise
distribution (or some other standard template) and \textit{denoise
and deform} the feature distribution until it converges to the original data distribution, which often has low intrinsic dimension. This process is computationally intractable if modeled {\color{revision} at just a single scale of noise \citep{Koehler2022-ed,chen2023-uo,Bovier2005-yo,Qin2023-wn}}, so it is typically broken into multiple incremental steps that 
denoise iteratively, as in \Cref{fig:diffusion_pipeline}: %
\begin{equation}
g\colon \vz = \wt{\vz}^{0}
\rightarrow \wt{\vz}^{1} \rightarrow \cdots \rightarrow \wt{\vz}^\ell \xrightarrow{\hspace{1mm} g^\ell \hspace{1mm}} \wt{\vz}^{\ell+1} \rightarrow  \cdots \rightarrow \wt{\vz}^{L} \xrightarrow{\hspace{1mm} g^{\post} \hspace{1mm}} \wh{\vx},
\label{eq:incremental-generation}
\end{equation}
where \(g^{\post}\) is a data post-processing map, and
\begin{equation}
\wt{\vz}^{\ell+1} = g^\ell(\wt{\vz}^\ell) = \wt{\vz}^\ell 
+ \tau\nabla \log q^{\ell}(\wt{\vz}^{\ell}), 
\end{equation} 
where $q^{\ell}$ is the density of \(\wt{\vz}^{\ell}\), i.e., the density of $\wt{\vz}^{L}$ after corruption with the $\ell$-th scale of Gaussian noise,
and $\nabla \log q^{\ell}$ is the so-called \textit{score
function} \citep{Hyvarinen2005-fi},  or equivalently  an estimate for the ``optimal denoising function'' for $q^{\ell}$ \citep{Efron2011-wn}. In practice, the score function is modeled using a generic black-box deep network.\footnote{The score function $\nabla \log q^{\ell}$ between two layers is typically learned by fitting relationships between $\wt{\vz}^\ell$ and $\wt{\vz}^{\ell+1}$, the data distribution at successive scales of corruption by Gaussian noise, from a large number of samples with a black-box deep network designed for denoising.} Diffusion models have shown effectiveness at learning and sampling from the data distribution \citep{karras2022elucidating,chen2022improved,rombach2022high}. However, despite
some recent efforts \citep{song2023consistency}, they generally do not establish any clear correspondence between the initial features and data samples. Hence, diffusion models themselves do not offer a parsimonious or interpretable
representation of the data distribution. Yet, conceptually, the above iterative denoising process \eqref{eq:incremental-generation} is compressing the feature distribution onto a targeted low-dimensional data distribution. \textit{In this work, we will show that if one were to compress and transform a distribution onto a standard mixture of (low-dimensional) Gaussians, the associated optimal denoising function takes an explicit form that is similar to the gradient of the rate reduction and to a transformer layer.} 
This provides a path to take a transformer-like encoder $f$ designed to compress the data distribution into a parsimonious and structured
representation, and derive its distributional inverse through a process analogous to \Cref{eq:incremental-generation}, yielding a white-box architecture for compressive autoencoding.

\begin{figure}
    \centering
    \includegraphics[width=\textwidth]{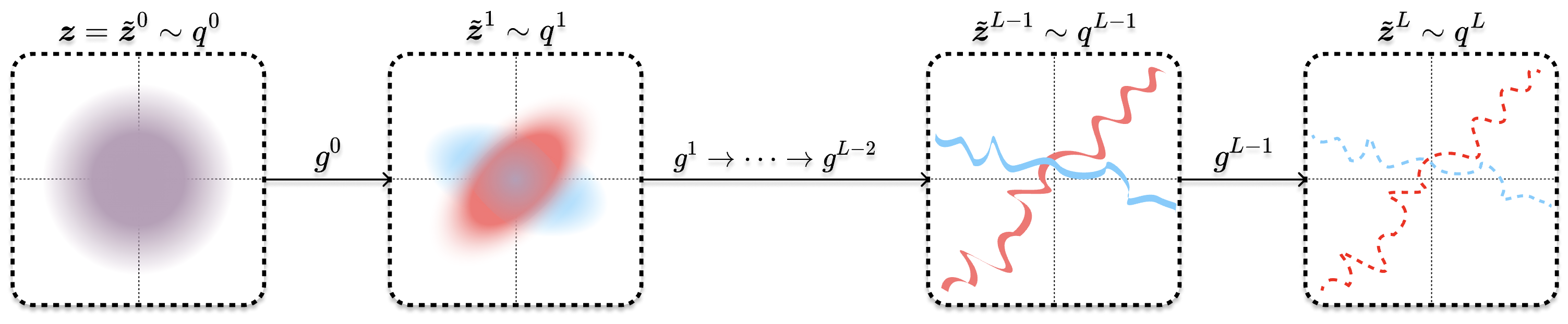}
    \caption{\small\textbf{Distribution flow in denoising-diffusion models.} Starting with generic noise \(\vz = \wt{\vz}^{0}\), the probability density of intermediate iterates is shaped towards the true distribution of \(\wt{\vz}^{L}\) locally and iteratively through the operators \(g^{\ell}\), which use the score function \(\nabla \log q^{\ell}\) at each layer \(\ell\).}
    \label{fig:diffusion_pipeline}
\end{figure}

\paragraph{Low-dimensionality promoting measures: sparsity and rate reduction.} In both of the previous popular methods, transformers and denoising-diffusion  models, a representation was learned implicitly as a byproduct of solving a downstream task (e.g., classification or generation/sampling)
using deep networks. The networks used are typically  chosen empirically. Therefore, it is difficult to rigorously ensure or impose any desired properties for the learned representation, except by trial and error. However, complementary to these popular empirical practices, a line of research has attempted to explicitly learn a desired representation of
the data distribution as a task in and of itself; this is most commonly done by
trying to explicitly identify and represent low-dimensional structures in the input data.
Classical examples of this paradigm include \textit{model-based} approaches such as
sparse coding \citep{olshausen1997sparse,chen2018sparse} and dictionary learning \citep{aharon2006k,Spielman2012-le,Gribonval2014-zr,zhai2020complete}, out of which grew
early attempts at designing and interpreting deep network architectures as learning a sparse representation 
\citep{Papyan2018-qc,Bruna2013-on}. More recent approaches build instead from a \textit{model-free}  perspective, where one learns a representation through a
sufficiently-informative pretext task such as compressing similar and
separating dissimilar data via contrastive learning \citep{tian2020makes,wang2022rethinking,bardes2022vicreg,shwartz2023compress}. Compared to black-box deep learning approaches, both model-based and model-free representation learning schemes have the advantage of being more interpretable: they allow users to explicitly design desired properties of the learned representation $\z$. To a large extent, the rate reduction framework \citep{OriginalMCR2,chan2021redunet,pai2022pursuit} strikes a good balance between the above model-based and model-free methods. Like contrastive learning, it aims to identify the data distribution by compressing similar/correlated data and separating dissimilar/uncorrelated data \citep{OriginalMCR2}. Meanwhile, like the model-based methods, it actively maps the data distribution to a family of desired representations, say a mixture of low-dimensional Gaussians  
\citep{ma2007segmentation,Vidal:Springer16}. 

\paragraph{Unrolled optimization: a unified paradigm for network interpretation \& design.} As we have discussed above, low-dimensionalty promoting  measures, such as sparsity or coding rate reduction, allow users to construct  white-box deep network
architectures
\citep{gregor2010learning,chan2021redunet} in a forward-construction fashion by
\textit{unrolling an optimization strategy for the chosen 
objective of the representations}, such that each layer of the constructed network implements an
iteration of the optimization algorithm
\citep{gregor2010learning,chan2021redunet,tolooshams2021stable}. In his recent work, \cite{hinton2022forwardforward} has also begun to hypothesize that the role of a deep network, with its forward pass, is likely to optimize certain feature goodness layer-wise. In this paradigm, the most challenging question is:
\begin{quotation}
\noindent\textit{What fundamental measure of goodness for the representations is a deep network trying to optimize in its forward pass?}
\end{quotation}
In the unrolled optimization paradigm, if the desired objectives are narrowly defined, say promoting sparsity alone \citep{Papyan2018-qc,Bruna2013-on}, it has so far proved difficult to arrive at network architectures that can achieve competitive practical performance on large real-world datasets. 
Other work has attempted to derive empirically-designed popular network architectures through unrolled optimization on a reverse-engineered learning objective for the representation, such as \citet{Yang2022-yi,Hoover2023-lc,De_Weerdt2023-ej}. In this case, the performance of the networks may remain intact, but the reverse-engineered representation learning objective is usually highly complex and not interpretable, and the properties of the optimal representation---or indeed the actually-learned representation---remain opaque. Such approaches do not retain the key desired benefits of unrolled optimization. 
{\em As we will argue in this work, to measure the goodness of a learned representation in terms of its intrinsic compactness and extrinsic simplicity, it is crucial to combine the measure of sparsity \citep{Papyan2018-qc,Bruna2013-on} and that of coding rate reduction  \citep{OriginalMCR2,chan2021redunet}}. 
As we will see, this combination will largely resolve the aforementioned limitations of extant methods that rely solely on sparsity or solely on rate reduction.

\subsection{Goals and Contributions of This Work} 
From the above discussion, we can observe that there has been an outstanding wide gap between the practice and theory of representation learning via deep networks. The fast advancement in the practice of deep learning has been primarily driven by empirical black-box models and methods that lack clear mathematical interpretations or rigorous guarantees. Yet almost all existing theoretical frameworks have only attempted to address limited or isolated aspects of practice, or only proposed and studied idealistic models that fall far short of producing practical performance that can compete with their empirical counterparts.  

\paragraph{Bridging the gap between theory and practice.} Therefore, the primary goal of this work is to remedy this situation with a more complete and unifying framework that has shown great promise in bridging this gap between theory and practice. On one hand, this new framework is able to provide a unified understanding of the many seemingly disparate approaches and methods based on deep networks, including compressive encoding/decoding (or autoencoding), rate reduction, and denoising-diffusion. {\color{revision}On the other hand, as we will see, this framework can guide us to derive or design deep network architectures that are not only mathematically fully interpretable but also obtain competitive performance on many learning tasks on large-scale real-world image or text datasets.}

\paragraph{A theory of white-box deep networks.} More specifically, we propose a unified objective, a principled  measure of goodness, for learning compact and structured  representations. 
For a learned representation, this objective aims to optimize both its intrinsic complexity in terms of coding rate reduction and its extrinsic simplicity in terms of sparsity. We call this objective the \textit{sparse rate reduction}, specified later in ~\eqref{eq:sparse-rr} and \eqref{eq:objective-sparse-rate-reduction-l1}. 
The intuition behind this objective is illustrated in \Cref{fig:sparse_rr_optima}.
To optimize this objective, we propose to learn a sequence of \textit{incremental mappings} that emulate unrolling certain gradient-descent-like iterative optimization scheme for the objective function. As we will see, this naturally leads to a transformer-like deep network architecture that is entirely a ``white box'' in the sense that its optimization objective, network operators, and learned representation are all  fully interpretable mathematically. 
We name such a white-box deep architecture  ``\ours{},'' or ``\ours{-Transformer},'' short for a \textbf{C}oding-\textbf{RATE} transformer. 
We also show mathematically that these incremental mappings are invertible in a distributional sense, and their inverses consist of essentially the same class of mathematical operators.
Hence a nearly identical \ours{} architecture can be used for realizing encoders, decoders, or together for auto-encoders.

\begin{figure}
    \centering
    \includegraphics[width=0.6\textwidth]{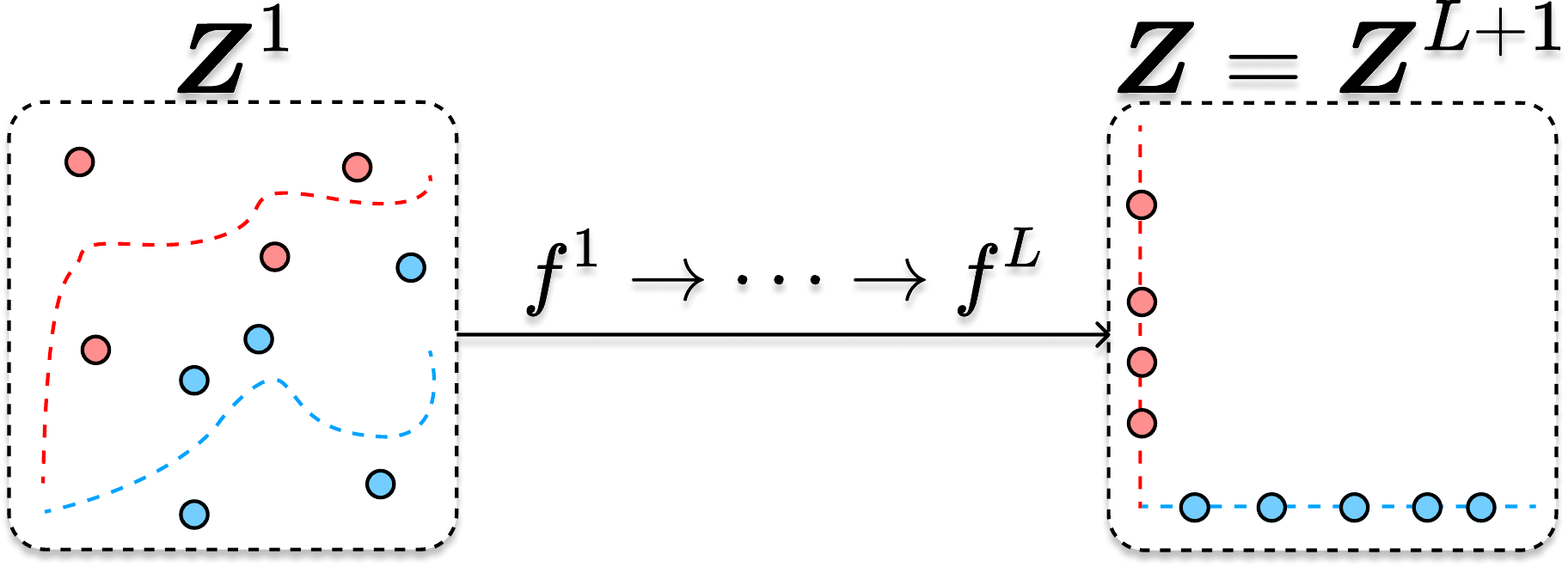}
    \caption{\small\textbf{The optima of the sparse rate reduction.} After pre-processing input data \(\vX\) into a sequence of tokens \(\vZ^{1}\), our \ours{} network attempts to optimize the sparse rate reduction of the token features \(\vZ = \vZ^{L + 1}\). The optimal representations, according to the sparse rate reduction objective, are {\color{purple!60!black}\textit{\textbf{linearized}}}---having low-dimensional linear subspace structure---{\color{green!60!black}\textit{\textbf{sparse}}}---where the subspaces are axis-aligned---and {\color{blue!60!black}\textit{\textbf{compressed}}}---adhering closely to that structure, with low or no noise. In the sequel, we discuss how \ours{} achieves {\color{revision} such} representations via constructing each layer to iteratively optimize the sparse rate reduction.}
    \label{fig:sparse_rr_optima}
\end{figure}

\paragraph{Practice of white-box deep networks.} To show that this framework can truly bridge the gap between theory and practice, we have conducted extensive experiments on both image and text data to evaluate the practical performance of the \ours{} model on a wide range of learning tasks and settings that conventional transformers have demonstrated strong performance. 
Surprisingly, despite its conceptual and structural simplicity, \ours{} has demonstrated competitive performance with respect to its black-box counterparts on {\em all} tasks and settings, including image classification via supervised learning \citep{dosovitskiy2020image},
unsupervised masked completion for imagery and language data \citep{He2021-lb,devlin2018bert,liu2019roberta}, self-supervised feature learning for imagery data \citep{caron2021emerging}, and language modeling via next-word prediction \citep{radford2018improving}.
Moreover, the \ours{} model demonstrates additional practical benefits: each layer and network operator statistically and geometrically meaningful, the learned model is significantly more interpretable compared to black-box transformers, and the features show semantic meaning,
i.e., they can be easily used to segment an object from its background and partition it into shared parts. 

Note that with limited resources, in this work we do not strive for state-of-the-art performance on all of the aforementioned tasks, which would require heavy engineering or extensive fine-tuning; nor can we implement and test our models at current industrial scales. Overall, our implementations for these tasks are basic and uniform, without significant task-specific customization. Nevertheless, we believe these experiments have convincingly verified that the derived white-box deep network \ours{} model is universally effective and sets a solid baseline for further engineering development and improvement.

\paragraph{Outline of the paper:} %
\begin{itemize}[leftmargin=1.0cm]
    \item In \Cref{sub:formulation}, we give a formal formulation for representation learning, both conceptually and quantitatively. We argue that a principled measure of goodness for a learned feature representation is the so-called \textit{sparse rate reduction} that simultaneously characterizes the representation's intrinsic information gain and its extrinsic sparsity. In \Cref{sub:unrolled_rep_learning}, we contend that the fundamental role of a deep network is to optimize such an objective by unrolling an iterative optimization scheme such as gradient descent.
    \item From \Cref{sub:compression} to \Cref{sub:architecture}, we show that a transformer-like deep architecture can be derived from unrolling an alternating minimization scheme for the sparse rate reduction objective. In particular, in \Cref{sub:compression}  we derive a multi-head self-attention layer as an unrolled gradient descent step to minimize the lossy coding rate of the token set with respect to a (learned) low-dimensional Gaussian mixture codebook. In \Cref{sub:sparse} we show that the multi-layer perceptron which immediately follows the multi-head self-attention in transformer blocks can be interpreted as (and replaced by) a layer which constructs a sparse coding of the token representations. This creates a new white-box, i.e., fully mathematically interpretable, transformer-like architecture called \ours{}, summarized in \Cref{sub:architecture}, where each layer performs a \textit{single step} of an alternating minimization algorithm to optimize the sparse rate reduction objective. 
    \item In \Cref{sec:autoencoding} we reveal a fundamental connection between compression via rate reduction and the diffusion-denoising process for learning a representation for the data distribution. In particular, we show that if one \textit{denoises} the tokens towards a family of low-dimensional subspaces, the associated score function assumes an explicit form similar to a self-attention operator seen in transformers. We also establish that the gradient descent of rate reduction essentially conducts structured denoising against the (learned) low-dimensional Gaussian mixture model for the tokens. This connection allows us to construct a white-box decoder based on a structured diffusion process, as a distributional inverse to the structured denoising process implemented by the \ours{} encoder. One can show that the decoder essentially shares the same architecture as the encoder, and they together form a symmetric white-box autoencoder that is fully mathematically interpretable. 
    \item In \Cref{sec:exp} we provide extensive experimental results to show that the \ours{} networks, despite being simple and often smaller, can already learn the desired compressed and sparse representations on large-scale real-world datasets, all while achieving performance on par with seasoned transformer networks on a wide variety of popular tasks and settings, including ViT for image classification, MAE for image completion, DINO for image segmentation with self-supervised learning, and BERT and GPT for text completion and prediction. 
    In addition, we demonstrate, both qualitatively and quantitatively, that the internal representations of \ours{} are more interpretable than vanilla vision transformers trained on image classification. 
\end{itemize}
At the end of the paper, in Appendices \ref{app:section-2} to \ref{sec:appendix-exp}, we provide adequate technical details and experimental details for the above sections, to ensure that  all our claims in the main body are verifiable and experiments are reproducible. Appendix \ref{app:code} gives PyTorch-like  pseudocode for our implementation of \ours{}.

%% file: sec_encoding.tex
\section{White-Box Encoding via Structured Lossy Compression}\label{sec:encoding}

In this section, we provide a technical formulation and justification for our new framework and approach. To wit, we provide a (gentle yet) complete derivation from first principles of our white-box transformer approach. While being a self-contained introduction to our framework, and providing a transparently interpretable transformer-like deep network architecture, it also foreshadows several connections between previously disparate technical approaches to representation learning. These we make clear in the next \Cref{sec:autoencoding} en route to extending our technical framework to autoencoding.

\paragraph{Notation.} We consider a general learning setup associated with real-world signals. We have some random variable \(\vX = \mat{\x_{1}, \dots, \x_{N}} \in \bR^{D \times N}\) which is our data source; each \(\x_{i} \in \bR^{D}\) is interpreted as a \textit{token}\footnote{For language transformers, tokens roughly correspond to words~\citep{vaswani2017attention}, while for vision transformers, tokens correspond to image patches~\citep{dosovitskiy2020image}.}, there are \(N\) tokens \(\vx_{i}\) in each data sample \(\vX\), and the \(\x_{i}\)'s may have arbitrary correlation structures. 
To obtain a useful representation of the input, we learn an \textit{encoder} mapping \(f \colon \bR^{D \times N} \to \bR^{d \times n}\). The features---that is, the output of the encoder---are denoted by the random variable \(\vZ \doteq f(\vX) \doteq \mat{\vz_{1}, \dots, \vz_{n}} \in \bR^{d \times n}\), whence each \(\vz_{i} \in \bR^{d}\) is a feature vector. The number of features \(n\) is typically the same as the number of tokens \(N\), or not much more (e.g., due to pre-processing), in which case there is a natural correspondence between feature vectors \(\vz_{i}\) and tokens \(\vx_{i}\). In the auto-encoding context, we also learn a \textit{decoder} mapping \(g \colon \bR^{d \times n} \to \bR^{D \times N}\), such that \(\vX \approx \vXh \doteq g(\vZ) \doteq \mat{\xh_{1}, \dots, \xh_{N}}\), whence each \(\xh_{i} \in \bR^{D}\) is the auto-encoding of token \(\x_{i}\).

As we have alluded to before, a  central question we want to answer in this work is the purpose of such an encoder and decoder in representation learning: namely, how should we design the encoder and decoder mappings to optimize a representation learning objective? As we will see, one specific form of the encoder \(f\) and the decoder \(g\), that can be naturally deduced through iterative  optimization of the objective, is composed of multiple basic operators, also known as \textit{layers} in the language of deep neural networks. In such cases, we write \(f = f^{L} \circ \cdots \circ f^{1} \circ f^{\pre}\) and \(g = g^{\post} \circ g^{L - 1} \circ \cdots \circ g^{0}\), where \(f^{\ell} \colon \bR^{d \times n} \to \bR^{d \times n}\) and \(g^{\ell} \colon \bR^{d \times n} \to \bR^{d \times n}\) are the \(\ell\th\) layer of the encoder and decoder respectively, and \(f^{\pre} \colon \bR^{D \times N} \to \bR^{d \times n}\) and \(g^{\post} \colon \bR^{d \times n} \to \bR^{D \times N}\) are the pre-~and post-processing layers respectively. The \textit{input} to the \(\ell\th\) layer of the encoder is denoted \(\vZ^{\ell} \doteq \mat{\vz_{1}^{\ell}, \dots, \vz_{n}^{\ell}} \in \bR^{d \times n}\), and the \textit{input} to the \(\ell\th\) layer of the decoder is denoted \(\wt{\vZ}^{\ell} \doteq \mat{\wt{\vz}_{1}^{\ell}, \dots, \wt{\vz}_{n}^{\ell}} \in \bR^{d \times n}\). In particular, \(\vZ^{\ell + 1} = f^{\ell}(\vZ^{\ell})\) and \(\wt{\vZ}^{\ell + 1} = g^{\ell}(\wt{\vZ}^{\ell})\). \Cref{fig:crate_autoencoding} depicts this overall process.

\subsection{Desiderata and Objective of Representation Learning}  \label{sub:formulation}

\paragraph{Representation learning via the principle of parsimony and consistency.} Following the framework of rate reduction \citep{chan2021redunet}, we contend that the goal of representation learning is to find a feature mapping \(f \colon \vX \in \bR^{D \times N} \to \vZ\in \bR^{d \times n}\) which transforms input data \(\vX \in \bR^{D \times N}\) with a potentially nonlinear and multi-modal distribution to a \textit{parsimonious} feature
representation \(\vZ \in \bR^{d \times n}\) \citep{ma2022principles}. As in~\citet{ma2022principles}, a complete desiderata for the learned representations ought to be:
\begin{enumerate}
   \item {\color{blue!60!black}\textit{\textbf{Compressed}}}: being strictly distributed according to some standard low-dimensional structures matching the intrinsic low-dimensionality of the data, so as to ensure a compact encoding of the data.
   \item {\color{purple!60!black}\textit{\textbf{Linearized}}}: the low-dimensional structures have (piecewise) linear geometry, so as to aid interpolation and extrapolation in the representation space.
   \item {\color{green!60!black}\textit{\textbf{Sparse}}}: the low-dimensional structures corresponding to different parts of the data distribution are statistically \textit{incoherent} or geometrically \textit{orthogonal}, and also \textit{axis-aligned}, so as to ensure a more compact encoding and aid downstream processing.
   \item {\color{orange!60!black}\textit{\textbf{Consistent}}}: for autoencoding/generative purposes, we desire that the learned representation is \textit{invertible}, in the sense that we can decode features to recover the corresponding input data, either on the level of individual samples or  distribution-wise.
\end{enumerate}
For the last item, specifically, we would also like to learn an inverse mapping: \(g \colon \vZ \in \bR^{d \times n} \to \vXh \in \bR^{D \times N}\) such that \(\vXh\) and \(\vX\) are quantitatively close in some sense. \Cref{fig:crate_autoencoding} illustrates the overall process and the desired four goals of such a  representation learning. In this section (\Cref{sec:encoding}), we will mainly show how to achieve the first three items on this list by developing an encoding scheme; we will address the last item in the next section (\Cref{sec:autoencoding}) by showing how the proposed encoding scheme can be naturally reversed.  

\begin{figure}
    \centering
    \includegraphics[width=\textwidth]{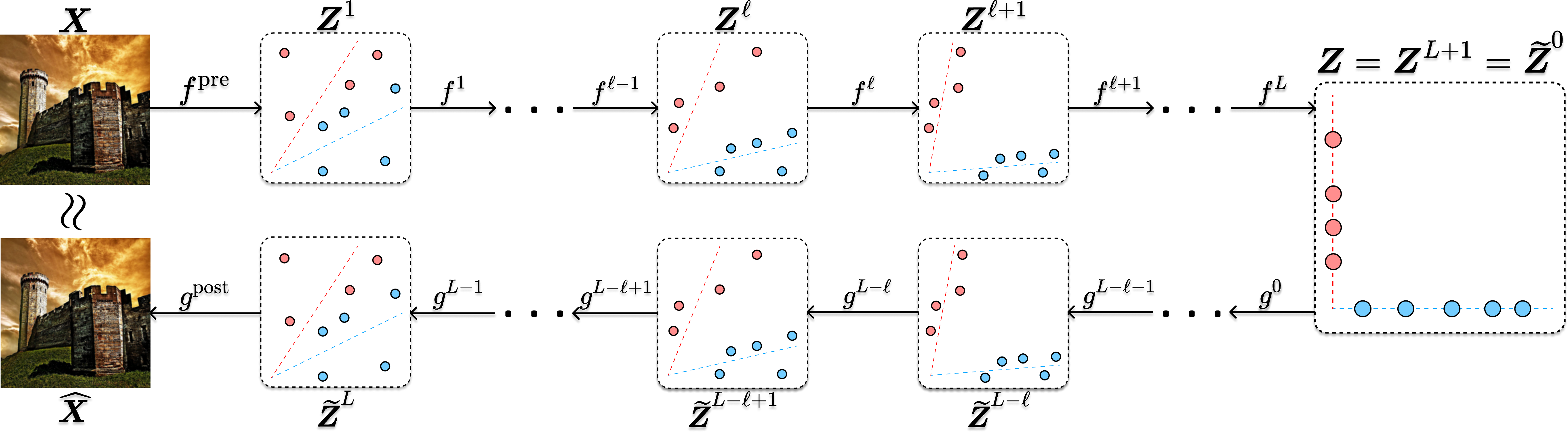}
    \caption{\small\textbf{The autoencoding process to be studied in \Cref{sec:encoding,sec:autoencoding}}. Each encoder layer \(f^{\ell}\) and decoder layer \(g^{L - \ell}\) are (partial) inverses of each other. Moreover, the overall representation \(\vZ = f(\vX)\) is parsimonious ({\color{blue!60!black}\textbf{compressed}}, {\color{purple!60!black}\textbf{linearized}}, and {\color{green!60!black}\textbf{sparse}}, as in \Cref{sub:formulation}), and the autoencoding is to be {\color{orange!60!black}\textbf{consistent}} in the sense that \(\vX \approx \wh{\vX}\).}
    \label{fig:crate_autoencoding}
\end{figure}

\paragraph{An objective which promotes parsimonious representations.} Previously, \citet{OriginalMCR2} have proposed to obtain parsimonious representations via maximizing the \textit{information gain} \citep{ma2022principles}, a principled measure of the information content of the features. A concrete instantiation of the information gain is the coding \textit{rate reduction} \citep{OriginalMCR2} of the features, i.e., 
\begin{equation}
   \Delta R(\vZ \mid \bm \Pi_{[K]}) = R(\vZ) - R^c(\vZ \mid \bm \Pi_{[K]}).
   \label{eqn:rate-reduction}
\end{equation}
The first term \(R(\vZ)\) in the above expression is an estimate of the lossy coding rate (i.e., \textit{rate distortion function}) for the whole set of features, when using a codebook adapted to Gaussians. More specifically, if we view the token feature vectors \((\vz_{i})_{i \in [n]}\) in \(\vZ \in \bR^{d\times n}\) as i.i.d.~samples from a single zero-mean Gaussian, an approximation of their (lossy) coding rate, subject to quantization precision \(\eps > 0\), is given in \citep{ma2007segmentation} as:
\begin{equation}\label{eq:coding_rate}
    R(\vZ) \doteq \frac{1}{2}\logdet{\I + \alpha\vZ\adj\vZ} = \frac{1}{2}\logdet{\I + \alpha\vZ\vZ\adj},\quad 
    \text{where}\ \alpha
    \doteq \frac{d}{n\eps^{2}}.
\end{equation}
The second term \(R^{c}\) in the rate reduction objective \eqref{eqn:rate-reduction} is also an estimate of the lossy coding rate, but under a different and more precise codebook---one which views the token feature vectors \((\vz_{i})_{i \in [n]}\) as i.i.d.~samples of a mixture of Gaussians, where assignment of tokens to a particular Gaussian is known and specified by the Boolean membership matrices \(\vPi_{[K]} = (\vPi_{k})_{k \in [K]}\), and the \(k\th\) Gaussian has \(n_{k}\) associated tokens. We obtain an estimate for the coding rate \(R^{c}\) as
\begin{equation}\label{eq:conditional_coding_rate_original}
    R^c(\vZ \mid \vPi_{[K]}) \doteq \frac{1}{2}\sum_{k=1}^{K} \logdet{\I + \gamma_{k} \vZ \vPi_{k} \vZ\adj}, \quad \text{where}\ \gamma_{k} \doteq \frac{d}{n_{k}\eps^{2}}.
\end{equation}
As shown in \citet{OriginalMCR2}, maximizing the rate reduction \(\Delta R\), i.e., the difference between \(R\) and \(R^{c}\), promotes that the token features \(\vz_{i}\) are compactly encoded as a mixture of low-dimensional Gaussian distributions, where different Gaussian are statistically \textit{incoherent}.

\paragraph{A generalized measure of rate reduction for tokens.} 
In more realistic and general scenarios, the  features \(\vZ\) can be a collection of tokens \((\vz_{i})_{i = 1}^{N}\) which have a sophisticated and task-specific joint distribution, which can encode rich information about the data\footnote{For example, co-occurrences between words in language data, or object parts in image data.} which we should also seek to capture in the final representation. 

To realize our above desiderata in this context---namely, seeking a compact representation of a complex joint distribution of the token features---we only require that \textit{the desired marginal distribution of individual tokens \(\vz_{i}\) should be a mixture of (say \(K\)) low-dimensional Gaussian distributions}. Without loss of generality, we may assume  that the \(k^{\mathrm{th}}\) Gaussian has mean \(\vZero \in \bR^{d}\), covariance \(\vSigma_{k} \succeq \vZero \in \bR^{d \times d}\), and support spanned by the orthonormal basis \(\vU_{k} \in \bR^{d \times p}\). We denote \(\vU_{[K]} = (\vU_{k})_{k=1}^{K}\) to be the set of all bases for the  Gaussians. In the sequel, we often identify the basis \(\vU_{k}\) with the subspace itself. 

For future reference, we provide a formal definition of this statistical model below. Note that we may incorporate random noise as a way to model benign deviations from the previously described idealized model.\footnote{Our noise model is standard and simple, but can be made more sophisticated at essentially no conceptual cost---the qualitative results will be the same.}

\paragraph{Low-Dimensional Gaussian Mixture Codebook:} \textit{
    Let \(\vZ = \mat{\vz_{1}, \dots, \vz_{n}} \in \bR^{d \times n}\) be a matrix-valued random variable. We impose the following statistical model on \({\vZ}\), parameterized by orthonormal bases \(\vU_{[K]} = (\vU_{k})_{k \in [K]} \in (\bR^{d \times p})^{K}\): each token \({\vz}_{i}\) has marginal distribution given by
    \begin{equation}\label{model:gaussian_tokens}
        {\vz}_{i} \stackrel{d}{=} \vU_{s_{i}}\valpha_{i}, \quad \forall i \in [n]
    \end{equation}
    where \((s_{i})_{i \in [n]} \in [K]^{n}\) are random variables corresponding to the subspace indices, and \((\valpha_{i})_{i \in [n]} \in (\bR^{p})^{n}\) are zero-mean Gaussian variables. If we optionally specify a noise parameter \(\sigma \geq 0\), we mean that we ``diffuse'' the tokens with Gaussian noise: by an abuse of notation, each token \({\vz}_{i}\) has marginal distribution given by
    \begin{equation}\label{model:gaussian_tokens_noise}
        \vz_{i} \stackrel{d}{=} \vU_{s_{i}}\valpha_{i} + \sigma \vw_{i}, \quad
        \forall i \in [n]
    \end{equation}
    where \((\vw_{i})_{i \in [n]} \in (\bR^{d})^{n}\) are i.i.d.~standard Gaussian variables, independent of $s_{i}$ and $\valpha_i$.
}

From the perspective of statistics, we may view $\vU_{[K]}$ as multiple ``principal subspaces'' \citep{Vidal:Springer16}, which, just as in principal component analysis, are preferred to be incoherent or nearly orthogonal to each other. From the perspective of signal processing, we may view \(\vU_{[K]}\) as ``local signal models'' for the input distribution. From the perspective of information theory, we may view the bases \(\vU_{[K]}\) as codebooks and the vectors 
\(\valpha_{ik} \doteq \vU_{k}\adj\vz_{i}\) as the ``codes'' of the token features \(\vz_{i}\) with respect to these codebooks. Motivated by \Cref{model:gaussian_tokens}, we desire these codes to have a Gaussian marginal distribution within each subspace; under this model, we can compute the coding rate of these codes, similar to \eqref{eq:coding_rate}, as  
\begin{equation}\label{eq:{eq:coding_rate}-subspace}
    R(\vU_{k}\adj \vZ) \doteq \frac{1}{2}\logdet{\I + \beta(\vU_k\adj\vZ)\adj(\vU_k\adj\vZ)}, \quad 
    \text{where}\ \beta \doteq \frac{p}{n\eps^{2}}.
\end{equation}

We emphasize that here, under \Cref{model:gaussian_tokens}, the joint distribution of such \(\vZ\) is underspecified, so the true optimal codebook for \(\vZ\) is unknown and so the lossy coding rate for \(\vZ\) is impossible to compute. However, since the desired marginal distribution of each token \(\vz_{i}\) is a mixture of low-dimensional Gaussians supported on subspaces \(\vU_{[K]}\), we may obtain an upper bound of the coding rate for the token set \(\vZ\), which we denote as \(R^{c}\), by projecting the tokens \(\vz_{i}\) onto these subspaces and summing up the coding rates on each subspace:
\begin{equation}\label{eq:def-mcr-parts}
    R^c(\vZ \mid \vU_{[K]}) \doteq \sum_{k=1}^{K}R(\vU_k\adj \vZ) =
    \frac{1}{2}\sum_{k=1}^{K}\logdet{\I +
    \beta(\vU_k\adj\vZ)\adj(\vU_k\adj\vZ)}.
\end{equation}

This form of the coding rate can be viewed as a generalization to the coding rate \(R^{c}\) in the original rate reduction objective defined in \Cref{eqn:rate-reduction}. In particular, the original  objective is defined with respect to a set of known membership labels \(\vPi_{[K]}\) specific to the particular data realization \(\vX\). In contrast, the  objective here is defined with respect to subspaces \(\vU_{[K]}\) which are in principle defined externally to any specific data realization, though they support the token feature distribution. Since a single token can have nonzero projection onto multiple subspaces \(\vU_{k}\), yet must belong to exactly one category defined by \(\vPi_{k}\), the coding rate defined in \Cref{eq:def-mcr-parts} may be viewed as a generalization of the coding rate defined in \Cref{eq:conditional_coding_rate_original}. We may correspondingly generalize the coding rate reduction $\Delta R$, obtaining:
\begin{equation}
   \Delta R(\vZ \mid \vU_{[K]}) \doteq R(\vZ) - R^c(\vZ \mid \vU_{[K]}).
\end{equation}

\begin{figure}
    \centering
    \includegraphics[width=0.98\textwidth]{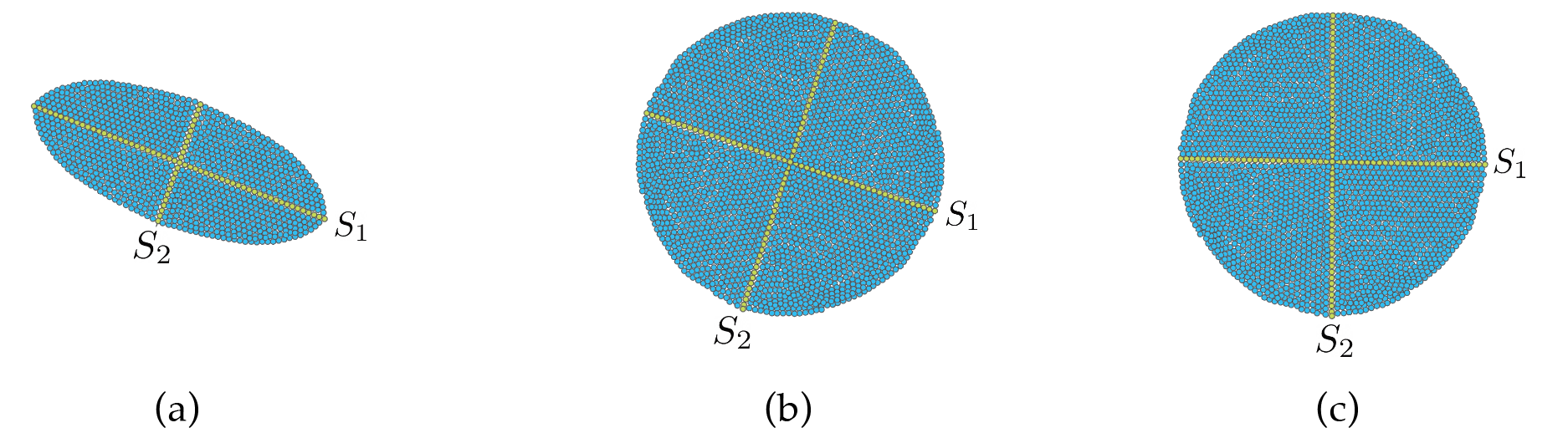}
    \vspace{-0.1in}\caption{\color{revision} \small\textbf{Comparison of three sets of  representations via rate reduction and sparsity.} Each $S_i$ represents one linear subspace, and the number of blue balls represents the difference between the coding rates $\Delta R(\vZ \mid \vU_{[K]}) = R(\vZ) - R^c(\vZ \mid \vU_{[K]})$.  }
    \label{fig:diagram_compare_compression_sparsification}
\end{figure}

\paragraph{Sparse rate reduction.} It is easy to see that the rate reduction is invariant to arbitrary joint rotations of the representations and subspaces \citep{ma2007segmentation}. In particular, optimizing the rate reduction may not naturally lead to axis-aligned (i.e., \textit{sparse}) representations. {\color{revision} For instance, consider the three sets of learned representations in \Cref{fig:diagram_compare_compression_sparsification}. The coding rate reduction increases from (a) to (b), but because it is invariant under rotations, remains the same from (b) to (c).} Therefore, we would like to transform the representations (and their supporting subspaces) so that the features $\vZ$ eventually become sparse\footnote{Concretely, having few nonzero entries.} with respect to the standard coordinates of the resulting representation space {\color{revision} as in \Cref{fig:diagram_compare_compression_sparsification}(c)}.

The combined rate reduction and sparsification process is illustrated in \Cref{fig:sparse_rr_optima} or \Cref{fig:crate_autoencoding}. Computationally, we may combine the above two goals into a unified objective for optimization:
\begin{equation}
    \max_{f \in \mathcal{F}} \mathbb{E}_{\vZ = f(\vX)}\big[ \Delta R(\vZ \mid \vU_{[K]}) - \lambda \norm{\bm{Z}}_0 \big], \label{eq:sparse-rr} 
\end{equation}
or equivalently, 
\begin{equation}\label{eq:objective-sparse-rate-reduction}
    \max_{f \in \mathcal{F}} \mathbb{E}_{\vZ = f(\vX)}\big[ R(\vZ) - {R}^c(\vZ \mid \vU_{[K]})- \lambda \norm{\bm{Z}}_0 \big],
\end{equation}
where the $\ell^{0}$ ``norm''{\color{revision}, defined as the number of nonzero values of the input vector/matrix,} promotes the sparsity of the learned token representations \(\vZ = f(\vX)\). \footnote{To simplify the notation, we will discuss the objective for one sample $\vX$ at a time with the understanding that we always mean to optimize the expectation.}

We call this objective ``\textit{sparse rate reduction.}'' In practice, one typically relaxes the $\ell^0$ norm to the $\ell^1$ norm for better computability \citep{Wright-Ma-2022}, obtaining:
\begin{equation}\label{eq:objective-sparse-rate-reduction-l1}
\max_{f \in \mathcal{F}} \mathbb{E}_{\vZ = f(\vX)}\big[ R(\vZ) - {R}^c(\vZ \mid \vU_{[K]})- \lambda \norm{\bm{Z}}_1 \big].
\end{equation}
 By a little abuse of language, we often refer to this objective function also as the \textit{sparse rate reduction}.

\begin{remark}[\textbf{Connections to likelihood maximization and energy-based models}]\label{rmk:energy_based_models}
    One natural interpretation of the Gaussian rate distortion \(R(\vZ)\) is as a lossy surrogate for the log-likelihood of \(\vZ\) under the assumption that the columns \(\vz_{i}\) are drawn i.i.d.~from a zero-mean Gaussian whose covariance is estimated using \(\vZ\) \citep{cover1999elements}. Similar interpretations hold for \(R^{c}\)---as a surrogate for the un-normalized log-likelihood of \(\vZ\) under the assumption that the columns of \(\vz_{i}\) are drawn from \Cref{model:gaussian_tokens}---and \(\Delta R\)---as the difference of these log-likelihoods. In some sense, the latter interpretations of the desired feature distribution are ``local,'' in that they manage the part of the feature distribution aligned with the \(\vU_{[K]}\).
    
    If we also interpret the sparse regularization term \(-\lambda \norm{\vZ}_{1}\) in this way, we obtain the interpretation that we prefer the features \(\vZ\) to have un-normalized log-density equal to \(-\lambda \norm{\vZ}_{1}\), so as to have density proportional to \(e^{-\lambda \norm{\vZ}_{1}}\). This is a more ``global'' interpretation of the feature distribution. In this way, regularization can be seen as ``exponentially tilting'' \citep{keener2010theoretical} the desired density towards one which is lower-entropy or more axis-aligned.

    One recently popular class of models which performs maximum-likelihood estimation is \textit{energy-based models} \citep{lecun2006tutorial}. 
    In particular, the overall objective function \eqref{eq:objective-sparse-rate-reduction-l1} has a natural interpretation as an ``energy function.'' In particular, if we assume that our surrogate likelihoods are {\color{revision} exact likelihoods} (up to constants), then the desired probability distribution of the feature set \(\vZ\) is known up to constants as 
    \begin{equation}
        p(\vZ \mid \vU_{[K]}) = Ce^{-E(\vZ \mid \vU_{[K]})} \doteq C \exp(-\lambda \|\vZ\|_1)\cdot \frac{\det(\I + \alpha\vZ\adj\vZ)}{ \prod_{k=1}^K \det\big(\I +
        \beta(\vU_k\adj\vZ)\adj(\vU_k\adj\vZ)\big)},
        \label{eqn:sparse-rate-reduction-density}
    \end{equation}
    where we define the energy function 
    \begin{equation}
        E(\vZ \mid \vU_{[K]}) = -[R(\vZ) - R^{c}(\vZ \mid \vU_{[K]}) - \lambda \norm{\vZ}_{1}],
    \end{equation}
    where the term $\det(\I + \alpha\vZ\adj\vZ)/(\prod_{k=1}^K \det(\I + \beta(\vU_k\adj\vZ)\adj(\vU_k\adj\vZ)))$ has a natural intrinsic geometric interpretation as the ratio of the ``volume'' of the whole feature set $\vZ$ and the product of ``volumes'' of its projections into the subspaces \citep{ma2007segmentation}. 
    
    Minimizing the above energy \(E(\vZ \mid \vU_{[K]})\) is {\color{revision}equivalent} to maximizing the sparse rate reduction objective \eqref{eq:objective-sparse-rate-reduction-l1}. In this sense, rate reduction-based approaches to representation learning are qualitatively similar to certain classes of energy-based models.
\end{remark}

\begin{remark}[\textbf{Intrinsic and extrinsic measures of goodness for the representations}]
Our notion of parsimony, as described above, desires the representations to have both \textit{intrinsic} and \textit{extrinsic} properties; that is, properties which are invariant to arbitrary rotations of the data distribution (e.g., compression and linearization), and those which are not (e.g., sparsity). There are separate long lines of work optimizing intrinsic measures of goodness for the representations \citep{OriginalMCR2,chan2021redunet,dai2022ctrl,pai2022pursuit} as well as extrinsic measures \citep{gregor2010learning,elad2010role,elad2010sparse,zhai2020complete,zhai2020understanding,tolooshams2021stable,Wright-Ma-2022}. Both classes of methods---that is, optimizing intrinsic and extrinsic measures of goodness of the representations---have individually been successful in learning compact and structured representations which are useful for downstream tasks. In this work, we combine and conceptually unify these perspectives on representation learning. In particular, our methodology seeks to optimize both intrinsic and extrinsic measures. Overall, we achieve even greater empirical success than previous white-box representation learning methods via learning intrinsically and extrinsically parsimonious representations.
\end{remark}

\begin{remark}[\textbf{Black-box representations learned through pretext tasks}] 
Representation learning has also been quantitatively studied as the byproduct of \textit{black-box} neural networks trained to solve pretext tasks, e.g., classification, contrastive learning, etc. Such end-to-end approaches do not explicitly attempt to learn parsimonious representations through the architecture or the objective. Meanwhile, we explicitly attempt to learn good representations which maximize the sparse rate reduction. Below, we give a concrete example of a conceptual separation between these two approaches, and their resulting representations.

Black-box representation learning has been most studied in the context of the supervised classification pretext task. Both empirical work and theoretical work has demonstrated that, under broad conditions, black-box neural networks trained with \textit{the cross-entropy loss} on \textit{supervised classification} have representations which obey \textit{neural collapse}~\citep{papyan2020prevalence, zhu2021geometric, fang2021exploring, yaras2022neural}, a phenomenon where representations of data from a given class are highly clustered around a single point, and the points from each class are maximally separated. 
\citet{wang2023understanding} (theoretically) and \citet{he2023law} (empirically) showed that a progressive neural collapse phenomenon, governed by a law of data separation, occurs from shallow to deep layers. 
This can be viewed as a form of ``compression'' of the features of each class towards a finite set of points, which form a geometric structure called a \textit{simplex equiangular tight frame}. This is distinguished from our approach to lossy compression through the sparse rate reduction in two particular ways.  
First, our representation for a data point is a token set, whereas commonly neural collapse is observed in cases where the representation is for a whole data point, so our representation is more fine-grained than those studied by neural collapse. 
Second, our proposed compression objective---sparse rate reduction---encourages the features to be \textit{diverse and expanded} within their supporting subspaces, and in particular \textit{not collapsed to individual points}. This is a more fundamental difference which suggests that our approach is at odds with neural collapse. More generally, our sparse rate reduction-based approach obtains qualitatively and conceptually different representations than black-box networks. %
\end{remark}

\subsection{Learning Parsimonious Representations  via Unrolled Optimization}\label{sub:unrolled_rep_learning}

\begin{figure}
     \centering
     \includegraphics[width=0.99\textwidth]{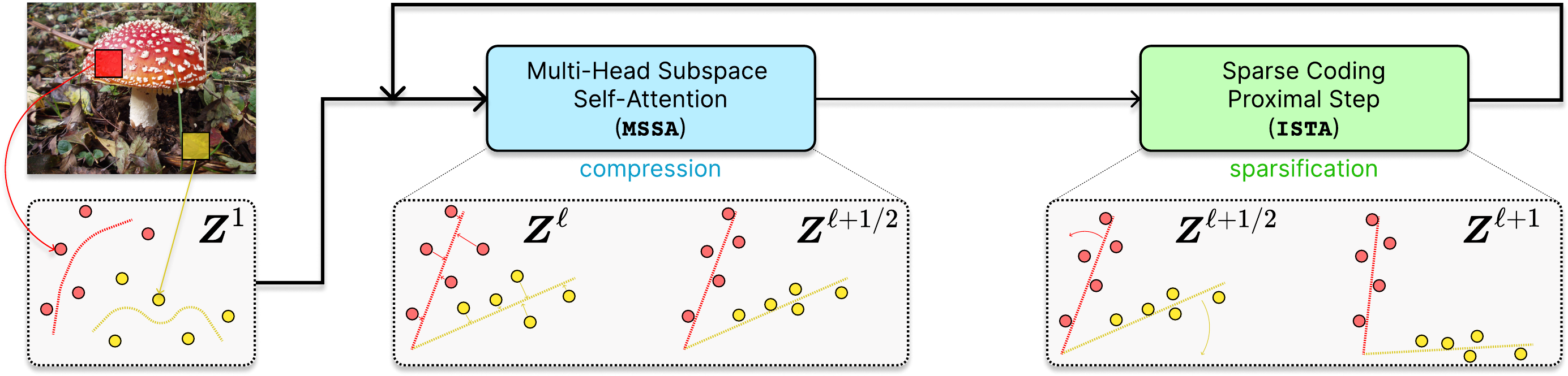}
     \caption{\small \textbf{The `main loop' of the \ourscaps{} white-box deep network design.} After preprocessing input data \(\vX\) into a sequence of tokens \(\vZ^1\), \ours{} constructs a deep network that transforms the data to a canonical configuration of low-dimensional subspaces by successive {\color{blue!60!black}\bf\textit{compression}} against a local model for the distribution, generating $\vZ^{\ell+1/2}$, and {\color{green!60!black}\bf\textit{sparsification}} against a global dictionary, generating $\vZ^{\ell+1}$. Repeatedly stacking these blocks and training the model parameters via backpropagation yields a powerful and interpretable representation of the data.
        }
        \label{fig:encoder_compression_sparsification}
\end{figure}

Although easy to state, each term of the sparse rate reduction objective proposed in the previous section, viz.
\begin{equation*}
    \max_{f \in \mathcal{F}} \mathbb{E}_{\vZ = f(\vX)}\big[ R(\vZ) - {R}^c(\vZ \mid \vU_{[K]})- \lambda \norm{\vZ}_1 \big], \tag{\ref{eq:objective-sparse-rate-reduction-l1}}
\end{equation*} 
can be computationally very challenging
to optimize. Hence it is natural to take an approximation approach that realizes the global transformation $f$  through a concatenation of multiple, say $L$, simple \textit{incremental and local} operations $f^\ell$ that push the representation distribution towards the desired parsimonious template distribution:
\begin{equation}
f\colon \vX
\xrightarrow{\hspace{1mm} f^{\pre} \hspace{1mm}} \vZ^1 \rightarrow \cdots \rightarrow \vZ^\ell \xrightarrow{\hspace{1mm} f^\ell \hspace{1mm}} \vZ^{\ell+1} \rightarrow  \cdots \to \vZ^{L+1} = \vZ,
\label{eq:incremental}
\end{equation}
where $f^{\pre}: \bR^{D \times N} \rightarrow \bR^{d \times n}$ is the pre-processing mapping that transforms the input token set $\vX \in \bR^{D \times N}$ to a first-layer representation $\vZ^{1} \in \bR^{d \times n}$, as in \Cref{fig:encoder_compression_sparsification}.

Each incremental \textit{forward mapping} $\vZ^{\ell + 1} = f^\ell(\vZ^\ell)$, or a ``layer'', transforms the token distribution to \textit{optimize} the above sparse rate reduction objective \eqref{eq:objective-sparse-rate-reduction-l1}, conditioned on a model, say a mixture of subspaces whose bases are \(\vU_{[K]}^{\ell}\), of the distribution of its input tokens $\vZ^\ell$:
\begin{equation}
    \max_{f^{\ell} \in \mathcal{F}^{\ell}}  \mathbb{E}_{\vZ^{\ell + 1} = f^{\ell}(\vZ^{\ell})}\big[R(\vZ^{\ell + 1}) - {R}^c(\vZ^{\ell + 1} \mid \vU_{[K]}^{\ell})- \lambda \norm{\vZ^{\ell + 1}}_{1}\big].
\end{equation}

Conceptually, if we follow the idea of the ReduNet \citep{chan2021redunet}, each $f^\ell$ should conduct a ``gradient-ascent'' like operation:
\begin{align}
    \vZ^{\ell+1} = f^{\ell}(\vZ^{\ell})
    &\approx \vZ^{\ell} + \eta \nabla \big[ R(\vZ^{\ell}) - {R}^c(\vZ^{\ell} \mid \vU_{[K]}^{\ell})- \lambda \norm{\vZ^{\ell}}_1 \big] ,\label{eqn:sparse-rate-reduction-gradient} \\
    &\approx \vZ^{\ell} + \eta \nabla \log p(\vZ^{\ell} \mid \vU_{[K]}^{\ell}), \label{eqn:sparse-rate-reduction-score}
\end{align}
where $p(\vZ \mid \vU_{[K]})$ was defined in \eqref{eqn:sparse-rate-reduction-density}. An acute reader might have noticed that the term $\nabla \log p(\vZ \mid \vU_{[K]})$ resembles that of a score function and the update \eqref{eqn:sparse-rate-reduction-score} resembles that of a \textit{denoising process}, i.e., it moves the current iterate \(\vZ^{\ell}\) towards the maximum-likelihood token set with respect to the signal model \(\vU_{[K]}^{\ell}\). We will thoroughly explore connections of the above gradient ascent operation to denoising and diffusion processes in \Cref{sec:autoencoding}. For now, we are interested in how to actually optimize the objective \Cref{eq:objective-sparse-rate-reduction-l1}.  

\paragraph{An alternating optimization strategy.} As already explored in the work of \citet{chan2021redunet}, it is difficult to directly compute the gradient and optimize the rate reduction term $\Delta R$,\footnote{This was part of the reason why the validity of ReduNet from \citet{chan2021redunet} could only be verified with small datasets -- it is difficult to scale the method up to produce competitive performance in practice.} let alone now with the non-smooth $\ell^1$ term $\|\vZ\|_1$. Nevertheless, from an optimization perspective, once we decide on using an incremental approach to optimizing
\eqref{eq:objective-sparse-rate-reduction-l1}, there are a variety of alternative optimization strategies. In this work we opt for perhaps the simplest possible choice that exploit the special structure of the objective. Given a model \(\vU_{[K]}^{\ell}\) for \(\vZ^{\ell}\), we opt for a two-step \textit{alternating minimization} process with a strong conceptual basis:
\begin{align}
    \vZ^{\ell + 1/2}&\ \text{is chosen to incrementally minimize}\ R^{c}(\vZ^{\ell + 1/2} \mid \vU_{[K]}^{\ell}), \label{eqn:compression} \\
    \vZ^{\ell + 1}&\ \text{is chosen to incrementally minimize}\left[\lambda \norm{\vZ^{\ell + 1}}_{0} - R(\vZ^{\ell + 1})\right].\label{eqn:sparsification}
\end{align}
For the first step \Cref{eqn:compression},  we \textit{compress} the tokens $\vZ^{\ell}$ via an approximate gradient step to minimize an estimate for the coding rate $R^c(\vZ^{\ell} \mid \vU_{[K]}^{\ell})$. Namely, \(R^{c}(\vZ^{\ell} \mid \vU_{[K]}^{\ell})\) measures the compression of \(\vZ^{\ell}\) against (i.e., adherence to) the statistical structure delineated in \Cref{model:gaussian_tokens} with subspace bases \(\vU_{[K]}^{\ell}\). Thus, taking a gradient step on \(R^{c}\) {\color{revision}with learning rate \(\kappa > 0\)} pushes the tokens towards having the desired statistics:
\begin{equation}
    \vZ^{\ell + 1/2} \approx \vZ^{\ell} - \kappa \nabla R^{c}(\vZ^{\ell} \mid \vU_{[K]}^{\ell}).
\end{equation}
Unfortunately, the gradient of the coding rate \(\nabla R^{c}\) is costly to compute, as it involves \(K\) separate matrix inverses, one for each of the \(K\) subspaces with basis \(\vU_{k}^{\ell}\). However, as we will formally derive in \Cref{sub:compression}, this gradient can be naturally approximated using a so-called \(\MSSA{\cdot}\) operator, which has a similar functional form to the multi-head self-attention operator \citep{vaswani2017attention} with \(K\) heads (i.e., one for each subspace, coming from each matrix inverse), yet has a more explicit interpretation as approximately the (negative) gradient of a compression measure $R^{c}(\vZ^{\ell} \mid \vU_{[K]}^{\ell})$. As a result, we obtain a transformed token set $\vZ^{\ell + 1/2}$ given by 
{\color{revision}
\begin{equation}
    \vZ^{\ell + 1/2} \doteq  
    (1 - \beta \kappa) \vZ^{\ell} + \beta\kappa \MSSA{\vZ^{\ell} \mid \vU_{[K]}^{\ell}} 
    \approx \vZ^{\ell} - \kappa \nabla R^{c}(\vZ^{\ell} \mid \vU_{[K]}^{\ell}),
\end{equation}
where $\beta={p}/{(n\eps^{2})}$ is defined in \eqref{eq:{eq:coding_rate}-subspace}.}

For the second step of \eqref{eqn:sparsification}, we \textit{sparsify} the compressed tokens, choosing \(\vZ^{\ell + 1}\) via a suitably-relaxed proximal gradient step to minimize the remaining term $\lambda \norm{\vZ^{\ell + 1}}_{1} - R(\vZ^{\ell + 1})$. As we will argue in detail in \Cref{sub:sparse}, we can find such a \(\vZ^{\ell + 1}\) by solving a sparse representation problem with respect to a sparsifying codebook, i.e., dictionary $\vD^{\ell}$:
\begin{equation}
  \vZ^{\ell+1} \approx \argmin_{{\vZ}} \left[\lambda \norm{\vZ}_1 + \frac{1}{2}\norm{\vZ^{\ell + 1/2} - \vD^{\ell} {\vZ}}_F^2 \right].
\end{equation}
In this work, we choose to implement this step with an iteration of the iterative shrinkage-thresholding algorithm (ISTA), which has classically been used to solve such sparse representation problems \citep{beck2009fast}. We call such an iteration the \(\ISTA{\cdot}\) operator, formally defined in \Cref{sub:sparse}. We obtain tokens $\vZ^{\ell + 1}$ given by 
\begin{equation}
    \vZ^{\ell + 1} \doteq \ISTA{\vZ^{\ell + 1/2} \mid \vD^{\ell}} \approx \argmin_{\vZ} \left[\lambda \norm{\vZ}_1 + \frac{1}{2}\norm{\vZ^{\ell + 1/2} - \vD^{\ell} {\vZ}}_F^2 \right].
\end{equation}
Both compression and sparsification are applied incrementally and repeatedly, as these operations form layers of the network 
\begin{equation}\label{eq:arch-encoder}
    f^{\ell} \colon \vZ^{\ell} \xrightarrow{\hspace{1mm} \texttt{MSSA} \hspace{1mm}} \vZ^{\ell + 1/2}  \xrightarrow{\hspace{1mm} \texttt{ISTA} \hspace{1mm}} \vZ^{\ell + 1}
\end{equation}
as in \Cref{eq:incremental}. \Cref{fig:compression_sparsification_oneiter} graphically demonstrates the idealized effect of one layer.

\begin{figure}
    \centering
    \includegraphics[width=\textwidth]{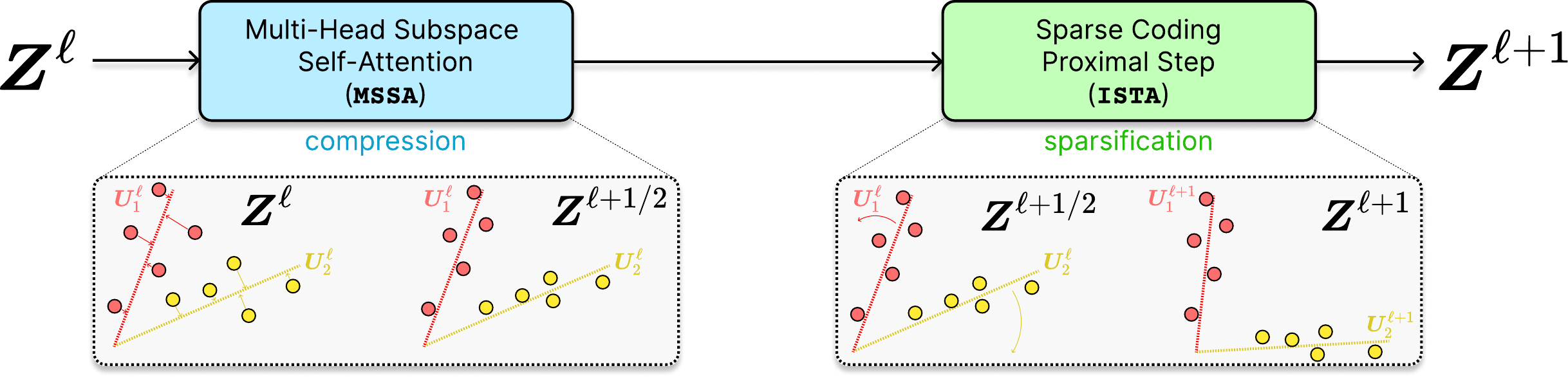}
    \caption{\small\textbf{The effect of one encoder layer \(f^{\ell}\) on the distribution of its input}. First, \(\vZ^{\ell}\) is compressed against the codebook \(\vU_{[K]}^{\ell}\) to obtain \(\vZ^{\ell + 1/2}\). Then, \(\vZ^{\ell + 1/2}\) is sparsified using the codebook \(\vD^{\ell}\) to obtain \(\vZ^{\ell + 1}\).}
    \label{fig:compression_sparsification_oneiter}
\end{figure}

\subsection{Self-Attention as Gradient Descent on Coding Rate of Tokens }\label{sub:compression}

In this subsection and the next, we provide technical details about each of the two steps mentioned in the \Cref{sub:unrolled_rep_learning}, in particular the precise forms of the \(\MSSA{\cdot}\) and \(\ISTA{\cdot}\) operators.

For the first step, where we compress the set of tokens against the \(K\) subspaces by minimizing the upper bound for the coding rate \(R^{c}\):
\begin{equation*}
    \vZ^{\ell + 1/2}\ \text{is chosen to incrementally minimize}\ R^{c}(\vZ^{\ell + 1/2} \mid \vU_{[K]}^{\ell}). \tag{\ref{eqn:compression}}
\end{equation*}
As in \Cref{sub:unrolled_rep_learning}, the compression operator takes an approximate gradient descent step on \(R^{c}\). The gradient of $R^c(\vZ \mid \vU_{[K]})$ is given by 
\begin{equation}
    \nabla R^c(\vZ \mid \vU_{[K]})
    = \beta\sum_{k=1}^K \vU_k(\vU_k\adj\vZ)\Big(\I +
    \beta(\vU_k\adj\vZ)\adj(\vU_k\adj\vZ)\Big)^{-1}.
    \label{eq:rate-gradient}
\end{equation}
The expression in \Cref{eq:rate-gradient} is highly expensive to compute exactly, since it requires \(K\) matrix inverses, making the use of naive gradient descent intractable on large-scale problems. Therefore, we seek an efficient approximation to this gradient; we choose to use the first-order Neumann series:
\begin{align}\label{eq:gd_rc_neumann}
    \nabla R^{c}(\vZ \mid \vU_{[K]}) 
    &\approx \beta \sum_{k = 1}^{K}\vU_{k}(\vU_{k}\adj\vZ)\big(\vI - \beta (\vU_{k}\adj\vZ)\adj(\vU_{k}\adj\vZ)\big) \\
    &{\color{revision}= \beta \left(\sum_{k = 1}^{K}\vU_{k}\vU_{k}\adj\right)\vZ -  \beta^2\sum_{k = 1}^{K} \vU_{k}(\vU_{k}\adj\vZ)(\vU_{k}\adj\vZ)\adj(\vU_{k}\adj\vZ)}.
\end{align}
The above approximate gradient expression \Cref{eq:gd_rc_neumann} approximates the residual of each projected token feature $\vU_k^* \vz_{i}$ regressed by other token features \(\vU_{k}\adj\vz_{j}\) \citep{chan2021redunet}. 
But, differently from \citep{chan2021redunet}, not all token features in this auto-regression are from the same subspace. Hence, to compress each token feature with token features from its own group, we can compute their similarity through an auto-correlation among the projected features as $(\vU_{k}\adj\vZ)\adj(\vU_{k}\adj\vZ)$ and convert it to a distribution of membership with a softmax, namely $\softmax{(\vU_{k}\adj\vZ)\adj(\vU_{k}\adj\vZ)}$. 
Thus, as we show in more detail in \Cref{app:proofs-compression}, if we only use similar tokens to regress and denoise each other, then a gradient step on the coding rate with learning rate \(\kappa\) can be naturally approximated as follows:
{\color{revision}\begin{equation}
    \vZ^{\ell + 1/2} = (1 - \beta \kappa) \vZ^{\ell} + \beta\kappa \MSSA{\vZ^{\ell} \mid \vU_{[K]}^{\ell}} \approx \vZ^{\ell} - \kappa\nabla R^{c}(\vZ^{\ell} \mid \vU_{[K]}^{\ell}),  \label{eq:gd-mcr-parts} 
\end{equation}}
where $\texttt{MSSA}$ is defined through an $\texttt{SSA}$ operator as:
\begin{align}
    \SSA{\vZ \mid \vU_{k}} 
    &\doteq (\vU_{k}\adj \vZ)\softmax{(\vU_{k}\adj\vZ)\adj(\vU_{k}\adj\vZ)}, \quad k \in [K], \label{eq:SSA} \\
    \MSSA{\vZ \mid \vU_{[K]}} 
    &\doteq \beta \mat{\vU_{1}, \dots, \vU_{K}}\mat{\SSA{\vZ \mid \vU_{1}} \\ \vdots \\ \SSA{\vZ \mid \vU_{K}}}.\label{eq:Multi-Head-SSA}
\end{align} 
Here the \texttt{SSA} operator in \eqref{eq:SSA} resembles the \textit{attention operator} in a typical transformer \citep{vaswani2017attention}, except that here the linear operators of value, key, and query  are all set to be \textit{the same} as the subspace basis, i.e., $\vV_{k} = \vK_{k} = \Q_{k} = \vU_k^*$. {We note that recently \cite{hinton2021represent} has surmised that it is more sensible to set the ``value, key, and query''  projection matrices in a transformer to be equal. Our derivation confirms this mathematically.} Hence, we name $\SSA{\spcdot\mid\vU_k}: \bR^{d\times n} \rightarrow \bR^{p\times n}$  the \textbf{S}ubspace \textbf{S}elf-\textbf{A}ttention (SSA) operator (more details and justification can be found in \Cref{eq:appendix-ssa-derivation} in \Cref{app:proofs-compression}). Then, the whole \texttt{MSSA} operator in \eqref{eq:Multi-Head-SSA}, formally defined as \(\MSSA{\spcdot \mid \vU_{[K]}} \colon \bR^{d \times n} \to \bR^{d \times n}\) and called the \textbf{M}ulti-Head \textbf{S}ubspace \textbf{S}elf-\textbf{A}ttention (MSSA) operator, aggregates the attention head outputs by averaging using model-dependent weights, similar in concept to the popular multi-head self-attention operator in existing transformer networks. The overall gradient step \eqref{eq:gd-mcr-parts} resembles the multi-head self-attention implemented with a skip connection in transformers. 

{\color{revision}
In our implementation, we find that replacing the term \(\beta\mat{\vU_{1},\dots,\vU_{K}}\) in the \(\texttt{MSSA}\) operator \Cref{eq:Multi-Head-SSA} with another trainable parameter \(\vW \in \bR^{d \times pK}\) largely speeds up the model training and optimization. Thus the \(\texttt{MSSA}\) block becomes
\begin{equation}\label{eq:mssa_trainable_w_main_paper}
    \MSSA{\vZ \mid \vU_{[K]}, \W} \doteq \vW\mat{\SSA{\vZ \mid \vU_{1}} \\ \vdots \\ \SSA{\vZ \mid \vU_{K}}}.
\end{equation}
}

\subsection{MLP as Proximal Gradient Descent for Sparse Coding of Tokens}\label{subsec:mlp-block-arch-design}

\label{sub:sparse}

In the previous subsection, we focused on how to compress a set of token features \(\vZ^{\ell}\) against a set of low-dimensional subspaces with orthonormal bases \(\vU_{[K]}^{\ell}\), obtaining a more compressed token set \(\vZ^{\ell + 1/2}\) which approximately minimizes \(R^{c}(\vZ^{\ell + 1/2} \mid \vU_{[K]}^{\ell})\). That is, we solved \Cref{eqn:compression} from \Cref{sub:unrolled_rep_learning}: 

\begin{equation*}
    \vZ^{\ell + 1/2}\ \text{is chosen to incrementally minimize}\ R^{c}(\vZ^{\ell + 1/2} \mid \vU_{[K]}^{\ell}). \tag{\ref{eqn:compression}}
\end{equation*}
Now, it remains to choose \(\vZ^{\ell + 1}\), by solving \Cref{eqn:sparsification} from \Cref{sub:unrolled_rep_learning}:
\begin{align}
    \vZ^{\ell + 1}\ \text{is chosen to incrementally minimize}\ 
    &\lambda \norm{\vZ^{\ell + 1}}_{0} - R(\vZ^{\ell + 1}) \nonumber \tag{\ref{eqn:sparsification}} \\ 
    &= \scalebox{0.875}{\(\lambda \norm{\vZ^{\ell + 1}}_{0} - \frac{1}{2}\logdet{\vI + \alpha (\vZ^{\ell + 1})\adj(\vZ^{\ell + 1})}\)}. 
\end{align}

On top of optimizing the remaining terms in the overall sparse rate reduction objective \Cref{eq:sparse-rr}, this step also serves an important conceptual role in itself. Namely, the term \(\norm{\vZ}_{0}\) in the objective \Cref{eqn:sparsification} serves to sparsify the compressed tokens, leading to a more compact and structured (i.e., \textit{parsimonious}) representation.  In addition, the coding rate \(R(\vZ)\) in \Cref{eqn:sparsification} promotes diversity and non-collapse of the representations, a highly desirable property. 

Similarly to \Cref{sub:unrolled_rep_learning}, the gradient \(\nabla R(\vZ)\) involves a matrix inverse \citep{chan2021redunet}, and thus naive proximal gradient to solve \Cref{eqn:sparsification} becomes intractable on large-scale problems. We therefore take a different, simplifying approach to trading off between representational diversity and sparsification: we posit a (complete) incoherent or orthogonal dictionary $\vD^{\ell} \in \bR^{d \times d}$, and ask to sparsify the intermediate iterates $\vZ^{\ell + 1/2}$ with respect to \(\vD^{\ell}\). That is, $\vZ^{\ell + 1/2} \approx \vD^{\ell} \vZ^{\ell + 1}$ where $\vZ^{\ell + 1}$ is more sparse; that is, it is a \textit{sparse encoding} of \(\vZ^{\ell + 1/2}\). The dictionary \(\vD^{\ell}\) is used to sparsify all tokens simultaneously. 
By the incoherence assumption, we have $(\vD^{\ell})\adj(\vD^{\ell}) \approx \vI$. Thus from \Cref{eq:coding_rate} we have 
\begin{equation}
    R(\vZ^{\ell + 1/2}) \approx R(\vD^{\ell}\vZ^{\ell + 1}) \approx R(\vZ^{\ell + 1}).
\end{equation}
Thus we aim to solve \Cref{eqn:sparsification} with the following program:
\begin{equation}
   \vZ^{\ell + 1} \approx \argmin_{\vZ}  \norm{\vZ}_0 \quad \mbox{subject to} \quad \vZ^{\ell + 1/2} = \vD^{\ell}\vZ.
   \label{eq:sparse}
\end{equation}
{\color{revision}The above sparse representation program is usually solved by relaxing it to an unconstrained convex program, known as LASSO \citep{tibshirani1996regression,Wright-Ma-2022}:} 
\begin{equation}
    \vZ^{\ell + 1} \approx \argmin_{\vZ} \left[\lambda \norm{\vZ}_1 + \frac{1}{2}\norm{\vZ^{\ell + 1/2} - \vD^{\ell} \vZ}_F^2 \right].
\end{equation}
In our implementation, we also add a non-negative constraint to $\vZ^{\ell + 1}$, and solve the corresponding non-negative LASSO:
\begin{equation}
    \vZ^{\ell + 1} \approx \argmin_{\vZ \geq \vZero} \left[\lambda\norm{\vZ}_1 + \frac{1}{2}\norm{\vZ^{\ell + 1/2} - \vD^{\ell} \vZ}_{F}^{2}\right].
    \label{eq:sparse-nonnegative}
\end{equation}
{\color{revision}We briefly justify the non-negativity constraint here. Given the dictionary $\bm{D}^{\ell}$, the $i$-th column of $\bm{Z}^{\ell+1}$ can be interpreted as a sparse code for approximating the $i$-th token — the $i$-th column of $\bm{Z}^{\ell+1/2}$. The non-negative value in $\bm{Z}^{\ell+1}$ indicates to what extent the dictionary atom is selected or not. There are both theoretical benefits \citep{zarka2019deep,guth2021phase} and empirical benefits \citep{sun2018supervised} to this modeling decision, mostly shown on classification problems, and validated in our own experiments in Table~\ref{tab:ablation-arch-variants-soft-thresholding}.}
We incrementally optimize \Cref{eq:sparse-nonnegative} by performing an unrolled {\em proximal gradient descent} step, known as an ISTA step \citep{beck2009fast}, to give the update:
\begin{align}\label{eq:ista-block}
    \vZ^{\ell + 1} 
    &= \ISTA{\vZ^{\ell + 1/2} \mid \vD^{\ell}}, \\
    \text{where} \quad \ISTA{\vZ \mid \vD} 
    &\doteq \operatorname{ReLU}(\vZ - \eta \D\adj(\vD\vZ - \vZ) - \eta \lambda \bm{1}).
\end{align}
In \Cref{app:proofs-sparse}, we will show one can arrive at a similar operator to the above ISTA-like update for optimizing \eqref{eqn:sparsification} by properly linearizing and approximating the coding rate \(R(\vZ)\).
    
\subsection{The Overall White-Box Transformer Architecture: CRATE}
\label{sub:architecture}

\begin{figure}
     \centering
     \includegraphics[width=0.8\textwidth]{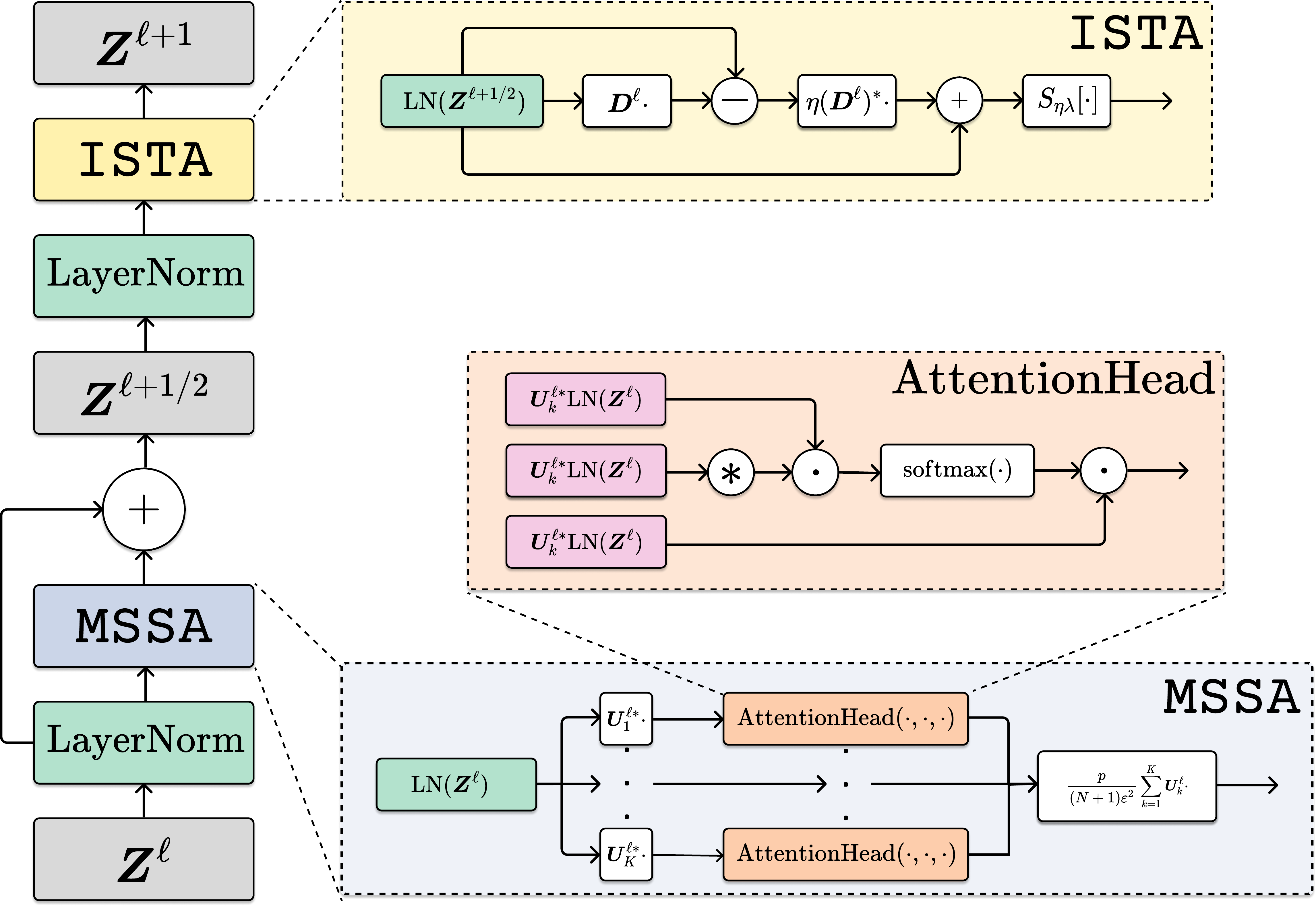}
    \caption{\small \textbf{One layer of the \ourscaps{} encoder architecture.} The full architecture is simply a concatenation of such layers, with some initial tokenizer, pre-processing head, and final task-specific head (i.e., a classification head).
    }
    \label{fig:arch}
\end{figure}

By combining the above two steps:
\begin{enumerate}[leftmargin=0.7cm]
    \item (\Cref{sub:compression}) Local compression of tokens within a sample towards a mixture-of-subspace structure, leading to the multi-head subspace self-attention block -- \texttt{MSSA};
    \item (\Cref{sub:sparse}) Global sparsification of token sets across all samples through sparse coding, leading to the sparsification block -- \texttt{ISTA};
\end{enumerate}
we can get the following rate-reduction-based transformer layer, illustrated in \Cref{fig:arch}, 
\begin{equation}
    \vZ^{\ell+1/2} \doteq \vZ^{\ell} + \texttt{MSSA}(\vZ^{\ell} \mid \vU_{[K]}^{\ell}), 
    \qquad 
    \vZ^{\ell+1}\doteq \texttt{ISTA}(\vZ^{\ell+1/2} \mid \D^\ell).
\end{equation}
Composing multiple such layers following the incremental construction of our representation in \Cref{eq:incremental}, we obtain a white-box transformer architecture that transforms the data tokens towards a compact and sparse union of incoherent subspaces. An overall flow of this architecture was shown in \Cref{fig:encoder_compression_sparsification}.

\begin{remark}[\textbf{Design choices in \ourscaps{}}]
    We note that in this work, at each stage of our network construction, we have chosen arguably the \textit{simplest possible} construction to use. We can substitute each part of this construction, so long as the new part maintains the same conceptual role, and obtain another white-box architecture. Nevertheless, our such-constructed architecture, called \ours{}, connecting to existing transformer models, is not only fully mathematically interpretable, but also obtains competitive results on real-world datasets, as we will see in Section \ref{sec:exp}.
\end{remark}

\begin{remark}[\textbf{The roles of the forward pass and backward propagation}]\label{sub:forward_backward}
    In contrast to other unrolled optimization approaches such as the ReduNet \citep{chan2021redunet}, we \textit{explicitly model} the distribution of each $\vZ^\ell$ and $\vZ^{\ell + 1/2}$ at each layer, either by a mixture of linear subspaces or sparsely generated from a dictionary. In \Cref{sub:unrolled_rep_learning}, we introduced the interpretation that at each layer \(\ell\), the learned bases for the subspaces \(\vU_{[K]}^{\ell}\) and the learned dictionaries \(\vD^{\ell}\) together serve as a \textit{codebook} or \textit{analysis filter} that encodes and transforms the intermediate representations at each layer \(\ell\). Since the input distribution to layer \(\ell\) is first modeled by \(\vU_{[K]}^{\ell}\) then transformed by \(\vD^{\ell}\), the input distribution to each layer is different, and so we require a separate code book at each layer to obtain the most parsimonious encoding. Parameters of these codebooks (i.e., the subspace bases and the dictionaries), heretofore assumed as being perfectly known, are actually learned from data (say via \textit{backward propagation} within end-to-end training).
    
    Hence, our methodology features a clear conceptual separation between forward ``optimization'' and backward ``learning'' for  the so-derived white-box deep neural network. Namely, in its forward pass, we interpret each layer as an operator which, conditioned on a learned model (i.e., a codebook) for the distribution of its input, transforms this distribution towards a more parsimonious representation. In its backward propagation, the codebook of this model, for the distribution of the input to each layer, is updated to better fit the input-output relationship. This conceptual interpretation implies a certain agnosticism of the model representations towards the particular task and loss; in particular, many types of tasks and losses will ensure that the models at each layer are fit, which ensures that the model produces parsimonious representations. To wit, we show in the sequel (\Cref{sec:exp}) that the \ours{} architecture promotes parsimonious representations and maintains layer-wise white-box operational characteristics on several different tasks, losses, and modalities.
\end{remark}

%% file: sec_autoencoding.tex
\section{White-Box Decoding via Structured Denoising and  Diffusion}\label{sec:autoencoding}

In \Cref{sec:encoding}, we have presented a principled metric for measuring the quality of learned representations---the sparse rate reduction \Cref{eq:sparse-rr}---and showed how to derive, via incremental optimization of this objective, a white-box transformer architecture (\ours{}) for general representation learning of high-dimensional data.
Conceptually, this corresponds to a (compressive) \textit{encoder}: $$f : \vX \rightarrow \vZ,$$ mapping high-dimensional data to representations preserving the distinct modes of intrinsic variability of the data.

For numerous reasons, ranging from being able to use the learned representations $\vZ$ for generation and prediction to having flexible avenues to learn the parameters $(\vU_{[K]}^{\ell})_{\ell \in [L]}$ of the white-box encoder $f$ from data, it is highly desirable to have a corresponding construction of a \textit{decoder}: $$g : \vZ \rightarrow  \vXh,$$ mapping the representations to approximations $\vXh$ of the original data distribution. 
However, it is challenging to construct a white-box decoder purely following the unrolled optimization framework that we have presented and exploited
in \Cref{sec:encoding} to derive the \ours{} encoder. Previous works, including notably the ReduNet of~\citet{chan2021redunet},
obtain white-box architectures for encoding only; on the other hand, models that have incorporated a decoder for learning (self-)consistent representations via autoencoding and \textit{closed-loop transcription} \citep{dai2022ctrl}, including in unsupervised settings, have leveraged
black-box deep network architectures for both the encoder $f$ and the decoder $g$ \citep{dai2023closed}, or limited-capacity architectures for the decoder $g$ \citep{tolooshams2021stable}.
Can compression alone, measured through the sparse rate reduction \Cref{eq:sparse-rr}, be used to derive a white-box decoder
architecture? And in such a white-box decoder architecture, what are the relevant operators for recovering the data distribution $\vXh \approx \vX$ from the representation $\vZ$, and can they be related to the operators in the encoder $f$?

In this section, we will resolve both of these fundamental questions affirmatively. We do this by establishing a powerful connection between \textit{compression}, around which we have derived the \ours{} encoder architecture, and \textit{diffusion-denoising}, the mathematical processes by which a data distribution is transformed into pure noise by incremental corruptions, and then recovered incrementally, using information about the data distribution at each corruption level. \Cref{fig:structured_diffusion} illustrates this connection with an intuitive example. This connection allows us to interpret the layers of the \ours{} encoder, which we have shown in \Cref{sec:encoding} perform compression against learnable local signal
models, say following \Cref{model:gaussian_tokens}, as performing denoising against the signal model. Since we are denoising against a highly structured input distribution, we call this process ``\textit{structured denoising}''. Given the model, this structured denoising process can be reversed in order to incrementally reconstruct the data distribution across several layers---we call this process ``\textit{structured diffusion}'', analogously but not identically to the denoising-diffusion process which underlies diffusion models. The structured denoising-diffusion processes naturally supply the construction of the first white-box decoder architecture for end-to-end representation learning.

\begin{figure}
    \centering
    \includegraphics[width=0.9\textwidth]{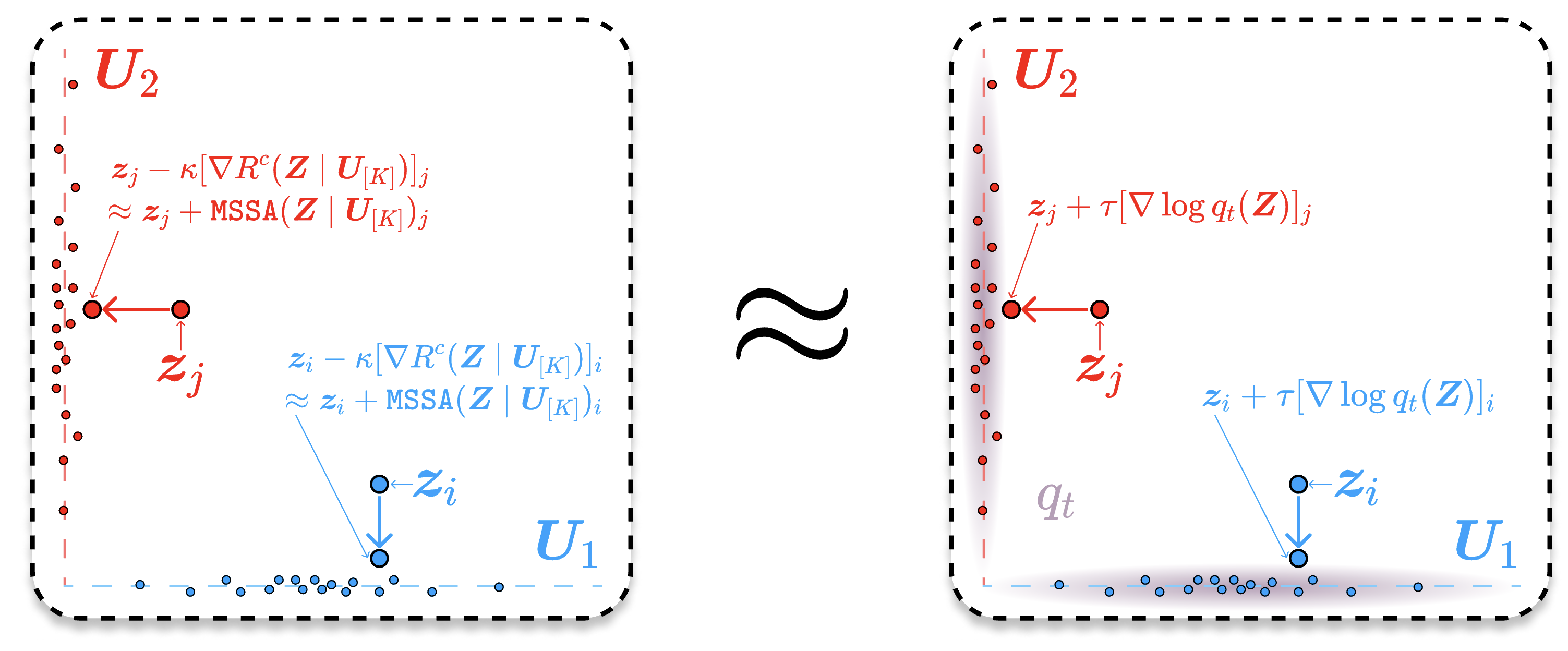}
    \caption{\small \textbf{Compression and denoising against the low-dimensional Gaussian mixture token model \Cref{model:gaussian_tokens} are equivalent.} \textit{Left:} the effect of compression against the low-dimensional Gaussian mixture model for tokens \Cref{model:gaussian_tokens}, i.e., taking gradient steps on the coding rate \(R^{c}(\cdot \mid \vU_{[K]})\)---or equivalently, using the \(\MSSA{\cdot \mid \vU_{[K]}}\) operator---which is shown in \Cref{thm:informal_rate_score} to be equivalent to projecting onto the \(\vU_{[K]}\). \textit{Right:} the effect of denoising against \Cref{model:gaussian_tokens}, i.e., taking gradient steps on the score function of the noisy model \Cref{model:gaussian_tokens_noise} at small noise levels \(\sigma\), or equivalently small times \(t\). Up to scaling factors (not pictured), these two operations are equivalent, and in any case have similar geometric and statistical interpretations as a projection onto the support of the data distribution. This connection motivates our structured denoising-diffusion framework, as elaborated in \Cref{sub:structured_diffusion}.}
    \label{fig:structured_diffusion}
\end{figure}

\subsection{Denoising-Diffusion against Low-Dimensional Structures}\label{sub:denoising}

In \Cref{sec:encoding}, we derived each layer $f^{\ell}$ of the encoder $f$ via \textit{compression} of the token distribution against a local signal model (i.e., the model \Cref{model:gaussian_tokens}), and \textit{sparsification} in the standard basis. 
To derive a corresponding white-box decoder $g$, we will make a connection between compression and \textit{denoising}, a problem with a rich mathematical theory and powerful implications for practical representation learning. 
In this section, we review the fundamental concepts of this theory in order to motivate our later developments.

\paragraph{One-step denoising via 
Tweedie's formula.} Consider, for simplicity, a single token \(\vz_{\natural}^{\ell}\) which has a particular marginal distribution, and define a noisy observation \(\vz^{\ell} \doteq \vz_{\natural}^{\ell} + \sigma^{\ell}\vw\), where \(\sigma^{\ell} > 0\) is a positive noise level, and \(\vw\) is a standard Gaussian vector independent of \(\vz_{\natural}^{\ell}\).
We imagine that $\vz_{\natural}^{\ell}$ represents the marginal distribution of any token at layer $\ell$ of the encoding process, and \(\vz^{\ell}\) has the same interpretation subject to a (small) Gaussian corruption.
To \textit{denoise} the observation $\vz^{\ell}$ is to recover, up to statistical limits, the signal (given by \Cref{model:gaussian_tokens}, which we will write here as $\vz_{\natural}^\ell$) from the noisy observation $\vz^{\ell}$.\footnote{In representation learning, we typically think of $\vz^{\ell}$ not as an ``observation'', but as a small perturbation off of the target model, whose structure matches our desiderata for representation learning. Similarly, rather than ``recovery'' of the structure from noisy observations, we are concerned with transforming the current distribution of the data to be closer to the target model. We will see in the next section how compression provides the bridge between these two perspectives; accordingly, we describe the denoising problem using language specific to either perspective according to context.}
In the mean-square sense, the optimal estimate is $\E{\vz_{\natural}^{\ell} \mid \vz^{\ell}}$. %
Letting $\vz \mapsto q^{\ell}(\vz)$ denote the density of $\vz^{\ell}$,\footnote{We emphasize that $q^{\ell}$ depends on the noise level $\sigma^{\ell}$, although we suppress this in the notation for concision.} 
Tweedie's formula \citep{Efron2011-wn} allows us to express this %
in closed-form:
\begin{equation}
    \E{\vz_{\natural}^{\ell} \mid \vz^{\ell}}
    = \vz^{\ell} + (\sigma^{\ell})^2 \nabla \log
    q^{\ell}(\vz^\ell).
    \label{eq:tweedie}
\end{equation}
Tweedie's formula expresses the optimal representation in terms of an additive correction (in general a nonlinear function of $\vz^{\ell}$) to the noisy observations by the gradient of the \textit{log-likelihood} of the distribution of the noisy observations, also known as the \textit{score function} \(\nabla \log q^{\ell}\) \citep{Hyvarinen2005-fi}. One may interpret Tweedie's formula as denoising via a gradient ascent step on the score function at noise level \(\sigma^{\ell}\). This connection is well-known in the areas of estimation theory and inverse problems \citep{Efron2011-wn,Stein1981-jv,Raphan2011-by,Milanfar2013-js,Kadkhodaie2020-fc,Venkatakrishnan2013-rt,Romano2017-rp}, and more recently has found powerful applications in the training of generative models for natural images \citep{Hyvarinen2005-fi,Vincent2011-dr,Sohl-Dickstein2015-kz,Song2020-xo,Song2020-hb}.

The practical question, of course, is then whether it is possible to efficiently {\em learn to denoise}. The additive correction with score function in \Cref{eq:tweedie} 
depends on the current noise level and the token distribution, and for general high-dimensional distributions (such as those of natural images, as above), this token distribution is unknown and can be prohibitively costly to compute. Nevertheless, in practice, the score function is often empirically modeled and approximated with a neural network (say a transformer), or another \textit{nonparametric} estimator, and estimated with a large number of samples and huge amounts of computation. Despite the empirical success of such diffusion-denoising methods in learning distributions of images \citep{rombach2022high}, there has been little  theoretical justification for why transformer-like architectures would be effective to model such score functions.

\paragraph{Denosing against a low-dimensional Gaussian mixture.} In the work of~\citet{Hyvarinen2005-fi}, the score function is used to learn a data distribution from a restricted \textit{parametric} family. As shown by \citet{Hyvarinen2005-fi}, for certain broad classes of parametric families, the score function is efficiently computable, e.g.\ for a mixture of Gaussians, independent component analysis models, over-complete dictionary learning, etc. Here (i.e., in this section and hereafter), we follow the same methodology. Namely, suppose that \(\vz_{\natural}^{\ell}\) has the low-dimensional Gaussian mixture distribution outlined in \Cref{model:gaussian_tokens}, so that \(\vz^{\ell}\) has the distribution outlined in \Cref{model:gaussian_tokens_noise} with noise level \(\sigma^{\ell}\). In this case, we can obtain a closed-form expression for
the {score function} $\nabla \log q^{\ell}$, which, when combined with Tweedie's formula \Cref{eq:tweedie} and some mild technical assumptions, %
gives the following approximation (shown in \Cref{app:proofs-denoising}):
\begin{equation}
    \E{\vz_{\natural}^{\ell} \mid \vz^{\ell}}
    \approx \mat{\vU_{1}, \dots, \vU_{K}}\left[\diag\mathopen{}\left(\softmax{\frac{1}{2(\sigma^{\ell})^{2}}\mat{\norm{\vU_{1}\adj\z^{\ell}}_{2}^{2} \\ \vdots \\ \norm{\vU_{K}\adj\z^{\ell}}_{2}^{2}}}\right) \otimes \vI\right]\mat{\vU_{1}\adj\z^{\ell} \\ \vdots \\ \vU_{K}\adj\z^{\ell}},
    \label{eq:mog-score-as-attention}
\end{equation}
where \(\otimes\) denotes the {\em Kronecker} product. In the small-noise limit $\sigma^{\ell} \to 0$, the operator implemented by \Cref{eq:mog-score-as-attention} becomes \textit{a projection of the observation $\vz^{\ell}$ onto the support of the distribution of the signal model $\vz_{\natural}^{\ell}$}, a significant characterization of the local behavior of denoising against the signal model \Cref{model:gaussian_tokens}.
Moreover, perhaps surprisingly, this operation is quite similar to the \texttt{MSSA} block derived in \Cref{sub:compression}, specialized to the case $n = 1$.
Indeed, the operation in \Cref{eq:mog-score-as-attention} resembles a self-attention layer in a standard transformer architecture with \(K\) heads, sequence length \(n = 1\), and the ``query-key-value'' constructs being replaced by a single linear projection \(\vU_{k}\adj\z^{\ell}\) of the token \(\z^{\ell}\). 

\paragraph{Stochastic denoising process.} The above approach only denoises the token \(\vz^{\ell}\) once. Much of the practical power of denoising via the score function, however, stems from the ability to \textit{iteratively denoise in small increments}. 
Starting with the token \(\vz^{\ell}\), given access to score functions of the distribution of \(\vz_{\natural}^{\ell}\) perturbed at at all noise levels up to \(\sigma^{\ell}\), iterative denoising of \(\vz^{\ell}\) produces \textit{new samples from the noiseless distribution of tokens $\vz_{\natural}^{\ell}$}. By Tweedie's formula \Cref{eq:tweedie}, this means that denoising $\vz^{\ell}$ is equivalent to representing the signal $\vz_{\natural}^\ell$ in a precise distributional sense.
In a simple instantiation, this representation process takes the following form \citep{Song2020-xo}. First, consider a \textit{diffusion process}, indexed by time $t \in [0, T]$ for $T = (\sigma^{\ell})^{2} >0$, which transforms the distribution of $\vz_{\natural}^{\ell}$ towards the noisy distribution of \(\vz^{\ell}\):
\begin{equation}
\begin{split}
    \diff \vz_t &= \diff \vw_t, \quad t \in [0, T], \\
    \vz_0 &\equid \vz_{\natural}^{\ell}.
\end{split}
\label{eq:ve-fwd-decoder}
\end{equation}
Here, $(\vw_t)_{t \in [0, T]}$ is a Wiener process, and we express this process in \Cref{eq:ve-fwd-decoder} as a stochastic differential equation (SDE); for background on SDEs, see \Cref{app:diffusion}. This SDE has a unique (strong, i.e., pathwise well-defined) solution which has distribution $\vz_t \equid \vz_{\natural}^{\ell} + \vw_{t}$. Recalling that \((\vw_{t})_{t \in [0, T]}\) is a Wiener process, \(\vw_{t}\) is unconditionally distributed as \(\sN(\vZero, t\vI)\), so that \(\vz_{T} = \vz_{(\sigma^{\ell})^{2}} \equid \vz^{\ell}\). %
As above, we write $q_t$ to denote the density of $\vz_t$. Then by the theory of time reversal for diffusion processes \citep{Haussmann1986-wa,Millet1989-ee}, the random process \((\vz_{t}^{\leftarrow})_{t \in [0, T]}\), where $\vz_{t}^{\leftarrow} \doteq \vz_{T - t}$, uniquely solves the following SDE: %
\begin{equation}
\begin{split}
    \diff \vz_t^{\leftarrow} &= \nabla \log q_{T - t}(\vz_t^{\leftarrow}) \diff t + \diff \vw^{\leftarrow}_t, \quad t \in [0, T], \\
    \vz_0^{\leftarrow} &\equid \vz^{\ell},
\end{split}
\label{eq:ve-bwd-decoder}
\end{equation}
where $\vw^{\leftarrow}_t$ is another Wiener process.\footnote{In the mathematical literature, both \Cref{eq:ve-fwd-decoder} and \Cref{eq:ve-bwd-decoder} are classified as (Markov) diffusion processes \citep{Bakry2016-pl}. By virtue of \Cref{eq:tweedie}, in this work we will refer to \Cref{eq:ve-fwd-decoder} as ``diffusion'' and \Cref{eq:ve-bwd-decoder} as ``denoising''.} Because $(\vz_{T-t})_{t \in [0, T]}$ solves \Cref{eq:ve-bwd-decoder}, it follows that this process yields a representation (via sampling) for $\vz_{\natural}^{\ell}$, as promised.
Crucially, it can be rigorously shown that an iterative denoising-diffusion process such as \Cref{eq:ve-bwd-decoder} is both necessary and sufficient for representing high-dimensional multimodal data distributions efficiently \citep{Koehler2022-ed,Ge2018-ny,Qin2023-wn}.

\paragraph{Deterministic denoising process.} In \Cref{eq:ve-bwd-decoder}, \(\vz_{t}^{\leftarrow}\) follows a noisy gradient flow to maximize (log-)likelihood of the current iterate against a smoothly time-varying, i.e.,  ``diffused,'' probabilistic model for the data distribution. We observe that each infinitesimal update of \Cref{eq:ve-bwd-decoder} is similar to Tweedie's formula \Cref{eq:tweedie}, which takes a single gradient step on a ``diffused'' log-likelihood to denoise. Thus, we interpret the process \Cref{eq:ve-bwd-decoder} as a \textit{stochastic} denoising process. In practice, and in particular towards our design of a corresponding \ours{} decoder, it will be useful to have a \textit{deterministic} analogue of this process. The probability flow ODE affords such a process: following \citet{Song2020-xo}, the dynamics of the probability density of \(\vz_{t}^{\leftarrow}\) in \Cref{eq:ve-bwd-decoder} is identical to that of the ODE:  
\begin{equation}
\begin{split}
    \diff \vz_t^{\leftarrow} &= \frac{1}{2} \nabla \log q_{T - t}(\vz_t^{\leftarrow}) \diff t, \quad t \in [0, T], \\
    \vz_0^{\leftarrow} &\equid \vz^{\ell}.
\end{split}
\label{eq:ve-bwd-decoder-pfode}
\end{equation}
It is significant that the representation for $\vz_{\natural}^{\ell}$ afforded by the deterministic process \Cref{eq:ve-bwd-decoder-pfode} can be characterized simply (again using \Cref{eq:tweedie}) as iterative denoising, across multiple noise scales. This leads to the core insight of diffusion-denoising:
\begin{quote}
\centering
   \textit{Denoising is equivalent to learning a representation of the data distribution.} 
\end{quote}

With these preliminaries in hand, we will 
develop a deeper link between compression and denoising that we can leverage to build a consistent encoder-decoder pair $f, g$ for general ($n > 1$) sets of tokens in the next section. %

\subsection{Parsimony and Consistency via Structured Denoising-Diffusion}\label{sub:structured_diffusion}

In \Cref{sub:denoising}, we have presented the core ideas of the denoising-diffusion theory in the context of the token marignal distribution of our signal model \Cref{model:gaussian_tokens}.
We now significantly extend the applicability of this theory, by connecting it to the compression framework we have introduced in \Cref{sec:encoding}.
We will use this extension to define our \textit{structured diffusion-denoising} paradigm, around which we derive a corresponding white-box decoder architecture for the encoder $f$. %

To this end, we study in detail a special case of the signal model \Cref{model:gaussian_tokens}: we assume that the 
indices $s_{i}$ specifying the subspace membership of each token $\vz_i^{\ell}$ are i.i.d.\ random variables, taking values in the set $\set{1, \dots, K}$ with equal probability, independently of the noises $\vw_i$ and the coefficients $\valpha_i$.
We have seen in \Cref{eq:mog-score-as-attention} in the previous subsection that for such a model with vanishing noise level $\sigma^{\ell} > 0$, denoising against the token marginal distribution implements a projection onto the support of the noiseless distribution.
We prove that this same property is shared by taking a gradient step on the compression objective $R^c(\vZ^{\ell} \mid \vU_{[K]}^{\ell})$, as in the construction of the \texttt{MSSA} block in our white-box encoder in \Cref{sub:compression}, confirming the qualitative picture in \Cref{fig:structured_diffusion}:

\begin{theorem}[Informal version of \Cref{lem:inverse-term} in \Cref{app:computations_rr_gradient}]\label{thm:informal_rate_score}
Suppose $\vZ^{\ell}$ follows the noisy Gaussian codebook model \Cref{model:gaussian_tokens_noise}, with infinitesimal noise level $\sigma^{\ell} > 0$ and subspace memberships
$s_{i}$ distributed as i.i.d.\ categorical random variables on the set of subspace indices $\set{1, \dots, K}$, independently of all other sources of randomness.
Suppose in addition that the number of tokens $n$, the representation dimension $d$, the number of subspaces $K$, and the subspace dimensions $p$ have relative sizes
matching those of practical transformer architectures including the \ours{} encoder (specified in detail in \Cref{ass:parameter_config}).
Then the negative compression gradient \(-\nabla_{\vz_{i}}R^{c}(\vZ^{\ell} \mid \vU_{[K]}^{\ell})\) points from $\vz_i^{\ell}$ to the nearest \(\vU_{k}^{\ell}\).
\end{theorem}

\Cref{thm:informal_rate_score} establishes in a representative special case of the Gaussian codebook model \Cref{model:gaussian_tokens} that at low noise levels, \textit{compression against the local signal model $\vU_{[K]}^{\ell}$ 
is equivalent to denoising against the local signal model}. 
Viewed through the lens of the deterministic denoising process \Cref{eq:ve-bwd-decoder-pfode}, this establishes a link between the gradient of the compression term $R^c$ and the score function for the Gaussian codebook model.
Most importantly, this allows us to understand the \texttt{MSSA} operators of the \ours{} encoder, derived in \Cref{sub:compression} from a different perspective, as realizing an incremental transformation of the data distribution towards the local signal model, via denoising.
This important property guarantees that a corresponding deterministic diffusion process---namely, the time reversal of the denoising process \Cref{eq:ve-bwd-decoder-pfode}---implies an inverse operator for the compression operation implemented by \texttt{MSSA}.

\paragraph{Using this connection to construct a principled autoencoder.} The above result suggests the following approach to constructing white-box autoencoding networks. 
Given the representation of the token distribution at layer $\ell$, namely $\vZ^{\ell}$, we construct a deterministic structured denoising process, identical to \Cref{eq:ve-bwd-decoder-pfode}, 
which compresses the data towards the local signal model at layer $\ell$ of the representation $f$, namely $\vU_{[K]}^{\ell}$. %
Using the equivalence asserted in \Cref{thm:informal_rate_score}, we can express this structured denoising process in terms of $R^c$, on small timescales $T>0$ (as we work out in detail in \Cref{app:structured-diffusion-calcs}): 
\begin{equation}\label{eq:compression_ode}
    \odif{\vZ^{\ell}(t)} = -\frac{1}{2(T-t)}\nabla R^{c}(\vZ^{\ell}(t) \mid \vU_{[K]}^{\ell})\odif{t}. 
\end{equation}
This process interpolates between the signal model \Cref{model:gaussian_tokens}, at $t = 0$, to a noisy version of the signal model at $t = T$.
By the same token, time reversal of diffusion processes gives a \textit{structured diffusion} process, which transforms the signal model to 
an incrementally more noisy version:
\begin{equation}\label{eq:compression_ode_diffusion}
    \odif{\wt{\vZ}^{\ell}(t)} = \frac{1}{2t}\nabla R^{c}(\wt{\vZ}^{\ell}(t) \mid \vU_{[K]}^{\ell})\odif{t}.
\end{equation}
These two processes are inverses of one another in a distributional sense. 
To use these structured denoising and diffusion processes for representation learning, we may ambitiously treat the 
first-layer distribution $\vZ^{1} = f^{\pre}(\vX)$ itself as being a small deviation off the distribution of the first local signal model $\vU_{[K]}^1$. 
In this way, the incrementally-constructed representation \Cref{eq:incremental}, which we have been studying at a single ``layer''
$\ell$ thus far, naturally leads to the following (completely formal) structured denoising process in which layer index $\ell$ and time $t$ are unified
into a single parameter, and where $\vZ(0) = \vZ^1$ is the preprocessed data distribution:
\begin{equation}\label{eq:structured-denoising-conceptual}
    \odif{\vZ(t)}= -\frac{1}{2(T - t)}\nabla R^{c}(\vZ(t) \mid \vU_{[K]}(t))\odif{t}.
\end{equation}
Similarly, we have the (completely formal) inverse process, a structured diffusion process:
\begin{equation}\label{eq:structured-diffusion-conceptual}
    \odif{\wt{\vZ}(t)}= \frac{1}{2t}\nabla R^{c}(\wt{\vZ}(t) \mid \vU_{[K]}(T-t))\odif{t}.
\end{equation}
These two equations provide a conceptual basis for transforming data to and from a structured, parsimonious representation, via the denoising-diffusion theory. On the one hand, their similar functional forms---unified through compression, via the compression gradient $\nabla R^c$ and \Cref{thm:informal_rate_score}---demonstrate that the operators necessary for structured denoising and structured diffusion take essentially the same form.
On the other hand, the connection we have made in \Cref{sub:compression} between the gradient of the compression term of the rate reduction objective and transformer-like network layers implies that a transformer-like architecture is sufficient for both compressive encoding and decoding.
We can therefore realize compressive autoencoding with a fully mathematically-interpretable network architecture, which we now instantiate.

\subsection{Structured Denoising-Diffusion via Invertible Transformer Layers}\label{sub:structured_diffusion_impl}

In \Cref{sec:encoding}, we described a method to construct a white-box transformer-like encoder network via unrolled optimization meant to compress the data against learned geometric and statistical structures, say against a distribution of tokens where each token is marginally distributed as a Gaussian mixture supported on \(\vU_{[K]}\). 
Armed with the structured diffusion-denoising correspondence outlined in \Cref{sub:structured_diffusion} and its connection to compression, we now 
generalize the \ours{} compressive encoder architecture to a full encoder-decoder pair, with essentially identical operators transforming the data distribution from layer to layer, and thus similarly interpretable operational characteristics.

\paragraph{A white-box encoder layer which implements structured denoising.} Recall that in \Cref{sub:structured_diffusion}, we established a continuous-time deterministic dynamical system which implements \textit{structured denoising}, in that it denoises the initial data towards the desired parsimonious structure:
\begin{equation}\label{eq:structured_denoising_ode_prediscretized}
    \odif{\vZ(t)}= -\frac{1}{2(T-t)}\nabla R^{c}(\vZ(t) \mid \vU_{[K]}(t))\odif{t}.
\end{equation}
In order to construct a network architecture, we use a first-order discretization (the discretization scheme being another design choice) of this process, 
which we describe in detail in \Cref{app:structured-diffusion-calcs}. This obtains the iteration
\begin{align}\label{eq:structured_denoising_ode_discretized}
    \vZ^{\ell + 1/2} 
    &\approx \vZ^{\ell} - \kappa \nabla R^{c}(\vZ^{\ell} \mid \vU_{[K]}^{\ell}), \\
    \text{in particular}\ \vZ^{\ell + 1/2}\label{eq:compression_structured_denoising}
    &=  \vZ^{\ell} + \MSSA{\vZ^{\ell} \mid \vU_{[K]}^{\ell}},
\end{align}
where \(\MSSA{\cdot}\) was defined in \Cref{eq:Multi-Head-SSA}. In order to perform structured denoising while ensuring that the representation structures (e.g., supporting subspaces) themselves are parsimonious (i.e., sparse), similar to \Cref{sec:encoding}, we insert a sparsification step for the features. Namely, we instantiate a learnable dictionary \(\vD^{\ell} \in \bR^{d \times d}\) and sparsify against it, obtaining 
\begin{align}
    \vZ^{\ell + 1}
    &\approx \argmin_{\vZ \geq \vZero}\left[\lambda \norm{\vZ}_{1} + \frac{1}{2}\norm{\vZ^{\ell + 1/2} - \vD^{\ell}\vZ}_{F}^{2}\right], \\
    \text{in particular}\ \vZ^{\ell + 1}\label{eq:sparsification_structured_denoising}
    &= \ISTA{\vZ^{\ell + 1/2} \mid \vD^{\ell}},
\end{align}
where \(\ISTA{\cdot}\) was defined in \Cref{eq:ista-block}. This yields a two step iteration for the \(\ell\th\) encoder layer \(f^{\ell}\), where \(\vZ^{\ell + 1} = f^{\ell}(\vZ^{\ell})\):
\begin{equation}\label{eq:encoder_layer}
    \vZ^{\ell + 1/2} = \vZ^{\ell} + \MSSA{\vZ^{\ell} \mid \vU_{[K]}^{\ell}}, \qquad \vZ^{\ell + 1} = \ISTA{\vZ^{\ell + 1/2} \mid \vD^{\ell}}.
\end{equation}
This is the same layer as in the \ours{} encoder, whose conceptual behavior we have illustrated in \Cref{fig:compression_sparsification_oneiter}. Thus, we have re-derived the \ours{} encoder layer from a useful alternate perspective, that of structured denoising. Namely, we have shown an equivalence between structured denoising and unrolled optimization in this case, which stems from the fact that the diffusion probability flow \Cref{eq:structured-denoising-conceptual} is conceptually and mechanically similar to gradient flow on the compression objective in certain regimes. 
Thus, we have demonstrated a useful conceptual connection between discretized diffusion processes and unrolled optimization as iteratively compressing or denoising the signal against the learned data structures.

\paragraph{A white-box decoder layer which implements structured diffusion.} Previously, we have shown how the structured denoising approach can recover the \ours{} encoder architecture, derived through unrolled optimization in \Cref{sec:encoding}, as an encoder. Now, we go beyond the purview of unrolled optimization. Namely, we use the connections concretely described in \Cref{sub:structured_diffusion} to construct a decoder which implements \textit{structured diffusion}.

Recall that the encoder is constructed from a discretization of the \textit{structured denoising} ODE given in \Cref{eq:structured-denoising-conceptual}. Its pathwise time reversal, which inverts the transformation of the data distribution induced by the structured denoising ODE, is given formally by the \textit{structured diffusion} ODE:
\begin{equation}\label{eq:structured_diffusion_ode_prediscretized}
    \odif{\wt{\vZ}(t)}= \frac{1}{2t}\nabla R^{c}(\wt{\vZ}(t) \mid \vU_{[K]}(T-t))\odif{t}.
\end{equation}
For this reason, we use the structured diffusion ODE as the backbone of our decoder.

Indeed, a first-order discretization of this ODE %
obtains the iteration:
\begin{equation}\label{eq:structured_diffusion_discretized}
    \wt{\vZ}^{\ell + 1} = \wt{\vZ}^{\ell + 1/2} - \MSSA{\wt{\vZ}^{\ell + 1/2} \mid \vV_{[K]}^{\ell}} \approx \wt{\vZ}^{\ell + 1/2} + \kappa \nabla R^{c}(\wt{\vZ}^{\ell + 1/2} \mid \vV_{[K]}^{\ell}),
\end{equation}
where \(\vV_{[K]}^{\ell} = (\vV_{1}^{\ell}, \dots, \vV_{K}^{\ell})\) and each \(\vV_{k}^{\ell} \in \bR^{d \times p}\) are the bases of the subspaces to ``anti-compress'' against. 
To invert the effect of a sparsifying \(\ISTA{\cdot}\) step, %
we instantiate a learnable \textit{synthesis} dictionary \(\vE^{\ell} \in \bR^{d \times d}\) and multiply by it, obtaining the iteration:
\begin{equation}
    \wt{\vZ}^{\ell + 1/2} = \vE^{\ell}\wt{\vZ}^{\ell}, \qquad \wt{\vZ}^{\ell + 1} = \wt{\vZ}^{\ell + 1/2} - \MSSA{\wt{\vZ}^{\ell + 1/2} \mid \vV_{[K]}^{\ell}}.
\end{equation}
This constructs the \((\ell + 1)^{\mathrm{st}}\) layer \(g^{\ell}\) of our decoder as 
\begin{equation}\label{eq:def_decoder_layer}
    \wt{\vZ}^{\ell + 1} = g^{\ell}(\wt{\vZ}^{\ell}) \doteq \vE^{\ell}\wt{\vZ}^{\ell} - \MSSA{\vE^{\ell}\wt{\vZ}^{\ell} \mid \vV_{[K]}^{\ell}}.
\end{equation}
A graphical depiction of the encoder and decoder layers, with layer normalization added to match the implementation, is found in \Cref{fig:crate_mae_layers}.

\begin{remark}[\textbf{Lack of weight-tying in autoencoding architecture}]
    Our derivation suggests simple pre-set values for \(\vV_{[K]}^{\ell}\) and \(\vE^{\ell}\), i.e., \(\vU_{[K]}^{L - \ell}\) and \((\vD^{\ell})\adj\) respectively. However, we do \textit{not} enforce these choices, or in any way share the weights between the encoder and decoder \textit{a priori} in our implementation. This is because the discretization of the structured denoising and structured diffusion dynamics in \Cref{eq:structured-denoising-conceptual,eq:structured-diffusion-conceptual} is good yet imperfect; similarly, the rotation perspective on the \(\ISTA{\cdot}\) (i.e., from \Cref{sub:compression}) is good yet imperfect, so the such-derived inversion above is also imperfect. Thus we should not expect a \(1\)-\(1\) correspondence between codebooks in the encoder and decoder. Correspondingly, each encoder layer and its corresponding decoder layer are only approximate mutual inverses.
\end{remark}

\begin{remark}[\textbf{Alternate instantiations of structured denoising-diffusion}]
    In this section, to construct a decoder we established a specific connection between denoising against the model \Cref{model:gaussian_tokens} and compression using a gradient step on the coding rate \(R^{c}\). Mechanically, this connection was used to justify an inversion of the \(\texttt{ISTA}\) compression operator defined in \Cref{sec:encoding}. However, as indicated in \Cref{rmk:energy_based_models}, one may interpret the whole sparse rate reduction \Cref{eq:objective-sparse-rate-reduction-l1} as a compression objective. Understanding to what degree an analogue of \Cref{thm:informal_rate_score} holds for the gradient of the sparse rate reduction would be quite interesting; in particular, it would inform an analogue of this structured denoising-diffusion framework which could be used to construct potentially more powerful white-box transformer-like architectures from the sparse rate reduction. In this work, however, we use arguably the simplest possible instantiation of our structured denoising-diffusion framework, in that we make the simplest possible connection between denoising and compression.
\end{remark}

\begin{figure}
    \centering
    \includegraphics[width=1.0\textwidth]{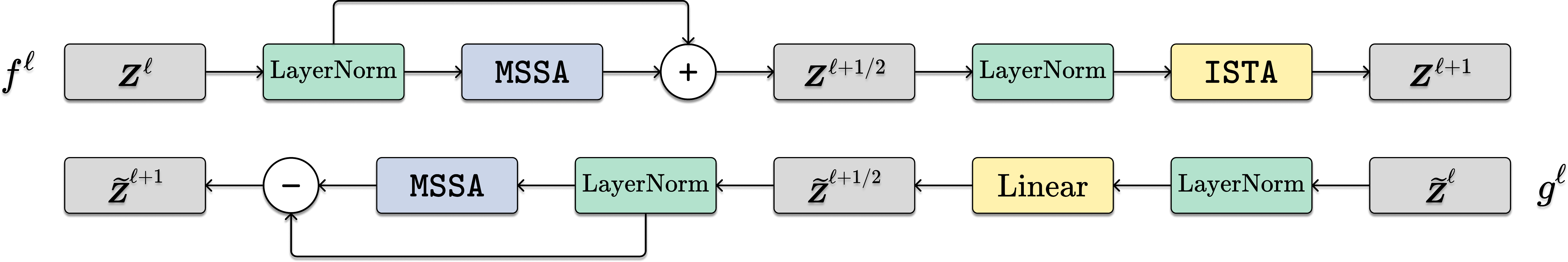}
    \caption{\small \textbf{Diagram of each encoder layer (\textit{top}) and decoder layer (\textit{bottom}) in \ourscaps{}.} Notice that the two layers are highly anti-parallel---each is constructed to do the operations of the other in reverse order. That is, in the decoder layer \(g^{\ell}\), the \(\ISTA{\cdot}\) block of \(f^{L - \ell}\) is partially inverted first using a linear layer, then the \(\MSSA{\cdot}\) block of \(f^{L - \ell}\) is reversed; this order unravels the transformation done in \(f^{L - \ell}\).%
    }
    \label{fig:crate_mae_layers}
\end{figure}

%% file: sec_exp.tex
\section{Experimental Evaluations }\label{sec:exp}

In this section, we conduct experiments to study the empirical performance of our proposed white-box transformer \ours{} on real-world datasets and tasks. As the analysis in \Cref{sec:encoding} and \Cref{sec:autoencoding} suggests, the compression/denoising and sparsification steps can each be achieved through various alternative design choices or optimization strategies. The current \ours{} adopts arguably the most basic choices. So our goal with the experiments is \textit{not} to compete, using such a rudimentary design, with other empirically-designed transformer architectures which are heavily engineered on these tasks. Rather, the goal is to convince the reader that \textit{our transformer-like architecture \ours{} obtains promising performance on large-scale datasets, while being fully white-box---both mathematically and empirically interpretable---and theoretically principled.}
We aim to achieve this goal through the following three demonstrations: 
\begin{itemize}
    \item First, we verify convincingly that the white-box transformer architecture, \ours{} is practically effective and can achieve strong performance on many large-scale real-world datasets and tasks. These include supervised and self-supervised learning tasks on both vision and natural language data: ViT, MAE, DINO, BERT, and GPT. 
    \item Second, unlike most empirically designed black-box networks that can usually only be evaluated through their end-to-end performance, we show that the white-box design of our network allows us to \textit{look inside} the deep architecture and verify if individual layers and operators of the learned network indeed perform their design objective---say performing incremental optimization for the objective \Cref{eq:objective-sparse-rate-reduction-l1}. 
    \item Third, we study the white-box transformers through post hoc interpretability. Extensive qualitative and quantitative experimental results demonstrate that the internal representations of the white-box transformer are much more interpretable, with clear and easily-extractable semantic meaning, compared to black-box transformers.
\end{itemize}
In the remainder of this section we highlight a selection of results; additional experimental details and results can be found in \Cref{sec:appendix-exp}. PyTorch-like pseudocode is given in \Cref{app:code}.

\subsection{Empirical Verification of \textsc{CRATE} on Many Practical Tasks}
\subsubsection{Supervised Image Classification via ViT}\label{subsubsec:image-classification-task}
We first study the performance of \ours{} architecture on supervised image classification, where we apply \ours{} as the backbone architecture. 

\paragraph{Model architecture.} 
We implement the architecture that is described in \Cref{sub:architecture}, with minor modifications that are described in \Cref{app:subsec-experiment-supervise-classification-details}. Let \(C\) be the number of classes. We use the pre-processing map:
\begin{equation}
    f^{\pre}(\vX) = \mat{\vz_{\cls}^{1}, \vW^{\pre}\vX} + \vE_{\pos},
\end{equation}
where \(\vz_{\cls}^{1} \in \bR^{d}\) is a trainable \textit{class token} parameter, \(\vX = \mat{\x_{1}, \dots, \x_{N}} \in \bR^{D \times N}\) is the input image patches, \(\vW^{\pre} \in \bR^{d \times D}\) is a trainable matrix, and \(\vE_{\pos} \in \bR^{d \times n}\) is a trainable positional encoding matrix, where \(n = N + 1\). We also consider a classification head at the output of the \ours{} model:
\begin{equation}
    f^{\head}(\mat{\vz_{1}, \dots, \vz_{n}}) = \vW^{\head}\vz_{1},
\end{equation}
where \(\vZ = \mat{\vz_{1}, \dots, \vz_{n}}\) is the output of the encoder \ours{}, \(\vz_{1}\) in particular is the feature corresponding to the class token \(\vz_{\cls}^{1}\), and \(\vW^{\head} \in \bR^{C \times d}\) is a trainable classification head. The output of \(f^{\head}\) is a set of unnormalized log-probabilities for the different classes. 

We consider different model sizes of \ours{} by varying the token dimension $d$, number of heads $K$, and the number of layers $L$. 
We consider four model sizes in this work: \ours{-Tiny}, \ours{-Small}, \ours{-Base}, and \ours{-Large}. 
We provide details about model configurations of \ours{} in Table~\ref{tab:model_configs_classification}.
More implementation details can be found in \Cref{app:subsec-experiment-supervise-classification-details}; PyTorch-style pseudocode can be found in \Cref{app:code}. %

\paragraph{Pre-training methodology.}
Let \(H \colon \Delta_{C} \times \Delta_{C} \to \bR\), where \(\Delta_{C}\) is the set of probability distributions on \([C]\), be the cross-entropy function, defined as
\begin{equation}\label{eq:cross_entropy}
    H(\vp, \vq) = \mathbb{E}_{x \sim \vp}[-\log \vq(x)] = -\sum_{c = 1}^{C}p_{c}\log(q_{c}).
\end{equation}
The pre-training loss is 
\begin{equation}\label{eq:loss_cross_entropy}
    L_{\mathrm{CE}}(f, f^{\head}) \doteq \mathbb{E}_{\vX, \vy}\mathopen{}\left[H\big(\vy, \operatorname{softmax}\big\{(f^{\head} \circ f)(\vX)\big\}\big)\right].
\end{equation}
We pre-train on ImageNet-1K~\citep{deng2009imagenet}, using the Lion optimizer \citep{chen2023symbolic}. More details about the training and datasets can be found in \Cref{app:subsec-experiment-supervise-classification-details}.

\paragraph{Fine-tuning methodology.} For fine-tuning (possibly on a different dataset with a different number of classes), we re-initialize \(f^{\head}\) using parameters with the appropriate value of \(C\), and train on the cross-entropy loss in \Cref{eq:loss_cross_entropy}, updating both \(f\) and \(f^{\head}\). Again, we train using the Lion optimizer \citep{chen2023symbolic}, this time on a variety of commonly used datasets: CIFAR10/CIFAR100 \citep{krizhevsky2009learning}, Oxford Flowers \citep{nilsback2008automated}, and Oxford-IIIT-Pets \citep{parkhi2012cats}. More details about the training and datasets can be found in \Cref{app:subsec-experiment-supervise-classification-details}.

\paragraph{Classification accuracy.} We study the empirical performance of the proposed networks by measuring their top-1 accuracy on ImageNet-1K as well as transfer learning performance on several widely used downstream datasets. 
We summarize the results in \Cref{tab:Comparison_with_Sota}. Our transfer learning methodology is to fine-tune using \Cref{eq:loss_cross_entropy} starting from the pre-trained \(f\) and a re-initialized \(f^{\head}\).

As our designed architecture leverages parameter sharing in both the attention block (\texttt{MSSA}) and the nonlinearity block (\texttt{ISTA}), our \ours{-Base} model (22.80 million) 
has a similar number of parameters to the ViT-Small (22.05 million), and \textit{less than 30\% of the parameters of an identically configured ViT-Base} (86.54 million). From \Cref{tab:Comparison_with_Sota}, we find that with a similar number of model parameters, our proposed network achieves similar ImageNet-1K and transfer learning performance as ViT, while having a simple and principled design. 
Moreover, with the same set of training hyperparameters, we observe promising scaling behavior in \ours{}---we consistently improve the performance by scaling up the model size. 
For comparison, directly scaling ViT on ImageNet-1K does not always lead to consistent performance improvement measured by top-1 accuracy, even while such scaling behavior is touted as one of the main benefits of the transformer architecture~\citep{dosovitskiy2020image}. 
Furthermore, we also study the performance of \ours{} when pre-training on a larger dataset---ImageNet-21K~\citep{deng2009imagenet}. We then fine-tune the pre-trained \ours{} models on downstream tasks including the ImageNet-1K dataset. The results are summarized in Table~\ref{tab:Comparison_with_VIT_IN21k}. 
In particular, the \ours{-B} achieves 79.5\% top-1 accuracy on ImageNet-1K.
To summarize, we achieve promising performance on real-world large-scale datasets by directly implementing our principled architecture.

\begin{table*}[t!]
\centering
\caption{\small Top-1 classification accuracy of \ours{} on various datasets with different model scales 
when pre-trained on ImageNet-1K. For ImageNet-1K/ImageNet-1K ReaL, we directly evaluate the top-1 accuracy. For other datasets, we use models that are pre-trained on ImageNet as initialization and the evaluate the transfer learning performance via fine-tuning.}
\label{tab:Comparison_with_Sota}
\small
    \setlength{\tabcolsep}{13.6pt}
\resizebox{0.98\textwidth}{!}{%
\begin{tabular}{@{}lcccc|cc@{}}
\toprule
\textbf{Model} & \ours{-T}  &  \ours{-S} & \ours{-B} & \ours{-L} & { \color{gray} ViT-T} &  { \color{gray}ViT-S } \\ 
\midrule
\midrule
 \# parameters & 6.09M & 13.12M & 22.80M & 77.64M & { \color{gray} 5.72M} & { \color{gray} 22.05M} \\
\midrule
 ImageNet-1K & 66.7 & 69.2 & 70.8 & 71.3 & { \color{gray} 71.5} & { \color{gray} 72.4} \\
 ImageNet-1K ReaL & 74.0 & 76.0 & 76.5 & 77.4 & { \color{gray} 78.3 } & { \color{gray} 78.4} \\
 \midrule
 CIFAR10 & 95.5 & 96.0 & 96.8 & 97.2 & { \color{gray} 96.6} & { \color{gray} 97.2} \\
 CIFAR100 & 78.9 & 81.0 & 82.7 & 83.6 & { \color{gray} 81.8} & { \color{gray} 83.2}\\
 Oxford Flowers-102 & 84.6 & 87.1 & 88.7 & 88.3 & { \color{gray} 85.1} & { \color{gray} 88.5}\\
 Oxford-IIIT-Pets & 81.4 & 84.9 & 85.3 & 87.4 & { \color{gray} 88.5} & { \color{gray} 88.6} \\
 \bottomrule
\end{tabular}%
}
\end{table*}

\subsubsection{Image Completion via Masked Autoencoding}
\label{subsubsec:image-completion-task}

Now, we study how we can use the full autoencoder architecture based on \ours{} described in \Cref{sec:autoencoding}, including both the \ours{} encoder defined in \Cref{sub:architecture} and the \ours{} decoder defined in \Cref{sub:structured_diffusion_impl}, to perform \textit{masked autoencoding} \citep{he2022masked}---an unsupervised masked image completion task. 
As shorthand, we call the \ours{}-masked autoencoder \ours{-MAE}.

\paragraph{Model architecture.}
We implement the encoder and decoder architectures described in \Cref{sec:autoencoding}, with a few changes detailed in \Cref{app:subsec-experiment-MAE}. We use the pre-processing map:
\begin{equation}
    f^{\pre}(\vX) = \vW^{\pre}\vX + \vE_{\pos},
\end{equation}
where \(\vX = \mat{\x_{1}, \dots, \x_{N}} \in \bR^{D \times N}\), \(\W^{\pre} \in \bR^{d \times D}\) is a trainable matrix, and \(\vE_{\pos} \in \bR^{d \times n}\) is a trainable positional encoding matrix, where \(n = N\). Similarly we use the post-processing map:
\begin{equation}
    g^{\post}(\wt{\vZ}) = \vW^{\post}(\wt{\vZ} - \vE_{\pos}),
\end{equation}
where \(\vW^{\post} \in \bR^{D \times d}\) is another trainable matrix. 

We consider different model sizes of \ours{} by varying the token dimension \(d\), number of heads \(K\), and number of layers \(L\); such parameters will be kept the same for the encoder and decoder. 
This choice is is different from the proposed MAE architecture in \citet{he2022masked}, which uses an asymmetric encoder-decoder setup where the decoder is significantly smaller and more lightweight. 
This asymmetric approach is conceptually messier, whereas our symmetric approach is totally in line with our white-box derivation. 
We consider three model sizes: \ours{-MAE}-Small, \ours{-MAE}-Base, and \ours{-MAE}-Large. \Cref{tab:model_configs} displays the model configurations and number of parameters, with a comparison to MAE, showing that \textit{\ours{} uses around 30\% of the parameters of MAE with the same model configuration (e.g., layers \(L\), token dimension \(d\), and number of heads \(K\))}.

\begin{figure}
    \centering
    \includegraphics[width=0.95\textwidth]{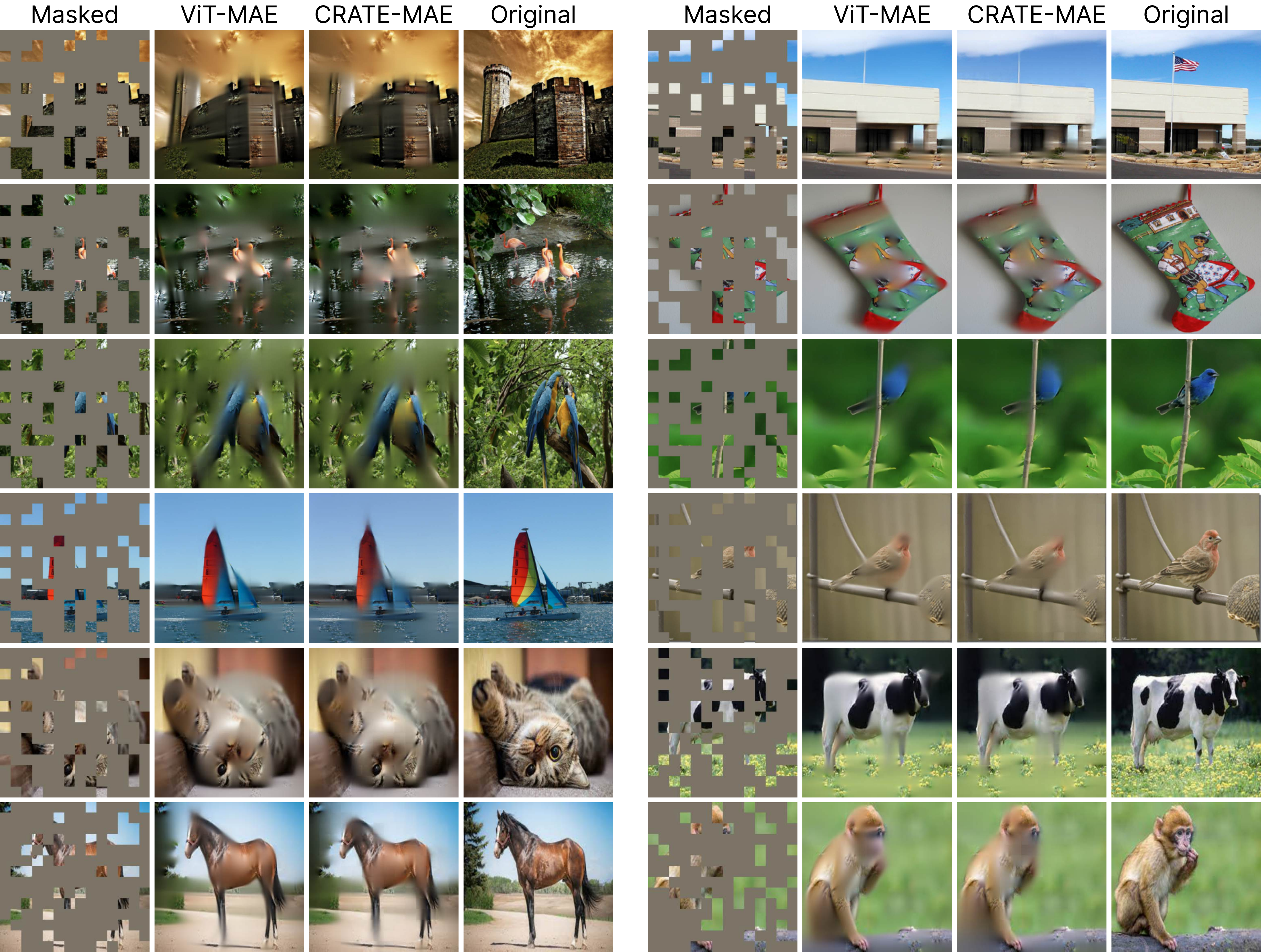}
    \caption{\small Autoencoding visualization comparison of \ours{}-Base to MAE-Base \citep{he2022masked} with 75\% patches masked. We observe that the reconstructions from \ours{}-Base are on par with the reconstructions from MAE-Base, despite using fewer than \(1/3\) of the parameters.%
    }
    \label{fig:mae_autoencoding}
\end{figure}

\paragraph{Pre-training methodology.}
\textit{Masked autoencoding}~\citep{he2022masked} ``masks out'' a large percentage of randomly selected input image tokens in the input \(\vX = \mat{\vx_{1}, \dots, \vx_{N}} \in \bR^{D \times N}\) and then attempts to reconstruct the whole image, measuring success by the resulting autoencoding reconstruction loss and performance on downstream tasks. 
The regular masked autoencoding (MAE) loss \citep{he2022masked} is defined as  
\begin{equation}\label{eq:mae_loss_unreg}
    L_{\mathrm{MAE}}(f, g) \doteq \mathbb{E}_{\vX, \Omega}\big[\norm{(g \circ f)(\operatorname{\texttt{Mask}}_{\Omega}(\vX)) - \vX}_{2}^{2}\big],
\end{equation}
where \(\texttt{Mask}_{\Omega}(\cdot)\) replaces all tokens in \(\vX\) whose indices are in \(\Omega\) with the same masked token. Namely, for each data point, we sample a subset \(\Omega \subseteq [N]\) of a fixed size, and set \(\vX^{\mathrm{mask}} = \mat{\vx_{1}^{\mathrm{mask}}, \dots, \vx_{N}^{\mathrm{mask}}}\), where \(\vx_{i}^{\mathrm{mask}} = \vx_{\mathrm{mask}}\) if \(i \in \Omega\) and \(\vx_{i}^{\mathrm{mask}} = \vx_{i}\) if \(i \notin \Omega\), where \(\vx_{\mathrm{mask}} \in \bR^{D}\) is a trainable mask token.

For pre-training, we employ the AdamW optimizer~\citep{loshchilov2017decoupled}. We mainly use ImageNet-1K \citep{deng2009imagenet} as the main pre-training dataset. More details about training and datasets can be found in \Cref{app:subsec-experiment-MAE}. 

\paragraph{Fine-tuning methodology.} Following \citet{he2022masked}, we fine-tune using supervised classification to examine the linear separability of the features. To do this, we use either full-model fine-tuning or linear probing. The classification head is a ``global pooling'' head:
\begin{equation}
    f^{\head}(\vZ) = \vW^{\head}\left(\frac{1}{n}\sum_{i = 1}^{n}\vz_{i}\right) = \frac{1}{n}\vW^{\head}\vZ\,\bm{1}_{n}, \qquad \text{where}\ \vZ = \mat{\vz_{1}, \dots, \vz_{n}},
\end{equation}
where \(\bm{1}_{n}\) is the all-ones vector in \(\bR^{n}\); note that this is different to the classification setup in \Cref{subsubsec:image-classification-task}, owing to not incorporating a class token. We train using the cross-entropy loss, where \(H\) is defined in \Cref{eq:cross_entropy}:
\begin{equation}
    L_{\mathrm{CE}}(f, f^{\head}) \doteq \mathbb{E}_{\vX, \vy}\mathopen{}\left[H(\vy, \operatorname{softmax}\{(f^{\head} \circ f)(\vX)\})\right].
\end{equation}
For full fine-tuning, we update both \(f\) and \(f^{\head}\). For linear probing, we update only \(f^{\head}\). We fine-tune \ours{-MAE} on several commonly used downstream datasets: CIFAR10/CIFAR100 \citep{krizhevsky2009learning}, Oxford Flowers \citep{nilsback2008automated}, and Oxford-IIIT-Pets \citep{parkhi2012cats}. For fine-tuning \(f^{\head}\), we use the AdamW optimizer on minibatches~\citep{loshchilov2017decoupled}. For linear probing, we use a logistic regresion solver to compute \(f^{\head}\) in one shot using the full dataset. More details about training can be found in \Cref{app:subsec-experiment-MAE}. 

\paragraph{Autoencoding performance.} 
We analyze the autoencoding reconstruction performance of \ours{-MAE} and compare with the original MAE in \citet{he2022masked}. 
In \Cref{fig:mae_autoencoding}, we qualitatively compare the masked autoencoding performance of \ours{-MAE}-Base to MAE-Base \citep{he2022masked}. We observe that both models are able to reconstruct the data well, despite \ours{} using less than a third of the parameters of MAE. In  \Cref{tab:reconstruction_loss}, we display the average reconstruction loss of \ours{-MAE}-Base and MAE-Base, showing a similar quantitative conclusion. %

\begin{table*}[t!]
    \centering
    \caption{\small Average reconstruction loss over the training and validation sets of ImageNet-1K as a function of architecture model size. We see that the performance of \ours{}-Base, while a bit worse than MAE-Base, obtains promising performance on the challenging masked autoencoding task. 
    }
    \label{tab:reconstruction_loss}
    \footnotesize
    \setlength{\tabcolsep}{13.6pt}
        \begin{tabular}{@{}lcc@{}}
            \toprule

            \textbf{Reconstruction Loss} & \ours{}-Base & MAE-Base
            \\ 
            \midrule
            \midrule
            Training & 0.266 &
            0.240
            \\
            \midrule
            Validation & 0.303 & 0.267 
            \\
            \bottomrule
        \end{tabular}%
\end{table*}

\paragraph{Fine-tuning performance.} 

In \Cref{tab:autoencode_finetuning_logreg}, we compare the fine-tuning and linear probing performance of \ours{-MAE} models pretrained on ImageNet-1K across different model sizes and datasets. We demonstrate that \ours{-MAE} has comparable finetuning performance to the usual MAE models while also saving many parameters. This demonstrates that our representations are highly geometrically and statistically structured (e.g., having linear structure), and carry semantic content that is useful for downstream tasks. 

\begin{table*}[t!]
    \centering
    \caption{\small Top-1 classification accuracy of \ours{-MAE}-Small, \ours{-MAE}-Base, (ViT-)MAE-Small, and (ViT-)MAE-Base when pre-trained on ImageNet-1K and fine-tuned on classification (either on the whole network, or using logistic regression a.k.a.~full-batch linear probing) for several smaller datasets. 
    Our results show that the representation learning in \ours{} is useful for downstream tasks on a level comparable to that of the regular MAE, on top of having several side benefits such as parameter efficiency and emergent properties (as discussed in the sequel).
    }
    \label{tab:autoencode_finetuning_logreg}
    \footnotesize
    \setlength{\tabcolsep}{13.6pt}
    \begin{tabular}{@{}lcc|cc@{}}
        \toprule 
        \textbf{Model} & \ours{-MAE}-S & \ours{-MAE}-B & {\color{gray} ViT-MAE-S} & {\color{gray} ViT-MAE-B} \\
        \midrule
        \midrule
        \# parameters & 25.4M & 44.6M & 47.6M & {\color{gray}143.8M} \\
        \midrule
        \underline{\textit{Fine-Tuning}} & & & \\
        CIFAR10 & 96.2 & 96.8 & {\color{gray} 97.6} & {\color{gray} 98.5} \\
        CIFAR100 & 79.0 & 80.3 & {\color{gray} 83.0} & {\color{gray} 87.0} \\
        Oxford Flowers-102 & 71.7 & 78.5 & {\color{gray} 84.2} & {\color{gray} 92.5} \\
        Oxford-IIIT-Pets & 73.7 & 76.7 & {\color{gray} 81.7} & {\color{gray} 90.3} \\
        \midrule 
        \underline{\textit{Linear Probing}} & & & \\
        CIFAR10 & 79.4 & 80.9 & {\color{gray} 79.9} & {\color{gray} 87.9} \\
        CIFAR100 & 56.6 & 60.1 & {\color{gray} 62.3} & {\color{gray} 68.0} \\
        Oxford Flowers-102 & 57.7 & 61.8 & {\color{gray} 66.8} & {\color{gray} 66.4} \\
        Oxford-IIIT-Pets & 40.6 & 46.2 & {\color{gray} 51.8} & {\color{gray} 80.1} \\
        \bottomrule
    \end{tabular}
\end{table*}

\subsubsection{Self-Supervised Learning via DINO Training Method}\label{sec:experiment-DINO}

\noindent
Self-supervised learning (SSL) has gained increasing popularity and become a dominant approach for pre-training deep networks. In this part, we study whether the success of SSL in transformers can be applied to the proposed \ours{} architecture. Specifically, we train \ours{} backbones with DINO~\citep{caron2021emerging}, a widely used SSL framework for ViTs that demonstrate great feature learning capabilities.

\paragraph{Model architecture.} 
To start with, we briefly introduce the self-supervised learning method DINO (self-\textbf{di}stillation with \textbf{no} labels)~\citep{caron2021emerging}. 
DINO applies knowledge distillation~\citep{hinton2015distilling}, with a ``teacher'' encoder network $f_{t}^{\mathrm{proj}} \circ f_{t}$ and ``student'' encoder network  $f_{s}^{\mathrm{proj}} \circ f_{s}$, to learn from two different ``views'' of an image. Here, for each ``backbone'' model \(f\) (either \(f_{s}\) or \(f_{t}\)), we can write (as in other experimental settings) \(f = f^{L} \circ \cdots \circ f^{1} \circ f^{\pre}\), where the pre-processing map $f^{\pre}: \bR^{D \times N} \to \bR^{d \times n}$ is defined as
\begin{equation}
    f^{\pre}(\vX) = \mat{\vz_{\cls}^{1}, \vW^{\pre}\vX} + \vE_{\pos},
\end{equation}
where $n = N + 1$ and $\vz_{\cls}^{1} \in \bR^{d}$,  $\vW^{\pre} \in \bR^{d \times D}$ and $\vE_{\pos} \in \bR^{d \times n}$ are trainable parameters, and each layer \(f^{\ell} \colon \bR^{d \times n} \to \bR^{d \times n}\) is a transformer layer, such as the \ours{} layer or ViT layer. Each projection map \(f^{\mathrm{proj}} \colon \bR^{d \times n} \to \bR^{d}\) (either \(f_{s}^{\mathrm{proj}}\) or \(f_{t}^{\mathrm{proj}}\)) is defined, as in \citet{caron2021emerging}, as 
\begin{equation}\label{eq:dino_head}
    f^{\mathrm{proj}}(\mat{\vz_{1}, \dots, \vz_{n}}) = \vW_{3}^{\mathrm{proj}}\operatorname{ReLU}(\vW_{2}^{\mathrm{proj}}\operatorname{ReLU}(\vW_{1}^{\mathrm{proj}}\vz_{1} + \vb_{1}^{\mathrm{proj}}) + \vb_{2}^{\mathrm{proj}}) + \vb_{3}^{\mathrm{proj}},
\end{equation}
where $\vW_{i}^{\mathrm{proj}} \in \bR^{d \times d}$ are trainable weights and $\vb_{i}^{\mathrm{proj}} \in \bR^{d}$ are trainable biases. 

In our experiments, for the ``backbone'' layers \(f^{\ell} \circ \cdots \circ f^{1}\) for the teacher and student models, we adopt the same model configuration as in the supervised learning experiments introduced in \Cref{subsubsec:image-classification-task}. Specially, due to the heavy computational demand for DINO, we train with two smaller model sizes for this task: \ours{-Small} and \ours{-Base} (note \ours{-Base} has a similar number of parameters as as ViT-Small). Refer to \citet{caron2021emerging} for more implementation details about DINO, including additional small task-specific alterations.

\paragraph{Pre-training methodology.} 

We parameterize \((f_{s}, f_{s}^{\mathrm{proj}})\) by a parameter set \(\vtheta_{s}\), and similarly parameterize \((f_{t}, f_{t}^{\mathrm{proj}})\) by a parameter set \(\vtheta_{t}\). Note that by construction, \(\vtheta_{s}\) and \(\vtheta_{t}\) parameterize the same architecture and thus are elements of the same Euclidean space. The training objective is to minimize
\begin{equation}\label{eq:dino-objective}
    L(\vtheta_{s}, \vtheta_{t}) \doteq \mathbb{E}_{\vX}\big[ H(P(\underbrace{\texttt{Aug}_{\texttt{g}}(\vX)}_{\text{view 1}}; \vtheta_{t}, \tau_{t}), P(\underbrace{\texttt{Aug}_{\texttt{l}}(\vX)}_{\text{view 2}}; \vtheta_{s}, \tau_{s}))\big],
\end{equation}
where $\texttt{Aug}_{\texttt{g}}(\cdot)$ denotes a \textit{global} view augmentation and  $\texttt{Aug}_{\texttt{l}}(\cdot)$ denotes a \textit{local} view augmentation\footnote{We follow the multi-crop strategy~\citep{caron2020unsupervised}. The views are distorted crops of the original image. A global view is a high-resolution (e.g. $224\times 224$) crop that covers a large (e.g. $>50\%$) area of the input image while a local view is of lower resolution (e.g. $96\times 96$) and covers a smaller (e.g. $<50\%$) area of the original image.} %
, and the probability $P(\cdot; \cdot, \cdot)$ is defined as
\begin{equation}\label{eq:dino-def-P}
    P(\vX; \vtheta, \tau) = \softmax{\tau^{-1}(f_{\vtheta}^{\head} \circ f_{\vtheta})(\vX)},
\end{equation}
where \(f_{\vtheta_{s}} = f_{s}\), and so on, and \(\tau > 0\) is the softmax temperature which controls how sharply the softmax concentrates about its largest values. By passing two different views (local and global) of the same image to the student network and teacher network respectively, the DINO objective is designed to encourage ``local-to-global'' correspondences---the features of the local view $\texttt{Aug}_{\texttt{l}}(\vX)$ for the student network are encouraged to be similar to the features of the global view $\texttt{Aug}_{\texttt{g}}(\vX)$ for the teacher network. 

The training is done in the following (slightly unintuitive) way using exponential moving averages:
\begin{align}
    \vtheta_{s}^{k + 1}
    &= \vtheta_{s}^{k} - \eta \nabla_{\vtheta_{s}}L(\vtheta_{s}^{k}, \vtheta_{t}^{k}), \\
    \vtheta_{t}^{k + 1}
    &= \lambda \vtheta_{t}^{k} + (1 - \lambda) \vtheta_{s}^{k + 1},
\end{align}
where \(\eta > 0\) is a learning rate, \(\lambda \in (0, 1)\) is an exponential moving average parameter, and \(k\) is the training iteration. \citet{caron2021emerging} applies the stop-gradient operator on the teacher network, so gradients are propagated only through the student network and not the teacher. In reality, \(L\) is replaced with an empirical version over mini-batches, and the iteration for \(\vtheta_{s}\) can use a more complex algorithm than stochastic gradient descent. In particular, we adopt a similar training recipe as reported in \citet{caron2021emerging}; we train on ImageNet-1K~\citep{deng2009imagenet} and use the AdamW optimizer \citep{loshchilov2017decoupled}, with a learning rate schedule that is composed of a linear warm-up phase followed by a cosine decay schedule. Due to the differences in the backbone architecture, we only alter the choice of base learning rates and weight decay parameters in our experiments. 
More details can be found in \Cref{app:subsec-experiment-dino}.

\paragraph{Fine-tuning methodology.}
Similar to the evaluation protocol of~\cite{caron2021emerging}, we evaluate the classification performance using the features of the pre-trained networks. Precisely, we consider the extracted class token feature \(\vz_{t}\) from the pre-trained teacher network \(f_{t}\) in the following way. Specifically, we obtain
\begin{equation}
    \vz_{t} \gets \vz_{1} \qquad \text{where}\ \mat{\vz_{1}, \dots, \vz_{n}} = f_{t}(\vX).
\end{equation}
From here, we fine-tune on supervised classification in two different ways.

First, we use the feature vectors \(\vz_{t}\) to form a weighted \(k\)-nearest neighbor classifier \citep{wu2018unsupervised}. That is, with \(C\) classes we build a classification head \(f_{k}^{\head} \colon \bR^{d} \to [C]\) which performs weighted \(k\)-nearest neighbors with respect to the training dataset. We do this for several choices of \(k\), and in our shown results we select the classifier \(f_{k}^{\head} \circ f_{t}\) with the highest accuracy on the training set. 

Second, we use the feature vectors \(\vz_{t}\) to build a linear probing model. That is, with \(C\) classes we build a classification head \(f^{\mathrm{head}} \colon \bR^{d} \to \bR^{C}\) defined as 
\begin{equation}
    f^{\head}(\vz) = \vW^{\head}\vz,
\end{equation}
for a trainable parameter \(\vW^{\head} \in \bR^{C \times d}\), by minimizing the cross-entropy loss \Cref{eq:loss_cross_entropy}
\begin{equation}
    L_{\mathrm{CE}}(f, f^{\head}) \doteq \mathbb{E}_{\vX, \vy}\mathopen{}\left[H(\vy, \operatorname{softmax}\{(f^{\head} \circ  f)(\vX)\})\right],
\end{equation}
over \(f^{\head}\) (leaving $f$ fixed). We do this by using a logistic regression solver to compute \(f^{\head}\) in one shot using the full fine-tuning dataset. 

The fine-tuning dataset in our experiments is the same as the training dataset, i.e., ImageNet-1K \citep{deng2009imagenet}.

\paragraph{Fine-tuning results.} 
We report the Top-1 test accuracy of both our fine-tuning methodologies on ImageNet-1K in~\Cref{tab:dino_linear_probe}. These results show that, with only minimal changes of hyper-parameter settings, the proposed \ours{} architecture can learn meaningful and informative features via self-supervised learning.

Later in Section \ref{sec:emergence}, we will show how the features learned by the self-supervised  DINO can already produce good image segmentation results, see Figure \ref{fig:maskcut_visualization}. 

\begin{table*}[t!]
    \centering
    \caption{\small Top-1 classification accuracy on ImageNet-1K of \ours{}-Small, \ours{}-Base, and ViT-Small when pre-trained using the DINO objective. }
    \label{tab:dino_linear_probe}
    \footnotesize
    \setlength{\tabcolsep}{13.6pt}
        \begin{tabular}{@{}lcccc@{}}
            \toprule
            \textbf{Model} & \#parameters & $k$-NN & Linear Probing
            \\ 
            \midrule
            \ours{}-Small/16 & 13.12M & 53.02\% &
            58.41\% 
            \\
            \midrule
            \ours{}-Base/8 & 22.80M & 59.91\% &
            67.42\%
            \\
            \midrule
            ViT-Small/16 & 22.05M & 69.36\% & 72.02\%  
            \\
            \midrule
            ViT-Small/8 & 22.05M & 71.98\% & 75.63\% 
            \\
            \bottomrule
        \end{tabular}%
\end{table*}

\subsubsection{Pre-Training Language Models via BERT and GPT}

We now study the performance of \ours{} architectures on text data. 
Specifically, we consider two widely used pre-training tasks for learning language representations---Bidirectional Encoder Representations from
Transformers (BERT) pre-training~\citep{devlin2018bert, liu2019roberta} and Generative Pre-Training (GPT) pre-training~\citep{radford2019language}.

\paragraph{\ourscaps{-BERT} model architecture.} 

For BERT pre-training, we apply the RoBERTa tokenizer~\citep{sennrich2015neural,liu2019roberta} with vocabulary size $V=50,265$ as the default tokenizer. We implement the \ours{} architecture that is used in image classification \Cref{subsubsec:image-classification-task}, except for the pre-processing layer and the task-specific head. Specifically, the pre-processing layer is described as:
\begin{equation}\label{eq:free-pre-text}
    f^{\pre}(\vX) = f^{\mathrm{embed}}(\vX) + \vE_{\pos},
\end{equation}
where $f^{\mathrm{embed}}$ is an embedding function which maps input tokens \(\vx_{i}\) to embedding vectors in \(\bR^{d}\), 
and $\vE_{\pos}\in\mathbb{R}^{d\times n}$ is the sinusoid positional embedding defined in \citet{vaswani2017attention}.
The final output layer takes the form of a linear head:
\begin{equation}\label{eq:f-head-text}
    f^{\head}(\vZ) = \vW^{\head}\vZ,
\end{equation}
where \(\vW^{\head} \in \bR^{V \times d}\).

We denote the \ours{} model with BERT pre-training as \ours{-BERT}.
We summarize the configurations of \ours{-BERT} models in \Cref{tab:configs-crate-text} (upper). For \ours{-BERT-Base}, we use the same width, number of attention heads, and depth as BERT-Base~\citep{liu2019roberta}. \ours{-BERT-Base} uses 60.9M parameters, while BERT-Base uses 110M parameters, nearly twice as many.

\paragraph{\ourscaps{-BERT} pre-training methodology.}
The training loss for \ours{}-BERT is masked language modeling (masked LM)\footnote{The original BERT pre-training tasks also contain next sentence prediction, but we discard this task due to observations shown by~\citet{liu2019roberta}.}, which is defined as
\begin{equation}\label{eq:training-loss-bert}
  L_{\text{BERT}}(f, f^{\head}) = \mathbb{E}_{\vX, \Omega}\left[\frac{1}{\abs{\Omega}}\sum_{i\in \Omega}  H(\bm{1}(\vx_{i}), \operatorname{softmax}\{(f^{\head} \circ f)(\vX)_{i}\})\right],
\end{equation}
where \((f^{\head} \circ f)(\vX)_{i} \in \bR^{V}\) denotes the \(i\th\) index of the model output---that is, the log-probabilities of each potential token in the vocabulary to take up the \(i\th\) index in the sequence \(\vX\). Similarly $\mathbf{1}(\vx_i) \in \{0, 1\}^{V}$ represents the one-hot embedding vector of the token \(\vx_{i}\) in the vocabulary. 

To pre-train, we use the Adam optimizer \citep{kingma2014adam}. We use the same pre-training dataset as BERT ~\citep{devlin2018bert}, including BooksCorpus~\citep{zhu2015aligning} and English Wikipedia.  Refer to \citet{devlin2018bert} for more details about the pre-training tasks.

\paragraph{\ourscaps{-GPT} model architecture.} 

For \ours{}-GPT pretraining, following the setup in the GPT-2 paper \citep{radford2019language}, we modify our \ours{} architecture to align with the next word prediction task used in GPT. 
Specifically, we replace the attention in \texttt{MSSA} with a causally masked self-attention, i.e., replacing 
\begin{align}
    \softmax{(\vU_{k}\adj\vZ)\adj(\vU_{k}\adj\vZ)} 
    &\to \softmax{\texttt{CausalMask}((\vU_{k}\adj\vZ)\adj(\vU_{k}\adj\vZ))}, \\
    \text{where}\ \texttt{CausalMask}(\vA)_{ij} 
    &= \begin{cases}A_{ij}, & i \leq j \\ -\infty, & i > j.\end{cases}
\end{align}
in \Cref{eq:SSA}. We use the same pre-processing and head architecture as \ours{-BERT}:
\begin{equation}
    f^{\pre}(\vX) = f^{\mathrm{embed}}(\vX) + \vE_{\pos}, \qquad f^{\head}(\vZ) = \vW^{\head}\vZ.
\end{equation}
We denote the \ours{} model for GPT pre-training as \ours{-GPT}. We summarize the configurations of GPT models in \Cref{tab:configs-crate-text} (lower), and set the model architecture parameters such that \ours{-GPT} has the same layer specifications as GPT2 with parameter size 60M (c.f.~GPT2's 124M, more than twice as many).

\paragraph{\ourscaps{-GPT} pre-training methodology.}
The training objective of GPT is next word prediction (causal language modeling, causal LM), which is defined as
\begin{equation}\label{eq:training-loss-gpt}
  L_{\text{GPT}}(f, f^{\text{head}}) = \mathbb{E}_{\vX}\bigg[\frac{1}{\abs{\vX} -1} \sum_{i = 1}^{\abs{\vX} - 1} H(\bm{1}(\vx_{i + 1}), \operatorname{softmax}\{(f^{\head} \circ f)(\vX)_{i}\}) \bigg],
\end{equation}
where $\abs{\vX}$ denotes the total number of tokens in the input sample $\vX$.
For GPT pretraining, we mainly follow the implementations in \citet{nanogpt}; in particular we set the block size as 1,024. 
We pre-train \ours{-GPT} models on OpenWebText~\citep{Gokaslan2019OpenWeb} using the Adam optimizer \citep{kingma2014adam}. Refer to \citet{radford2019language} for more details about the pre-training tasks.

As usual, we defer the pre-training setup details of \ours{}-BERT and \ours{}-GPT to \Cref{app:subsec-experiment-text}.

\paragraph{\ourscaps{-BERT} evaluation results.} 
Pre-training results of \ours{-BERT} is summarized in Table~\ref{tab:bert-eval}. We use the officially released BERT checkpoints for BERT evaluation~\citep{huggingface_bert}. We evaluate the \ours{-BERT} model by fine-tuning on downstream tasks from the GLUE benchmark~\citep{wang2018glue}. 
\textit{Single-task Finetuning} of Table~\ref{tab:bert-eval} means that we fine-tune the model only on the task we evaluate. 
In addition, following previous work~\citep{liu2019roberta}, models pre-trained only on the masked language modeling task should be further fine-tuned on MNLI before evaluating on MRPC, STS-B and RTE to be competitive to BERT. That is, first fine-tuning on MNLI improves the performance on these tasks.
Therefore, we also investigate this training recipe and include the corresponding results in the \textit{Further Finetuning after MNLI} section of \Cref{tab:bert-eval}. 
The \textit{AVG} score is the average of all GLUE tasks shown in the table.\footnote{We exclude the problematic CoLA task, as it has been reported multiple times that the performance depends heavily on the seed chosen for BERT \citep{huggingface_cola_issue}.}
We also present the pre-training loss and GLUE evaluation scores of \ours{-BERT} with respect to progressing pre-training steps in \Cref{fig:crate-text-evals}~(left and middle). 
The progressing GLUE average scores are calculated under the \textit{single-task finetuning} setting.\footnote{We use 3 epochs (instead of 8) for the MNLI task for all evaluation except the 24,000 and 30,000 steps because it suffices to get a good score.}
From Table~\ref{tab:bert-eval}, we find that \ours{-BERT-Base} can achieve competitive performance compared to BERT models realized by standard transformers.

\begin{table*}[t!]
\small
\setlength{\tabcolsep}{3pt} 
\def\arraystretch{1.1}
    \centering
    \caption{\small Performance on the GLUE benchmark (higher is better). All scores are obtained using the HuggingFace evaluation script~\citep{huggingface_transformers}. 
    F1 scores are reported for QQP and MRPC, Spearman correlations are reported for STS-B, and accuracy scores are reported for the other tasks.}
    \begin{tabular}{l@{\hspace{6pt}}@{\hspace{6pt}}c@{\hspace{6pt}}c@{\hspace{6pt}}c@{\hspace{6pt}}c@{\hspace{6pt}}c@{\hspace{6pt}}c@{\hspace{6pt}}c@{\hspace{6pt}}c@{\hspace{6pt}}c@{\hspace{6pt}}c}
    \hline
    &  \textbf{SST-2} & \textbf{MRPC} & \textbf{STS-B} & \textbf{QQP} & \textbf{MNLI} & \textbf{QNLI} & \textbf{RTE} & \textbf{WNLI} & \textbf{Avg} \\[-0.5em]
    & \scriptsize \textsc{Acc} & \scriptsize \textsc{F1} & \scriptsize \textsc{Sprm} & \scriptsize \textsc{F1} & \scriptsize \textsc{M/MM} & \scriptsize \textsc{Acc} & \scriptsize \textsc{Acc} & \scriptsize \textsc{Acc} & \scriptsize \textsc{All} \\
     \hline
     \multicolumn{10}{c}{\textit{Single-task Finetuning}} \\
     \hline
     \scriptsize BERT-Medium (41M) & 90.5 & 87.8 & 87.6 & 86.2 & 80.8/81.2 & 88.6 & 64.6 & 30.1 & 77.5 \\
     \scriptsize BERT-Base (110M) & 92.3 & 90.9 & 88.5 & 88.2 & 84.4/84.6 & 91.7 & 64.3 & 53.5 & 82.0 \\
     \scriptsize \ours{-BERT-Base} (61M) & 85.9 & 80.0 & 73.0 & 85.8 & 76.1/75.3 & 83.4 & 53.8 & 56.3 & 74.4 \\
     \hline
     \multicolumn{10}{c}{\textit{Further Finetuning after MNLI}} \\
     \hline
     \scriptsize BERT-Medium (41M) & \color{gray} 90.5 & 89.6 & 88.0 & \color{gray} 86.2 & \color{gray} 80.8/81.2 & \color{gray} 88.6 & 72.6 & \color{gray} 30.1 & 78.6 \\
     \scriptsize BERT-Base (110M) & \color{gray} 92.3 & 90.3 & 86.8 & \color{gray} 88.2 & \color{gray} 84.4/84.6 & \color{gray} 91.7 & 79.4 & \color{gray} 53.5 & 83.5 \\
     \scriptsize \ours{-BERT-Base} (61M)  & \color{gray} 85.9 & 84.7 & 83.2 & \color{gray} 85.8 & \color{gray} 76.1/75.3 & \color{gray} 83.4 & 65.0 & \color{gray} 56.3 & 77.3 \\
     \hline
    \end{tabular}
    \label{tab:bert-eval}
\end{table*}

\begin{figure}[t!]
     \centering
     \begin{subfigure}[b]{0.32\textwidth}
         \centering
    \includegraphics[width=\textwidth]{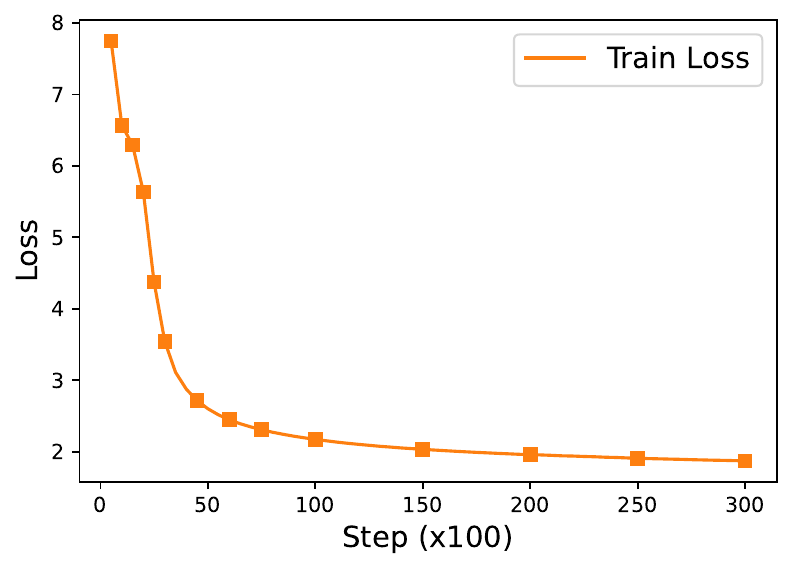}
     \end{subfigure}
     \begin{subfigure}[b]{0.32\textwidth}
         \centering
    \includegraphics[width=\textwidth]{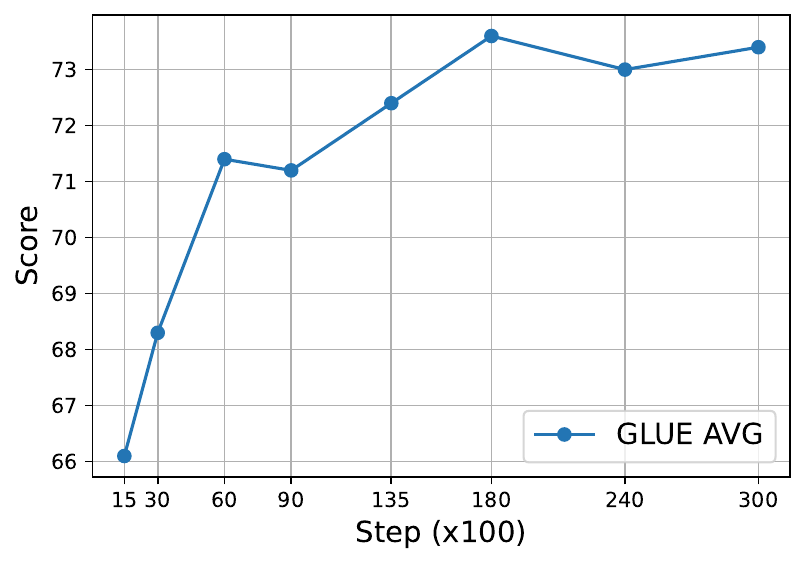}
     \end{subfigure}
     \begin{subfigure}[b]{0.32\textwidth}
         \centering
    \includegraphics[width=\textwidth]{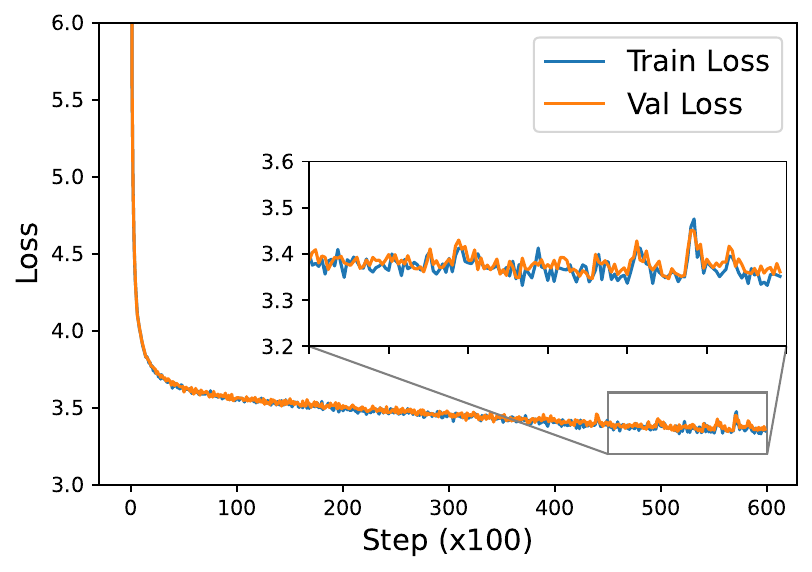}
     \end{subfigure}
        \caption{\small (\textit{left}) The training loss curve of \ours{}-BERT-Base trained on the Wikipedia and BookCorpus datasets. 
        (\textit{middle}) GLUE single-task finetuning scores evaluated after each pretraining checkpoint of \ours{}-BERT-Base. 
        (\textit{right}) The loss curve of \ours{}-GPT-Base trained on the OpenWebText dataset.}
        \label{fig:crate-text-evals}
\end{figure}

\paragraph{\ourscaps{-GPT} evaluation results.} 
We evaluate \ours{-GPT} models by measuring the zero-shot cross-entropy loss on different datasets. 
We evaluate \ours{-GPT-Base} on OpenWebText (OWT) as well as other datasets as shown in \citep{radford2019language}. 
We compare with the official release of GPT2-Base which is pre-trained on WebText \citep{huggingface_gpt}, while \ours{-GPT-Base} is pre-trained on OWT. 
{\color{revision} Meanwhile, in order to compare \ours{-GPT} and the GPT2 model under similar number of model parameters, we consider the GPT2 with a smaller model size, denoted by GPT2-Small, and apply the same training setup described in \citet{nanogpt}.}
Therefore, we first fine-tune GTP2-Base on OWT before evaluating on the OWT test dataset.
The evaluation results on the test datasets are shown in Table~\ref{tab:gpt-eval}. 
We also present the pre-training loss of \ours{-GPT} with respect to progressing
pre-training steps in Figure~\ref{fig:crate-text-evals}~(right).
From both Table~\ref{tab:gpt-eval} and Figure~\ref{fig:crate-text-evals}~(right), we can observe that the \ours{} architecture can be applied as backbones for generative language models and achieve competitive performance compared to GPT2. 

For more technical details about the evaluation of the pre-trained \ours{}-BERT and \ours{}-GPT, please refer to \Cref{app:subsec-experiment-text}.

\begin{table}[t!]
\def\arraystretch{1.1}
    \small
    \caption{\small Zero-shot cross-entropy loss of the \ours{}-GPT2-Base model and GPT2-Small, GPT2-Base model evaluated on the test split of the datasets (lower is better). 
    }
    \centering
    \begin{tabular}{ccccccc}
    \hline
    & \#parameters & \textbf{OWT} & \textbf{LAMBADA} & \textbf{WikiText} & \textbf{PTB} & \textbf{Avg} \\
     \hline
     GPT2-Base  & {\color{revision}124M} & 2.85 & 4.12 & 3.89 & 4.63 & 3.87 \\
     {\color{revision}GPT2-Small } &  {\color{revision}64M} & {\color{revision}3.04} & {\color{revision}4.49} & {\color{revision}4.31} & {\color{revision}5.15} & {\color{revision}4.25} \\
     \ours{}-GPT2-Base & {\color{revision}60M} & 3.37 & 4.91 & 4.61 & 5.53 & 4.61 \\
     \hline
    \end{tabular}
    \label{tab:gpt-eval}
\end{table}

\paragraph{Generated texts by \ourscaps{-BERT} and \ourscaps{-GPT}.} 
We also explored some qualitative examples generated by \ours{}-BERT and \ours{}-GPT, as shown in Figure \ref{fig:eval-gpt-qual}. We observe a similar behavior of \ours{} and the standard Transformers, where both answers make sense and the standard transformer slightly outperforms \ours{}.

\begin{figure}
    \centering
    \includegraphics[width=0.95\linewidth]{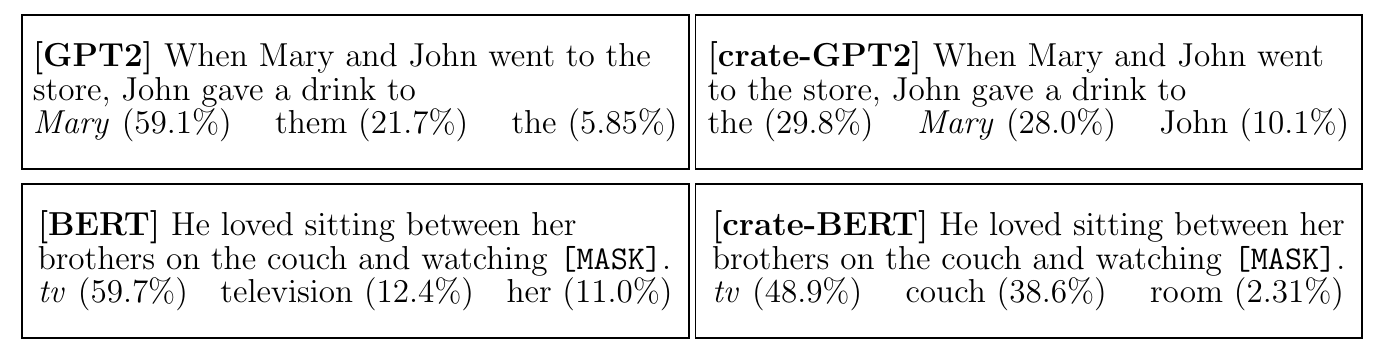}
    \caption{\small Examples of predictions made by our models and the official models. The token \textit{in Italic} is the ground truth token. We compare \ours{}-GPT2-Base to GPT2-Base on the next word prediction task and \ours{}-BERT-Base to BERT-Medium on the masked language modeling task.
    }
    \label{fig:eval-gpt-qual}
\end{figure}

\subsection{Analysis and Visualization of Learned CRATE Layers}\label{subsec:exp-in-depth-analysis}
We now study the internal representations of \ours{} at different layers---\(\{\vZ^{\ell}\}_{\ell=1}^{L}\). 
We measure the designed representation learning objective, sparse rate reduction \Cref{eq:objective-sparse-rate-reduction}, of token representations from different layers. 
We mainly consider the \ours{} encoder architectures on image classification tasks (\Cref{subsubsec:image-classification-task}) for the experiments in this subsection.

\paragraph{Do layers of \ourscaps{} achieve their design goals?} 
As described in \Cref{sub:compression} and \Cref{sub:sparse}, the \texttt{MSSA} block is designed to optimize the compression term  $R^{c}(\vZ \mid \vU_{[K]})$ and the \texttt{ISTA} block to sparsify the token representations (corresponding to the sparsification term $\|\vZ\|_0$). 
To understand whether \ours{} indeed optimizes these terms, for each layer $\ell$, we measure (i) the compression term $R^{c}(\vZ^{\ell+1/2} \mid \vU_{[K]}^{\ell})$ on the \texttt{MSSA} block outputs $\vZ^{\ell+1/2}$; and (ii) sparsity $\|\vZ^{\ell+1}\|_0$ on the \texttt{ISTA} block outputs $\vZ^{\ell+1}$. 
Specifically, we evaluate these two terms by using training/validation samples from ImageNet-1K. Both terms are evaluated at the per-sample level and averaged over \(B = 1000\) samples.

In \Cref{fig:exp-rc-sparisty-small} we show  the plots of these two key measures at all layers for the learned \ours{-small} model. We find that as the layer index $\ell$ increases, both the compression and the sparsification terms improve in most cases. The increase in the sparsity measure of the last layer is caused by the extra linear layer for classification.\footnote{ Note that the learned sparse (tokens) features need to be mixed in the last layer for predicting the class. The phenomenon of increase in the sparsity measure  at the last layer suggests that each class of objects may be associated with a number of features, and some of these features are likely to be shared across different classes.} 
These results suggest that \ours{} aligns well with the original design goals: once learned, it essentially learns to gradually compress and sparsity the representations through its layers. In addition, we also measure the compression and sparsification terms on \ours{} models with different model sizes as well as intermediate model checkpoints and the results are shown by plots in \Cref{fig:appendix-exp-rc-sparisty-all-model-size} of \Cref{subsec:appendix-exp-results}. The observations are very consistent across all different model sizes---both the compression and sparsification terms improve in most scenarios. Models with more layers tend to optimize the objectives more effectively, confirming our understanding of each layer's roles. 

To see the effect of learning, we present the evaluations on \ours{-Small} trained with different numbers of epochs in \Cref{fig:exp-rc-sparisty-small-epochs}. 
When the model is not trained enough (e.g.~untrained), the architecture does not optimize the objectives effectively. However, during training---learning better subspaces $\vU_{[K]}^{\ell}$ and dictionaries $\vD^{\ell}$---the designed blocks  start to optimize the objectives much more effectively.

\begin{figure}[t!]
     \centering
     \begin{subfigure}[b]{0.47\textwidth}
         \centering
    \includegraphics[width=\textwidth]{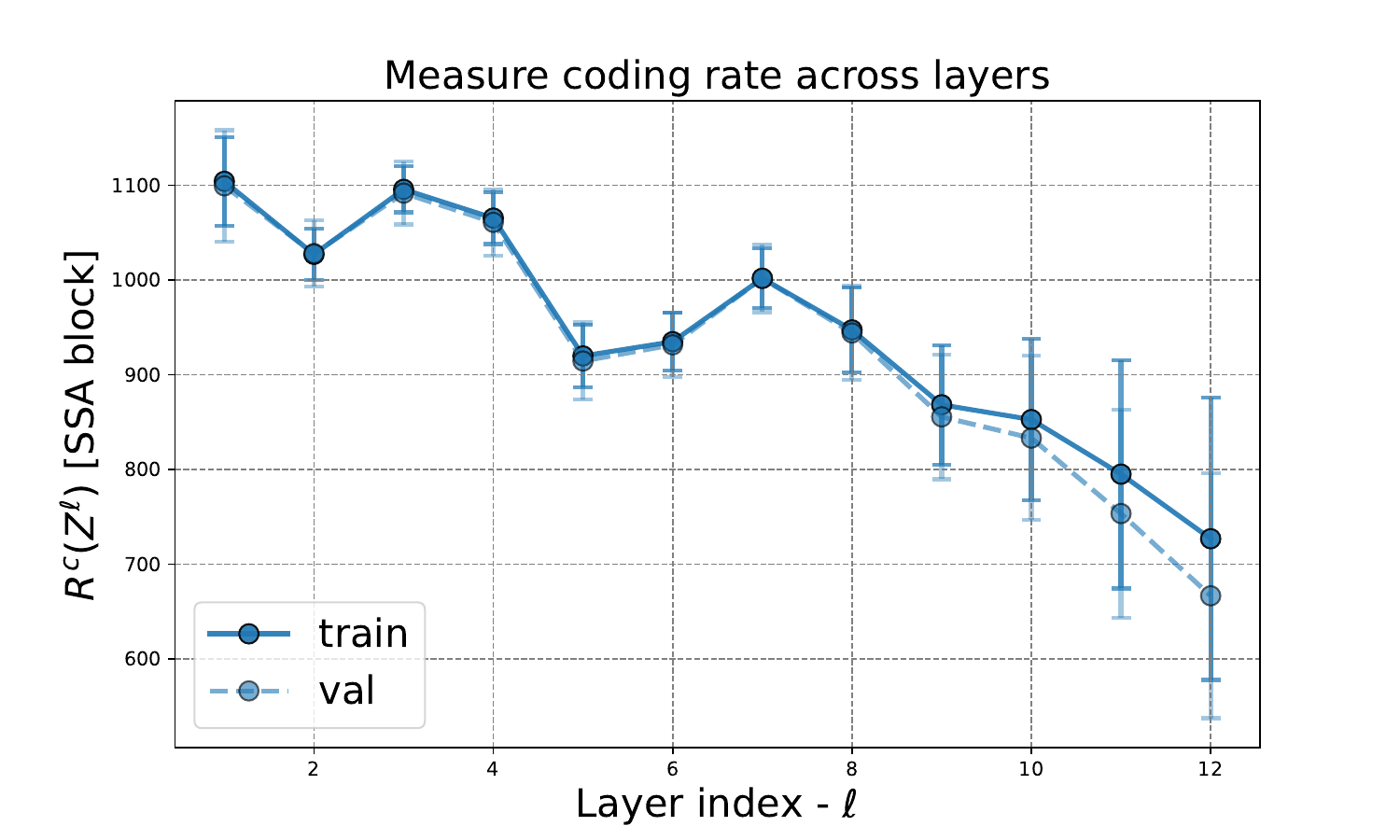}
     \end{subfigure}
     \begin{subfigure}[b]{0.482\textwidth}
         \centering
    \includegraphics[width=\textwidth]{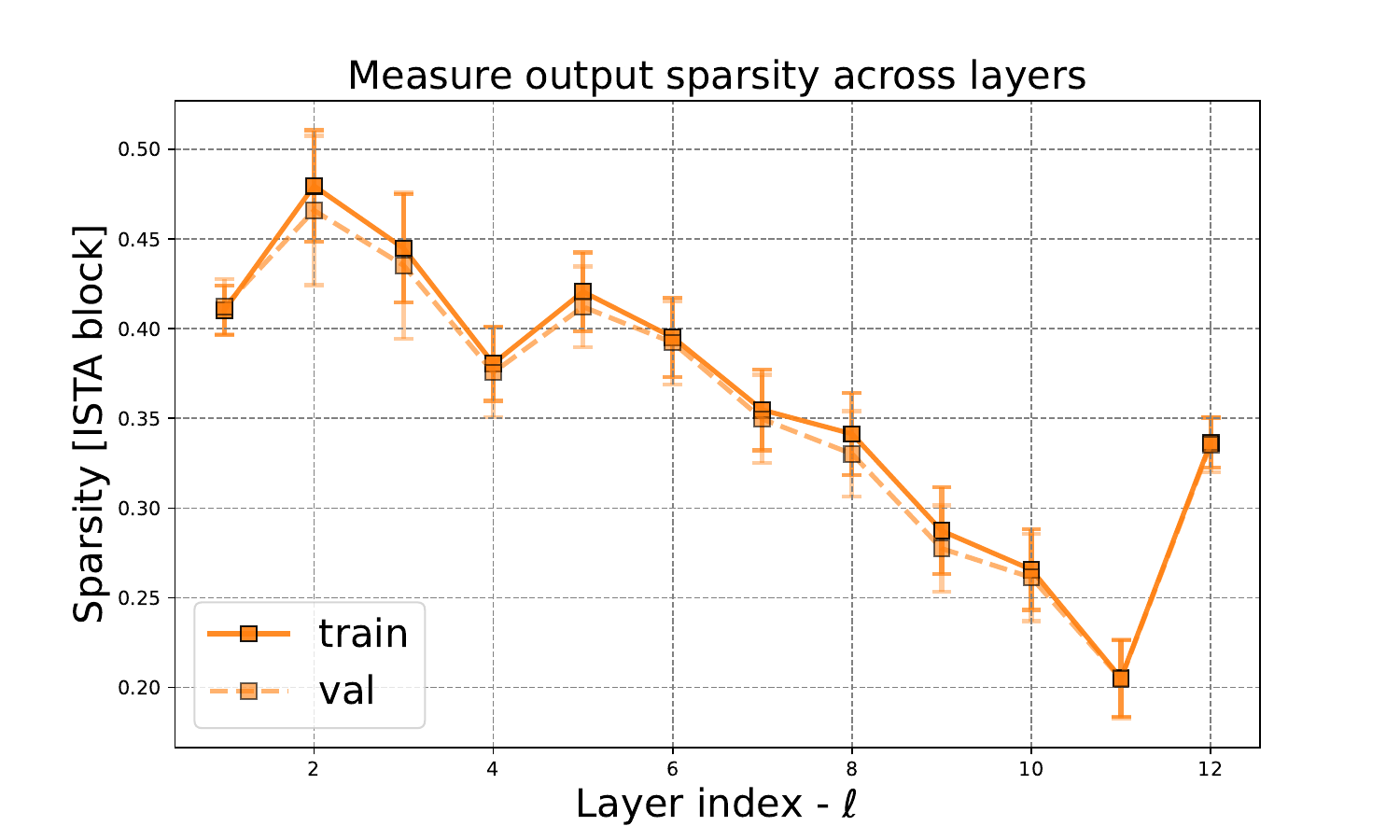}
     \end{subfigure}
        \caption{\small \textit{Left}: The compression term $R^{c}(\vZ^{\ell+1/2} \mid \vU_{[K]}^{\ell})$ of the \texttt{MSSA} outputs at different layers. \textit{Right}: the sparsity of the \texttt{ISTA} output block, $\|\vZ^{\ell+1}\|_0 / (d\cdot N)$, at different layers. 
        (Model: \ours{-Small}).}
        \label{fig:exp-rc-sparisty-small}
\end{figure}

\begin{figure}[t!]
     \centering
     \begin{subfigure}[b]{0.49\textwidth}
         \centering
    \includegraphics[width=\textwidth]{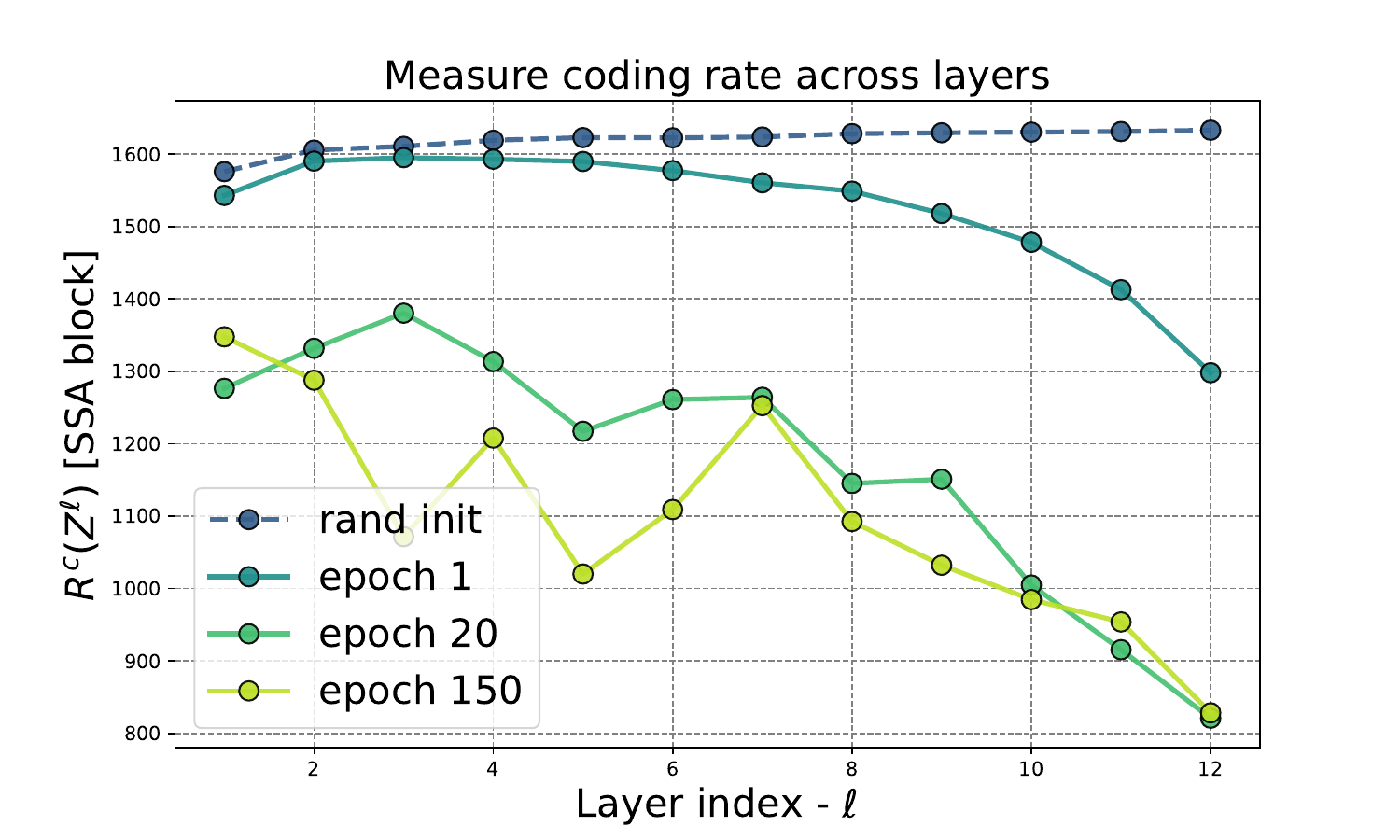}
     \end{subfigure}
     \begin{subfigure}[b]{0.49\textwidth}
         \centering
    \includegraphics[width=\textwidth]{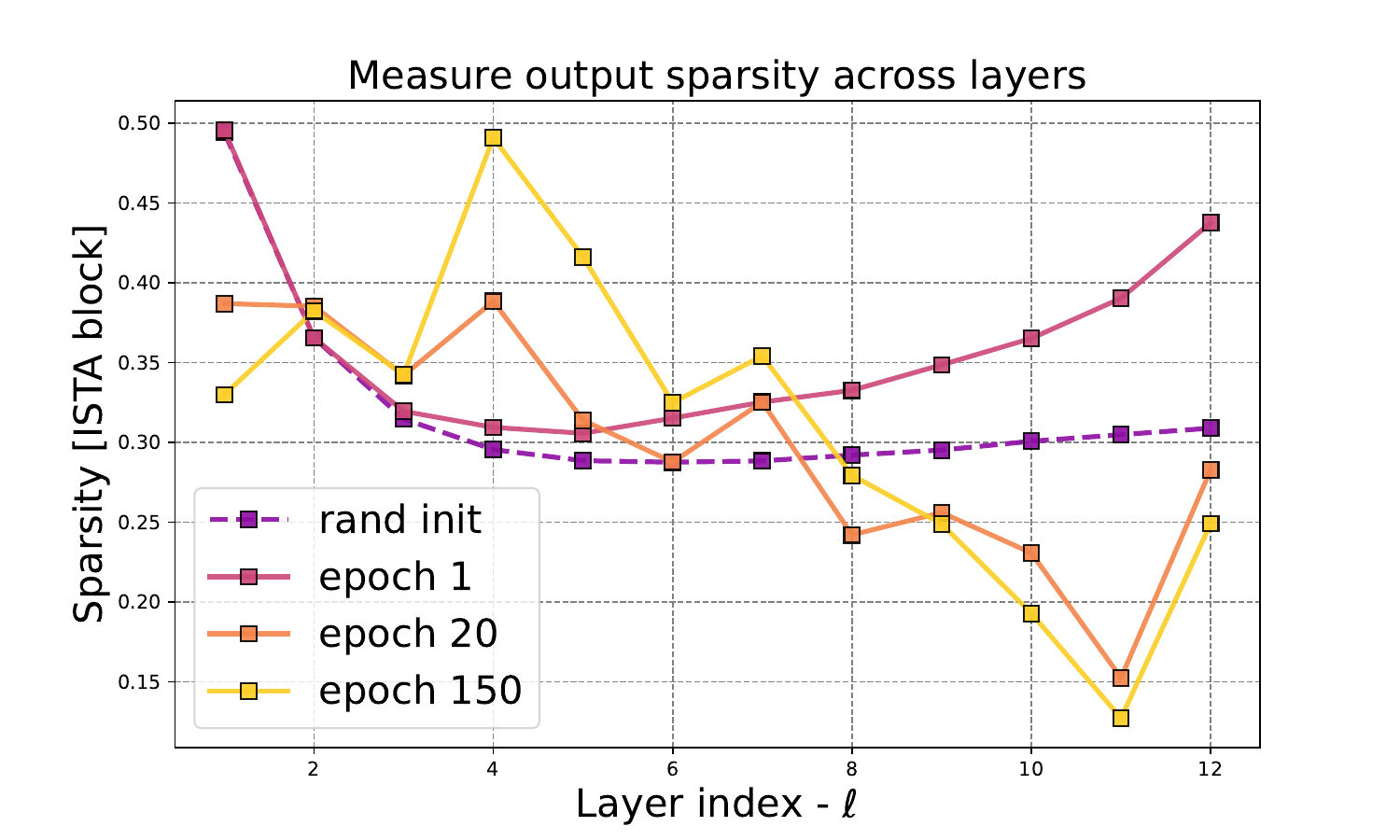}
     \end{subfigure}
        \caption{\small The compression term $R^c(\vZ^{\ell + 1/2} \mid \vU_{[K]}^{\ell})$ (\textit{left}) and sparsification term $\|\vZ^{\ell + 1}\|_0 / (d \cdot N)$ (\textit{right}) across models trained with different numbers of epochs. (Model: \ours{-Base}).}
        \label{fig:exp-rc-sparisty-small-epochs}
\end{figure}

\begin{figure}[ht]
     \centering
     \begin{subfigure}[b]{0.24\textwidth}
         \centering
    \includegraphics[width=\textwidth]{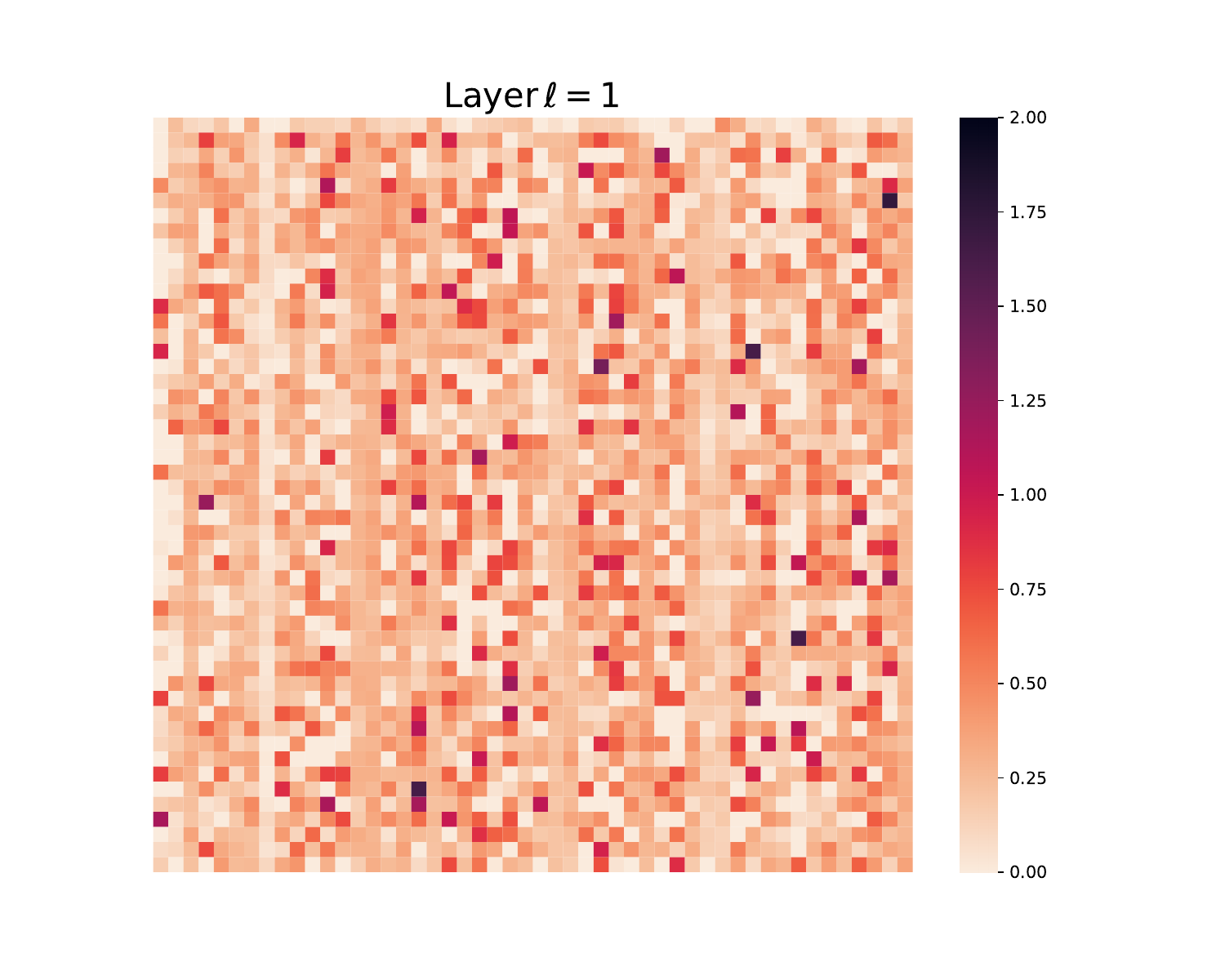}
         \caption{$\ell=1$.}
     \end{subfigure}
     \begin{subfigure}[b]{0.24\textwidth}
         \centering
    \includegraphics[width=\textwidth]{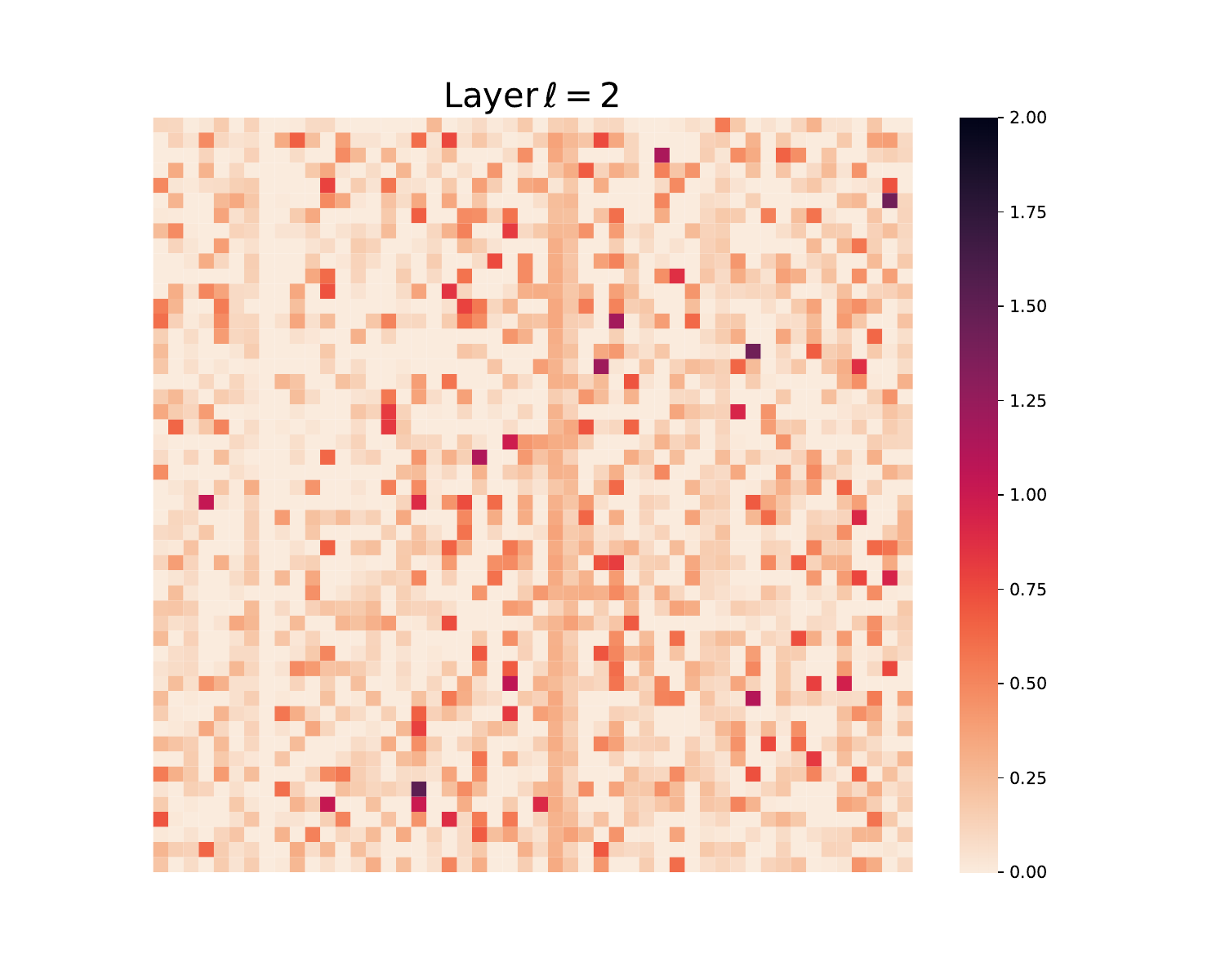}
         \caption{$\ell=2$.}
     \end{subfigure}
     \begin{subfigure}[b]{0.24\textwidth}
         \centering
    \includegraphics[width=\textwidth]{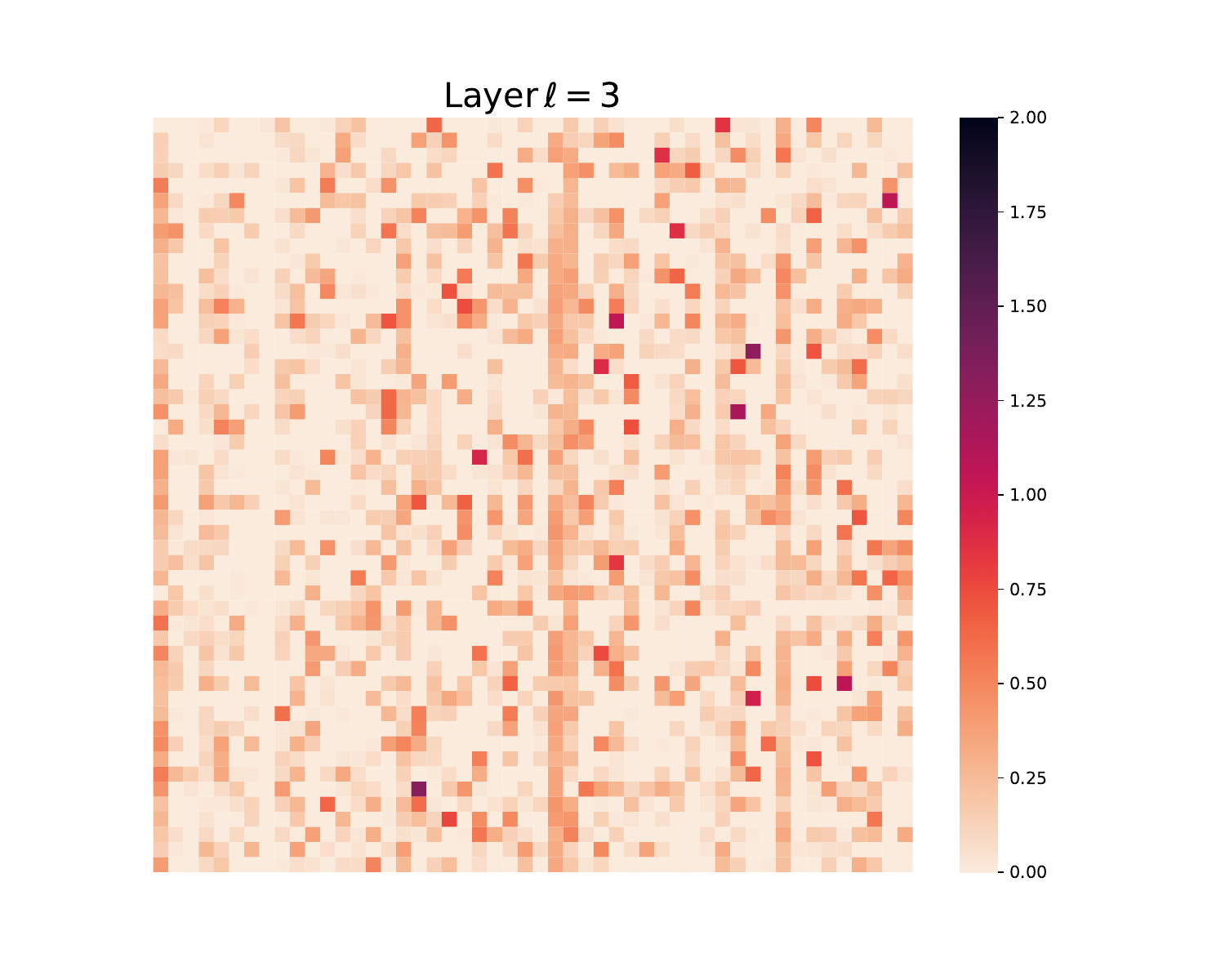}
         \caption{$\ell=3$.}
     \end{subfigure}
     \begin{subfigure}[b]{0.24\textwidth}
         \centering
    \includegraphics[width=\textwidth]{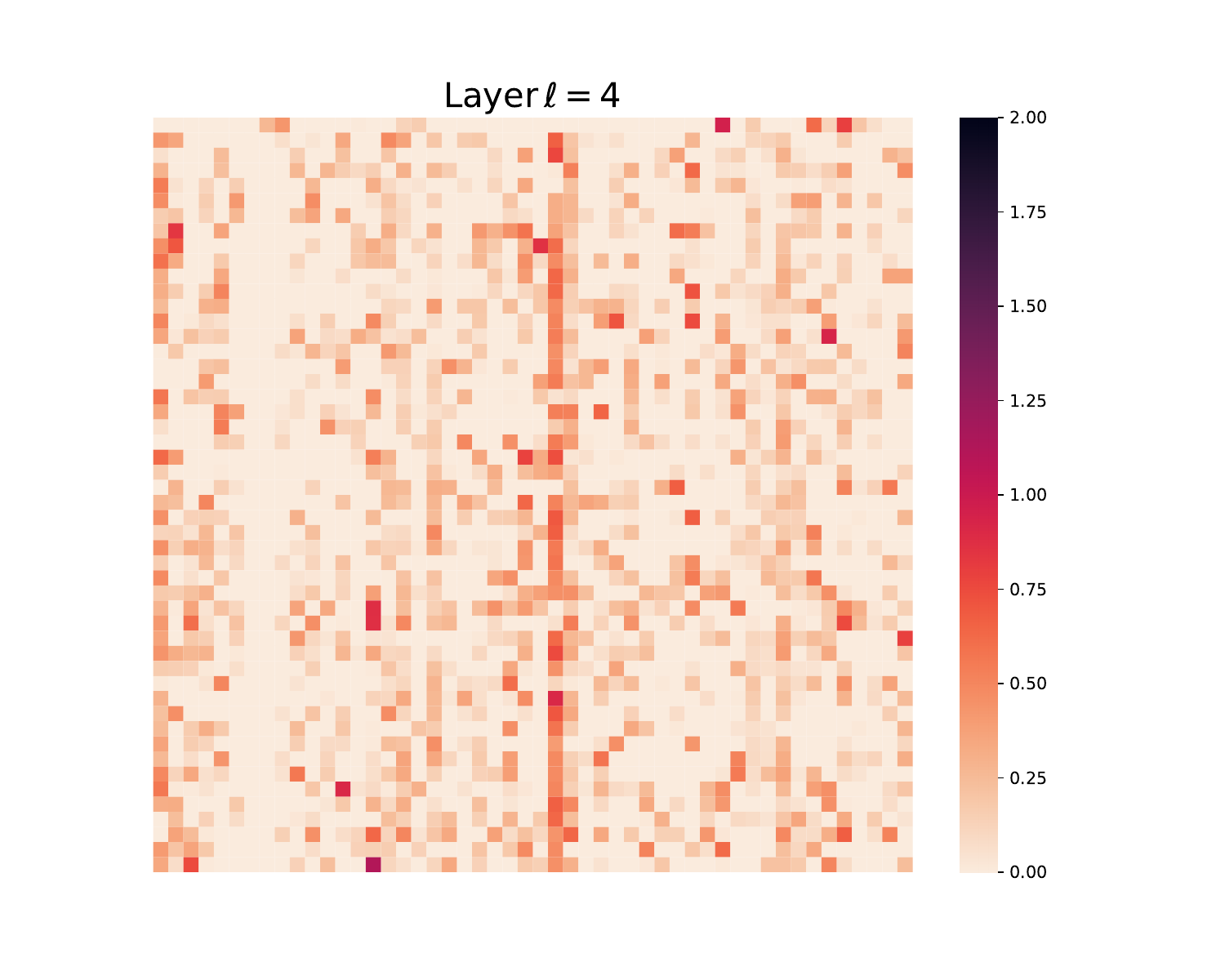}
         \caption{$\ell=4$.}
     \end{subfigure}
     \begin{subfigure}[b]{0.24\textwidth}
         \centering
    \includegraphics[width=\textwidth]{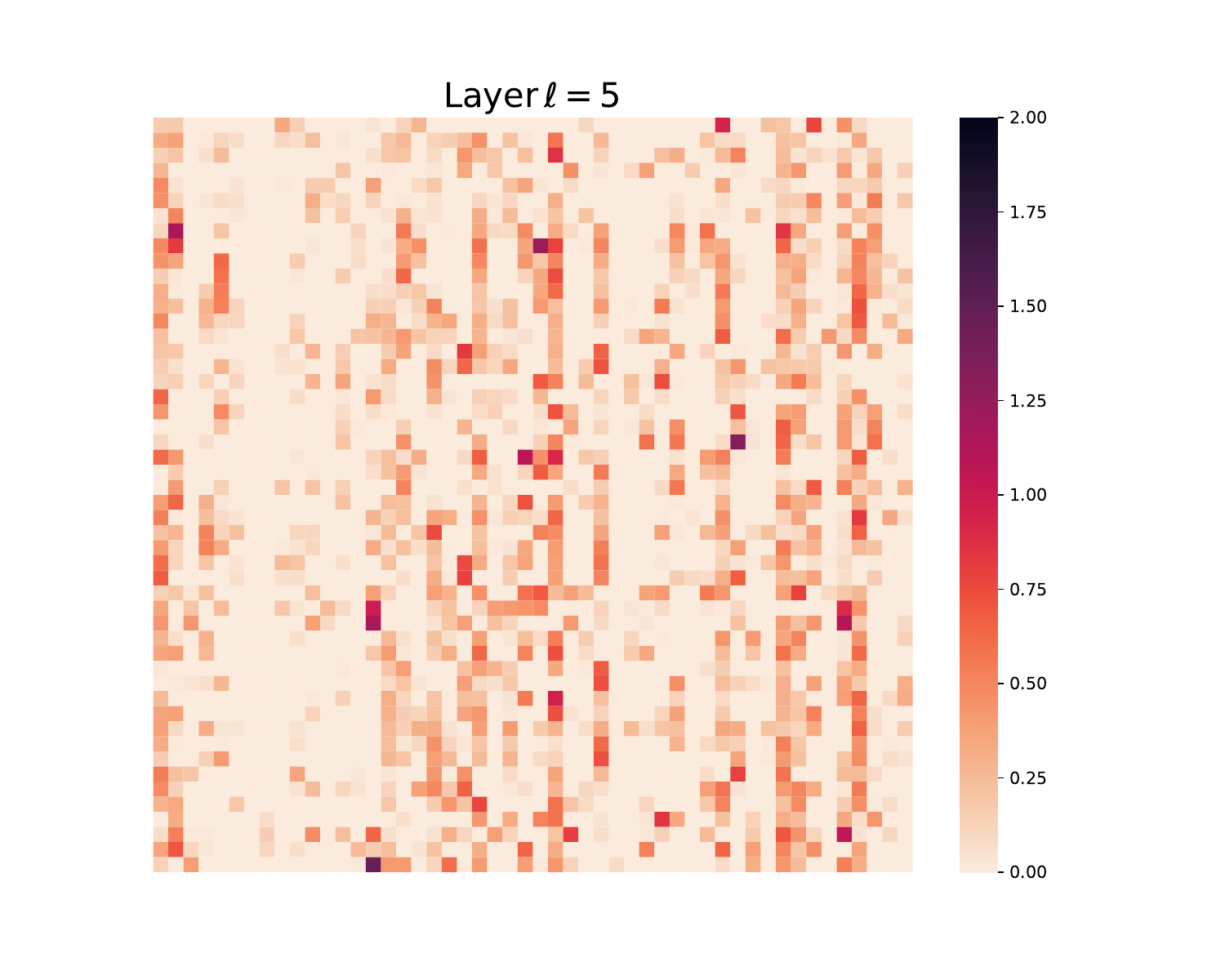}
         \caption{$\ell=5$.}
     \end{subfigure}
     \begin{subfigure}[b]{0.24\textwidth}
         \centering
    \includegraphics[width=\textwidth]{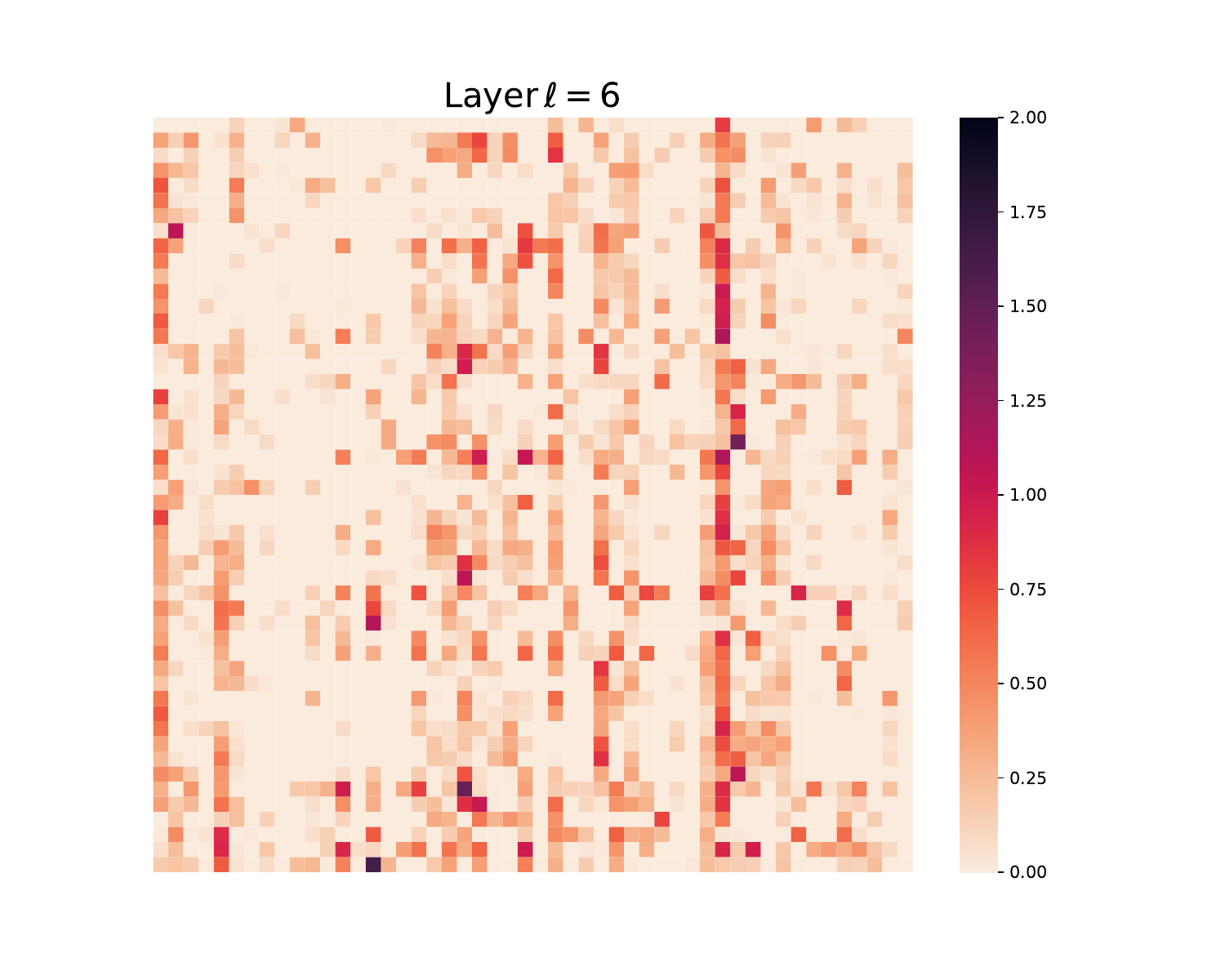}
         \caption{$\ell=6$.}
     \end{subfigure}
     \begin{subfigure}[b]{0.24\textwidth}
         \centering
    \includegraphics[width=\textwidth]{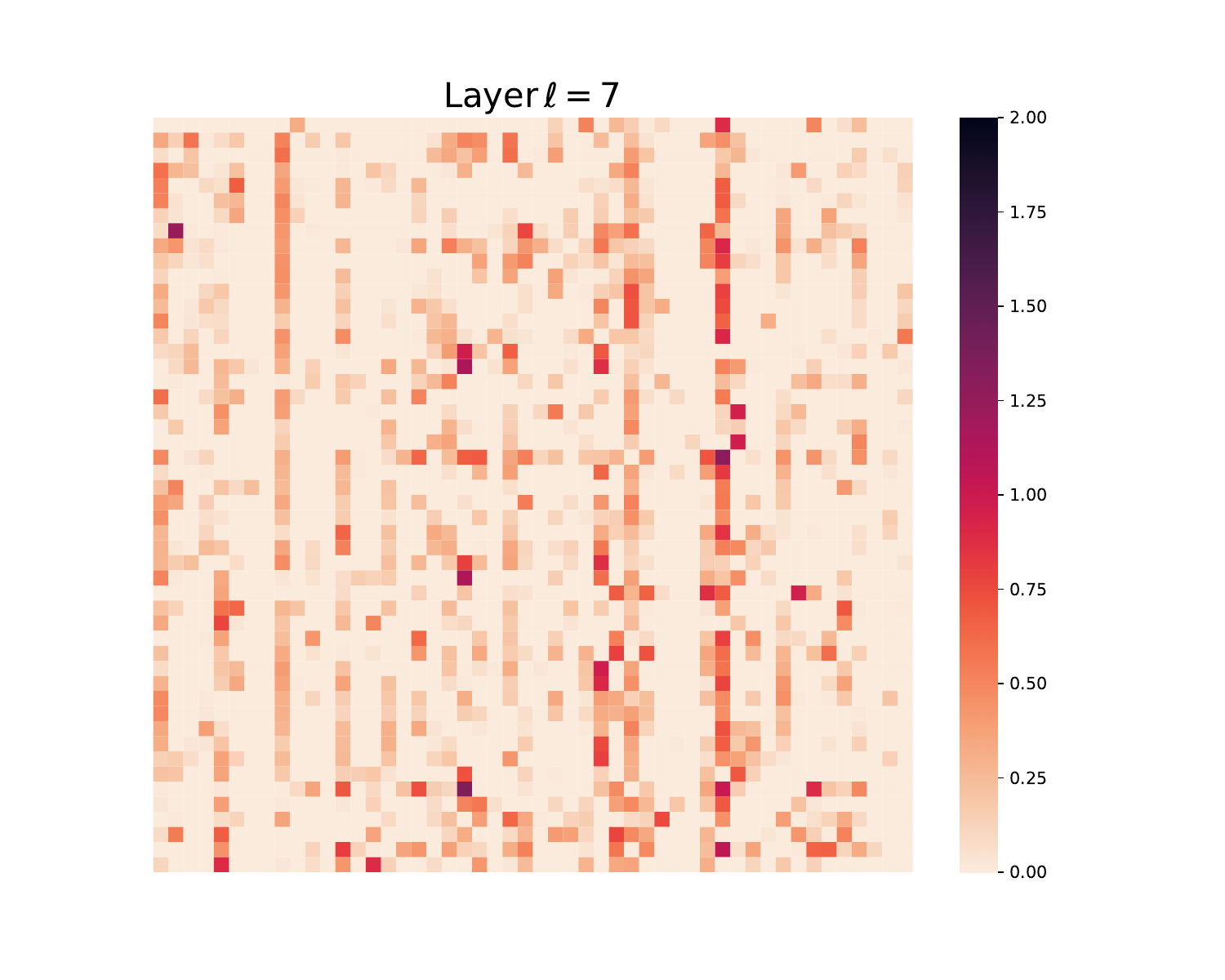}
         \caption{$\ell=7$.}
     \end{subfigure}
     \begin{subfigure}[b]{0.24\textwidth}
         \centering
    \includegraphics[width=\textwidth]{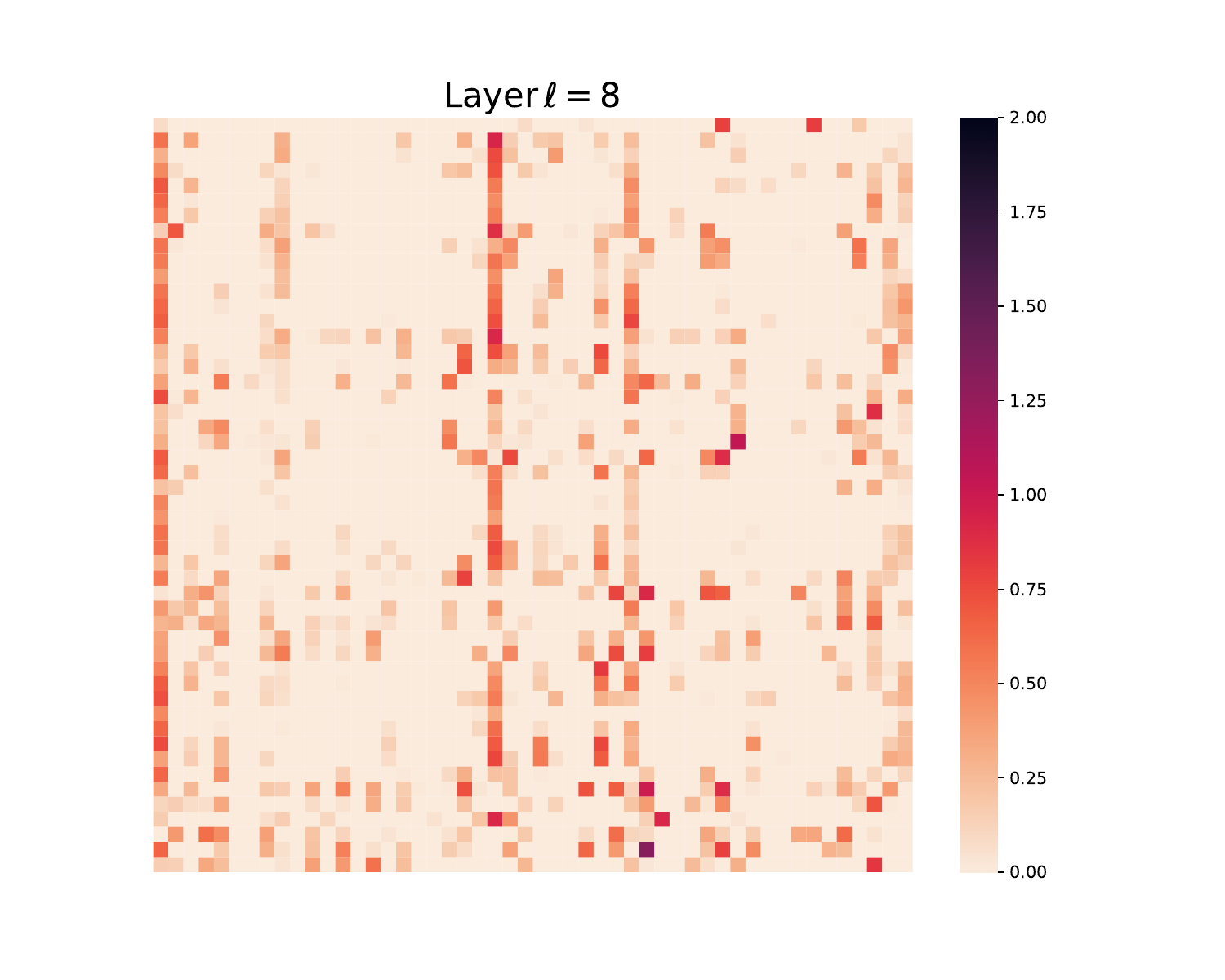}
         \caption{$\ell=8$.}
     \end{subfigure}
     \begin{subfigure}[b]{0.24\textwidth}
         \centering
    \includegraphics[width=\textwidth]{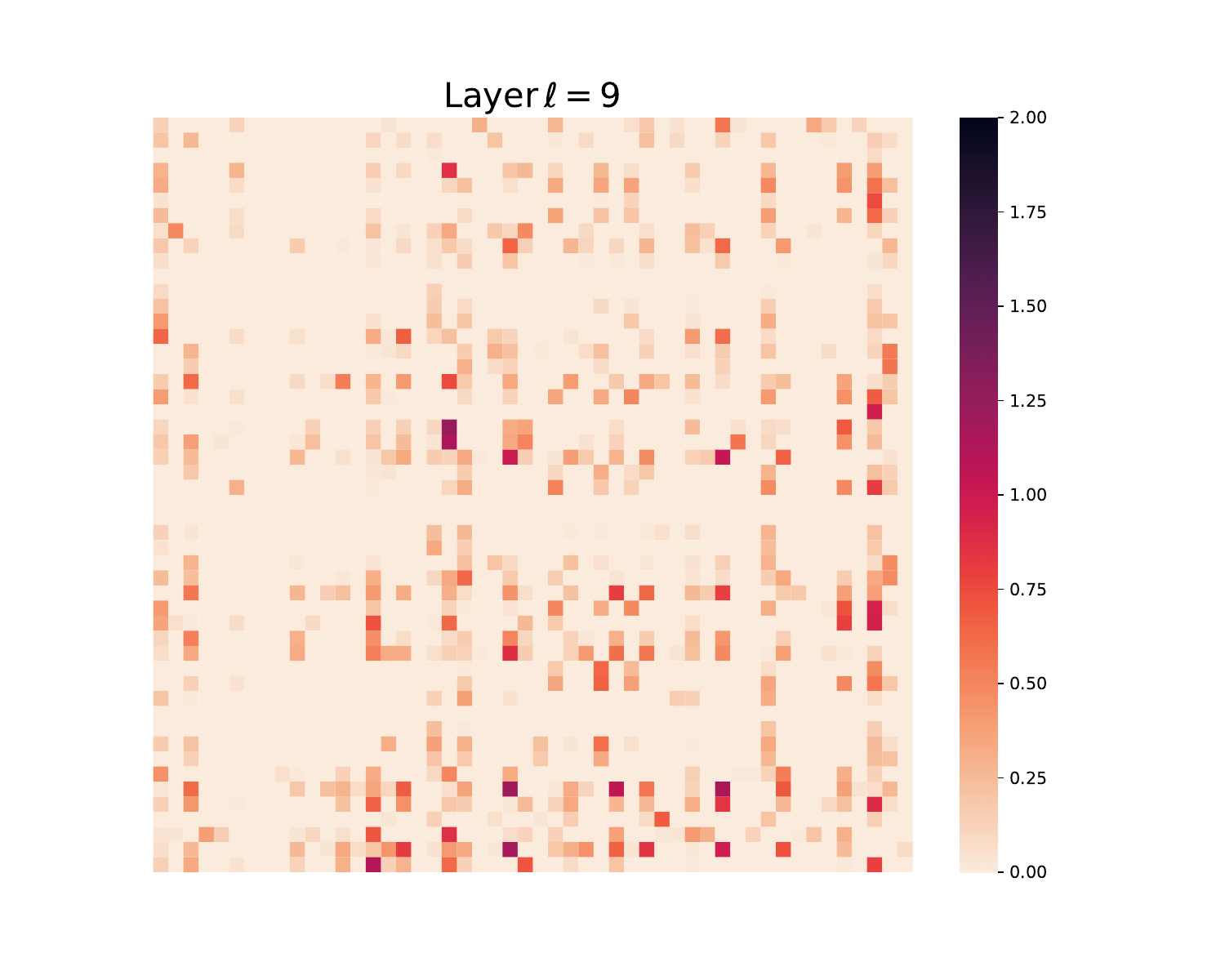}
         \caption{$\ell=9$.}
     \end{subfigure}
     \begin{subfigure}[b]{0.24\textwidth}
         \centering
    \includegraphics[width=\textwidth]{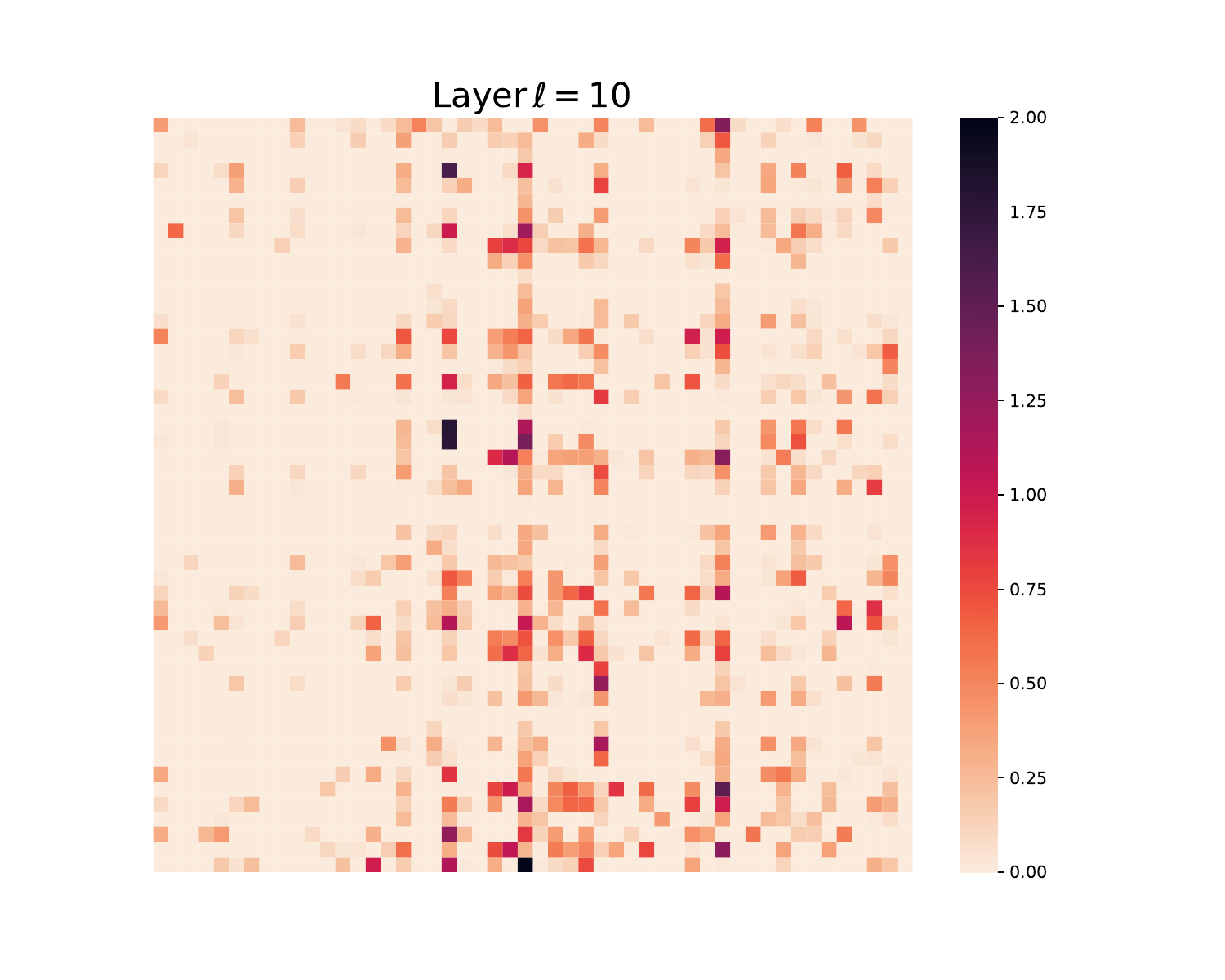}
         \caption{$\ell=10$.}
     \end{subfigure}
     \begin{subfigure}[b]{0.24\textwidth}
         \centering
    \includegraphics[width=\textwidth]{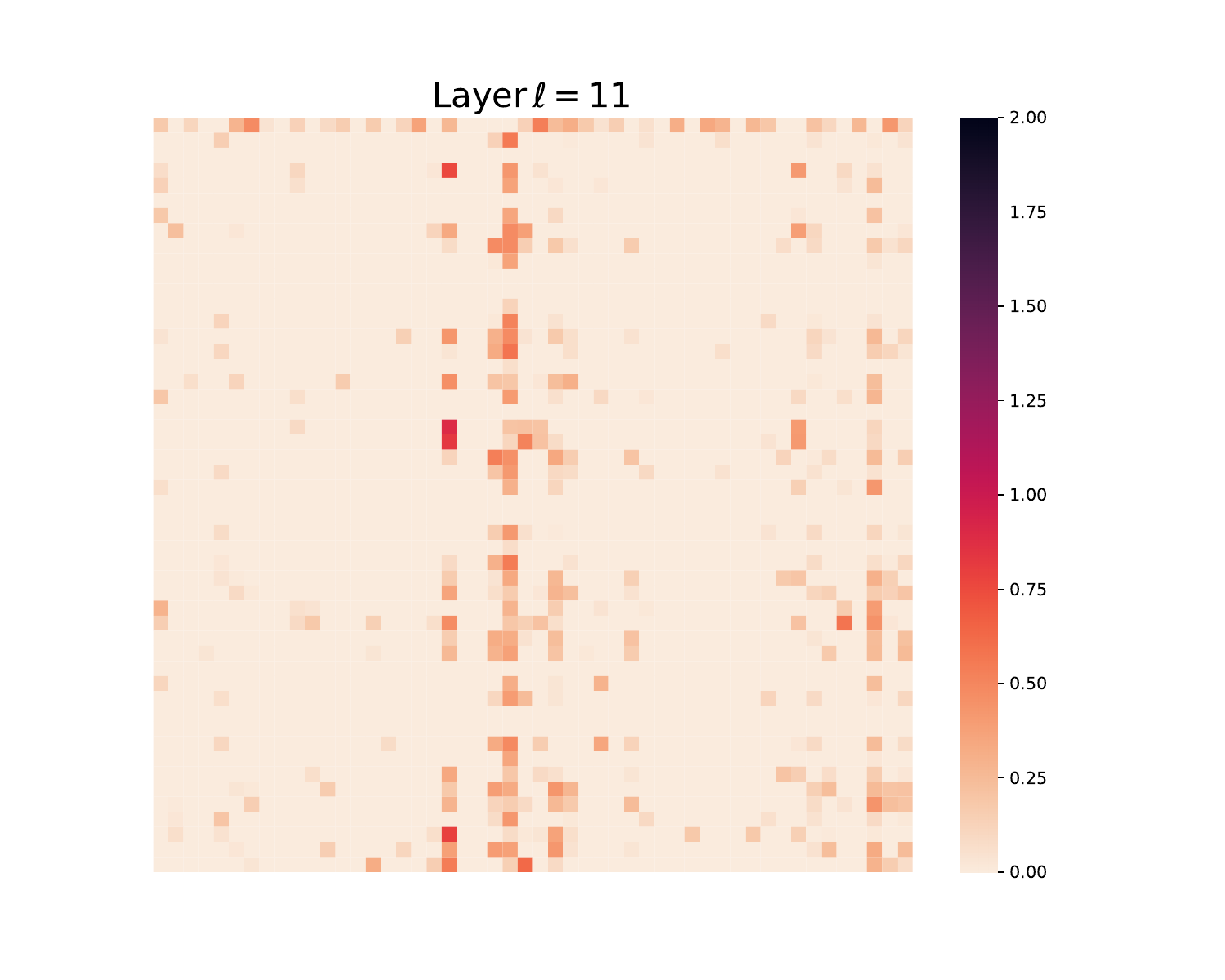}
         \caption{$\ell=11$.}
     \end{subfigure}
     \begin{subfigure}[b]{0.24\textwidth}
         \centering
    \includegraphics[width=\textwidth]{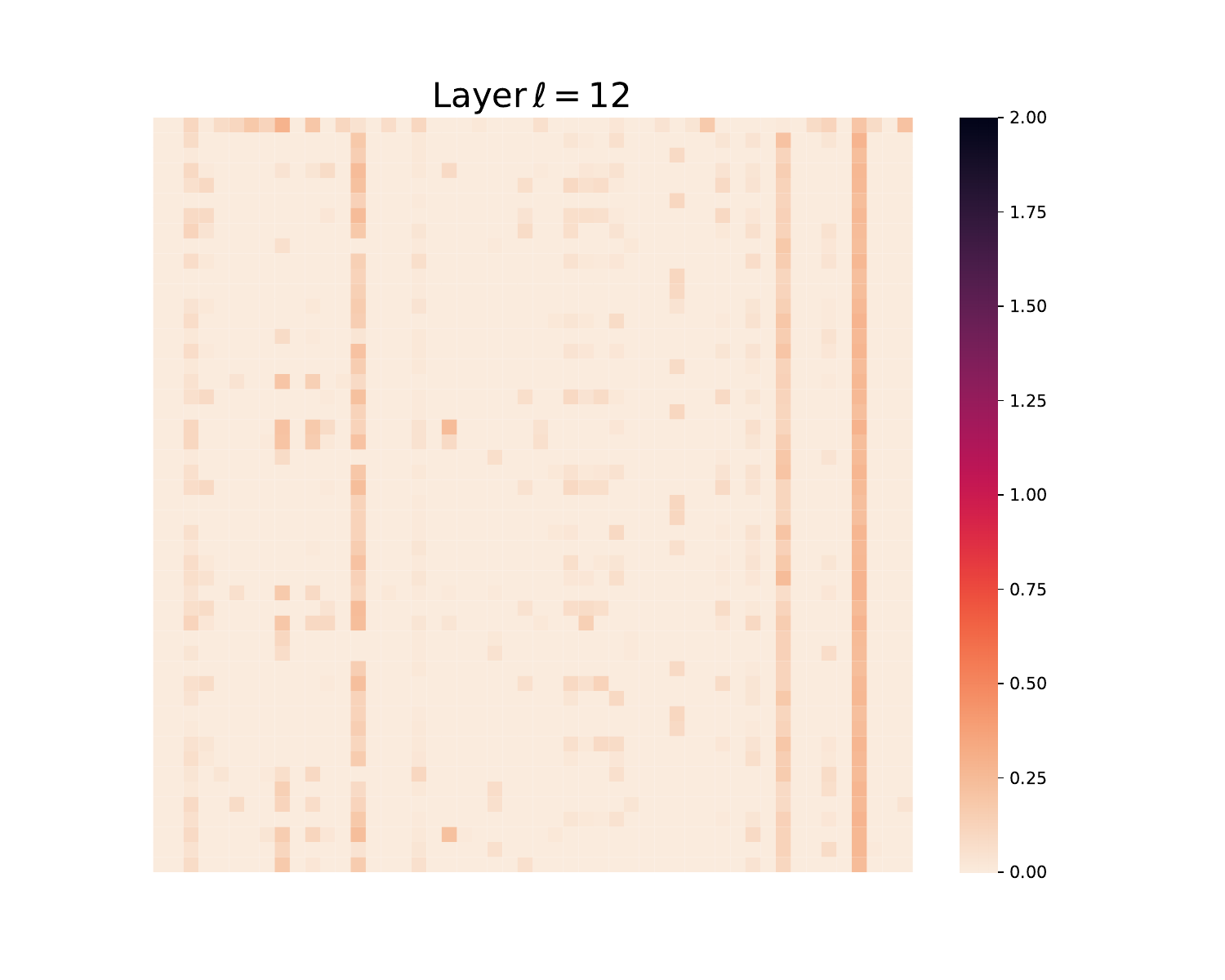}
         \caption{$\ell=12$.}
     \end{subfigure}
        \caption{\small Visualizing layer-wise token $\vZ^{\ell}$ representations at each layer $\ell$. To enhance the visual clarity, we randomly extract a 50$\times$50 sub-matrix from $\vZ^{\ell}$ for display purposes. (Model: \ours{-Tiny})}
        \label{fig:appendix-exp-ista-sparsity-heatmap}
\end{figure}

\paragraph{Visualizing layer-wise token representations.} 
To gain a better understanding of the token representations of \ours{}, we visualize the output of each \texttt{ISTA} block at layer $\ell$ in \Cref{fig:appendix-exp-ista-sparsity-heatmap}. 
Specifically, we visualize the $\vZ^{\ell+1}$ via heatmap plots. 
We observe that the output $\vZ^{\ell+1}$ becomes more sparse as the layer increases. 
Moreover, besides the sparsity, we also find that $\vZ^{\ell+1}$ becomes more structured (i.e., low-rank), which indicates that the set of token representations become closer to linear subspaces, confirming our mental picture of the geometry of each layer (as in \Cref{fig:encoder_compression_sparsification} and \Cref{fig:crate_autoencoding}).

\begin{figure}[t!]
     \centering
     \begin{subfigure}[b]{0.24\textwidth}
         \centering
    \includegraphics[width=\textwidth]{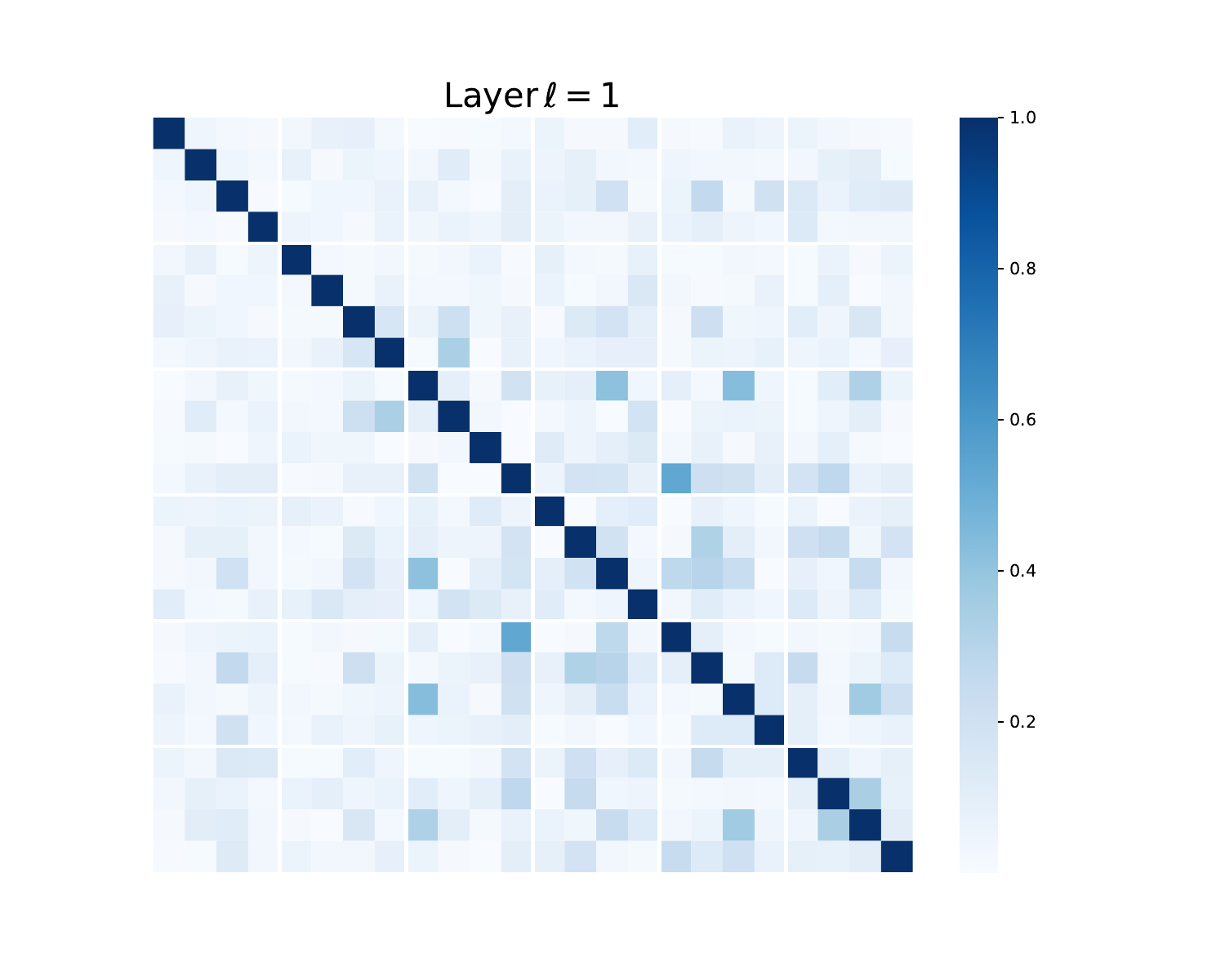}
         \caption{$\ell=1$.}
     \end{subfigure}
     \begin{subfigure}[b]{0.24\textwidth}
         \centering
    \includegraphics[width=\textwidth]{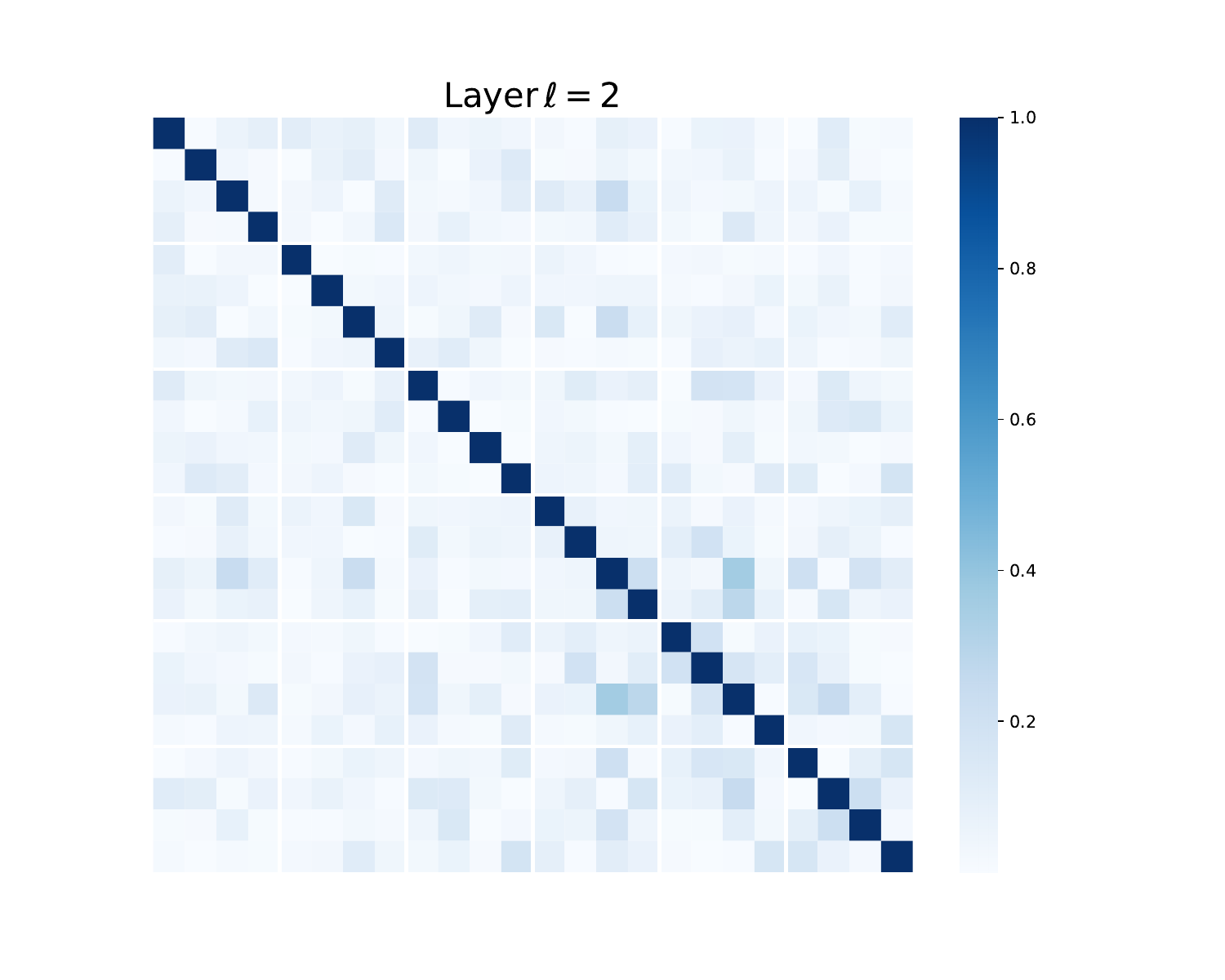}
         \caption{$\ell=2$.}
     \end{subfigure}
     \begin{subfigure}[b]{0.24\textwidth}
         \centering
    \includegraphics[width=\textwidth]{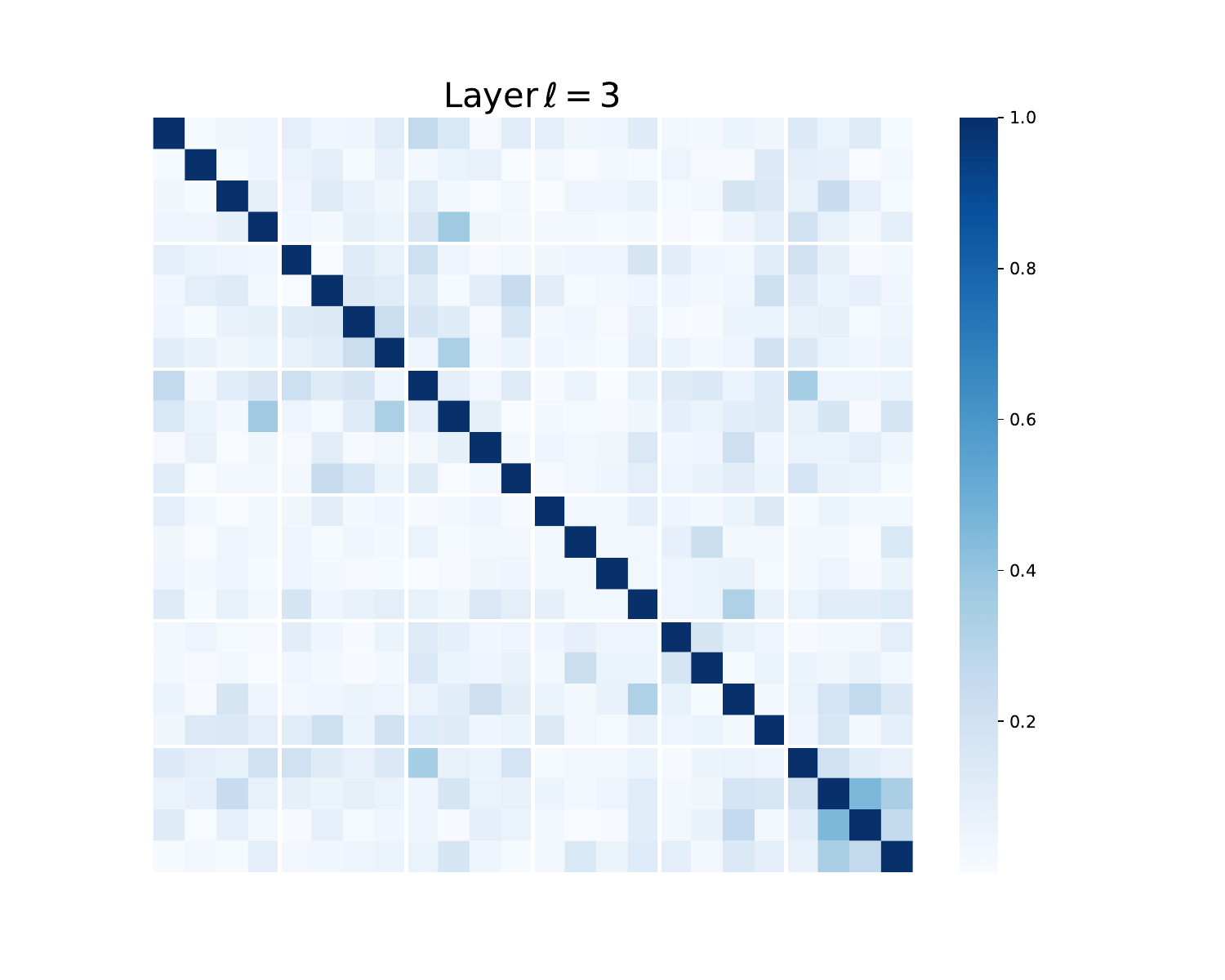}
         \caption{$\ell=3$.}
     \end{subfigure}
     \begin{subfigure}[b]{0.24\textwidth}
         \centering
    \includegraphics[width=\textwidth]{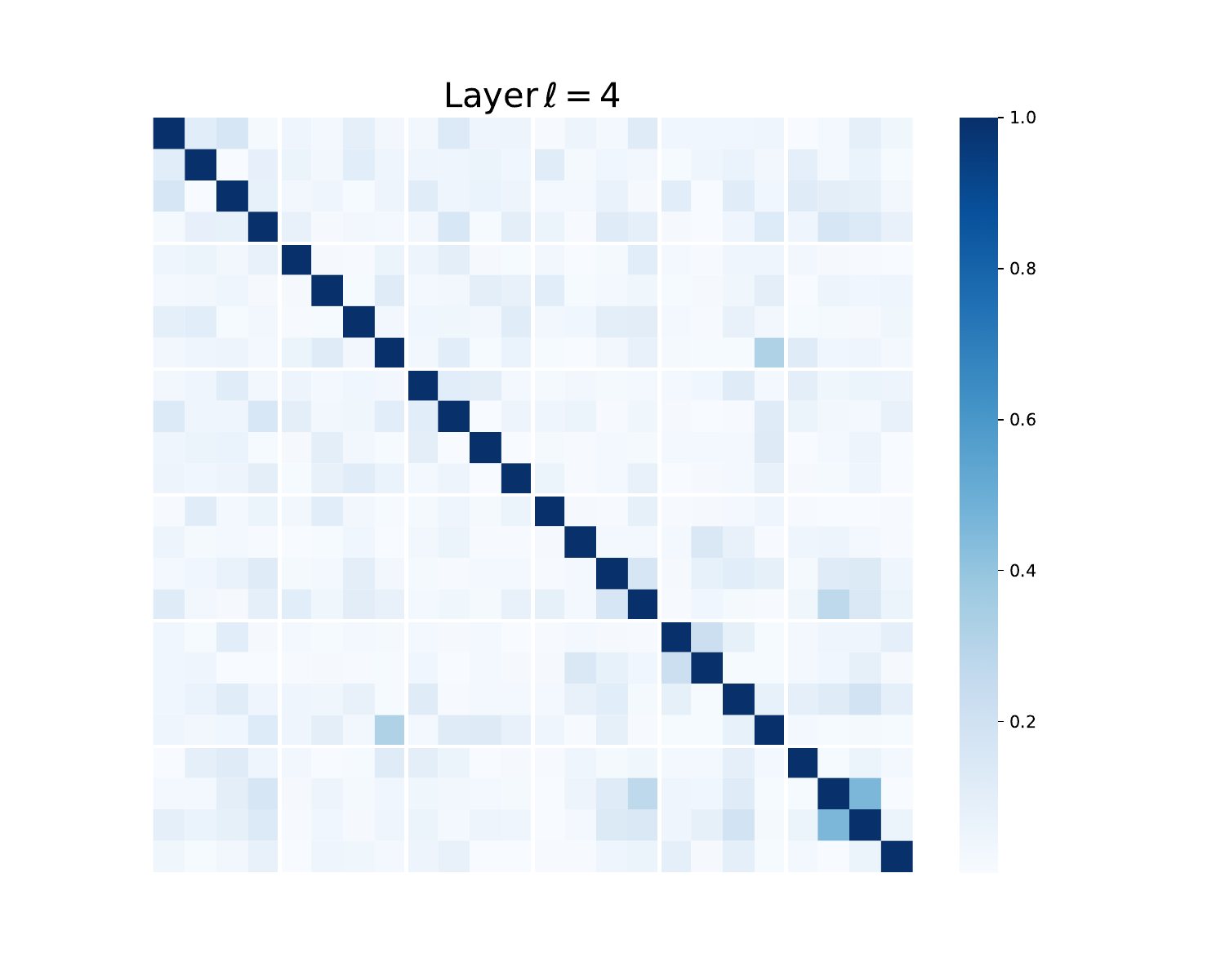}
         \caption{$\ell=4$.}
     \end{subfigure}
     \begin{subfigure}[b]{0.24\textwidth}
         \centering
    \includegraphics[width=\textwidth]{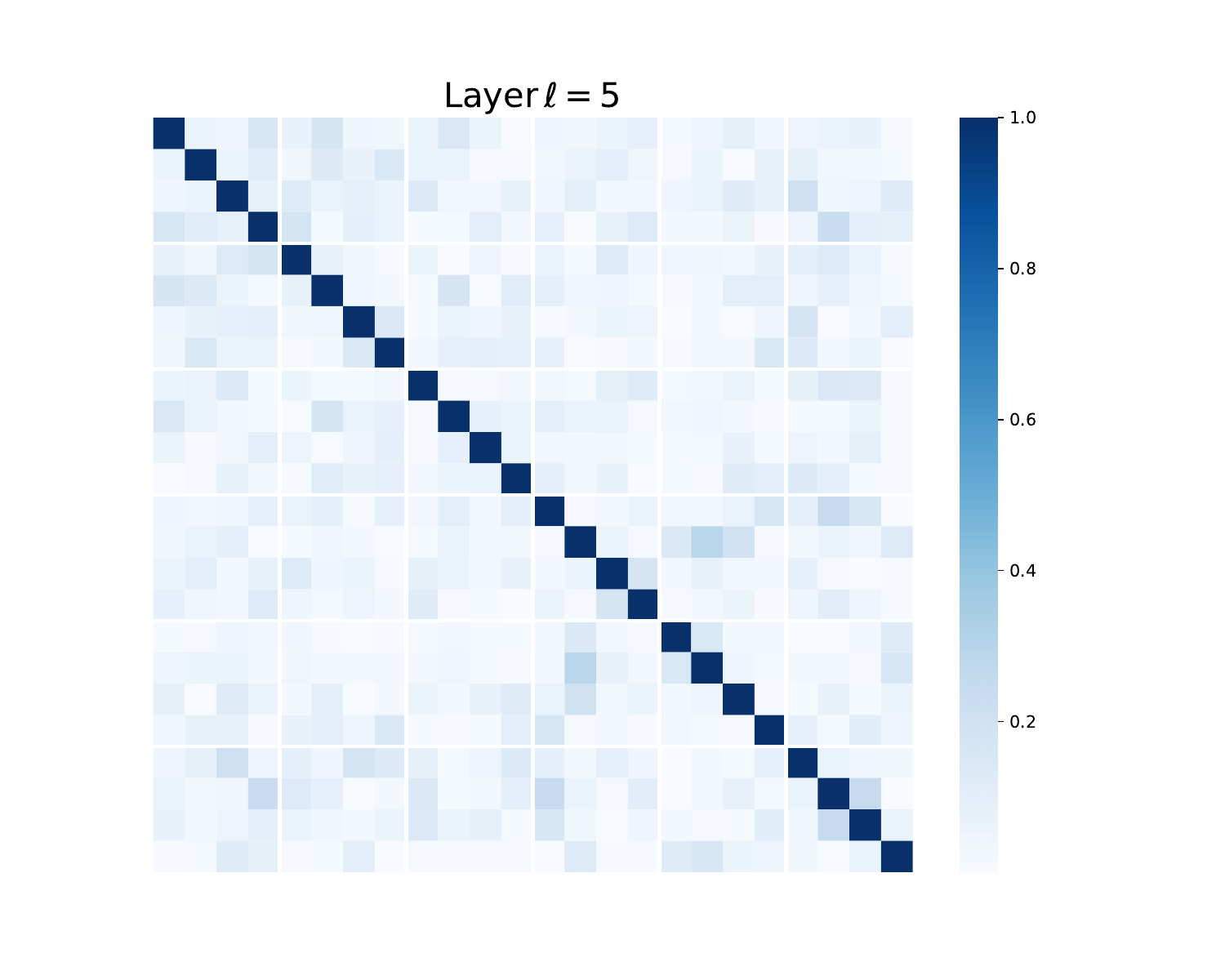}
         \caption{$\ell=5$.}
     \end{subfigure}
     \begin{subfigure}[b]{0.24\textwidth}
         \centering
    \includegraphics[width=\textwidth]{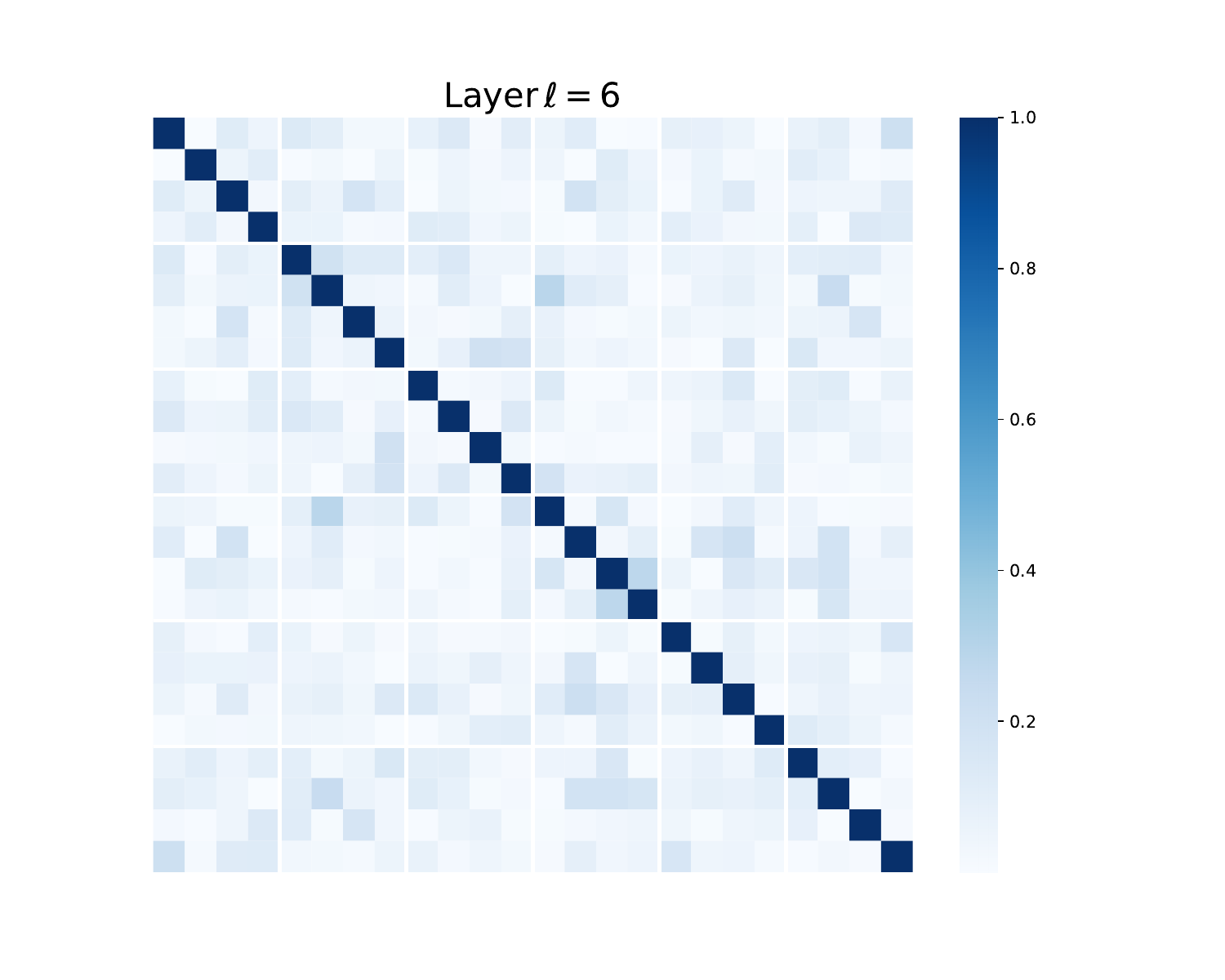}
         \caption{$\ell=6$.}
     \end{subfigure}
     \begin{subfigure}[b]{0.24\textwidth}
         \centering
    \includegraphics[width=\textwidth]{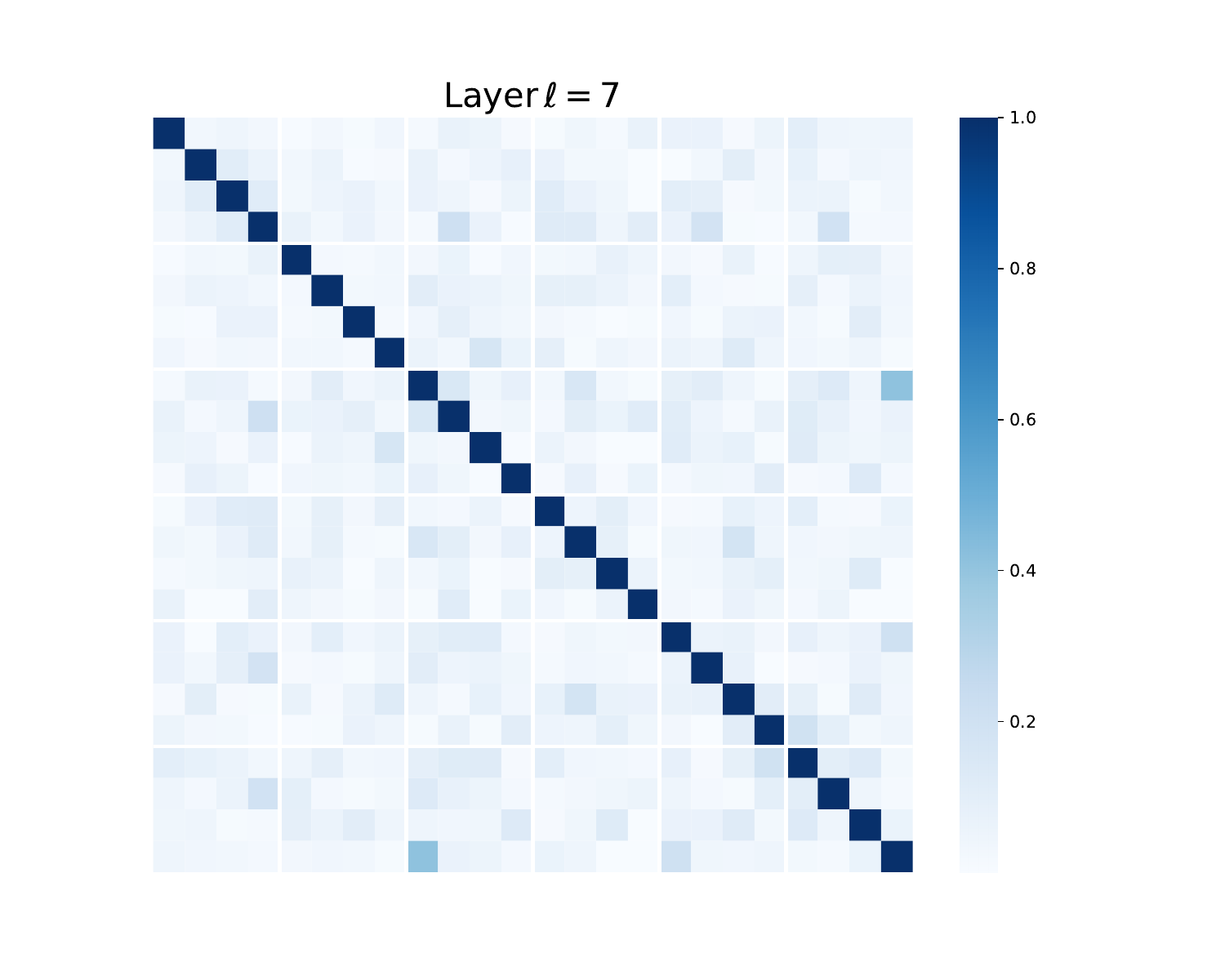}
         \caption{$\ell=7$.}
     \end{subfigure}
     \begin{subfigure}[b]{0.24\textwidth}
         \centering
    \includegraphics[width=\textwidth]{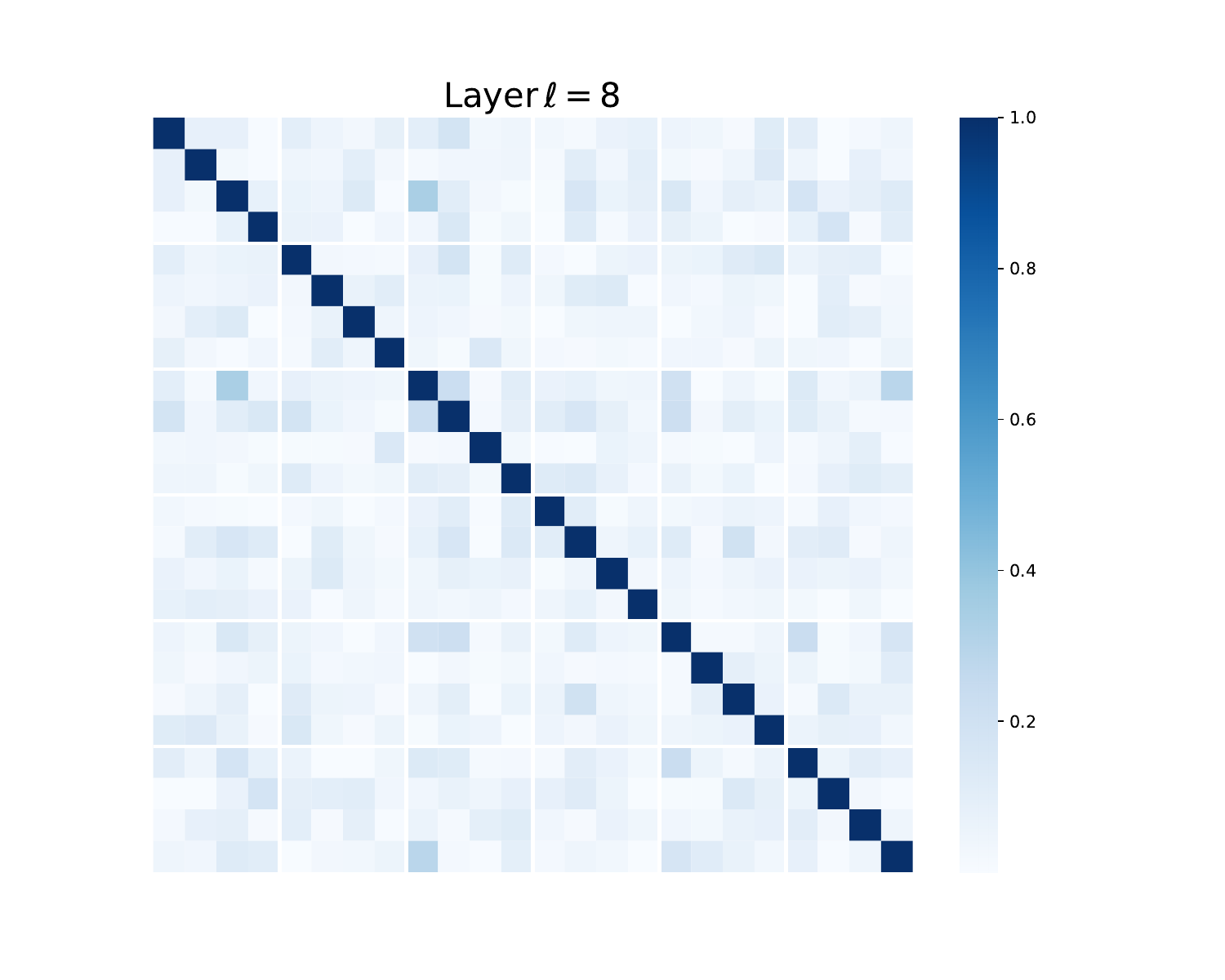}
         \caption{$\ell=8$.}
     \end{subfigure}
     \begin{subfigure}[b]{0.24\textwidth}
         \centering
    \includegraphics[width=\textwidth]{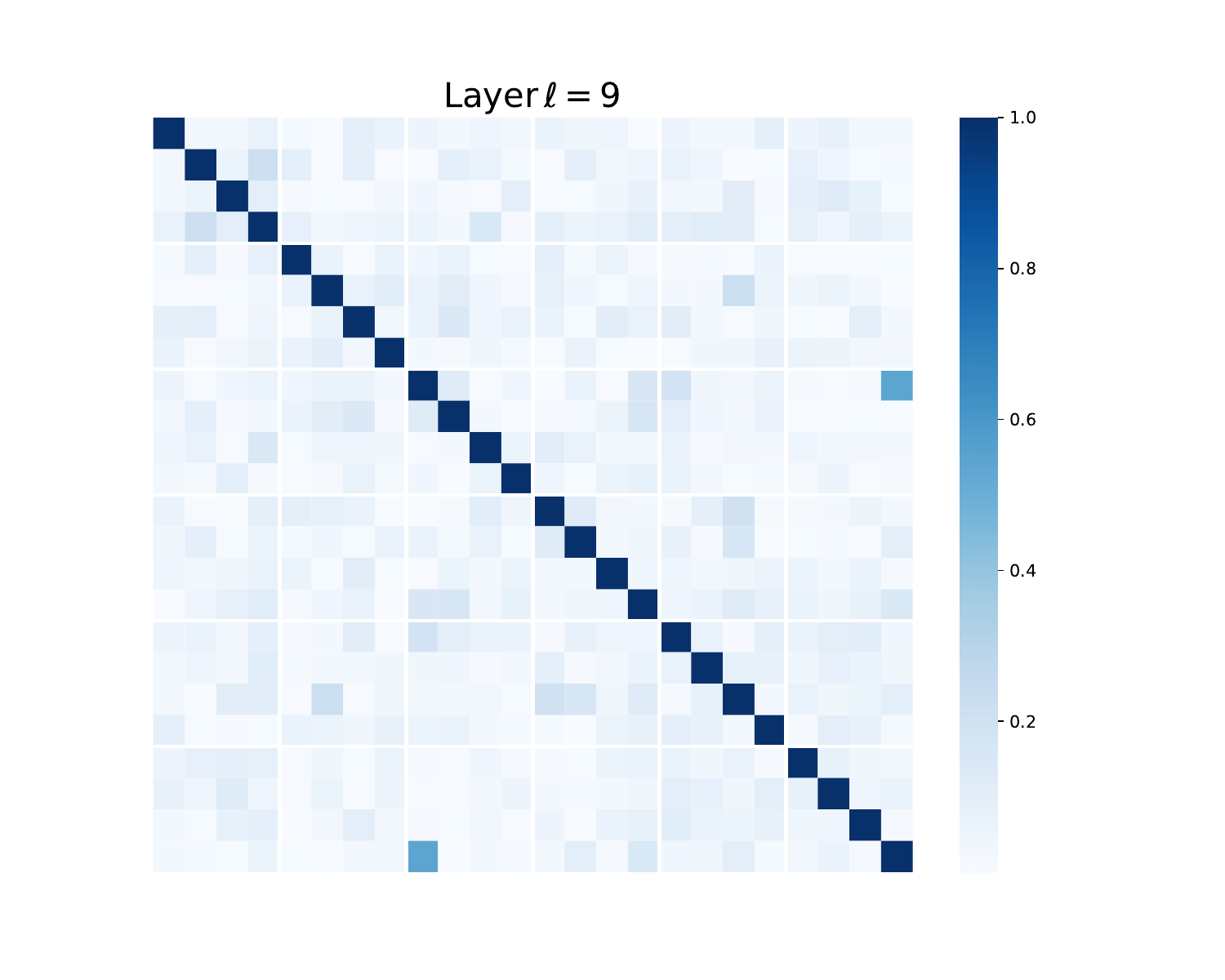}
         \caption{$\ell=9$.}
     \end{subfigure}
     \begin{subfigure}[b]{0.24\textwidth}
         \centering
    \includegraphics[width=\textwidth]{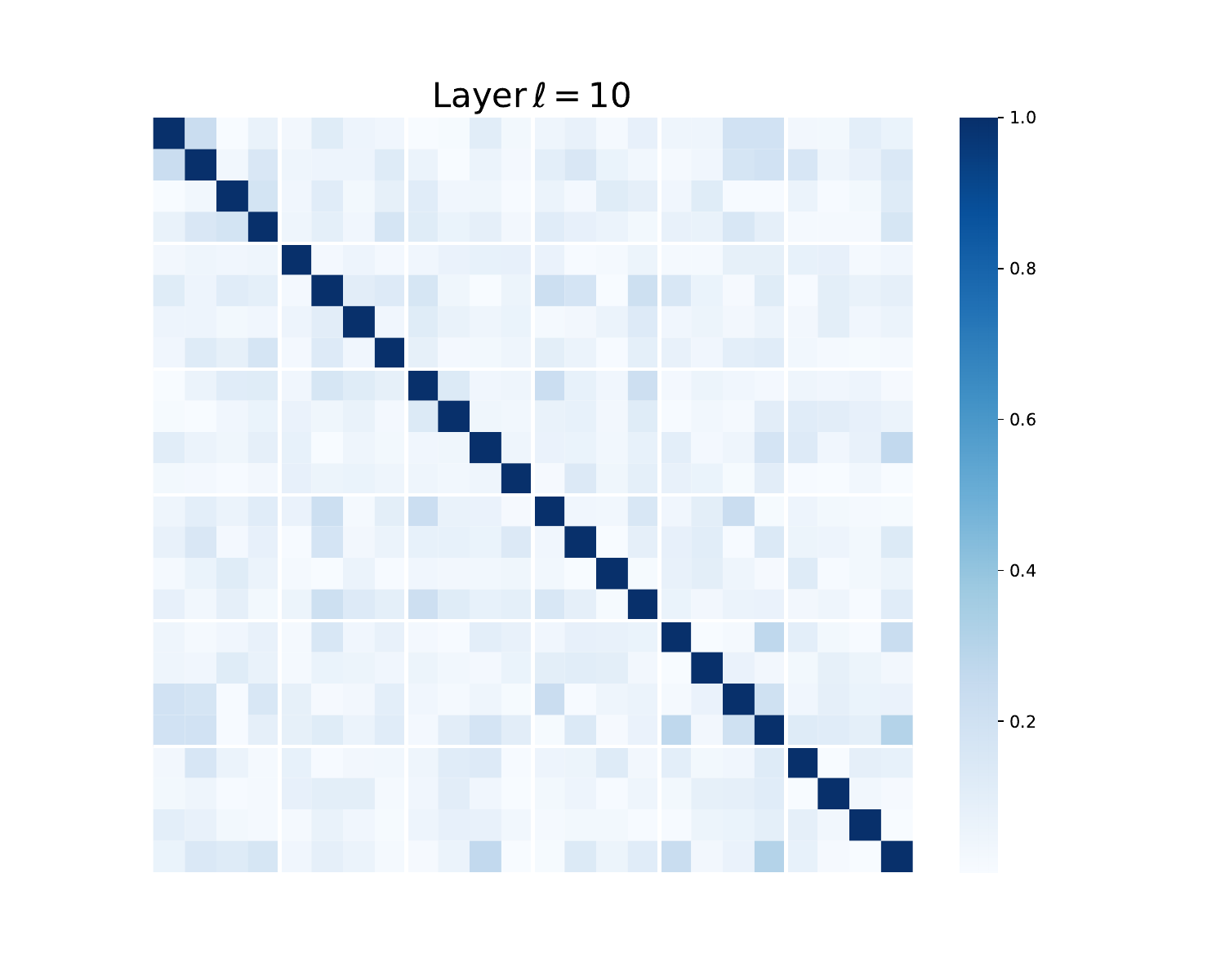}
         \caption{$\ell=10$.}
     \end{subfigure}
     \begin{subfigure}[b]{0.24\textwidth}
         \centering
    \includegraphics[width=\textwidth]{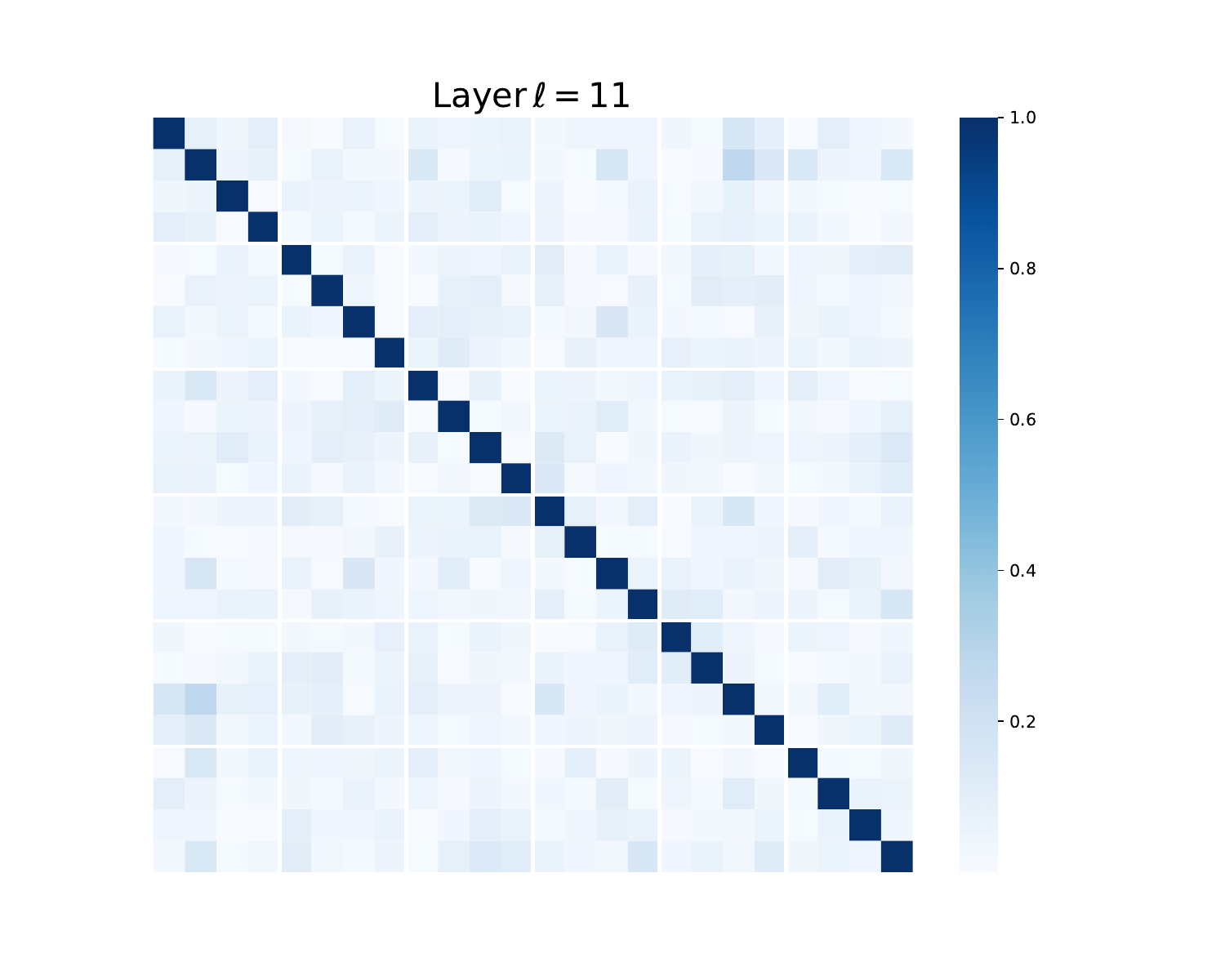}
         \caption{$\ell=11$.}
     \end{subfigure}
     \begin{subfigure}[b]{0.24\textwidth}
         \centering
    \includegraphics[width=\textwidth]{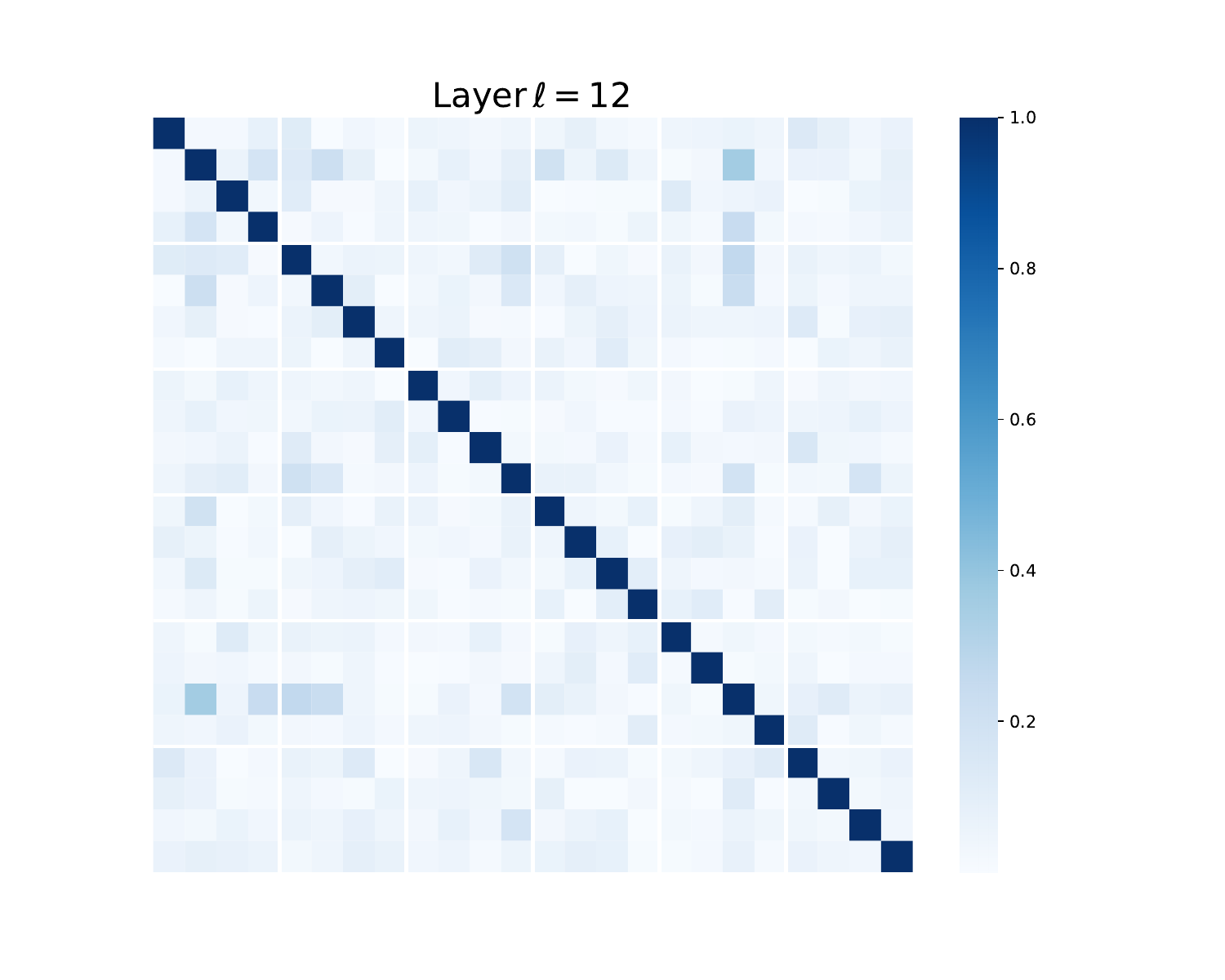}
         \caption{$\ell=12$.}
     \end{subfigure}
        \caption{\small We visualize the $[\vU_1^{\ell}, \dots, \vU_K^{\ell}]^{\adj}[\vU_1^{\ell}, \dots, \vU_K^{\ell}] \in \bR^{pK \times pK}$ at different layers.  
        The $(i, j)$-th block in each sub-figure corresponds to $(\vU_i^{\ell})\adj\vU_j^{\ell}$ for $i, j \in [K]$ at a particular layer $\ell$. To enhance the visual clarity, for each subspace $\vU_i$, we randomly pick 4 directions for display purposes. (Model: \ours{-Tiny})}
        \label{fig:appendix-exp-visualize-UiUj}
\end{figure}

\paragraph{Visualizing layer-wise subspaces in multi-head self-attention.} 
We now visualize the $\vU_{[K]}^{\ell}$ matrices used in the \texttt{MSSA} block. 
In \Cref{sub:compression}, we assumed that $\vU_{[K]}^{\ell}$ were incoherent to capture different ``views'' of the set of tokens. 
In Fig.~\ref{fig:appendix-exp-visualize-UiUj}, we first normalize the columns in each $\vU_k^{\ell}$, then we visualize the $[\vU_1^{\ell}, \dots, \vU_K^{\ell}]^{\adj}[\vU_1^{\ell}, \dots, \vU_K^{\ell}] \in \bR^{pK \times pK}$. 
The $(i, j)\th$ block in each sub-figure corresponds to $(\vU_i^{\ell})\adj\vU_j^{\ell}$ for $i, j \in [K]$ at a particular layer $\ell$. 
We find that the learned $\vU_{[K]}^{\ell}$ are approximately incoherent, which aligns well with our assumptions. One interesting observation is that the $\vU_{[K]}^{\ell}$ becomes more incoherent when the layer index $\ell$ is larger, which suggests that the token representations are more separable. This mirrors the situation in other popular deep networks \citep{he2022law}.

\subsection{Emergence of Semantic Properties in Learned CRATE Attention Maps}
\label{sec:emergence}

\noindent
In this subsection, we analyze the attention maps within \ours{} models trained on vision tasks. 
Previous work~\citep{caron2021emerging} use the self-attention in vision transformers to study semantic segmentation of the input image. 
\citet{caron2021emerging} demonstrated that a specific self-supervised training method, named DINO, can lead to the emergence of segmentation in vision transformers. On the other hand, ViTs trained with supervised learning do not have such properties. 
In contrast, as we will present in this subsection, we find that \textit{the white-box design of \ours{} leads to the emergence of segmentation properties in the network’s self-attention maps, solely through a minimalistic supervised training recipe---the supervised classification training used in vanilla supervised ViTs}. 
We quantify the segmentation properties of \ours{} both qualitatively and quantitatively and compare the results with ViTs. 
Through extensive evaluations, we find that the self-attention maps in white-box transformers (\ours{}) are much more interpretable than vanilla black-box vision transformers.

\subsubsection{Experimental Setup}
\paragraph{Model architecture} We utilize the \ours{} model as described in~\Cref{sub:architecture} at sizes -S/8 and -B/8 (that is, \ours{-Small} and \ours{-Base} with patch size \(8 \times 8\)). We similarly adopt the ViT model from \citet{dosovitskiy2020image} using the same scales (-S/8 and -B/8), ensuring consistent configurations between them. 
Refer to Table~\ref{tab:Comparison_with_VIT_IN21k} for detailed model performance evaluations on classification tasks.

\paragraph{Datasets.} In this subsection, all visual models are trained for classification tasks, using the methodology described in \Cref{subsubsec:image-classification-task}, on the complete ImageNet dataset~\citep{deng2009imagenet}, commonly referred to as ImageNet-21K. This dataset comprises 14,197,122 images distributed across 21,841 classes. We apply the MaskCut~\citep{wang2023cut} pipeline on the COCO {val2017}~\citep{lin2014microsoft}, which consists of 5,000 RGB images, and assess our models' performance for both object detection and instance segmentation tasks. \textit{All evaluation procedures are unsupervised, and we do not update the model weights during the detection and segmentation evaluation processes.}

\begin{figure}
    \centering
    \includegraphics[width=0.95\linewidth]{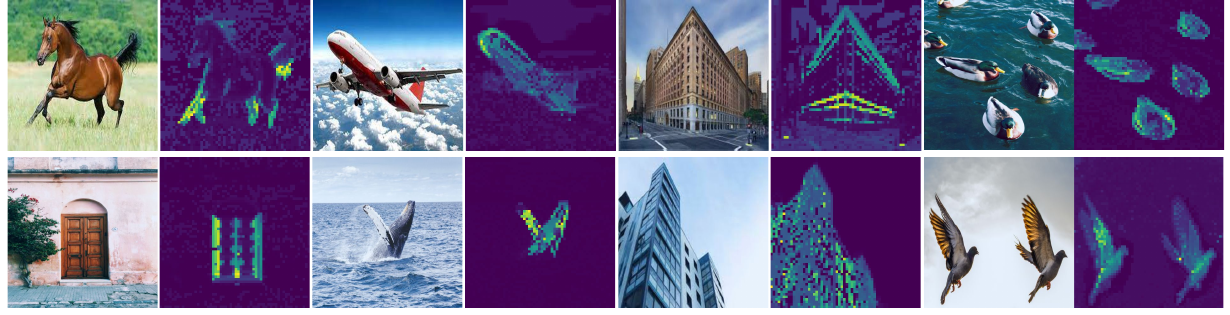}
    \caption{\small \textbf{Self-attention maps from a supervised \ourscaps{} with $8\times
    8$ patches} trained using classification. 
    The \ours{} architecture automatically learns to perform object segmentation without a complex self-supervised training recipe or any fine-tuning with segmentation-related annotations. 
    For each image pair, we visualize the original image on the left and the self-attention map of the image on the right. 
    }
    \label{fig:segmentation-teaser}
\end{figure}

\subsubsection{Measuring the Emergence of Segmentation}
\label{subsubsec:segmentation_measurement}
\paragraph{Visualizing self-attention maps.} To qualitatively measure the emergence phenomenon, we adopt the attention map approach based on the [\texttt{CLS}] token, which has been widely used as a way to interpret and visualize transformer-like architectures~\citep{abnar2020quantifying, caron2021emerging}. 
Indeed, we use the same methodology as \cite{abnar2020quantifying,caron2021emerging}, noting that in \ours{} the query-key-value matrices are all the same; a more formal accounting is deferred to \Cref{app:attn_maps}. 
The visualization results of self-attention maps are summarized in \Cref{fig:segmentation-teaser}. 
We observe that the self-attention maps of the \ours{} model correspond to semantic regions in the input image. 
Our results suggest that the \ours{} model encodes a clear semantic segmentation of each image in the network's internal representations, which is similar to the self-supervised method DINO~\citep{caron2021emerging}. In contrast, as shown in~\Cref{appendix_fig:crate_vs_vit} in the Appendices, the vanilla ViT trained on supervised classification does not exhibit similar segmentation properties.

\paragraph{Object detection and fine-grained segmentation.}\label{subsec_method_maskcut}
To further validate and evaluate the rich semantic information captured by \ours{}, we employ MaskCut~\citep{wang2023cut}, a recent effective approach for object detection and segmentation that does not require human annotations. As usual, we provide a more detailed methodological description in \Cref{app:attn_maps}.
This procedure allows us to extract more fine-grained segmentation for  an image based on the token representations learned by different methods. 
In Figure~\ref{fig:maskcut_visualization}, we visualize the fine-grained segmentation produced by MaskCut on features from ViT trained by classification, \ours{} trained by classification as in Section \ref{subsubsec:image-classification-task}, and \ours{} trained via DINO as in Section \ref{sec:experiment-DINO}, respectively. We compare the segmentation and detection performance in Table~\ref{tab: maskcut_main_result}.
Based on these results, we observe that MaskCut on features learned with supervised ViT  typically fails to produce good  segmentation masks, for example, the first image in Figure~\ref{fig:maskcut_visualization} and the ViT-S/8 row in Table~\ref{tab: maskcut_main_result}. On the other hand, we notice that, regardless of the training task  (supervised or unsupervised), \ours{} is able to capture semantically meaningful boundaries of the main objects in an image.
Compared to ViT, \ours{} provides better internal representation tokens for both segmentation and detection.

\begin{figure}[t!]
    \centering
    \includegraphics[width=1.0\linewidth]{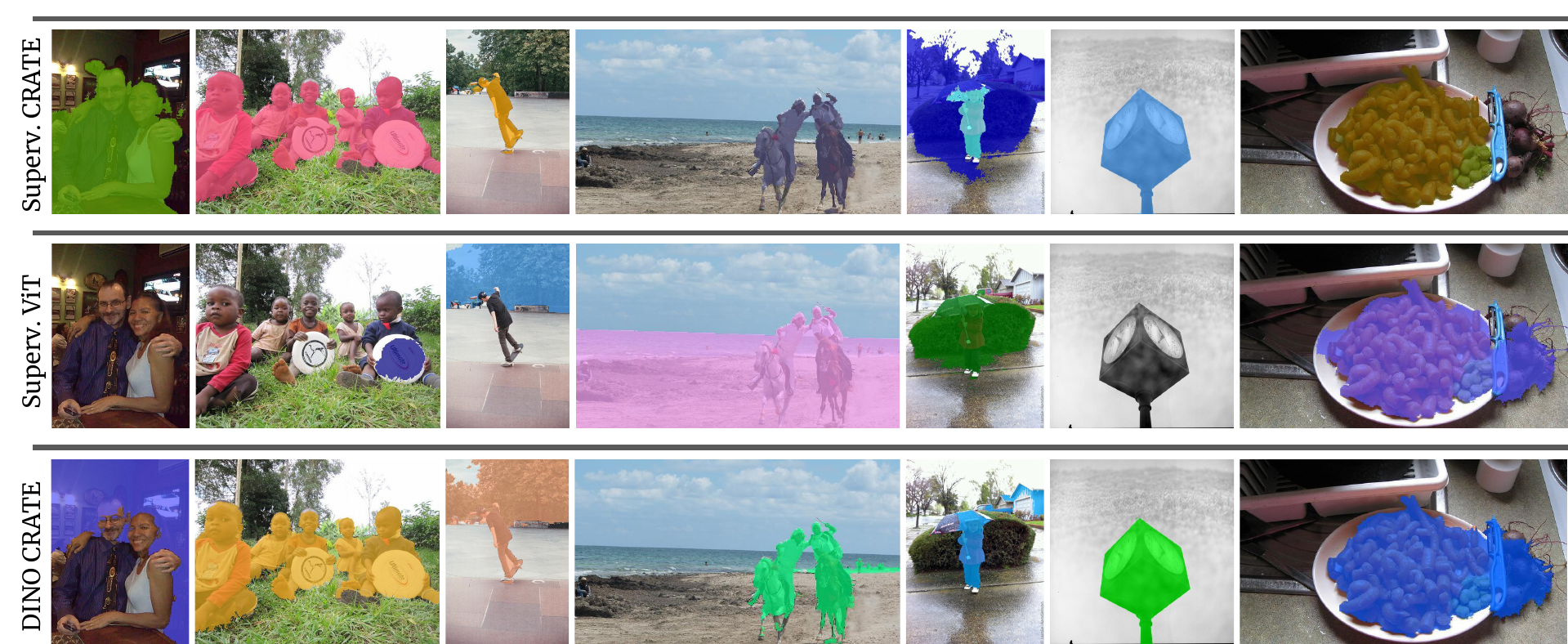}
    \caption{\small \textbf{Visualization of segmentation with MaskCut on COCO {val2017}~\citep{lin2014microsoft}.}
    \textit{Top and Bottom Rows}: \ours{} architecture clearly detects main objects in the image when trained using either supervised classification technique in Section \ref{subsubsec:image-classification-task} or the DINO self-supervised technique in Section \ref{sec:experiment-DINO} ~\citep{caron2021emerging}. \textit{Middle Row}: Note that compared to \ours{}, ViT trained via classification often fails to detect main  objects in the images (columns 2, 3, 4). }
    \label{fig:maskcut_visualization}
\end{figure}

\begin{table*}[t!]
\caption{\small \textbf{Object detection and fine-grained segmentation via MaskCut on COCO {val2017}~\citep{lin2014microsoft}}. We consider models with different scales and evaluate the average precision measured by COCO's official evaluation metric. 
The first four models are pre-trained with image classification tasks under label supervision; the bottom three models are pre-trained via the DINO self-supervised technique~\citep{caron2021emerging}. \ours{} conclusively performs better than the ViT at detection and segmentation metrics when both are trained using supervised classification. 
}
\centering
\small
\setlength{\tabcolsep}{8pt}
\resizebox{0.8\textwidth}{!}{\begin{tabular}{@{}llcccccccc@{}}
\toprule
 &  & \multicolumn{3}{c}{Detection} &  \multicolumn{3}{c}{Segmentation} \\ 
Model & Train & AP$_{50}$ & AP$_{75}$ & AP & AP$_{50}$ & AP$_{75}$ & AP  \\ 
\midrule
\ours{}-S/8 & Supervised & 2.9 & 1.0 & 1.1 & 1.8 & 0.7 & 0.8 \\
\ours{}-B/8 & Supervised & 2.9& 1.0 & 1.3 & 2.2 & 0.7 & 1.0 \\
ViT-S/8 & Supervised & 0.1& 0.1 & 0.0 & 0.0 & 0.0 & 0.0 \\
ViT-B/8 & Supervised & 0.8 & 0.2 & 0.4 & 0.7 & 0.5 & 0.4 \\
\midrule
\rowcolor{lightgray!30}
\color{gray} \ours{}-B/8 & \color{gray} DINO & \color{gray}3.5 & \color{gray}1.1 & \color{gray}1.6 & \color{gray}2.4 & \color{gray}0.5 & \color{gray}1.0 \\
\rowcolor{lightgray!30}
\color{gray} ViT-S/8 & \color{gray} DINO & \color{gray}5.0 & \color{gray}2.0 & \color{gray}2.4 & \color{gray}4.0 & \color{gray}1.3 & \color{gray}1.7 \\
\rowcolor{lightgray!30}
\color{gray} ViT-B/8 & \color{gray} DINO & \color{gray}5.1 & \color{gray}2.3 & \color{gray}2.5 & \color{gray}4.1 & \color{gray}1.3 & \color{gray}1.8 \\
\bottomrule
\end{tabular}}
\label{tab: maskcut_main_result}
\end{table*}

\begin{figure}[t!]
    \centering
    \includegraphics[width=\textwidth]{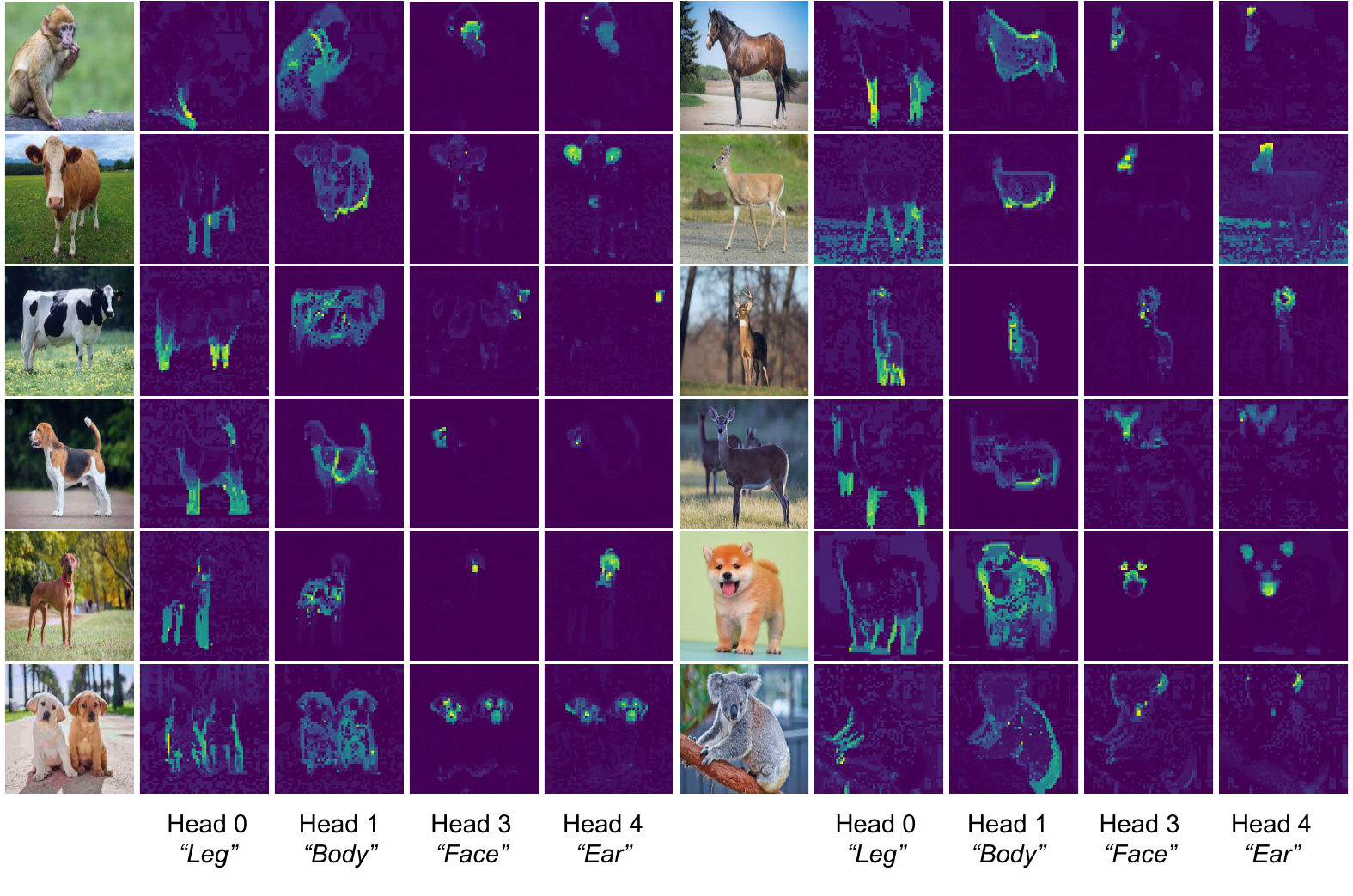}
    \caption{\small \textbf{Visualization of semantic heads.}  We forward a mini-batch of images through a supervised \ours{} and examine the attention maps from all the heads in the penultimate layer. 
    We visualize a selection of attention heads to show that certain heads convey specific semantic meaning, i.e. \textit{head 0} $\leftrightarrow$ \textit{``Legs''}, \textit{head 1} $\leftrightarrow$ \textit{``Body''}, \textit{head 3} $\leftrightarrow$ \textit{``Face''}, \textit{head 4} $\leftrightarrow$ \textit{``Ear''}. }
    \label{fig:semantic_heads}
\end{figure}

\subsubsection{Analysis of Segmentation in \ourscaps{}}
\paragraph{Segmentation emerges through minimalistic design.}
Our empirical results demonstrate that self-supervised learning, as well as the specialized design options in DINO~\citep{caron2021emerging} (e.g., momentum encoder, student and teacher networks, self-distillation, etc.) are not necessary for the emergence of segmentation. We contend that an equally-promising approach to promote segmentation properties in transformer is to \textit{design the transformer architecture with the structure of the input data in mind}. This finding of \ours{} represents the \textit{first supervised vision model with emergent segmentation properties}, and establishes white-box transformers as a promising direction for interpretable data-driven representation learning in foundation models. 

\paragraph{Identifying semantic properties of attention heads.}\label{subsec:semantic_meaning}
We are interested in capturing the semantic meaning of certain attention \textit{heads}; this is an important task for interpretability and is already studied for language transformers \citep{olsson2022context}.
Intuitively, each head captures certain features of the data. 
Given a \ours{} model, we first forward an input image (e.g.\ a horse image as in \Cref{fig:semantic_heads}) and select four attention heads which seem to have semantic meaning by manual inspection. After identifying the attention heads, we visualize the self-attention map of these heads on other input images. 
Interestingly, we find that each of the selected attention heads captures a different part of the object, and even a different semantic meaning. 
For example, the attention head displayed in the first column of \Cref{fig:semantic_heads} captures the legs of different animals, and the attention head displayed in the last column captures the ears and head. 
This parsing of the visual input into a part-whole hierarchy has been a fundamental goal of learning-based recognition architectures since deformable part models~\citep{Felzenszwalb2005-ma,Felzenszwalb2008-ui} and capsule networks~\citep{Hinton2011-yb,Sabour2017-xo}---strikingly, it emerges from the white-box design of \ours{} within our simple supervised training setup.%

%% file: sec_conclusion.tex
\section{Conclusions and Open Directions}\label{sec:conclusion}

In this paper, we propose a new theoretical framework that allows us to derive a deep transformer-like network architecture, named \ours{}, as an incremental optimization scheme to learn compressed and sparse representation of the input data, measured by a principled objective for information gain: {\em the sparse rate reduction}. The so-derived and learned deep architectures are not only fully mathematically interpretable but also consistent on a layer-by-layer level with their design objective. 
{\color{revision} Despite being arguably the simplest among all possible designs, these networks already demonstrate strong performance -- slightly worse than existing black-box models -- on large-scale real-world (image or text) datasets and tasks (discriminative and generative), and in supervised, unsupervised, and self-supervised settings.}
{\color{revision} 
In fact, recent follow-up work~\citep{yang2024scaling} suggested that one can significantly improve the performance of \textsc{crate}-like white-box models with a light training recipe on vision tasks, and to a large extent close the gap between white-box models such as \textsc{crate} and the conventional Transformers on some tasks. 
Still, a gap persists. It would be interesting to investigate how to construct  white-box models that can achieve same level of performance as black-box models, but we leave this for future work.
}

We believe this work truly helps \textit{bridge the gap between theory and practice} of deep neural networks, as well as helps unify seemingly separate approaches to learning and representing data distributions through the perspective of data compression. These approaches include, but are not limited to: rate reduction, denoising-diffusion, and transformers. As we have seen, in these approaches, the role of each layer of the associated deep neural networks can be interpreted mathematically as operators to incrementally compress (denoise) and transform the data distribution according to its low-dimensional structures (modeled as Gaussian mixtures). In particular, our work provides a principled justification for slightly earlier empirical syntheses of some of these approaches, for example the recent success of Transformer-type architectures for denoising in the context of diffusion models \citep{Peebles2022DiT}.%

To a large extent, our work suggests that a universal way to effectively and efficiently learn a data distribution with intrinsically low-dimensional structures in a high-dimensional space is through {\em incremental compression}, as already stated as a slogan by~\citet{Wright-Ma-2022}:
\begin{quotation}
    \centering
    \textit{We compress to learn, and we learn to compress.}
\end{quotation}
This work and the previous work on ReduNet by \cite{chan2021redunet} together strongly suggest that various deep neural networks are simply means to an end---compression. Different optimization schemes for compression are manifested as different network architectures and operators, e.g.\ LeNet, LISTA, ResNet, ReduNet, Transformers, which are now  elucidated by CRATE. 

From a practical standpoint, it is clear that this new framework can now provide us with theoretical guidelines to design and justify new, potentially more powerful,  deep architectures for representation learning that are based on more advanced optimization techniques or strategies, instead of the basic gradient-based technique used to construct \ours{}. Also, in this work, we have only used \ours{} to learn deterministic autoencoding models. We believe that this framework can be extended to learn structured representations of the data distribution based on more advanced structured diffusion and denoising. This could potentially lead to more \textit{controllable and consistent} generative methods, since improving the consistency and efficiency of existing diffusion-based generative models remains a challenging open problem \citep{song2023consistency}. 

Furthermore, as suggested by~\citet{dai2022ctrl,ma2022principles}, for any system to learn a low-dimensional data distribution \textit{autonomously and continuously}, one needs to integrate the encoder and decoder not only as an autoencoder but as a \textit{closed-loop transcription}:
\begin{equation}
     \x \in \mathbb{R}^D \xrightarrow{\hspace{2mm} f(\x)\hspace{2mm}} \z  \in \mathbb{R}^d\xrightarrow{\hspace{2mm} g(\z) \hspace{2mm}} \widehat{\x} \in \mathbb{R}^D \xrightarrow{\hspace{2mm} f(\x)\hspace{2mm}} \widehat{\z} \in \mathbb{R}^d,
     \label{eqn:closed-loop}
\end{equation}
which allows the system to correct itself by minimizing the discrepancy between the internal representations $\z$ of the data and $\widehat{\z}$ of the generated approximation, without any external help (by human) to compare between the data $\x$ and $\widehat \x$. So far, the effectiveness of such a closed-loop transcription system has only been demonstrated with black-box deep architectures such as the ResNet~\citep{dai2022ctrl,tong2022incremental}. It would be possible now to build a closed-loop transcription with purely white-box encoder and decoder designs. This could ultimately enable us to develop \textit{complete}  learning systems that are capable of learning autonomously and continuously an internal data representation that is fully mathematically interpretable, controllable, and eventually self-consistent, or ``aligned'',  with the true external data distribution. The so-learned representation, just like our acquired memory, can support both discriminative and generative tasks, and serve both recognition and prediction purposes. 

Last but not least, we hope that this work provides an alternative viewpoint that may help clarify the ultimate capabilities of modern artificial intelligence (AI) systems, which are typically based on deep networks such as transformers. Just as with all other natural phenomena or technical innovations that were once ``black boxes'' to people, significant confusion and anxiety is arising in society about the potential or implications of emerging new AI systems, including the recent large language models (LLMs) such as GPT-4~\citep{openai2023gpt4}, large image generation models such as Midjourney and DALL-E 3~\citep{Betker2023improving}, and many other multi-modal large language models (MLLMs)~\citep{liu2023visual,gpt4v}. 
From the perspective of this work, in which we developed a clear understanding of transformer-type learning systems, these large models are unlikely to do anything beyond purely mechanical data compression (encoding) and interpolation (decoding). %
That is, this work suggests that regarding essence of these existing large AI models, however magical and mysterious they might appear to be, 
\begin{quotation}
    \centering
    \textit{compression is all there is.}
\end{quotation}
 This leaves us with a much more profound and challenging question: whether \textit{compression alone} is all one needs for a system, natural or artificial, to develop general intelligence or even sentience? We believe that compression, employed correctly, is merely the very first step of a long path towards general intelligence.

%% file: app_technical_compression.tex
\subsection{Companion to \Cref{sub:compression}} \label{app:proofs-compression}

We now give the derivation of the approximation alluded to in \Cref{sub:compression}. In the process,
we make some simple approximations and technical assumptions. The validity of these assumptions may be explored, and the approximations refined, altogether providing a more complex (and possibly more performant) resulting self-attention like operator. For the sake of technical clarity and simplicity in this work, we make perhaps the \textit{simplest possible choices}. As a result, we \textit{do not} claim that our network is optimally designed, but rather that the principles we develop in this work (compression, denoising, sparsification, unrolled optimization) can provide the backbone for far superior and more interpretable network architectures in the future on sundry tasks. As it is, with our straightforward, simple, and interpretable design, we still obtain meaningful conceptual results and very solid empirical performance.

Recall that \(\beta = p/(n\eps^{2})\).

\begin{approximation}\label{thm:grad_codingrate_approx} 
    Let \(\vZ \in \bR^{d \times n}\) have unit-norm columns, and \(\vU_{[K]} = (\vU_1, \dots, \vU_K) \) such that each \(\vU_{k} \in \bR^{d \times p}\) is an orthogonal matrix, the \((\vU_{k})_{k = 1}^{K}\) are incoherent, and the columns of \(\vZ\) approximately lie on \(\bigcup_{k = 1}^{K}\mathrm{Span}(\vU_{k})\). Let \(\kappa > 0\). Then
    \begin{equation}
        \vZ - \kappa \nabla_{\vZ}R^{c}(\vZ \mid \vU_{[K]}) \approx (1 - \kappa\beta)\vZ + \kappa\beta\MSSA{\vZ \mid \vU_{[K]}},
    \end{equation}
    where as in \Cref{sub:compression} we have
    \begin{align}
        \SSA{\vZ \mid \vU_{k}} 
        &= (\vU_{k}\adj\vZ)\softmax{(\vU_{k}\adj\vZ)\adj(\vU_{k}\adj\vZ)}, \\
        \MSSA{\vZ \mid \vU_{[K]}} 
        &= \beta\mat{\vU_{1}, \dots, \vU_{K}}\mat{\SSA{\vZ \mid \vU_{1}} \\ \vdots \\ \SSA{\vZ \mid \vU_{K}}},
    \end{align}
    where \(\softmax{\cdot}\) is the softmax operator (applied to each column of an input matrix), i.e.,
    \begin{align}
        \softmax{\vv} 
        &= \frac{1}{\sum_{i = 1}^{n}e^{v_{i}}}\mat{e^{v_{1}} \\ \vdots \\ e^{v_{n}}}, \\
        \softmax{\mat{\vv_{1}, \dots, \vv_{K}}} 
        &= \mat{\softmax{\vv_{1}}, \dots, \softmax{\vv_{K}}}.
    \end{align}
\end{approximation}

\begin{proof}
    According to \Cref{eq:rate-gradient}, the gradient $\nabla_{\vZ}{R}^c(\vZ \mid \vU_{[K]})$ is
    \begin{equation}
     \nabla_{\vZ}{R}^c(\vZ \mid \vU_{[K]}) 
         = \beta \sum_{k=1}^K \vU_k\vU_k\adj\vZ\bp{\I +
    \beta(\vU_k\adj\vZ)\adj(\vU_k\adj\vZ)}^{-1}.
    \end{equation}
    Notice that according to \citep{chan2021redunet}, the gradient is precisely the residual of a ridge regression for each (projected) token $\vU_k\adj\vz_{i}$ using other projected tokens \(\vU_{k}\adj\vz_{j}\) as the regressors, hence being the residual of an auto-regression. 
    
    However, as we have seen in the work of ReduNet \citep{chan2021redunet}, computing the inverse $\bp{\I +
    \beta(\vU_k^*\vZ)^*(\vU_k^*\vZ)}^{-1}$ can be expensive. Hence for computational efficiency, we may approximate it with the first order term of its von Neumann expansion: 
    \begin{align}
    \nabla_{\vZ}{R}^c(\vZ \mid \vU_{[K]}) 
        & = \beta \sum_{k=1}^K \vU_k\vU_k^*\vZ\Big({\I +
    \beta(\vU_k^*\vZ)^*(\vU_k^*\vZ)}\Big)^{-1} \\
     & \approx \beta \sum_{k=1}^K \vU_k\vU_k^*\vZ\Big({\I - 
    \beta(\vU_k^*\vZ)^*(\vU_k^*\vZ)}\Big) \\
        &=  \beta\sum_{k=1}^{K}\vU_k\Big({\vU^*_k\vZ - \beta\vU^*_k\vZ [(\vU_{k}\adj\vZ)\adj(\vU_{k}\adj\vZ)]}\Big)
        \label{eqn:approximation-1}
    \end{align}
    Notice that the term \((\vU_{k}\adj\vZ)\adj(\vU_{k}\adj\vZ)\) is the auto-correlation among the projected tokens. As the tokens $\vZ$ may be from different subspaces, we would prefer to use only tokens that belong to the \textit{same} subspace to regress and compress themselves. Hence we may convert the above correlation term into a subspace-membership indicator with a softmax operation, whence \eqref{eqn:approximation-1} becomes
    \begin{eqnarray}
        \nabla_{\vZ}{R}^c(\vZ \mid \vU_{[K]}) 
        &\approx & \beta\sum_{k=1}^{K}\vU_k\Big({\vU^*_k\vZ - \beta\vU^*_k\vZ [(\vU_{k}\adj\vZ)\adj(\vU_k\adj\vZ)]}\Big)\\
        &\approx & \beta\sum_{k=1}^{K}\vU_k\vU_k^*\vZ - \beta^{2}\sum_{k=1}^{K}\vU_k\Big(\vU_k^*\vZ \softmax{(\vU_{k}\adj\vZ)\adj(\vU_{k}\adj\vZ)}\Big)
        \label{eqn:approximation-2}
    \end{eqnarray}
    
    Then, we can rewrite the above approximation to the gradient of $R^c$ as: 
    \begin{align}\label{eq:appendix-ssa-derivation}
        \nabla_{\vZ}{R}^c(\vZ \mid \vU_{[K]}) 
        &\approx \beta\sum_{k=1}^{K}\vU_k\vU_k^*\vZ - \beta^{2}\sum_{k=1}^{K}\vU_k\bp{\vU_k^*\vZ \softmax{(\vU_{k}\adj\vZ)\adj(\vU_k\adj\vZ)}} \\
         &= \beta\sum_{k=1}^{K}\vU_k\vU_k^*\vZ - \beta^{2}\sum_{k=1}^{K}\vU_k\SSA{\vZ\mid\vU_k} \\
         &= \beta\underbrace{\bp{\sum_{k=1}^{K}\vU_k\vU_k^*}\vZ}_{\approx \vZ} -\beta^{2}\mat{\vU_1, \cdots, \vU_K}
         \mat{\SSA{\vZ\mid\vU_1} \\ \vdots \\  \SSA{\vZ\mid\vU_K} } \\
         &\approx\beta \vZ - \beta^{2}\mat{\vU_{1}, \cdots, \vU_{K}} \mat{\SSA{\vZ \mid \vU_{1}} \\ \vdots \\ \SSA{\vZ \mid \vU_{K}}} \\
         &= \beta \vZ - \beta \MSSA{\vZ \mid \vU_{[K]}}
    \end{align}
    Thus the gradient descent step with learning rate \(\kappa > 0\) gives
    {\color{revision}
    \begin{equation}\label{eq:appendix-mssa-block-variant}
        \vZ - \kappa\nabla_{\vZ}R^{c}(\vZ \mid \vU_{[K]}) \approx (1 - \kappa\beta)\vZ + \kappa\beta \MSSA{\vZ \mid \vU_{[K]}}.
    \end{equation}
    }
\end{proof}

%% file: app_technical_mlp.tex
\subsection{Companion to \Cref{sub:sparse}} \label{app:proofs-sparse}

In the main text \Cref{sub:sparse}, our connection between the second step of the alternating minimization and the LASSO optimization was high-level and heuristic. In some sense, the choice to pose the minimization step as a LASSO was a \textit{simple, reliable, and interpretable choice} which works well in practice, but is nonetheless not backed up by rigorous theoretical justification. 

{\color{revision}
In the following subsection, we provide a mathematical justification for a reformulation of the minimization step using a majorization-minimization framework. We further show that the associated unrolled optimization step bears a strong resemblance to the ISTA step.
This confirms our earlier discussion --- we took the \textit{simplest possible choice} in designing \ours{}, but by more rigorous derivation we can uncover \textbf{alternative operators---which differ from the ISTA operator defined in \eqref{eq:ista-block} in Section~\ref{subsec:mlp-block-arch-design}---}that nonetheless have the same conceptual function and may perform better in practice.
}

\paragraph{Assumptions.} In this section, we present a rigorous optimization
analysis of an incremental minimization approach to the objective
\Cref{eqn:sparsification}. We will show that under two simplifying assumptions,
namely
\begin{enumerate}
    \item The columns of $\vZ^{\ell + 1/2}$ are normalized, in the sense that
        $\diag((\vZ^{\ell+1/2})\adj \vZ^{\ell + 1/2}) = \bm{1}$;\footnote{This is
            a natural assumption in transformer-type architectures such as
            \ours{} due to the use of LayerNorm blocks---although these blocks
            (indeed, as we use them in \ours{}) include trainable mean and
            scale offsets as well as an additional mean subtraction operation
            \citep{phuong2022formal}, they are initialized to have zero mean and
            unit norm, hence this assumption corresponds to an analysis of the
        network at its initialization. }
    \item We have $d \geq n$,\footnote{This assumption is without loss
            of generality, as we will see in the analysis below. The reason is that
            $\vZ\adj \vZ$ and $\vZ \vZ\adj$ have the same nonzero eigenvalues
            regardless of the shape of $\vZ$, which implies that $\log\det(\vI + \alpha
            \vZ\adj \vZ) = \log\det(\vI + \alpha \vZ\vZ\adj)$. In particular,
            interpreting the norms appropriately (with a slight abuse of notation), we
            have $\varphi(\vZ) = \varphi(\vZ\adj)$, so for the purposes of analysis
            we can always proceed as though $\vZ$ is a tall matrix (as long as
            we do not use any special properties of  $\alpha$ in our
        derivation).} and the columns of $\vZ^{\ell
        + 1/2}$ are
        orthogonal, so that $(\vZ^{\ell+1/2})\adj \vZ^{\ell + 1/2} =
        \vI$.\footnote{This assumption is strictly stronger than the previous
            one, and strictly stronger than an assumption of incoherence on the
            columns. It corresponds to the representation $\vZ^{\ell+1/2}$ being non-collapsed, which we expect to hold at initialization due to the projections $\vU_{[K]}$ being random. %
        }
\end{enumerate}
the approach leads to an update iteration that is equal to a slightly
simplified version of the ISTA block \Cref{eq:ista-block}. We see this as a
justification for our derivation in \Cref{sub:sparse}, which obtained the ISTA
block by introducing an additional simplifying assumption on the distribution
of the data at layer $\ell$.

\paragraph{Analysis.} Following \Cref{eq:sparse-nonnegative}, we will consider
the natural relaxation of the $\ell^0$ ``norm'' to the $\ell^1$ norm, and
incorporate a nonnegativity constraint.  Consider the objective
\begin{equation}
    \varphi(\vZ) = \lambda \norm{\vZ}_1 + \chi_{\set{\vZ \geq \vZero}}(\vZ)
    - \underbrace{ \frac{1}{2} \log\det\left( \vI + \alpha \vZ\adj \vZ
    \right)}_{R(\vZ)},
\end{equation}
where $\vZ \in \bR^{d \times n}$ and $\alpha = d / (n \veps^2)$, and
$\chi_{\set{\vZ \geq \vZero}}$ denotes the characteristic function for the set
of elementwise-nonnegative matrices $\vZ$. As in
\Cref{app:proofs-compression}, we calculate
\begin{equation}
    \nabla_{\vZ} R(\vZ) = \alpha \vZ \left( \vI + \alpha \vZ\adj \vZ \right)\inv.
\end{equation}
We consider an incremental optimization scheme for the highly nonlinear and
nonconvex objective $\varphi$. Following \Cref{sub:compression}, we optimize
locally at a ``post-compression'' iterate $\vZ^{\ell + 1/2}$. We follow the
standard proximal majorize-minimize framework \citep{Wright-Ma-2022} for
incremental/local optimization: this begins with the second-order Taylor
expansion for the smooth part of $\varphi$ in a neighborhood of the
current iterate $\vZ^{\ell + 1/2}$:
\begin{equation}
    \begin{split}
        R(\vZ)
        =
        R(\vZ^{\ell + 1/2}) 
        &+ \ip*{\nabla_{\vZ} R(\vZ^{\ell + 1/2})}{\vZ - \vZ^{\ell + 1/2}} \\
        &+ \int_0^1 (1 - t) \ip*{\vZ - \vZ^{\ell+1/2}}{\nabla^2 R(\vZ_t)\left(
                \vZ - \vZ^{\ell+1/2}
            \right)
        }\ \mathrm{d}t,
    \end{split}
    \label{eq:R-quadratic-taylor}
\end{equation}
where for any $\vZ \in \bR^{d \times n}$, $\vZ_t = t \vZ^{\ell+1/2} + (1-t)
\vZ$. The proximal majorization-minimization approach alternates two steps to
minimize $\varphi$:
\begin{enumerate}
    \item First, use assumptions on $\vZ^{\ell + 1/2}$ to derive an upper bound
        on the operator norm of the Hessian $\nabla^2 R(\vZ)$ over the
        effective domain of the optimization problem. We will write $L$ for
        this (uniform) upper bound. This yields a quadratic upper bound for the
        smooth part of the objective $\varphi$.
    \item Then, alternately minimize the \textit{smooth part} of the quadratic
        upper bound as a function of $\vZ$, and take a \textit{proximal step}
        on the nonsmooth part. It can be shown \citep{Wright-Ma-2022} that
        corresponds to the iteration
        \begin{equation}
            \vZ^+ = \prox{\frac{\lambda}{L} (\norm{}_1 + \chi_{\set{\vZ \geq \vZero}} ) }\left(
                \vZ + \frac{1}{L} \nabla_{\vZ} R(\vZ)%
            \right)
            \label{eq:mm-iter}
        \end{equation}
        In the alternating minimization setting of this paper for optimizing
        \Cref{eq:sparse-rr}, we only take one such step, starting at
        $\vZ^{\ell+1/2}$.
\end{enumerate}
We will instantiate this program below, showing quantitative error bounds
related to our assumptions above as necessary. Rather than directly applying
the iteration \Cref{eq:mm-iter}, we will derive it below under our
aforementioned assumptions.

Starting at \Cref{eq:R-quadratic-taylor}, our first task is to upper bound the
quadratic residual. 
This corresponds to estimating
\begin{align}
    &\ip*{\vZ - \vZ^{\ell+1/2}}{\nabla^2 R(\vZ_t)\left(
            \vZ - \vZ^{\ell+1/2}
        \right)
    } \\
    &\qquad\leq
    \sup_{t \in [0,1]}
    \norm*{
        \nabla^2 R(\vZ_t)
    }_{\ell^2 \to \ell^2}
    \norm*{
            \vZ - \vZ^{\ell+1/2}
    }_{F}^2
\end{align}
with Cauchy-Schwarz. 
Using
\Cref{lem:logdet-hessian}, we can estimate the operator norm term in the
previous bound in terms of properties of $\vZ^{\ell + 1/2}$. We need to bound 
\begin{equation}
    \alpha \sup_{\norm{\vDelta}_{F} \leq 1}
    \norm*{
        \left(
            \vDelta
            - \alpha \vZ_t(\vI + \alpha \vZ_t\adj \vZ_t)\inv (\vZ_t\adj \vDelta +
            \vDelta\adj \vZ_t)
        \right)(\vI + \alpha \vZ_t\adj \vZ_t)\inv
    }_{F},
\end{equation}
and \Cref{lem:logdet-grad-lipschitz} gives that this term is no larger than
$9\alpha / 4$ for any $\vZ$ and any $t$.
With this
estimate and \Cref{eq:R-quadratic-taylor}, we have a quadratic upper bound for
$-R(\vZ)$:
\begin{equation}
    -R(\vZ)
    \leq
    -R(\vZ^{\ell + 1/2}) 
    + \ip*{-\nabla_{\vZ} R(\vZ^{\ell + 1/2})}{\vZ - \vZ^{\ell + 1/2}} \\
    + \frac{9\alpha}{8} \norm*{
        \vZ - \vZ^{\ell+1/2}
    }_{F}^2.
\end{equation}
Meanwhile, by our assumptions above, we have
\begin{equation}
    -\nabla_{\vZ} R(\vZ^{\ell + 1/2})
    = -\alpha \vZ^{\ell+1/2} \left( \vI + \alpha \vI \right)\inv
    = -\frac{\alpha}{1 + \alpha} \vZ^{\ell + 1/2}.
\end{equation}
We now minimize the preceding quadratic upper bound as a
function of $\vZ$. Differentiating, the minimizer $\vZ_{\mathrm{opt}}$ is
calculated as
\begin{equation}
    \vZ_{\mathrm{opt}}
    = \left(1 + \frac{4}{9(1 + \alpha)}\right) \vZ^{\ell+1/2},
\end{equation}
and it is well-known that the proximal operator of the sum of
$\chi_{\set{\vZ\geq \vZero}}$ and $\lambda \norm{}_{1}$ is simply the one-sided
soft-thresholding operator \citep{Wright-Ma-2022}
\begin{equation}
    \prox{\chi_{\set{\vZ\geq\vZero}} + \lambda \norm{}_1}\left( \vZ \right)
    =
    \max \set{
        \vZ - \lambda \One, \vZero
    },
\end{equation}
where the maximum is applied elementwise. As in \Cref{sub:sparse}, we may write
this elementwise maximum simply as $\operatorname{ReLU}$.  Thus, one step of
proximal majorization-minimization under our simplifying assumptions takes the form
\begin{equation}
    \vZ^{\ell+1}
    =
    \operatorname{ReLU}\left(
        \left(1 + \frac{4}{9(1 + \alpha)}\right) \vZ^{\ell+1/2}
        - \frac{4\lambda}{9\alpha}\One
    \right),
\end{equation}
Finally, we point out one additional elaboration which introduces the
dictionary $\vD$ that appears in the ISTA block in \Cref{sub:sparse}. 
Notice that for any orthogonal $\vD$, one has $R(\vD \vZ) = R(\vZ)$ for every
$\vZ$. This symmetry implies equivariance properties of $\nabla_{\vZ} R(\vZ)$
and $\nabla^2_{\vZ} R(\vZ)$: for every $\vZ$ and every $\vDelta$ and every
orthogonal $\vD$,
\begin{align}
    \vD \nabla_{\vZ} R(\vZ) &= \nabla_{\vZ} R(\vD \vZ), \\
    \ip{\vD \vDelta}{\nabla^2_{\vZ} R(\vZ) \left(\vD \vDelta\right)}
    &= 
    \ip{\vDelta}{\nabla^2_{\vZ} R(\vD \vZ) \left(\vDelta\right)}.
\end{align}
Hence the quadratic Taylor expansion \Cref{eq:R-quadratic-taylor} can be
written equivalently as
\begin{equation}
    \begin{split}
        R(\vZ)
        =
        R(\vD\adj \vZ^{\ell + 1/2}) 
        &+ \ip*{\nabla_{\vZ} R(\vD\adj \vZ^{\ell + 1/2})}{\vZ - 
        {\color{revision} \vD\adj \vZ^{\ell + 1/2}}} \\
        &+ \int_0^1 (1 - t) \ip*{\vZ - {\color{revision} \vD\adj\vZ^{\ell+1/2}}}{\nabla^2 R(\vD\adj\vZ_t)\left(
                \vZ - {\color{revision} \vD\adj \vZ^{\ell+1/2}}
            \right)
        }\ \mathrm{d}t,
    \end{split}
\end{equation}
for any orthogonal $\vD$. The significance of this is that we have obtained an
expression equivalent to \Cref{eq:R-quadratic-taylor}, but with
$\vZ^{\ell+1/2}$ replaced by $\vD\adj \vZ^{\ell + 1/2}$; moreover, because our
approximation arguments above are not affected by left-multiplication of
$\vZ^{\ell + 1/2}$ by an orthogonal matrix (this operation does not change the
norms of the columns of $\vZ^{\ell + 1/2}$, or their correlations, and hence the
matrix's incoherence), we can apply exactly the same line of reasoning above to
obtain that an equivalent proximal majorization-minimization iteration is given
by 
\begin{equation}\label{eq:prox_maj_min_iteration}
    \vZ^{\ell+1}
    =
    \operatorname{ReLU}\left(
        \left(1 + \frac{4}{9(1 + \alpha)}\right) \vD\adj \vZ^{\ell+1/2}
        - \frac{4\lambda}{9\alpha} \One
    \right),
\end{equation}
for any orthogonal dictionary $\vD$. This gives an update quite similar to the
ISTA block \Cref{eq:ista-block} in the case where the dictionary used in
\Cref{sub:sparse} is orthogonal, but without a skip connection. 
{\color{revision} The operator defined in \eqref{eq:prox_maj_min_iteration} is not exactly the same as the ISTA operator defined in \eqref{eq:ista-block} but exhibits substantial similarities.}

We thus obtain a natural
white-box version of this part of the architecture, along with the natural
interpretation \textit{that its purpose is to sparsify the compressed tokens
$\vZ^{\ell+1/2}$ {\color{revision} with respect to} a (learnable) dictionary}, which accords with recent
empirical studies \citep{Li2023-ig}.

\paragraph{Other architectures?} As we mentioned at the start of this section,
the preceding derivation is performed in the most elementary possible setting
in order to demonstrate the majorization-minimization approach for layer
design. More precise approximations or assumptions may lead to superior layer
designs that better optimize the target objective \Cref{eq:sparse-rr} (and in
particular \Cref{eqn:sparsification}). We mention two here:
\begin{enumerate}
    \item \textbf{Beyond exactly-incoherent features}: our derivations above assumed that the
        incoming representations $\vZ^{\ell+1/2}$ were already maximal for the
        expansion term $R$ in \Cref{eqn:sparsification}. It is desirable to obtain
        a `perturbative' derivation, which applies in cases where
        $\vZ^{\ell+1/2}$ is not fully orthogonal, but instead near-orthogonal,
        in particular \textit{incoherent} \citep{Wright-Ma-2022}. The
        derivations above can be adapted to this setting; the perturbation
        bounds become slightly more delicate, and the ultimate layer
        \Cref{eq:prox_maj_min_iteration} changes to involve additional
        normalization.
    \item \textbf{Beyond orthogonal dictionaries}: The symmetries of the
        expansion term $R$ in \Cref{eqn:sparsification} may be followed to lead to
        a pair of dictionaries $\vD$ and $\vD'$ and an objective that
        sparsifies $\vD \vZ \vD'$. This type of transformation is suggestive of
        popular architectures that mix over tokens
        \citep{Tolstikhin2021-yh,Trockman2022-po}, however we consider the
        simpler form $\vD \vZ$ in this work. In addition, we have
        focused for simplicity on orthogonal dictionaries $\vD$; as in the
        previous bullet, one may consider in a similar way dictionaries $\vD$
        which are complete and near-orthogonal. Adapting the derivation to
        \textit{overcomplete dictionaries} is an interesting future direction
        that we expect to improve the scalability of \ours{}; one avenue to
        achieve this could be increasing the number of projections $\vU_{[K]}$
        and their embedding dimensions.
\end{enumerate}

\subsubsection{Auxiliary Lemmas}

\begin{lemma}
    \label{lem:logdet-hessian}
    Consider the function
    \begin{equation}
        R(\vZ) = \frac{1}{2} \log \det \left( \vI + \alpha \vZ\adj \vZ \right),
    \end{equation}
    where $\alpha > 0$ is a constant. Then we have
    \begin{equation}
        \nabla_{\vZ} R(\vZ) = \alpha \vZ \left( \vI + \alpha \vZ\adj \vZ\right)
        \inv,
    \end{equation}
    and the Hessian operator $\nabla_{\vZ}^{2}R(\vZ) \colon \mathbb{R}^{d \times n} \to \mathbb{R}^{d \times n}$ satisfies that for any $\vDelta \in \mathbb{R}^{d \times n}$, 
    \begin{align}
        &\nabla^2_{\vZ} R(\vZ)\left( \vDelta \right) \\
        &=
        \alpha\vDelta \left( \vI + \alpha \vZ\adj \vZ \right)\inv
        - \alpha^2 \vZ 
        \left( \vI + \alpha \vZ\adj \vZ \right)\inv
        \left( \vZ\adj \vDelta + \vDelta\adj \vZ \right)
        \left( \vI + \alpha \vZ\adj \vZ \right)\inv.
    \end{align}
    \begin{proof}
        The gradient calculation follows from \citep{OriginalMCR2}, for example.
        For the Hessian, we use the usual approach to calculating derivatives:
        if $\vDelta$ is any matrix with the same shape as $\vZ$ and $t > 0$,
        \begin{equation}
            \nabla_{\vZ}^2 R(\vZ)\left(
                \vDelta
            \right) = 
            \dac\left[t\mapsto \nabla_{\vZ} R(\vZ + t
            \vDelta)\right],
        \end{equation}
        valid since $R$ is smooth. We have
        {\color{revision}
        \small
        \begin{align}
            &\nabla_{\vZ} R(\vZ + t \vDelta)\\
            =&
            \alpha(\vZ + t \vDelta) \left( \vI + \alpha (\vZ +
            t\vDelta)\adj(\vZ + t\vDelta) \right)\inv \\
            =&
            \alpha(\vZ + t \vDelta) \left( \vI + \alpha \vZ\adj \vZ
                + \alpha t\left[
                    \vZ\adj \vDelta + \vDelta\adj \vZ + t \vDelta\adj \vDelta
                \right]
            \right)\inv \\
            =&
            \alpha(\vZ + t \vDelta) \left( 
                \vI + 
                \alpha t
                \left(\vI + \alpha \vZ\adj \vZ\right)\inv
                \left[
                    \vZ\adj \vDelta + \vDelta\adj \vZ + t \vDelta\adj \vDelta
                \right]
            \right)\inv
            \left( \vI + \alpha \vZ\adj \vZ \right)\inv
            \\
            =&
            \alpha(\vZ + t \vDelta) \left( 
                \sum_{k = 0}^\infty
                (- \alpha t)^k
                \left(
                    \left(\vI + \alpha \vZ\adj \vZ\right)\inv
                    \left[
                        \vZ\adj \vDelta + \vDelta\adj \vZ + t \vDelta\adj \vDelta
                    \right]
                \right)^k
            \right)
            \left( \vI + \alpha \vZ\adj \vZ \right)\inv,
        \end{align}
        }

        \noindent
        where in the fourth line we require that $t$ is sufficiently close to
        $0$ in order to invoke the Neumann series. First, notice that the term
        involving $\vDelta\adj \vDelta$ does not play a role in the final
        expression: after we differentiate with respect to $t$ and take a limit
        $t \to 0$, terms arising due to differentiation of $t \mapsto t
        \vDelta\adj \vDelta$ go to zero, because whenever the summation index
        $k > 0$ we have a term $(-\alpha t)^k$ that goes to zero as $t \to 0$.
        We thus obtain with the product rule
        \begin{align}
            &\dac\left[t\mapsto \nabla_{\vZ} R(\vZ + t
            \vDelta)\right]\\
            =\,\, 
            &\alpha\vDelta \left( \vI + \alpha \vZ\adj \vZ \right)\inv 
             - \alpha^2 \vZ 
            \left( \vI + \alpha \vZ\adj \vZ \right)\inv
            \left( \vZ\adj \vDelta + \vDelta\adj \vZ \right)
            \left( \vI + \alpha \vZ\adj \vZ \right)\inv.
        \end{align}
    \end{proof}
\end{lemma}

\begin{lemma}
    \label{lem:logdet-grad-lipschitz}
    One has
    \begin{equation}
        \sup_{\norm{\vDelta}_{F} \leq 1}
        \norm*{
            \left(
                \vDelta
                - \alpha \vZ_t(\vI + \alpha \vZ_t\adj \vZ_t)\inv (\vZ_t\adj \vDelta +
                \vDelta\adj \vZ_t)
            \right)(\vI + \alpha \vZ_t\adj \vZ_t)\inv
        }_{F}
        \leq
        \frac{9}{4}.
    \end{equation}
    \begin{proof}
        Fix $\vDelta$ satisfying $\norm{\vDelta}_{F}\leq 1$. By the
        triangle inequality,
        \begin{align}
            &\norm*{
                \left(
                    \vDelta
                    - \alpha \vZ_t(\vI + \alpha \vZ_t\adj \vZ_t)\inv (\vZ_t\adj \vDelta +
                    \vDelta\adj \vZ_t)
                \right)(\vI + \alpha \vZ_t\adj \vZ_t)\inv
            }_{F}
            \\
            &\leq
            \norm*{
                \vDelta (\vI + \alpha \vZ_t\adj \vZ_t)\inv
            }_{F}
            +
            \alpha \norm*{
                \vZ_t(\vI + \alpha \vZ_t\adj \vZ_t)\inv (\vZ_t\adj \vDelta +
                \vDelta\adj \vZ_t)
                (\vI + \alpha \vZ_t\adj \vZ_t)\inv
            }_{F}.
        \end{align}
        For the first term, we note that
        \begin{equation}
            \norm*{
                \vDelta (\vI + \alpha \vZ_t\adj \vZ_t)\inv
            }_{F}
            =
            \norm*{
                \left(
                    (\vI + \alpha \vZ_t\adj \vZ_t)\inv
                    \kron \vI
                \right) \vec(\vDelta)
            }_{F},
        \end{equation}
        and since $(\vI + \alpha \vZ_t\adj \vZ_t)\inv \preceq \vI$,
        we obtain from Cauchy-Schwarz\footnote{Recall that the eigenvalues of a Kronecker product
            of symmetric matrices are the tensor product of the eigenvalues
        (with multiplicity).}
        \begin{equation}
            \norm*{
                \vDelta (\vI + \alpha \vZ_t\adj \vZ_t)\inv
            }_{F}
            \leq \norm{\vDelta}_{F}.
        \end{equation}
        We can use a similar idea to control the second term. We have from the
        triangle inequality
        \begin{align}
            &\norm*{
                \vZ_t(\vI + \alpha \vZ_t\adj \vZ_t)\inv (\vZ_t\adj \vDelta +
                \vDelta\adj \vZ_t)
                (\vI + \alpha \vZ_t\adj \vZ_t)\inv
            }_{F} \\
            &\quad\leq
            \norm*{
                \vZ_t(\vI + \alpha \vZ_t\adj \vZ_t)\inv 
                \vZ_t\adj \vDelta
                (\vI + \alpha \vZ_t\adj \vZ_t)\inv
            }_{F} \\
            &\qquad+
            \norm*{
                (\vI + \alpha \vZ_t\adj \vZ_t)\inv \vZ_t\adj
                \vDelta 
                (\vI + \alpha \vZ_t\adj \vZ_t)\inv
                \vZ_t\adj
            }_{F}.
        \end{align}
        For the first term, we have
        \begin{align}
            &\norm*{
                \vZ_t(\vI + \alpha \vZ_t\adj \vZ_t)\inv 
                \vZ_t\adj \vDelta
                (\vI + \alpha \vZ_t\adj \vZ_t)\inv
            }_{F} \\
            &\quad=
            \norm*{
                \left(
                    (\vI + \alpha \vZ_t\adj \vZ_t)\inv
                    \kron
                    \vZ_t(\vI + \alpha \vZ_t\adj \vZ_t)\inv \vZ_t\adj
                \right)
                \vec(\vDelta)
            }_{F} \\
            &\quad\leq
            \sigma_{\max}\left(
                (\vI + \alpha \vZ_t\adj \vZ_t)\inv
            \right)
            \sigma_{\max}\left(
                \vZ_t(\vI + \alpha \vZ_t\adj \vZ_t)\inv \vZ_t\adj
            \right)
            \norm{\vDelta}_{F} \\
            &\quad\leq \frac{1}{\alpha} \norm{\vDelta}_{F}.
        \end{align}
        The last estimate follows from a computation using the SVD of $\vZ_t$.
        Meanwhile, we have for the second term by a similar argument (using the
        fact that the singular values of $\vA$ and $\vA\adj$ are identical for
        any matrix $\vA$)
        \begin{align}
            \norm*{
                (\vI + \alpha \vZ_t\adj \vZ_t)\inv \vZ_t\adj
                \vDelta 
                (\vI + \alpha \vZ_t\adj \vZ_t)\inv
                \vZ_t\adj
            }_{F}
            &\leq
            \sigma_{\max}\left( (\vI + \alpha \vZ_t\adj \vZ_t)\inv \vZ_t\adj
            \right)^2
            \norm{\vDelta}_{F} \\
            &\leq
            \frac{1}{4\alpha} \norm{\vDelta}_{F},
        \end{align}
        where once again the estimate follows from a computation involving the
        SVD of $\vZ_t$ (together with the fact that the function $\sigma
        \mapsto \sigma / (1 + \alpha \sigma^2)$ is bounded on $\sigma \geq 0$
        by $1/(2 \sqrt{\alpha})$). Putting it together, we have obtained
        \begin{equation}
            \norm*{
                \left(
                    \vDelta
                    - \alpha \vZ_t(\vI + \alpha \vZ_t\adj \vZ_t)\inv (\vZ_t\adj \vDelta +
                    \vDelta\adj \vZ_t)
                \right)(\vI + \alpha \vZ_t\adj \vZ_t)\inv
            }_{F}
            \leq \frac{9}{4} \norm{\vDelta}_{F},
        \end{equation}
        which gives the claim after taking suprema.

    \end{proof}
\end{lemma}

%% file: app_technical_diffusion.tex
\subsection{An Overview of Diffusion Processes}\label{app:diffusion}

In this section, we give an overview of the basics of time-reversible It\^{o} diffusion processes, the mathematical foundation for diffusion models. This is to make this paper more self-contained by providing knowledge about general diffusion processes that we will apply to our special models. The coverage adapts that of \cite{millet1989integration,Song2020-xo,karras2022elucidating}. 

Consider a generic It\^{o} diffusion process \((\vz(t))_{t \in [0, T]}\), where \(\vz(t)\) is an \(\R^{m}\)-valued random variable, given by the SDE
\begin{equation}\label{eq:forward_general_diffusion}
    \odif{\vz(t)} = b(\vz(t), t)\odif{t} + \Sigma(\vz(t), t)\odif{\vw(t)}, \qquad \vz(0) \sim P, \qquad \forall t \in [0, T]
\end{equation}
where \(\vw\) is a Brownian motion and \(P\) is some probability measure on \(\R^{m}\) (in this case representing the data distribution). Here the \textit{drift coefficient} \(b \colon \R^{m} \times \R \to \R^{m}\) and \textit{diffusion coefficient} \(\Sigma \colon \R^{m} \times \R \to \R^{m \times m}\) are functions. To make sense of \Cref{eq:forward_general_diffusion} and also verify the existence of strong (i.e., pathwise well-defined) solutions, we need some regularity on them, and we choose the following assumption:
\begin{enumerate}[label={A\arabic*.}, leftmargin=*]
    \item \(b\) and \(\Sigma\) have some spatial smoothness and do not grow too fast, i.e., there is a constant \(K \geq 0\) such that for all \(\vx, \wt{\vz} \in \R^{m}\) we have
    \begin{align}
        &\sup_{t \in [0, T]}\bs{\norm{\Sigma(\vx, t) - \Sigma(\wt{\vz}, t)}_{F} + \norm{b(\vx, t) - b(\wt{\vz}, t)}_{2}} \leq K\norm{\vx - \wt{\vz}}_{2} \\
        &\sup_{t \in [0, T]}\bs{\norm{\Sigma(\vx, t)}_{F} + \norm{b(\vx, t)}_{2}} \leq K(1 + \norm{\vx}_{2}).
    \end{align}
\end{enumerate}
In general, \(\vz(t)\) may not have a density w.r.t.~the Lebesgue measure on \(\R^{m}\). For example, suppose that \(P\) is supported on some low-dimensional linear subspace (or even a Dirac delta measure), and take \(\Sigma\) to be the orthoprojector onto this subspace. Then \(\vz(t)\) will be supported on this subspace for all \(t\) and thus not have a density w.r.t.~the Lebesgue measure. Thus, when further discussing processes of the type \Cref{eq:forward_general_diffusion}, we make the following assumption
\begin{enumerate}[resume, label={A\arabic*.}, leftmargin=*]
    \item \(\vz(t)\) has a probability density function \(p(\cdot, t)\) for all \(t > 0\). 
    
    This is guaranteed by either of the following conditions \citep{millet1989integration}:
    \begin{enumerate}[label={A2.\arabic*}]
        \item \(b\) and \(\Sigma\) are differentiable in \((\vx, t)\) and have H\"{o}lder-continuous derivatives, and \(P\) has a density w.r.t.~the Lebesgue measure;
        \item The event 
        \begin{equation}
            \{\text{\(\rank(\Sigma(\vz(s), s)) = m\) for all \(s\) in some neighborhood of \(0\)}\}
        \end{equation}
        happens \(P\)-almost surely.
    \end{enumerate}
\end{enumerate}

Define \(\Psi \colon \R^{m} \times \R \to \R^{m \times m}\) by
\begin{equation}
    \Psi(\vx, t) \doteq \Sigma(\vx, t)\Sigma(\vx, t)\adj.
\end{equation}
To discuss time-reversibility, we also need the following local integrability condition, which is another measure of sharp growth of the coefficients (or precisely their derivatives):
\begin{enumerate}[resume, label={A\arabic*.}, leftmargin=*]
    \item The functions \((\vx, t) \mapsto \nabla_{\vx} \cdot (\Psi(\vx, t)p(\vx, t))\) are integrable on sets of the form \(D \times [t_{0}, 1]\) for \(t_{0} > 0\) and \(D\) a bounded measurable subset of \(\R^{m}\):
    \begin{align}
        \int_{t_{0}}^{1}\int_{D}\norm{\nabla_{\vx} \cdot (\Psi(\vx, t)p(\vx, t))}_{2}\odif{\vx}\odif{t} < \infty.
    \end{align}
\end{enumerate}
To write the notation out more explicitly, 
\begin{align}
    &\nabla_{\vx} \cdot (\Psi(\vx, t)p(\vx, t)) = \mat{\nabla_{\vx} \cdot (\Psi^{1}(\vx, t)p(\vx, t)) \\ \vdots \\ \nabla_{\vx} \cdot (\Psi^{m}(\vx, t)p(\vx, t))}  \\
    \text{where} \qquad 
    &\nabla_{\vx} \cdot (\Psi^{i}(\vx, t)p(\vx, t)) = \sum_{j = 1}^{m}\pdv*{[\Psi^{ij}(\vx, t)p(\vx, t)]}{x_{j}}
\end{align}
where \(\Psi^{i}\) is the \(i\th\) row of \(\Psi\) transposed to a column, and \(\Psi^{ij}\) is the \((i, j)\th\) entry of \(\Psi\). Note that \cite{millet1989integration} phrases this in terms of an local integrability condition on each \(\abs{\nabla_{\vx} \cdot (\Psi^{i}(\vx, t)p(\vx, t))}\), which would naturally give a local integrability condition on \(\norm{\nabla_{\vx}\cdot (\Psi(\vx, t)p(\vx, t))}_{\infty}\). However, all norms on \(\R^{m}\) are equivalent, and so this leads to a local integrability condition for \(\norm{\nabla_{x} \cdot (\Psi(\vx, t)p(\vx, t))}_{2}\) as produced. Note that the assumptions do not guarantee that the involved derivatives exist, in which case they are taken in the distributional (e.g., weak) sense, whence they should exist \citep{millet1989integration}.

Under assumptions A1---A3, \cite{millet1989integration} guarantees the existence of another process \((\wt{\vz}(t))_{t \in [0, T]}\) such that the laws of \(\vz(t)\) and \(\wt{\vz}(T - t)\) are the same for all \(t \in [0, T]\). This process \((\wt{\vz}(t))_{t \in [0, T]}\) is called the \textit{time reversal} of \((\vz(t))_{t \in [0, T]}\), and is shown to have law given by 
\begin{equation}\label{eq:backward_general_diffusion}
    \odif{\wt{\vz}(t)} = b^{\leftarrow}(\wt{\vz}(t), t)\odif{t} + \Sigma^{\leftarrow}(\wt{\vz}(t), t)\odif{\vw^{\leftarrow}(t)}, \qquad \wt{\vz}(0) \sim p(\cdot, T), \qquad \forall t \in [0, T]
\end{equation}
where \(\vw^{\leftarrow}(t)\) is an independent Brownian motion and
\begin{align}
    b^{\leftarrow}(\vx, t)
    &= -b(\vx, T - t) + \frac{\nabla_{\vx}\cdot[\Psi(\vx, T - t)p(\vx, T - t)]}{p(\vx, T - t)} \\
    &= -b(\vx, T - t) + \nabla_{\vx} \cdot \Psi(\vx, T - t) + \Psi(\vx, T - t)[\nabla_{\vx}\log p(\vx, T - t)], \\
    \Sigma^{\leftarrow}(\vx, t)
    &= \Sigma(\vx, T - t).
\end{align}

We would next like to develop an ODE which transports the probability mass \(P\) in the same way as \Cref{eq:forward_general_diffusion} --- namely, find another process \((\ol{\vz}(t))_{t \in [0, T]}\) which has deterministic dynamics, yet has the same law as \((\vz(t))_{t \in [0, T]}\). \cite{Song2020-xo} looks at the Fokker-Planck equations (which can be defined, at least in a weak sense, under assumptions A1--A2) and manipulates them to get the following dynamics for \(\ol{\vz}(t)\):
\begin{align}\label{eq:general_diffusion_ode}
    \odif{\ol{\vz}(t)} 
    &= \ol{b}(\ol{\vz}(t), t)\odif{t}, \qquad \ol{\vz}(0) \sim P, \qquad \forall t \in [0, T], \\
    \text{where} \qquad \ol{b}(\vx, t)
    &= b(\vx, t) - \frac{1}{2}\cdot \frac{\nabla_{\vx} \cdot [\Psi(\vx, t)p(\vx, t)]}{p(\vx, t)} \\
    &= b(\vx, t) - \frac{1}{2}\nabla_{\vx}\cdot \Psi(\vx, t) - \frac{1}{2}\Psi(\vx, t)[\nabla_{x}\log p(\vx, t)].
\end{align}
Now to get a similar process for \(\wt{\vz}(t)\), namely a process \((\ol{\wt{\vz}}(t))_{t \in [0, T]}\) which evolves deterministically yet has the same law as \((\wt{\vz}(t))_{t \in [0, T]}\), we may either take the time reversal of \Cref{eq:general_diffusion_ode} or apply the Fokker-Planck method to \Cref{eq:backward_general_diffusion}, in both cases obtaining the same dynamics:
\begin{align}\label{eq:backward_general_diffusion_ode}
    \small
    \odif{\ol{\wt{\vz}}(t)}
    &= \ol{b}^{\leftarrow}(\ol{\wt{\vz}}(t), t)\odif{t}, \qquad \ol{\wt{\vz}}(0) \sim p(\cdot, T), \qquad \forall t \in [0, T], 
\end{align}
where
\begin{align}
    \ol{b}^{\leftarrow}(\vx, t)
    &= -\ol{b}(\vx, T - t) \\
    &= -b(\vx, T - t) + \frac{1}{2}\cdot \frac{\nabla_{\vx} \cdot [\Psi(\vx, T - t)p(\vx, T - t)]}{p(\vx, T - t)} \\
    &= -b(\vx, t) + \frac{1}{2}\nabla_{\vx}\cdot \Psi(\vx, T - t) + \frac{1}{2}\Psi(\vx, T - t)[\nabla_{\vx}\log p(\vx, T - t)].
\end{align}
The quantity \(\nabla_{\vx}\log p(\vx, t)\) is of central importance; it is denoted the \textit{score at time \(t\)}, and we use the notation \(s(\vx, t) \doteq \nabla_{x}\log p(\vx, t)\) for it. With this substitution, we have the following dynamics for our four processes:
\begin{align}
    \odif{\vz(t)} 
    &= b(\vz(t), t)\odif{t} + \Sigma(\vz(t), t)\odif{\vw(t)}, \quad \vz(0) \sim P \\
    \odif{\wt{\vz}(t)} 
    &= [-b(\wt{\vz}(t), T - t) + \nabla_{\vx} \cdot \Psi(\wt{\vz}(t), T - t) + \Psi(\wt{\vz}(t), T - t)s(\wt{\vz}(t), T - t)]\odif{t} \\
    &\qquad + \Sigma(\wt{\vz}(t), T - t)\odif{\vw^{\leftarrow}(t)}, \quad \wt{\vz}(0) \sim p(\cdot, T) \\
    \odif{\ol{\vz}(t)}
    &= \bs{b(\ol{\vz}(t), t) - \frac{1}{2}\nabla_{\vx}\cdot \Psi(\ol{\vz}(t), t) - \frac{1}{2}\Psi(\ol{\vz}(t), t)s(\ol{\vz}(t), t)}\odif{t}, \quad \ol{\vz}(0) \sim P \\
    \odif{\ol{\wt{\vz}}(t)}
    &= \bigg[-b(\ol{\wt{\vz}}(t), T - t) + \frac{1}{2}\nabla_{\vx}\cdot \Psi(\ol{\wt{\vz}}(t), T - t) \\
    &\qquad + \frac{1}{2}\Psi(\ol{\wt{\vz}}(t), T - t)s(\ol{\wt{\vz}}(t), T - t)\bigg]\odif{t}, \quad \ol{\wt{\vz}}(0) \sim p(\cdot, T).
\end{align}

In practice, one fits an estimator for \(s(\cdot, \cdot)\) and estimates \(p(\cdot, T)\) and runs a discretization of either \Cref{eq:backward_general_diffusion} or \Cref{eq:backward_general_diffusion_ode} to sample approximately from \(P\). One common instantiation used in diffusion models \citep{karras2022elucidating} is the so-called \textit{variance-exploding} diffusion process, which has the coefficient settings
\begin{equation}
    b(\vx, t) = 0, \qquad \Sigma(\vx, t) = \sqrt{2}\vI
\end{equation}
which implies that
\begin{equation}
    \Psi(\vx, t) = 2\vI.
\end{equation}
This means that the four specified processes are of the form
\begin{align}
    \odif{\vz(t)} 
    &= \sqrt{2}\odif{\vw(t)}, \quad \vz(0) \sim P \\
    \odif{\wt{\vz}(t)} 
    &= 2s(\wt{\vz}(t), T - t)\odif{t} + \sqrt{2}\odif{\vw^{\leftarrow}(t)}, \quad \wt{\vz}(0) \sim p(\cdot, T) \\
    \odif{\ol{\vz}(t)}
    &= -s(\ol{\vz}(t), t)\odif{t}, \quad \ol{\vz}(0) \sim P \\
    \odif{\ol{\wt{\vz}}(t)}
    &= s(\ol{\wt{\vz}}(t), T - t), \quad \ol{\wt{\vz}}(0) \sim p(\cdot, T).
\end{align}
Notice that the determinstic flows are actually gradient flows on the score, which concretely reveals a connection between sampling and optimization, and thus between diffusion models (precisely those which use the probability flow ODE to sample) and unrolled optimization networks.

In this context, we can also establish the connection between diffusion networks and iterative denoising. In the variance-exploding setting, we have
\begin{equation}
    \vz(t) \sim \mathcal{N}(\vz(0), 2t\vI),
\end{equation}
which can be easily computed using results from, e.g., \cite{sarkka2019applied}. Thus \(\vz(t)\) is a noisy version of \(\vz(0)\), with noise level increasing monotonically with \(t\), and sampling \(\vz(0)\) from \(\vz(t)\) conceptually removes this noise. Concretely, \textit{Tweedie's formula} \citep{efron2011tweedie} says that the optimal denoising function \(\E{\vz(0) \mid \vz(t)}\) has a simple form in terms of the score function:
\begin{equation}
    \E{\vz(0) \mid \vz(t)} = \vz(t) + 2t\cdot s(\vz(t), t).
\end{equation}
In other words, the score function \(s\) points from the current iterate \(\vz(t)\) to the value of the optimal denoising function, so it is a negative multiple of the conditionally-expected noise. Following the score by (stochastic) gradient flow or its discretization is thus equivalent to iterative denoising.

%% file: app_technical_tweedie.tex
\subsection{Companion to \Cref{sub:denoising}} \label{app:proofs-denoising}

Here we produce the approximation result alluded to in \Cref{sub:denoising}.
To simplify the proofs, we use the following notation correspondences: \(\x \mapsto \z^{\ell}\), \(\z \mapsto \z_{\natural}^{\ell}\), and \(\sigma \mapsto \sigma^{\ell}\).

\begin{approximation}\label{thm:opt_denoiser_multi_subspaces}
    Suppose that \(\vz\) follows the low-dimensional Gaussian mixture model \Cref{model:gaussian_tokens} with respect to subspaces \(\vU_{[K]}\), so that \(\vx = \vz + \sigma \vw\), where \(\vw\) is a standard Gaussian vector independent of \(\vz\), follows \Cref{model:gaussian_tokens_noise}. That is, for some random index \(s \sim \pi\) where \(\pi\) is a distribution on \([K]\) and \(s\) is chosen independently of all other randomness, we have
    \begin{equation}
        \vz = \vU_{s}\valpha,
    \end{equation}
    where \(\valpha\) is a zero-mean Gaussian vector with diagonal covariance \(\vLambda\). For each \(k \in [K]\), define 
    \begin{equation}
        \vSigma_{i} \doteq \operatorname{Cov}[\vz \mid s = k] = \vU_{k}\vLambda\vU_{k}\adj.
    \end{equation}
    Assume (for the sake of normalization) that 
    \begin{equation}
        \frac{\pi_{i}}{\sqrt{\det(\vSigma_{i} + \sigma^{2}\vI)}} = \frac{\pi_{j}}{\sqrt{\det(\vSigma_{j} + \sigma^{2}\vI)}}, \qquad \text{for all}\ 1 \leq i \leq j \leq K.
    \end{equation}
    Then we have
    \begin{equation}
        \mathbb{E}[\z \mid \x] \approx \mat{\vU_{1}, \dots, \vU_{K}}\left[\diaglr{\softmax{\frac{1}{2\sigma^{2}}\mat{\norm{\vU_{1}\adj\x}_{2}^{2} \\ \vdots \\ \norm{\vU_{K}\adj\x}_{2}^{2}}}} \otimes \I_{p}\right]\mat{\vU_{1}\adj\x \\ \vdots \\ \vU_{K}\adj\x},
    \end{equation}
    where \(\otimes\) denotes the Kronecker product, i.e., the block matrix defined by
    \begin{equation}
        \bm{A} \otimes \bm{B} = \mat{A_{11}\bm{B} & \cdots & A_{1n}\bm{B} \\ \vdots & \ddots & \vdots \\ A_{m1}\bm{B} & \cdots & A_{mn}\bm{B}}.
    \end{equation}
\end{approximation}

\begin{proof}
    Define \(\vM_{k} \doteq (\vSigma_{k} + \sigma^{2}\vI)^{-1/2}\). From \Cref{thm:score_multi_subspaces}, we have
    \begin{align}
        \nabla_{\x}\log q(\x)
        &= -\sum_{k = 1}^{K}\e_{k}\adj\softmax{-\frac{1}{2}\mat{\norm{\vM_{1}\adj\x}_{2}^{2} \\ \vdots \\ \norm{\vM_{K}\adj\x}_{2}^{2}}}\vM_{k}\vM_{k}\adj\x \\
        &= -\sum_{k = 1}^{K}\e_{k}\adj\softmax{-\frac{1}{2\sigma^{2}}\mat{\norm{\sigma \vM_{1}\adj\x}_{2}^{2} \\ \vdots \\ \norm{\sigma \vM_{K}\adj\x}_{2}^{2}}}\vM_{k}\vM_{k}\adj\x \\
        &= -\sum_{k = 1}^{K}\e_{k}\adj\softmax{\frac{1}{2\sigma^{2}}\mat{\norm{\x}_{2}^{2} - \norm{\sigma \vM_{1}\adj\x}_{2}^{2} \\ \vdots \\ \norm{\x}_{2}^{2} - \norm{\sigma \vM_{K}\adj\x}_{2}^{2}}}\vM_{k}\vM_{k}\adj\x.
    \end{align}
    Now define \(\P_{k} \doteq \I_{d} - \sigma \M_{k}\), and let \(\vU_{k}^{\perp} \in \bR^{d \times (d - p)}\) be an orthogonal complement of \(\vU_{k}\). Then we have
    \begin{align}
        \P_{k}
        &= \I_{d} - \sigma \M_{k} \\
        &= \I_{d} - \sigma \left(\vSigma_{k} + \sigma^{2}\I_{d}\right)^{-1/2} \\
        &= \I_{d} - \sigma \left(\mat{\vU_{k} & \vU_{k}^{\perp}}\mat{\vLambda & \vZero \\ \vZero & \vZero}\mat{\vU_{k}\adj \\ (\vU_{k}^{\perp})\adj} + \sigma^{2}\I_{d}\right)^{-1/2} \\ 
        &= \I_{d} - \sigma \left(\mat{\vU_{k} & \vU_{k}^{\perp}}\mat{\vLambda + \sigma^{2}\I_{p} & \vZero \\ \vZero & \sigma^{2}\I_{d - p}}\mat{\vU_{k}\adj \\ (\vU_{k}^{\perp})\adj}\right)^{-1/2} \\ 
        &= \I_{d} - \mat{\vU_{k} & \vU_{k}^{\perp}} \mat{\sigma(\vLambda + \sigma^{2}\I_{p})^{-1/2} & \vZero \\ \vZero & \sigma \cdot (\sigma^{2})^{-1/2}\I_{d - p}}\mat{\vU_{k}\adj \\ (\vU_{k}^{\perp})\adj} \\ 
        &= \I_{d} - \mat{\vU_{k} & \vU_{k}^{\perp}} \mat{(\sigma^{-2}\vLambda + \I_{p})^{-1/2} & \vZero \\ \vZero & \I_{d - p}}\mat{\vU_{k}\adj \\ (\vU_{k}^{\perp})\adj} \\ 
        &= \mat{\vU_{k} & \vU_{k}^{\perp}} \mat{\I_{p} - (\sigma^{-2}\vLambda + \I_{p})^{-1/2} & \vZero \\ \vZero & \vZero}\mat{\vU_{k}\adj \\ (\vU_{k}^{\perp})\adj} \\
        &\approx \mat{\vU_{k} & \vU_{k}^{\perp}} \mat{\I_{p} & \vZero \\ \vZero & \vZero}\mat{\vU_{k}\adj \\ (\vU_{k}^{\perp})\adj} \\
        &= \vU_{k}\vU_{k}\adj.
    \end{align}
    Thus \(\vP_{k}\) is approximately a projection when \(\sigma\) is small. Under this algebraic relation, we have
    \begin{align}
        &\nabla_{\x}\log q(\x) \\
        &= -\sum_{k = 1}^{K}\e_{k}\adj\softmax{\frac{1}{2\sigma^{2}}\mat{\norm{\x}_{2}^{2} - \norm{\sigma \vM_{1}\adj\x}_{2}^{2} \\ \vdots \\ \norm{\x}_{2}^{2} - \norm{\sigma \vM_{K}\adj\x}_{2}^{2}}}\vM_{k}\vM_{k}\adj\x \\
        &= -\frac{1}{\sigma^{2}}\sum_{k = 1}^{K}\e_{k}\adj\softmax{\frac{1}{2\sigma^{2}}\mat{\norm{\x}_{2}^{2} - \norm{(\I_{d} - \vP_{1})\adj\x}_{2}^{2} \\ \vdots \\ \norm{\x}_{2}^{2} - \norm{(\I_{d} - \vP_{K})\adj\x}_{2}^{2}}}(\I_{d} - \vP_{k})(\I_{d} - \vP_{k})\adj\x \\
        &\approx -\frac{1}{\sigma^{2}}\sum_{k = 1}^{K}\e_{k}\adj\softmax{\frac{1}{2\sigma^{2}}\mat{\norm{\vP_{1}\adj\x}_{2}^{2} \\ \vdots \\ \norm{\vP_{K}\adj\x}_{2}^{2}}}(\I_{d} - \vP_{k})(\I_{d} - \vP_{k})\adj\x \\
        &\approx -\frac{1}{\sigma^{2}}\sum_{k = 1}^{K}\e_{k}\adj\softmax{\frac{1}{2\sigma^{2}}\mat{\norm{\P_{1}\adj\x}_{2}^{2} \\ \vdots \\ \norm{\P_{K}\adj\x}_{2}^{2}}}(\I_{d} - \P_{k})^{*}\x \\
        &= -\frac{\x}{\sigma^{2}}\sum_{k = 1}^{K}\e_{k}\adj\softmax{\frac{1}{2\sigma^{2}}\mat{\norm{\P_{1}\adj\x}_{2}^{2} \\ \vdots \\ \norm{\P_{K}\adj\x}_{2}^{2}}} + \frac{1}{\sigma^{2}}\sum_{k = 1}^{K}\e_{k}\adj\softmax{\frac{1}{2\sigma^{2}}\mat{\norm{\P_{1}\adj\x}_{2}^{2} \\ \vdots \\ \norm{\P_{K}\adj\x}_{2}^{2}}}\P_{k}\adj\x \\
        &= -\frac{1}{\sigma^{2}}\x + \frac{1}{\sigma^{2}}\sum_{k = 1}^{K}\e_{k}\adj\softmax{\frac{1}{2\sigma^{2}}\mat{\norm{\P_{1}\adj\x}_{2}^{2} \\ \vdots \\ \norm{\P_{K}\adj\x}_{2}^{2}}}\P_{k}\adj\x \\
        &\approx -\frac{1}{\sigma^{2}}\x + \frac{1}{\sigma^{2}}\sum_{k = 1}^{K}\e_{k}\adj\softmax{\frac{1}{2\sigma^{2}}\mat{\norm{\vU_{1}\adj\x}_{2}^{2} \\ \vdots \\ \norm{\vU_{K}\adj\x}_{2}^{2}}}\vU_{k}\vU_{k}\adj\x \\
        &= -\frac{1}{\sigma^{2}}\x + \frac{1}{\sigma^{2}}\mat{\vU_{1}, \cdots, \vU_{K}}\left[\diaglr{\softmax{\frac{1}{2\sigma^{2}}\mat{\norm{\vU_{1}\adj\x}_{2}^{2} \\ \vdots \\ \norm{\vU_{K}\adj\x}_{2}^{2}}}} \otimes \I_{p}\right]\mat{\vU_{1}\adj\x \\ \vdots \\ \vU_{K}\adj\x}.
    \end{align}
    Plugging this into Tweedie's formula, we have
    \begin{equation}
        \mathbb{E}[\z \mid \x] \approx \mat{\vU_{1}, \cdots, \vU_{K}}\left[\diaglr{\softmax{\frac{1}{2\sigma^{2}}\mat{\norm{\vU_{1}\adj\x}_{2}^{2} \\ \vdots \\ \norm{\vU_{K}\adj\x}_{2}^{2}}}} \otimes \I_{p}\right]\mat{\vU_{1}\adj\x \\ \vdots \\ \vU_{K}\adj\x}.
    \end{equation}
\end{proof}

\begin{remark}
   Although \Cref{thm:opt_denoiser_multi_subspaces} is stated as an approximation rather than as a proposition, we believe it should be possible without too much extra work to convert it into a statement of asymptotic equivalence as $\sigma \to 0$ (in particular, holding for $\sigma$ below the smallest (nonzero) eigenvalue of $\vLambda$). Most approximations taken in the derivation of \Cref{thm:opt_denoiser_multi_subspaces} can immediately be turned into asymptotic claims; the only slightly delicate point is treating the softmax, which can be accomplished using standard ``high temperature'' convergence behavior of the softmax function
   \citep{Gao2017-sd}
   (in particular, as $\sigma \to 0$ in our expressions, the softmax concentrates on the ``best head''). 
\end{remark}

\subsubsection{Auxiliary Lemmas}

\begin{proposition}\label{thm:score_multi_subspaces}
    Suppose that \(\vz\) follows the low-dimensional Gaussian mixture model \Cref{model:gaussian_tokens} with respect to subspaces \(\vU_{[K]}\), so that \(\vx = \vz + \sigma \vw\), where \(\vw\) is a standard Gaussian vector independent of \(\vz\), follows \Cref{model:gaussian_tokens_noise}. That is, for some random index \(s \sim \pi\) where \(\pi\) is a distribution on \([K]\) and \(s\) is chosen independently of all other randomness, we have
    \begin{equation}
        \vz = \vU_{s}\valpha,
    \end{equation}
    where \(\valpha\) is a zero-mean Gaussian vector with diagonal covariance \(\vLambda\). For each \(k \in [K]\), define 
    \begin{equation}
        \vSigma_{i} \doteq \operatorname{Cov}[\vz \mid s = k] = \vU_{k}\vLambda\vU_{k}\adj, \qquad \vM_{k} \doteq (\vSigma_{k} + \sigma^{2}\vI)^{-1/2}.
    \end{equation}
    Assume (for the sake of normalization) that 
    \begin{equation}
        \pi_{i}\det(\vM_{i}) = \pi_{j}\det(\vM_{j}), \qquad \text{for all}\ 1 \leq i \leq j \leq K.
    \end{equation}
    Then, letting \(\vx \mapsto q(\vx)\) be the density of \(\vx\), we have
    \begin{align}
        &\nabla_{\x}\log q(\x) \\
        &= -\mat{\M_{1}, \cdots, \M_{K}}\left[\diaglr{\softmax{-\frac{1}{2}\mat{\norm{\M_{1}\adj\x}_{2}^{2} \\ \vdots \\ \norm{\M_{K}\adj\x}_{2}^{2}}}} \otimes \I_{d}\right]\mat{\M_{1}\adj\x \\ \vdots \\ \M_{K}\adj\x}.
    \end{align}
\end{proposition}
\begin{proof}
    By the law of total probability, we have
    \begin{align}
        \nabla_{\x}\log q(\x) 
        &= \nabla_{\x} \log \sum_{k = 1}^{K}q(\x \mid k)\pi_{k} \\
        &= \frac{\sum_{k = 1}^{K}\pi_{k}\nabla_{\x}q(\x \mid k)}{\sum_{k = 1}^{K}q(\x \mid k)\pi_{k}}
    \end{align}
    where \(q(\x \mid k)\) is the conditional density of \(\x\) given the event \(\{s = k\}\). To compute this quantity, note that \textit{conditional on the value of \(s\)}, we have
    \begin{equation}
        \x = \vz_{s} + \sigma \w \sim \mathcal{N}(\vZero, \vSigma_{s} + \sigma^{2}\I_{d}).
    \end{equation}
    Thus we have
    \begin{equation}
        q(\x \mid k) = \frac{1}{\sqrt{(2\pi)^{d}\det(\vSigma_{k} + \sigma^{2}\I_{d})}}\explr{-\frac{1}{2}\x\adj(\vSigma_{k} + \sigma^{2}\I_{d})^{-1}\x},
    \end{equation}
    This gives
    \begin{equation}
        \nabla_{\x}q(\x \mid k) = -q(\x \mid k) \cdot (\vSigma_{k} + \sigma^{2}\I_{d})^{-1}\x.
    \end{equation}
    Putting this all together, we get
    \begin{align}
        &\nabla_{\x}\log q(\x) \\
        &= -\frac{\sum_{k = 1}^{K}q(\x \mid k)\pi_{k} \cdot (\vSigma_{k} + \sigma^{2}\I_{d})^{-1}\x}{\sum_{k = 1}^{K}q(\x \mid k)\pi_{k}} \\
        &= -\frac{\sum_{k = 1}^{K}\pi_{k}\det(\vSigma_{k} + \sigma^{2}\I_{d})^{-1/2}\explr{-\frac{1}{2}\x\adj(\vSigma_{k} + \sigma^{2}\I_{d})^{-1}\x}\cdot (\vSigma_{k} + \sigma^{2}\I_{d})^{-1}\x}{\sum_{k = 1}^{K}\pi_{k}\det(\vSigma_{k} + \sigma^{2}\I_{d})^{-1/2}\explr{-\frac{1}{2}\x\adj(\vSigma_{k} + \sigma^{2}\I_{d})^{-1}\x}} \\
        &= -\frac{\sum_{k = 1}^{K}\pi_{k}\det(\vM_{k})\explr{-\frac{1}{2}\x\adj\vM_{k}\vM_{k}\adj\x}\cdot \vM_{k}\vM_{k}\adj\x}{\sum_{k = 1}^{K}\pi_{k}\det(\vM_{k})\explr{-\frac{1}{2}\x\adj\vM_{k}\vM_{k}\adj\x}} \\
        &= -\frac{\sum_{k = 1}^{K}\pi_{k}\det(\vM_{k})\explr{-\frac{1}{2}\norm{\vM_{k}\adj\x}_{2}^{2}}\cdot \vM_{k}\vM_{k}\adj\x}{\sum_{k = 1}^{K}\pi_{k}\det(\vM_{k})\explr{-\frac{1}{2}\x\adj\vM_{k}\vM_{k}\adj\x}}.
    \end{align}
    Given our assumption that 
    each \(\pi_{k}\det(\M_{k})\) is the same, we have
    \begin{align}
        &\nabla_{\x}\log q(\x) \\
        &= -\frac{\sum_{k = 1}^{K}\pi_k\det(\vM_{k})\explr{-\frac{1}{2}\norm{\vM_{k}\adj\x}_{2}^{2}}\cdot \vM_{k}\vM_{k}\adj\x}{\sum_{k = 1}^{K}\pi_k\det(\vM_{k})\explr{-\frac{1}{2}\norm{\vM_{k}\adj\x}_{2}^{2}}} \\
        &= -\frac{\sum_{k = 1}^{K}\explr{-\frac{1}{2}\norm{\vM_{k}\adj\x}_{2}^{2}}\cdot \vM_{k}\vM_{k}\adj\x}{\sum_{k = 1}^{K}\explr{-\frac{1}{2}\norm{\vM_{k}\adj\x}_{2}^{2}}} \\
        &= -\sum_{k = 1}^{K}\e_{k}\adj\softmax{-\frac{1}{2}\mat{\norm{\vM_{1}\adj\x}_{2}^{2} \\ \vdots \\ \norm{\vM_{K}\adj\x}_{2}^{2}}}\vM_{k}\vM_{k}\adj\x \\
        &= -\mat{\vM_{1}, \dots, \vM_{K}}\left[\diaglr{\softmax{-\frac{1}{2}\mat{\norm{\vM_{1}\adj\x}_{2}^{2} \\ \vdots \\ \norm{\vM_{K}\adj\x}_{2}^{2}}}} \otimes \I_{d}\right]\mat{\vM_{1}\adj\x \\ \vdots \\ \vM_{K}\adj\x}.
    \end{align}
\end{proof}

%% file: app_technical_compression_denoising.tex
\subsection{Companion to \Cref{sub:structured_diffusion}} \label{app:computations_rr_gradient}

In this section, we prove a formal version of the result \Cref{thm:informal_rate_score} stated in \Cref{sub:structured_diffusion}. That is, we examine a basic yet representative instantiation of the signal model \Cref{model:gaussian_tokens_noise}, and show that under this model, in a
natural regime of parameter scales motivated by the architecture of \ours{}
applied to standard image classification benchmarks, the operation implemented
by taking a gradient step on the compression term of the sparse rate reduction
objective \Cref{eq:sparse-rr} corresponds to a projection operation at
quantization scales $\veps^2$ proportional to the size of the deviation.
This leads us in particular to a formal version of the result
\Cref{thm:informal_rate_score}.

\paragraph{Signal model.}

We consider an instantiation of the model \Cref{model:gaussian_tokens_noise}, elaborated here. That is, we fix a distribution over tokens $\vZ \in \bbR^{d \times n}$ induced by
the following natural signal model: each token $\vz_i$ is drawn independently from the
normalized isotropic Gaussian measure on one of $K$ $p$-dimensional
subspaces with orthonormal bases $\vU_1, \dots, \vU_K \in
\bbR^{d \times p}$,\footnote{More precisely, $\vz_i$ is distributed according to the
pushforward of the normalized isotropic Gaussian measure $\sN(\vZero, \tfrac{1}{p}\vI)$ on $\bbR^p$ by the bases $\vU_k$.}
which comprise the low-dimensional structure in the observed tokens, then
corrupted with i.i.d.\ Gaussian noise $\sN(\vZero, \tfrac{\sigma^2}{d} \vI)$;
the subspace each token is drawn from is selected uniformly at random,
independently of all other randomness in the problem.
This signal model therefore corresponds to the setting of uncorrelated tokens,
with maximum entropy coordinate distributions within subspaces.
It is natural to first develop our
theoretical understanding of the connection between compression and the score
function in the uncorrelated setting, although in general, the ability of \ours{} to capture correlations in the
data through the MSSA block is essential. In connection with the latter issue, we note that our proofs will generalize straightforwardly to the setting of ``well-dispersed'' correlated tokens: see the discussion in \Cref{rmk:assumptions}.

We make the further following assumptions within this model:
\begin{enumerate}
    \item Inspired by an ablation in \Cref{fig:appendix-exp-visualize-UiUj}, which suggests that the learned
        \ours{} model on supervised classification on Imagenet has signal models $\vU_k$
        which are near-incoherent, we will assume that the subspaces $\vU_k$ have pairwise orthogonal column spaces.
        Our proofs will generalize straightforwardly to the setting where the subspaces are merely incoherent: see the discussion in \Cref{rmk:assumptions}.
    \item We assume that the relative scales of these parameters conform to the
        \ours{}-Base settings, trained on Imagenet: from
        \Cref{tab:model_configs}, these parameters are
        \begin{enumerate}
            \item $d = 768$;
            \item $n = 196$;
            \item $K = 12$;
            \item $p = d / K = 64$.
        \end{enumerate}
        In particular, $d \gg n \gg p$ and $Kp=d$. 
\end{enumerate}
These precise parameter values will not play a role in our analysis. We merely require the following quantitative relationships between the parameter values, which are more general than the above precise settings.
\begin{assumption}\label{ass:parameter_config}
    We have $\veps \leq 1$, \(\vU_{k}^{\top}\vU_{k^{\prime}} = \Ind{k = k^{\prime}}\vI\) for all \(k \neq k^{\prime}\), and the following parameter settings and scales:
    \begin{itemize}
        \item \(d \geq n \geq p \geq K \geq 2\);
        \item \(Kp = d\);
        \item \(C_{1}\sqrt{n \log n} \leq \frac{1}{2}n/K\), where \(C_{1}\) is the same as the universal constant \(C_{1}\) in the statement of \Cref{prop:binom_concentration};
        \item \(6C_{2}^{2}n \leq d\), where \(C_{2}\) is the same as the universal constant \(C_{3}\) in the statement of \Cref{prop:gaussian_covariance_concentration_opnorm};
        \item \(2C_{4}^{2}n \leq d\), where \(C_{4}\) is the same as the universal constant \(C_{1}\) in \Cref{prop:xktxk_concentration_opnorm};
    \end{itemize}
\end{assumption}
\textbf{note:} there is no self-reference, as the third inequality is not used to prove \Cref{prop:binom_concentration}, the fourth is not used to prove \Cref{prop:gaussian_covariance_concentration_opnorm}, and the fifth is not used to prove \Cref{prop:xktxk_concentration_opnorm}.

The first and second inequalities together imply in particular that \(p \geq n / K\). The third inequality implies that \(C_{1}\sqrt{n \log n} < n/K\).  The first, second, and and third inequalities together imply that \(p > C_{1}\sqrt{n \log n}\), and that \(0 < n/K - C_{1}\sqrt{n \log n} < n/K < n/K + C_{1}\sqrt{n \log n} < n\).

These inequalities are verifiable in practice if one wishes to explicitly compute the absolute constants \(C_{1}, C_{2}, C_{3}, C_{4}\), and indeed they hold for our \ours{}-Base model.

Formally, let $\mu(K, p, \sigma^2)$ denote the probability measure on $\bbR^{d \times n}$ 
corresponding to the noisy Gaussian mixture distribution specified above.
We let $\vZ_{\base} \sim \mu$ denote an observation distributed according to
this signal model: formally, there exists a (random) map $i \mapsto s_{i}$, for $i
\in [n]$ and $s_{i} \in [K]$, such that
\begin{equation}
    \vz_{\base i} = \vU_{s_{i}} \valpha_i + \vdelta_i, \quad i = 1, \dots, n,
    \label{eq:signal-model-vector-form}
\end{equation}
where $\vDelta = \begin{bmatrix} \vdelta_1 & \hdots & \vdelta_{n} \end{bmatrix}
\simiid \sN(\vZero, \tfrac{\sigma^2}{d} \vI)$, and (independently) $\valpha_i
\simiid \sN(\vZero,
\tfrac{1}{p}\vI)$.
It is convenient to write this observation model in block form. 
To this end, let $K_k = \sum_{i=1}^n \Ind{s_{i} = k}$ for $k \in [K]$
denote the number of times the $k$-th subspace is represented amongst the
columns of $\vZ_{\base}$ (a random variable).
Then by rotational invariance of the Gaussian distribution, we have
\begin{equation}
    \vZ_{\base}
    \equid
    \begin{bmatrix}
        \vU_1 \vA_1 & \hdots & \vU_K \vA_K
    \end{bmatrix}
    \vPi
    + \vDelta,
    \label{eq:signal-model-block-form}
\end{equation}
where $\equid$ denotes equality in distribution, $\vPi \in \bbR^{n \times n}$ is
a uniformly random permutation matrix, and each $\vA_k \in \bbR^{p \times K_k}$. We also %
define \(\vX_{\natural}\) to be the noise-free version of \(\vZ_{\base}\).

Because of this equality in distribution, we will commit the mild abuse of
notation of identifying the block representation
\Cref{eq:signal-model-block-form} with the observation model
\Cref{eq:signal-model-vector-form} that follows the distribution $\mu$.

\paragraph{Denoising in the uncorrelated tokens model.}
In the uncorrelated tokens model \Cref{eq:signal-model-block-form}, the marginal
distribution of each column of $\vZ_{\base}$ is identical, and equal to an
equiproportional mixture of (normalized) isotropic Gaussians on the subspaces
$\vU_1, \dots \vU_k$, convolved with the noise distribution $\sN(\vZero,
\tfrac{\sigma^2}{d} \vI)$. 
This marginal distribution was studied in \Cref{app:proofs-denoising}, where it was shown that
when the perturbation level $\sigma^2 \to 0$, the score function for this
marginal distribution approximately implements a projection operation onto the
nearest subspace $\vU_k$.

Hence, we can connect compression, as implemented in the MSSA block of the
\ours{} architecture, to denoising in the uncorrelated tokens model by showing
that at similar local scales, and for suitable settings of the model parameters,
the compression operation implements a projection onto the low-dimensional
structure of the distribution, as well.

\paragraph{Compression operation.} 
The MSSA block of the \ours{} architecture arises from taking an (approximate)
gradient step on the $R^c$ term of the sparse rate reduction objective
\Cref{eq:sparse-rr}. This term writes
\begin{equation}
    R^c(\vZ \mid \vU_{[K]}) = \frac{1}{2} \sum_{k=1}^K \log\det \left( \vI + \beta
    (\vU_k \adj \vZ)\adj \vU_k\adj \vZ \right),
\end{equation}
where
\begin{align}
    \beta &= \frac{p}{n \veps^2},
\end{align}
and $\veps > 0$ is the quantization error. 
Calculating the gradient, we have
\begin{equation}
    \nabla_{\vZ} R^c(\vZ \mid \vU_{[K]})
    =
    \sum_{k=1}^K \vU_k\vU_k\adj \vZ \left(
        \beta\inv \vI + (\vU_k\adj \vZ)\adj \vU_k\adj \vZ
    \right)\inv.
    \label{eq:deltar-gradient}
\end{equation}
Minimizing the sparse rate reduction objective corresponds to taking a gradient
descent step on $R^c(\spcdot \mid \vU_{[K]})$. Performing this operation at the
observation from the uncorrelated tokens model $\vZ_{\base}$, the output can be
written as
\begin{equation}
    \vZ^+ = \vZ_{\base} - \eta \nabla R^c(\vZ_{\base} \mid \vU_{[K]}),
    \label{eq:tokens-compressed}
\end{equation}
where $\eta > 0$ is the step size.

\paragraph{Main result on projection.} We will see shortly that the behavior of
the compression output \Cref{eq:tokens-compressed} depends on the relative
scales of the perturbation about the low-dimensional structure $\sigma^2$ and
the target quantization error $\veps^2$.

\begin{theorem}\label{lem:inverse-term}
    There are universal constants $C_{1}, C_{2}, C_{3}, C_{4} > 0$ such that the following holds.
    Suppose \Cref{ass:parameter_config} holds, and moreover suppose that 
    $\sigma \leq 1$ and $C_{1}\beta \sigma \leq \half$.
    Then with probability at least $1 - K C_{2}\bigl( e^{-C_{3} d} + e^{-C_{4}n/K} + 
    n^{-2}\bigr)$,
    it holds
    \begin{align}
        &\norm*{
        \vZ^{+}
        -
        \left[
            \left(\vDelta - \eta\sP_{\vU_{[K]}}(\beta\vDelta \vPi\adj) \vPi\right) + \frac{1 +\beta\inv - \eta}{1 + \beta^{-1}}\vX_{\natural}
        \right]}
        \\
        &\leq C_{5}
        K\eta\left(
        \sigma^2 \beta^2 + \sigma(1 + \sqrt{n/d}) + \sqrt{K} \beta
        \sigma^2(1 + \sqrt{n/d}) + \sqrt{n/d}
        \right).
    \end{align}
    Here,
    $\sP_{\vU_{[K]}}$ implements a projection onto the relevant subspaces for
    each token in the limiting case as $\veps \to 0$, and is precisely defined
    in \Cref{eq:proj-defn-1,eq:proj-defn-2}.
\end{theorem}

We give the proof of \Cref{lem:inverse-term} below. First, we make three remarks
on interpreting the result, our technical assumptions, and our analysis.

\begin{remark}
    \Cref{lem:inverse-term} admits the following interesting interpretation in
    an asymptotic setting, where we can identify the leading-order behavior of
    the gradient and confirm our hypothesis about the connection between
    compression and score-following.
    Choose $\eta = \beta\inv$, so that the guarantee in \Cref{lem:inverse-term}
    incurs some cancellation, and moreover delineate more precise dependencies
    on the RHS of the guarantee:
    \begin{align}
        &\norm*{
        \vZ^{+}
        -
        \left[
            \left(\vDelta - \sP_{\vU_{[K]}}(\vDelta \vPi\adj)
            \vPi\right) + \frac{1}{1 + \beta^{-1}}\vX_{\natural}
        \right]}
        \\
        &\quad\ltsim
        \frac{nK^2 \veps^2}{d}\left(
        \frac{\sigma^2 d^2}{n^2K^2 \veps^4} + \sigma(1 + \sqrt{n/d}) +
        \frac{d\sigma^2}{n\sqrt{K}\veps^2} (1 + \sqrt{n/d}) + \sqrt{n/d}
        \right)
        \\
        &\quad\ltsim
        K^{3/2}\sigma^2 
        + \frac{\sigma^2 d}{n\veps^2} 
        + \frac{nK^2}{d}\left(
        \sigma + \sqrt{\frac{n}{d}}
        \right)\veps^2,
    \end{align}
    where we used \Cref{ass:parameter_config}, which implies $p = d/K$ and
    $n/d\leq 1$.
    We will check in due course whether we have satisfied the hypotheses of
    \Cref{lem:inverse-term}, so that this guarantee indeed applies.
    To this end, we optimize this bound as a function of $\veps > 0$, since this is a
    parameter of the compression model.
    The optimal $\veps$ is straightforward to compute using calculus: it
    satisfies
    \begin{align}
        \veps^2 &= \sqrt{\left.\frac{\sigma^2 d}{n}\right/\frac{K^2 n}{d} \left(
        \sigma + \sqrt{\frac{n}{d}}\right)}
        \\
        &=
        \frac{\sigma d}{n K \sqrt{\sigma + \sqrt{\frac{n}{d}}}},
    \end{align}
    and the value of the residual arising from \Cref{lem:inverse-term} with this
    choice of $\veps$ is no larger than an absolute constant multiple of
    \begin{equation}
        K^{3/2}\sigma^2 
        + \sqrt{
            \frac{K^2 \sigma^2 d}{n} \left(
                \frac{n\sigma}{d} + \left( \frac{n}{d} \right)^{3/2}
            \right)
        }
        =
        K \sigma \left(
        \sqrt{K} \sigma
        + \sqrt{\sigma + \sqrt{\frac{n}{d}}}
        \right).
    \end{equation}
    Moreover, with this choice of $\veps$, $\beta$ satisfies
    \begin{equation}
        \beta\inv = \frac{\veps^2 nK}{d} = \sqrt{\frac{\sigma}{1 +
        \sqrt{\frac{n}{d\sigma^2}}}}.
    \end{equation}
    In particular, the condition $\beta \sigma \ltsim 1$ in
    \Cref{lem:inverse-term} demands
    \begin{equation}
        \sqrt{\sigma + \sqrt{\frac{n}{d}}} \ltsim 1,
    \end{equation}
    which holds for sufficiently small $\sigma$ and sufficiently large $d \geq
    n$, showing that \Cref{lem:inverse-term} can be nontrivially applied in this
    setting.
    If we consider a simplifying limiting regime where $n, d \to +\infty$ such
    that $n/d \to 0$ and $n/K \to +\infty$, we observe the following asymptotic
    behavior of the guarantee of \Cref{lem:inverse-term}:
    \begin{align}
        \norm*{
        \vZ^{+}
        -
        \left[
            \left(\vDelta - \sP_{\vU_{[K]}}(\vDelta \vPi\adj)
            \vPi\right) + \frac{1}{1 + \sqrt{\sigma}}\vX_{\natural}
        \right]}
        \ltsim
        K  \sigma^{3/2} \left( 1 + \sqrt{K\sigma} \right).
    \end{align}
    This demonstrates that a gradient step on $R^c$ performs denoising: there is
    a noise-level-dependent shrinkage effect applied to the signal
    $\vX_{\natural}$, which vanishes as $\sigma \to 0$, and meanwhile the noise
    term $\vDelta$ is reduced.

    Moreover, as $\sigma \to 0$, we can express the limiting form of
    $\sP_{\vU_{[K]}}$ exactly as an orthogonal projection, since this drives
    $\beta\inv \to 0$:
    following \Cref{eq:proj-defn-1,eq:proj-defn-2}, we have here
    \begin{equation}
        \sP_{U_{[K]}}
        =
        \begin{bmatrix}
            \sP_1 & \hdots & \sP_K   
        \end{bmatrix},
    \end{equation}
    where
    \begin{equation}
        \sP_k \to 
        \sum_{k' \neq k}
        \vU_{k'} \proj{\im(\vA_{k'})^\perp} \vU_{k'}\adj.
    \end{equation}
    This shows that, in an asymptotic sense, a gradient step on $R^c$ serves to \textit{suppress
    the effect of the perturbation applied to the observations $\vZ_{\base}$
    about the local signal model $\vX_{\natural}$.} This verifies our claim
    previously that in this setting, there is a correspondence between a
    score-following algorithm and a compression-based approach: locally, both
    project the observations onto the structures of the signal model.

    It can be shown moreover that the shrinkage effect on $\vX_{\natural}$
    demonstrated here appears as a consequence of using the $R^c$
    ``compression'' term for the gradient step in \ours{}; when the gradient
    step is taken instead on the full $\Delta R$ rate reduction objective (which
    is computationally prohibitive, of course), there is zero shrinkage, and
    perfect denoising is performed for a wider variety of step sizes $\eta$ than the choice made here. We see the introduction of this shrinkage
    effect this as the price of constructing an efficient and interpretable
    network architecture. In practice, the ISTA block of \ours{} counteracts
    this shrinkage effect, which is anyways minor at reasonable parameter scales.
\end{remark}

\begin{remark}\label{rmk:assumptions}
     We have made two assumptions which may not hold exactly in practice: namely, we have assumed that the \(\vU_{k}\)'s have orthogonal columns, namely \(\vU_{k}^{\top}\vU_{k^{\prime}} = \Ind{k = k^{\prime}}\vI\), and we have assumed that the linear combination coefficients \(\vA_{k}\) that form the matrix \(\vX_{\natural}\) are i.i.d.~samples from Gaussian distributions. Both these assumptions can be made more realistic, at the cost of additional (non-instructive) complexity in the analysis; we briefly go over how.
     
     Relaxing the orthogonality condition \(\vU_{k}^{\top}\vU_{k^{\prime}} = \Ind{k = k^{\prime}}\vI\) to near-orthogonality, namely \(\norm{\vU_{k}^{\top}\vU_{k^{\prime}} - \Ind{k = k^{\prime}}\vI} \leq \nu\) for a small \(\nu\), as observed in practice (recall \Cref{fig:appendix-exp-visualize-UiUj}) would introduce additional small error terms in the proof, say polynomial in \(\nu\). The magnitudes of these errors could in principle be precisely tracked, whence one could obtain a similar result to \Cref{lem:inverse-term}.

    Secondly, we have assumed that the \(\vA_{k}\)'s have independent columns which are sampled from (the same) Gaussian distribution. However, in the conceptual framework for \ours{}, we exploit the joint distribution (and in particular the correlations) between the tokens in order to obtain good performance for our model. Our analysis is not completely agnostic to this fact; as we will see, the proof of \Cref{lem:inverse-term} only leverages the independence of the columns of each \(\vA_{k}\)'s in order to obtain high-probability upper bounds on the smallest and largest singular value of the token matrices. If these bounds were ensured by some other method, such as appropriate normalization and incoherence, a similar conclusion to \Cref{lem:inverse-term} could hold in the more realistic correlated tokens model. Going beyond well-conditioned token matrices for each subspace would require additional modeling assumptions, and additional investigative experimental work to determine a realistic basis for such assumptions.
\end{remark}

\begin{remark}
    We have not attempted to optimize constants or rates of concentration in the
    proof of \Cref{lem:inverse-term}, preferring instead to pursue a
    straightforward analysis that leads to a qualitative interpretation of the
    behavior of the rate reduction gradient in our model problem. Minor
    improvements to the concentration analysis would enable the parameter
    scaling requirements in \Cref{ass:parameter_config} to be relaxed slightly,
    and the probability bound in \Cref{lem:inverse-term} that scales as $K /
    n^2$ can easily be improved to any positive power of $1/n$.
\end{remark}

\begin{proof}[Proof of \Cref{lem:inverse-term}]
    We start by noticing that, by orthonormality of the subspaces $\vU_k$,
    we have by \Cref{eq:signal-model-block-form}
    \begin{equation}
        \vU_k \adj \vZ_{\base}
        =
        \begin{bmatrix}
            \vZero & \hdots & \vA_k & \hdots & \vZero
        \end{bmatrix}
        \vPi
        + \vU_k\adj \vDelta,
    \end{equation}
    so that
    \begin{equation}
    \small
        \left(
        \beta\inv \vI + (\vU_k\adj \vZ_{\base})\adj \vU_k\adj
        \vZ_{\base}
        \right)\inv
        =
        \vPi\adj
        \left(
        \underbrace{
            \begin{bmatrix}
                \beta\inv\vI  &&&& \\
                &\ddots &&& \\
                && \beta\inv \vI + \vA_k\adj \vA_k && \\
                &&&\ddots & \\
                &&&&\beta\inv \vI  \\
            \end{bmatrix}
        }_{\vD_k}
        + \vXi_k
        \right)\inv
        \vPi,\label{eq:ak-defn}
    \end{equation}
    because permutation matrices are orthogonal matrices, and where the
    perturbation $\vXi_k$ is defined by
    \begin{equation}
        \vXi_k = 
        \vPi \vDelta\adj \vU_k \vU_k\adj \vDelta\vPi\adj
        +
        \begin{bmatrix}
            \vZero & \hdots & \vDelta_1\adj \vU_k \vA_k & \hdots & \vZero \\
            \vdots & & \vdots & & \vdots \\
            \vA_k\adj \vU_k\adj \vDelta_1 & \hdots & \vDelta_k\adj \vU_k \vA_k +
            \vA_k\adj \vU_k\adj \vDelta_k & \hdots & \vA_k\adj \vU_k\adj \vDelta_K\\
            \vdots & & \vdots & & \vdots \\
            \vZero & \hdots & \vDelta_K\adj \vU_k \vA_k & \hdots & \vZero \\
        \end{bmatrix},
        \label{eq:xik-defn}
    \end{equation}
    and where we have defined (implicitly) in addition
    \begin{equation}
        \begin{bmatrix}
            \vDelta_1 & \hdots & \vDelta_K
        \end{bmatrix}
        = \vDelta \vPi\adj.
    \end{equation}
    The matrix $\vD_k \succ \vZero$, so we can write
    \begin{equation}
        \left(
        \beta\inv \vI + (\vU_k\adj \vZ_{\base})\adj \vU_k\adj
        \vZ_{\base}
        \right)\inv
        =
        \vPi\adj
        \vD_k\inv
        \left(
        \vI + \vXi_k\vD_k\inv
        \right)\inv \vPi,
    \end{equation}
    from which it follows
    \begin{align}
        &\vU_k\adj \vZ_{\base}
        \left(
        \beta\inv \vI + (\vU_k\adj \vZ_{\base})\adj \vU_k\adj
        \vZ_{\base}
        \right)\inv
        \\
        &\quad=
        \left(
        \begin{bmatrix}
            \vZero & \hdots & \vA_k(\beta\inv \vI + \vA_k\adj \vA_k)\inv &
            \hdots & \vZero
        \end{bmatrix}
        + \vU_k\adj \vDelta\vPi\adj \vD_k\inv
        \right)
        \left(
        \vI + \vXi_k\vD_k\inv
        \right)\inv \vPi.
    \end{align}
    The task before us is therefore to control $\norm{\vXi_k \vD_k \inv} <
    1$, in order to apply the Neumann series to further simplify this
    expression. We will do this in stages: first, we invoke several
    auxiliary lemmas to construct a high-probability event on which the
    random quantities in the preceding expression are controlled about their
    nominal values; next, we show that the Neumann series can be applied on
    this event and a main term extracted; finally, we simplify this main
    term further in order to establish the claimed expression.

    \paragraph{High-probability event construction.} In order to achieve the appropriate control on all random quantities, we would like to construct a high-probability event on which the random quantities are not too large. By \Cref{prop:delta_concentration_opnorm,prop:xk_concentration_opnorm,prop:xktxk_concentration_opnorm} and union bound, there exist universal constants \(C_{i} > 0\) for which
    \begin{equation}
        \Pr*{
            \begin{array}{rl}
                & \norm{\vDelta} \leq \sigma(C_{1} + \sqrt{n/d}) \\ 
                \forall k \in [K] \colon& \norm{\vA_{k}} \leq 1 + C_{2}\sqrt{n/d} \\ 
                \forall k \in [K] \colon& \norm{\vA_{k}^{\top}\vA_{k} - \vI} \leq C_{3}\sqrt{n/d}
            \end{array}
        } \geq 1 - C_{4}K(e^{-C_{5}d} + e^{-C_{6}n/K} + n^{-2}).
    \end{equation}
    The event we compute the probability of, which we denote \(E^{\star}\), is precisely the good event that we want. Formally, 
    \begin{equation}
        E^{\star} \doteq \bc{
            \begin{array}{rl}
                & \norm{\vDelta} \leq \sigma(C_{1} + \sqrt{n/d}) \\ 
                \forall k \in [K] \colon& \norm{\vA_{k}} \leq 1 + C_{2}\sqrt{n/d} \\ 
                \forall k \in [K] \colon& \norm{\vA_{k}^{\top}\vA_{k} - \vI} \leq C_{3}\sqrt{n/d}
            \end{array}
        }. \label{eq:good_event}
    \end{equation}
    We know that \(E^{\star}\) occurs with high probability, and are able to strongly control the random quantities to the degree desired.

    \paragraph{Main term extraction.}
    By \Cref{lem:neumann-op-norm} and our hypotheses on the problem parameters, we have on \(E^{\star}\) that
    \begin{equation}
        \norm{\vXi_k \vD_k \inv} 
        \leq
        C \beta \sigma
        <
        1.
    \end{equation}  
    We can therefore apply the Neumann series to obtain
    \begin{align}
        &\vU_k\adj \vZ_{\base}
        \left(
        \beta\inv \vI + (\vU_k\adj \vZ_{\base})\adj \vU_k\adj
        \vZ_{\base}
        \right)\inv
        \\
        =&
        \scalebox{0.8}{\(\left(
        \begin{bmatrix}
            \vZero & \hdots & \vA_k(\beta\inv \vI + \vA_k\adj \vA_k)\inv &
            \hdots & \vZero
        \end{bmatrix}
        + \vU_k\adj \vDelta\vPi\adj \vD_k\inv
        \right)
        \left(
        \vI - 
        \vXi_k \vD_k^{-1}
        +
        \sum_{j=2}^\infty
        (-1)^j \left( \vXi_k \vD_k^{-1} \right)^j
        \right) \vPi\)}.
    \end{align}
    Again on \(E^{\star}\), we have
    \begin{equation}
        \norm*{
            \sum_{j=2}^\infty
            (-1)^j \left( \vXi_k \vD_k^{-1} \right)^j
        }
        \leq
        \sum_{j=2}^\infty
        \norm*{
            \vXi_k \vD_k^{-1}
        }^j
        \leq
        C (\beta \sigma)^2
        \frac{
            1
        }{
            1 - C \beta \sigma
        }
        \leq
        C' (\beta \sigma)^2.
    \end{equation}
    Moreover, as in the proof of \Cref{lem:neumann-op-norm}, we have on the previous event
    that
    \begin{equation}
        \norm*{
            \vU_k\adj \vDelta\vPi\adj \vD_k\inv
        }
        \leq
        C \beta \sigma.
    \end{equation}
    Thus, if we define a ``main term''
    \begin{equation}
        \vM_k
        =
        \left[
            \begin{bmatrix}
                \vZero & \hdots & \vA_k(\beta\inv \vI + \vA_k\adj \vA_k)\inv &
                \hdots & \vZero
            \end{bmatrix}
            \left(
            \vI - 
            \vXi_k \vD_k^{-1}
            \right)
            + \vU_k\adj \vDelta\vPi\adj \vD_k\inv
            \right]\vPi,
    \end{equation}
    we have on the same event as previously
    \begin{equation}
        \norm*{
            \vU_k\adj \vZ_{\base}
            \left(
            \beta\inv \vI + (\vU_k\adj \vZ_{\base})\adj \vU_k\adj
            \vZ_{\base}
            \right)\inv
            - \vM_k
        }
        \leq C (\beta \sigma)^2.
    \end{equation}
    To conclude, we need only study this main term, since $\vU_k$ has operator norm $1$.

    \paragraph{Simplifying the main term.}
    Our approach will be to control the main term $\vM_k$ around a simpler
    expression, using basic perturbation theory; by the triangle inequality
    for the operator norm, this will give control of the desired gradient
    term.
    After distributing, $\vM_k$ is a sum of three terms; we will start with
    the simplest term. We first compute
    \begin{equation}
        \vU_k\adj \vDelta\vPi\adj \vD_k\inv
        =
        \vU_k\adj
        \begin{bmatrix}
            \beta \vDelta_1 & \hdots 
            \vDelta_k \left(
            \beta\inv \vI + \vA_k\adj \vA_k\right)\inv
            & \hdots & \beta \vDelta_K
        \end{bmatrix}.
    \end{equation}
    We are going to argue that the residual
    \begin{equation}
        \norm*{
            \vU_k\adj \vDelta\vPi\adj \vD_k\inv
            -
            \vU_k\adj
            \begin{bmatrix}
                \beta \vDelta_1 & \hdots &
                \vZero
                & \hdots & \beta \vDelta_K
            \end{bmatrix}
        }
    \end{equation}
    is small. To this end, note that by the fact that $\vU_k$ has unit
    operator norm,
    \begin{align}
        &\norm*{
            \vU_k\adj \vDelta\vPi\adj \vD_k\inv
            -
            \vU_k\adj
            \begin{bmatrix}
                \beta \vDelta_1 & \hdots &
                \vZero
                & \hdots & \beta \vDelta_K
            \end{bmatrix}
        } \\
        &\leq
        \norm*{
            \begin{bmatrix}
                \vZero & \hdots &
                \vDelta_k \left(
                \beta\inv \vI + \vA_k\adj \vA_k\right)\inv
                & \hdots & \vZero
            \end{bmatrix}
        }
        \\
        &=
        \norm*{
            \vDelta_k \left(
            \beta\inv \vI + \vA_k\adj \vA_k\right)\inv
        } \\
        &\leq
        \norm*{
            \vDelta_k 
        }
        \norm*{
            \left(
            \beta\inv \vI + \vA_k\adj \vA_k\right)\inv
        }.
    \end{align}
    By \Cref{eq:good_event} and \Cref{eq:pseudoinverse-ipart-small} from
    \Cref{lem:blockwise-pinv}, the second term 
    here is controlled on $E^\star$.
    For the first term, we note that by definition and the fact that
    the unit sphere is invariant to rotations (and permutations are
    rotations),
    \begin{align}
        \norm*{\vDelta}
        =
        \sup_{\norm{\vu}_2 \leq 1}
        \norm*{\vDelta \vu}_2
        &=
        \sup_{\norm{\vu}_2 \leq 1}
        \norm*{\begin{bmatrix}
            \vDelta_1 & \hdots & \vDelta_K
        \end{bmatrix} \vu
        }_2
        \\
        &=
        \sup_{\norm{\vu}_2 \leq 1}
        \norm*{
            \sum_{i=1}^K
            \vDelta_i \vu_i
        }_2,
    \end{align}
    where $\vu_i$ are coordinate-subset-induced partitions of the vector
    $\vu$ induced by those of $\vDelta\vPi\adj$.
    This yields immediately
    \begin{equation}
        \norm*{\vDelta}
        \leq
        \sup_{\norm{\vu}_2 \leq 1}
        \sum_{i=1}^K 
        \norm*{
            \vDelta_i \vu_i
        }_2
        \leq
        \left( \max_{k \in [K]}\, \norm{\vDelta_k}\right)
        \sup_{\norm{\vu}_2 \leq 1}
        \sum_{i=1}^K \norm{\vu_i}_2
        \leq
        \sqrt{K}
        \left( \max_{k \in [K]}\, \norm{\vDelta_k}\right),
        \label{eq:control-op-by-blocks}
    \end{equation}
    by the triangle inequality and inequalities for $\ell^p$ norms.
    Similarly, choosing a specific $\vu$ in the operator norm expression,
    namely one that is supported entirely on one of the coordinate
    partitions $\vu_i$, shows that
    \begin{align}
        \norm*{\vDelta}
        &\geq
        \norm*{
            \vDelta_i \vu_i
        }_2
    \end{align}
    for any $i$, whence
    \begin{equation}
        \max_{k \in [K]}\, \norm{\vDelta_k}
        \leq
        \norm*{\vDelta}.
        \label{eq:blocks-opnorm-controlled-overall}
    \end{equation}
    It follows that we control the first term above on $E^\star$.
    Combining this reasoning, we conclude from the above
    \begin{align}
        &\norm*{
            \vU_k\adj \vDelta\vPi\adj \vD_k\inv
            -
            \vU_k\adj
            \begin{bmatrix}
                \beta \vDelta_1 & \hdots &
                \vZero
                & \hdots & \beta \vDelta_K
            \end{bmatrix}
        }\\
        \leq&
        \sigma(C + \sqrt{n/d})
        \left(
        \frac{1}{1 + \beta\inv}
        +
        \frac{C' \sqrt{n/d}}{1 + \beta\inv}
        \right)
        \\
         \ltsim &\sigma(1 + C\sqrt{n/d}),
    \end{align}
    where the second line uses \Cref{ass:parameter_config} to 
    remove the higher-order residual.

    next, we recall that $\vXi_k$ is a sum of two terms; we will do one term
    at a time for concision. We have first
    \begin{align}
        &\begin{bmatrix}
            \vZero & \hdots & \vA_k(\beta\inv \vI + \vA_k\adj \vA_k)\inv &
            \hdots & \vZero
        \end{bmatrix}
        \vPi \vDelta\adj \vU_k \vU_k\adj \vDelta\vPi\adj
        \\
        &=
        \begin{bmatrix}
            \vZero & \hdots & \vA_k(\beta\inv \vI + \vA_k\adj \vA_k)\inv &
            \hdots & \vZero
        \end{bmatrix}
        \begin{bmatrix}
            \vDelta_1\adj \\
            \vdots\\
            \vDelta_K\adj
        \end{bmatrix}
        \vU_k \vU_k\adj
        \begin{bmatrix}
            \vDelta_1\adj \\
            \vdots\\
            \vDelta_K\adj
        \end{bmatrix}\adj
        \\
        &=
        \vA_k(\beta\inv \vI + \vA_k\adj \vA_k)\inv \vDelta_k\adj \vU_k
        \vU_k\adj
        \begin{bmatrix}
            \vDelta_1 &
            \hdots&
            \vDelta_K
        \end{bmatrix}.
    \end{align}
    We then multiply this term by $\vD_k\inv$ on the right to get the term
    that appears in $\vM_k$ (ignoring the multiplication on the right by
    $\vPi$, because it does not change operator norms). 
    In particular, we will control
    \begin{equation}
        \norm*{
            \vA_k(\beta\inv \vI + \vA_k\adj \vA_k)\inv \vDelta_k\adj \vU_k
            \vU_k\adj
            \begin{bmatrix}
                \vDelta_1 &
                \hdots&
                \vDelta_K
            \end{bmatrix}\vD_k\inv
        },
    \end{equation}
    showing that this term is small.
    We will accomplish this with the block diagonal structure of $\vD_k$:
    indeed, this gives that $\vD_k\inv$ is obtained by blockwise inversion,
    and hence
    \begin{align}
        &\norm*{
            \vA_k(\beta\inv \vI + \vA_k\adj \vA_k)\inv \vDelta_k\adj \vU_k
            \vU_k\adj
            \begin{bmatrix}
                \vDelta_1 &
                \hdots&
                \vDelta_K
            \end{bmatrix}\vD_k\inv
        }
        \\
        &\quad=
        \scalebox{0.9}{\(\norm*{
            \vA_k(\beta\inv \vI + \vA_k\adj \vA_k)\inv \vDelta_k\adj \vU_k
            \vU_k\adj
            \begin{bmatrix}
                \beta\vDelta_1
                &\hdots&
                \vDelta_k \left( \beta\inv \vI + \vA_k\adj \vA_k \right)\inv
                &\hdots&
                \beta \vDelta_K
            \end{bmatrix}
        }\)}
        \\
        &\quad\leq
        \sqrt{K}\max \Bigl\{
            \norm*{
                \vA_k(\beta\inv \vI + \vA_k\adj \vA_k)\inv \vDelta_k\adj \vU_k
                \vU_k\adj\vDelta_{k}(\beta\inv \vI + \vA_k\adj \vA_k)\inv
            },
            \\
            &\quad\qquad\qquad
            \max_{k' \neq k}\,
            \beta \norm*{
                \vA_k(\beta\inv \vI + \vA_k\adj \vA_k)\inv \vDelta_k\adj \vU_k
                \vU_k\adj\vDelta_{k'}
            }
            \Bigr\},
    \end{align}
    where the second line uses \Cref{eq:control-op-by-blocks}.
    We will give a coarse control of this term---the error could be improved
    further by exploiting more thoroughly independence of the blocks
    $\vDelta_k$ to show that the maximum over $k' \neq k$ in the last line
    of the preceding expression is smaller.
    We have by submultiplicativity of the operator norm
    and the triangle inequality
    \begin{align}
        &\norm*{
            \vA_k(\beta\inv \vI + \vA_k\adj \vA_k)\inv \vDelta_k\adj \vU_k
            \vU_k\adj\vDelta_{k}(\beta\inv \vI + \vA_k\adj \vA_k)\inv
        }\\
        \leq&
        \left(
        \frac{1}{1 + \beta \inv}
        +
        \frac{C \sqrt{n/d}}{1 + \beta \inv}
        \right)^2
        \norm*{
            \vDelta_k\adj \vU_k \vU_k\adj\vDelta_{k}
        }
        \\
        \leq&
        \left(1 + C \sqrt{n/d} \right)
        \norm*{
            \vDelta_k\adj \vU_k \vU_k\adj\vDelta_{k}
        },
    \end{align}
    where the first line uses \Cref{lem:blockwise-pinv}, and the second line
    uses \Cref{ass:parameter_config} to simplify the residual as above.
    We have meanwhile
    from the definition of $E^\star$
    \begin{align}
        \norm*{
            \vDelta_k\adj \vU_k \vU_k\adj\vDelta_{k}
        }
        &\leq
        \norm*{
            \vDelta_k
        }^2
        \ltsim \sigma^2\left(1 + \sqrt{n/d}\right),
    \end{align}
    because $\vU_k\vU_k\adj$ is an orthogonal projection, and again using
    \Cref{ass:parameter_config} to simplify the residual.
    We can argue analogously to simplify the other term in the maximum
    appearing above, and this yields
    \begin{align}
        &\norm*{
            \vA_k(\beta\inv \vI + \vA_k\adj \vA_k)\inv \vDelta_k\adj \vU_k
            \vU_k\adj
            \begin{bmatrix}
                \vDelta_1 &
                \hdots&
                \vDelta_K
            \end{bmatrix}\vD_k\inv
        }
        \\
        &\quad\ltsim
        \sqrt{K}\beta\sigma^2 \left(1 + C \sqrt{n/d} \right) \left( 1 +
        \sqrt{n/d} \right),
    \end{align}
    where we used the fact that $\veps \leq 1$ and the rest of
    \Cref{ass:parameter_config} implies that $\beta \geq 1$.
    This residual simplifies using \Cref{ass:parameter_config} to
    \begin{equation}
        \norm*{
            \vA_k(\beta\inv \vI + \vA_k\adj \vA_k)\inv \vDelta_k\adj \vU_k
            \vU_k\adj
            \begin{bmatrix}
                \vDelta_1 &
                \hdots&
                \vDelta_K
            \end{bmatrix}\vD_k\inv
        }
        \ltsim
        \sqrt{K}\beta\sigma^2 \left(1 + C \sqrt{n/d} \right).
    \end{equation}

    next, we examine the last term, which is the other summand arising in
    the definition of $\vXi_k$. We have
    \begin{align}
        &\scalebox{0.85}{\(\begin{bmatrix}
            \vZero & \hdots & \vA_k(\beta\inv \vI + \vA_k\adj \vA_k)\inv &
            \hdots & \vZero
        \end{bmatrix}
        \begin{bmatrix}
            \vZero & \hdots & \vDelta_1\adj \vU_k \vA_k & \hdots & \vZero \\
            \vdots & & \vdots & & \vdots \\
            \vA_k\adj \vU_k\adj \vDelta_1 & \hdots & \vDelta_k\adj \vU_k \vA_k +
            \vA_k\adj \vU_k\adj \vDelta_k & \hdots & \vA_k\adj \vU_k\adj \vDelta_K\\
            \vdots & & \vdots & & \vdots \\
            \vZero & \hdots & \vDelta_K\adj \vU_k \vA_k & \hdots & \vZero \\
        \end{bmatrix}
        \)}
        \\
        &=
        \scalebox{0.95}{\(\vA_k(\beta\inv \vI + \vA_k\adj \vA_k)\inv
        \begin{bmatrix}
            \vA_k\adj \vU_k\adj \vDelta_1
            & \hdots 
            & 
            \left(
            \vDelta_k\adj \vU_k \vA_k + \vA_k\adj\vU_k\adj \vDelta_k
            \right)
            & \hdots 
            &%
            \vA_k\adj \vU_k\adj \vDelta_K
        \end{bmatrix}\)}.
    \end{align}
    Now multiplying on the right by $\vD_k\inv$ gives the term (again
    ignoring the right-multiplication by $\vPi$, which does not affect the
    operator norm)
    \begin{equation}
        \scalebox{0.8}{\(\vA_k(\beta\inv \vI + \vA_k\adj \vA_k)\inv
        \begin{bmatrix}
            \beta
            \vA_k\adj \vU_k\adj \vDelta_1
            & \hdots 
            & 
            \left(
            \vDelta_k\adj \vU_k \vA_k + \vA_k\adj\vU_k\adj \vDelta_k
            \right)
            \left(
            \beta\inv \vI + \vA_k \adj \vA_k
            \right)\inv
            & \hdots 
            &%
            \beta \vA_k\adj \vU_k\adj \vDelta_K
        \end{bmatrix}\)}.
        \label{eq:other-term-of-xik}
    \end{equation}
    We will argue that this term is close to the term
    \begin{equation}
        \vA_k(\beta\inv \vI + \vA_k\adj \vA_k)\inv
        \begin{bmatrix}
            \beta
            \vA_k\adj \vU_k\adj \vDelta_1
            & \hdots 
            & 
            \vZero
            & \hdots 
            &%
            \beta \vA_k\adj \vU_k\adj \vDelta_K
        \end{bmatrix}
        \label{eq:other-term-of-xik-nom}
    \end{equation}
    in operator norm. The argument is similar to the preceding arguments:
    for this, it suffices to bound
    \begin{equation}
        \scalebox{0.9}{\(\norm*{
            \begin{bmatrix}
                \vZero
                & \hdots 
                & 
                \vA_k(\beta\inv \vI + \vA_k\adj \vA_k)\inv
                \left(
                \vDelta_k\adj \vU_k \vA_k + \vA_k\adj\vU_k\adj \vDelta_k
                \right)
                \left(
                \beta\inv \vI + \vA_k \adj \vA_k
                \right)\inv
                & \hdots 
                &%
                \vZero
            \end{bmatrix}
        }\)},
    \end{equation}
    which is the same as controlling the operator norm of the nonzero block.
    Using submultiplicativity and \Cref{lem:blockwise-pinv} along with the
    simplifications we have done above leveraging
    \Cref{ass:parameter_config}, we obtain
    \begin{align}
        &\norm*{
            \vA_k(\beta\inv \vI + \vA_k\adj \vA_k)\inv
            \left(
            \vDelta_k\adj \vU_k \vA_k + \vA_k\adj\vU_k\adj \vDelta_k
            \right)
            \left(
            \beta\inv \vI + \vA_k \adj \vA_k
            \right)\inv
        }
        \\
        &\quad\leq
        \left( 1 + C \sqrt{n/d} \right)
        \norm*{
            \vDelta_k\adj \vU_k \vA_k + \vA_k\adj\vU_k\adj \vDelta_k
        }.
    \end{align}
    Meanwhile, on $E^\star$ we have the operator norm of $\vDelta_k$ and
    $\vA_k$ controlled, using again
    \Cref{eq:blocks-opnorm-controlled-overall}. Applying then 
    the triangle inequality and submultiplicativity, we obtain
    \begin{equation}
        \norm*{
            \vDelta_k\adj \vU_k \vA_k + \vA_k\adj\vU_k\adj \vDelta_k
        }
        \ltsim
        \sigma
        \left(
        1 + \sqrt{n/d}
        \right),
    \end{equation}
    again simplifying with \Cref{ass:parameter_config}.
    This shows that \Cref{eq:other-term-of-xik} is close to
    \Cref{eq:other-term-of-xik-nom}
    with deviations of the order $\ltsim \sigma(1 + \sqrt{n/d})$.

    \paragraph{Aggregating the previous results.}
    Combining our perturbation analysis above, we have established control
    \begin{align}
        &\Bigl\|
        \vM_k
        -
        \Bigl[
            \left(
            \vI - 
            \vA_k(\beta\inv \vI + \vA_k\adj \vA_k)\inv\vA_k\adj
            \right)
            \vU_k\adj
            \begin{bmatrix}
                \beta \vDelta_1
                &\hdots&
                \vZero
                &\hdots&
                \beta \vDelta_K
            \end{bmatrix}
            \\
            &\quad+
            \begin{bmatrix}
                \vZero & \hdots & \vA_k(\beta\inv \vI + \vA_k\adj
                \vA_k)\inv & \hdots & \vZero
            \end{bmatrix}
            \Bigr]\vPi
        \Bigr\|
        \\
        &\ltsim
        \sigma(1 + \sqrt{n/d})
        +
        \sqrt{K} \beta \sigma^2(1 + \sqrt{n/d} ).
    \end{align}
    It is convenient to include one additional stage of simplification here:
    namely, we use \Cref{lem:blockwise-pinv} once more to simplify the
    second term in the nominal value of $\vM_k$ appearing here.
    namely, we have (arguing as we have above, once again)
    \begin{align}
        &\norm*{
            \begin{bmatrix}
                \vZero & \hdots & \vA_k(\beta\inv \vI + \vA_k\adj
                \vA_k)\inv & \hdots & \vZero
            \end{bmatrix}
            -
            \begin{bmatrix}
                \vZero & \hdots & \frac{1}{1+\beta\inv}\vA_k & \hdots & \vZero
            \end{bmatrix}
        }
        \\
        &\quad=
        \norm*{
            \frac{1}{1+\beta\inv}\vA_k - \vA_k\left(
            \beta\inv \vI + \vA_k\adj \vA_k
            \right)\inv
        }
        \\
        &\quad\ltsim
        \sqrt{n/d},
    \end{align}
    from which it follows
    \begin{align}
        &\Bigl\|
        \vM_k
        -
        \Bigl[
            \left(
            \vI - 
            \vA_k(\beta\inv \vI + \vA_k\adj \vA_k)\inv\vA_k\adj
            \right)
            \vU_k\adj
            \begin{bmatrix}
                \beta \vDelta_1
                &\hdots&
                \vZero
                &\hdots&
                \beta \vDelta_K
            \end{bmatrix}
            \\
            &\quad+
            \begin{bmatrix}
                \vZero & \hdots & \frac{1}{1+\beta\inv}\vA_k & \hdots & \vZero
            \end{bmatrix}
            \Bigr]\vPi
        \Bigr\|
        \\
        &\ltsim
        \sigma(1 + \sqrt{n/d})
        +
        \sqrt{K} \beta \sigma^2(1 + \sqrt{n/d} ) + \sqrt{n/d}.
    \end{align}
    Meanwhile, recall the residual of scale $\ltsim (\sigma \beta)^2$
    arising when we controlled the gradient term around $\vM_k$:
    \begin{equation}
        \norm*{
            \vU_k\adj \vZ_{\base}
            \left(
            \beta\inv \vI + (\vU_k\adj \vZ_{\base})\adj \vU_k\adj
            \vZ_{\base}
            \right)\inv
            - \vM_k
        }
        \leq C (\beta \sigma)^2.
    \end{equation}
    Combining these two bounds with the triangle inequality controls the
    gradient term around its nominal value. Now, we sum these errors over
    $k$ (again with the triangle inequality) to obtain control of the
    aggregate gradient around its nominal value.
    We introduce notation to concisely capture the
    accumulations of the (approximate) orthogonal projections arising in the
    nominal value of the main term: for each $k \in [K]$, define
    \begin{equation}
        \sP_k = 
        \sum_{k' \neq k}
        \vU_{k'}
        \left(
        \vI - 
        \vA_{k'}(\beta\inv \vI + \vA_{k'}\adj \vA_{k'})\inv\vA_{k'}\adj
        \right)
        \vU_{k'}\adj,
        \label{eq:proj-defn-1}
    \end{equation}
    and define an overall (approximate) projection operator (which acts on block matrices
    partitioned compatibly with the class sizes $n_k$, as in
    \Cref{eq:signal-model-block-form}) by
    \begin{equation}
        \sP_{U_{[K]}}
        =
        \begin{bmatrix}
            \sP_1 & \hdots & \sP_K   
        \end{bmatrix}.
        \label{eq:proj-defn-2}
    \end{equation}
    Then the above argument implies
    \begin{align}
        &\;\norm*{
            \nabla_{\vZ} R^c(\vZ_{\base} \mid \vU_{[K]})
            -
            \sP_{\vU_{[K]}}(\beta \vDelta \vPi\adj) \vPi
            - \frac{1}{1 + \beta\inv}\vX_{\natural}
        }
        \\
        \ltsim&\;
        K\left(
        \sigma^2 \beta^2 + \sigma(1 + \sqrt{n/d}) + \sqrt{K} \beta
        \sigma^2(1 + \sqrt{n/d}) + \sqrt{n/d}
        \right),
    \end{align}
    which is enough to conclude.

\end{proof}

\subsubsection{Key Auxiliary Lemmas}

In this section we state and prove two key concentration inequalities that are used in the proof of the main theorem. They rely on simpler results which will be conveyed in subsequent subsections.

\begin{lemma}\label{lem:blockwise-pinv}
    There exist universal constants \(C, C' > 0\) such that the following holds. Let \(d, p, n, K \in \N\) be such that \Cref{ass:parameter_config} holds. Let \(\vA_{k}\), \(k \in [K]\), be defined as above. Let \(E^{\star}\) be the good event defined in \Cref{eq:good_event}. If \(E^{\star}\) occurs, then for \(k \in [K]\) we have
    \begin{equation}
        \norm*{(\beta^{-1}\vI + \vA_{k}^{\top}\vA_{k})^{-1} - \frac{1}{1 +
        \beta^{-1}} \vI} 
        \leq \frac{C\sqrt{n/d}}{(1 + \beta^{-1})}.
        \label{eq:pseudoinverse-ipart-small}
    \end{equation}
    and in addition
    \begin{equation}
        \norm*{\vA_{k}(\beta^{-1}\vI + \vA_{k}^{\top}\vA_{k})^{-1} - \frac{1}{1
        + \beta^{-1}}\vA_{k}} \leq \frac{C'\sqrt{n/d}}{(1 + \beta^{-1})}.
        \label{eq:pseudoinverse-small}
    \end{equation}
\end{lemma}
\begin{proof}
    Since \(E^{\star}\) holds, for all \(k \in [K]\) we have
    \begin{equation}
        \norm{\vA_{k}} \leq 1 + C_{1}\sqrt{n/d}, \qquad \norm{\vA_{k}^{\top}\vA_{k} - \vI} \leq C_{2}\sqrt{n/d}.
    \end{equation}
    By \Cref{ass:parameter_config}, we have \(\norm{\vA_{k}^{\top}\vA_{k} - \vI} < 1\), so \(\vA_{k}^{\top}\vA_{k}\) is well-conditioned. Write
    \begin{equation}
        \vXi \doteq \vA_{k}^{\top}\vA_{k} - \vI,
    \end{equation}
    so that
    \begin{align}
        (\beta^{-1}\vI + \vA_{k}^{\top}\vA_{k})^{-1}
        &= ((1 + \beta^{-1})\vI + \vXi)^{-1} \\
        &= \frac{1}{1 + \beta^{-1}}\left(\vI + \frac{1}{1 + \beta^{-1}}\vXi\right)^{-1} \\ 
        &= \frac{1}{1 + \beta^{-1}}\sum_{j = 0}^{\infty}\left(-\frac{1}{1 + \beta^{-1}}\right)^{j}\vXi^{j} \\
        &= \frac{1}{1 + \beta^{-1}}\vI + \frac{1}{1 + \beta^{-1}}\sum_{j = 1}^{\infty}\left(-\frac{1}{1 + \beta^{-1}}\right)^{j}\vXi^{j}
    \end{align}
    by the Neumann series. This gives us
    \begin{align}
        \norm*{(\beta^{-1}\vI + \vA_{k}^{\top}\vA_{k})^{-1} - \frac{1}{1 + \beta^{-1}}\vI}
        &= \norm*{\frac{1}{1 + \beta^{-1}}\sum_{j = 1}^{\infty}\left(-\frac{1}{1 + \beta^{-1}}\right)^{j}\vXi^{j}} \\
        &\leq \frac{1}{1 + \beta^{-1}}\sum_{j = 1}^{\infty}\left(\frac{1}{1 + \beta^{-1}}\right)^{j}\norm{\vXi}^{j} \\ 
        &\leq \frac{1}{1 + \beta^{-1}}\sum_{j =
        1}^{\infty}\left(\frac{C_{2}\sqrt{n/d}}{1 + \beta^{-1}}\right)^{j} \\
        &= \frac{C_{2}\sqrt{n/d}}{(1 + \beta^{-1})(1 + \beta^{-1} - C_{2}\sqrt{n/d})}.
    \end{align}
    Meanwhile, by \Cref{ass:parameter_config}, it holds
    \begin{align}
        C_2 \sqrt{n/d} \leq \sqrt{1/6},
    \end{align}
    so it follows
    \begin{equation}
        \frac{C_{2}\sqrt{n/d}}{1 + \beta^{-1} - C_{2}\sqrt{n/d}}
        \leq
        2C_2 \sqrt{n/d}.
    \end{equation}
    By the submultiplicativity of the operator norm, we thus have
    \begin{align}
        \norm*{\vA_{k}(\beta^{-1}\vI + \vA_{k}^{\top}\vA_{k})^{-1} - \frac{1}{1 + \beta^{-1}}\vA_{k}}
        &\leq \norm{\vA_{k}}\norm*{(\beta^{-1}\vI + \vA_{k}^{\top}\vA_{k})^{-1} - \frac{1}{1 + \beta^{-1}}\vI} \\
        &\leq \frac{[1 + C_{1}\sqrt{n/d}]C_{2}\sqrt{n/d}}{(1 + \beta^{-1})(1 + \beta^{-1} - C_{2}\sqrt{n/d})} \\ 
        &\leq 2\frac{[1 + C_{1}\sqrt{n/d}]C_{2}\sqrt{n/d}}{1 + \beta^{-1}} \\
        &= 2\frac{C_{2}\sqrt{n/d} + C_{1}C_{2}n/d}{1 + \beta^{-1}}.
    \end{align}
    By \Cref{ass:parameter_config}, we have that there exists some absolute constant \(C_{3} > 0\) with \(C_{3} \cdot n/d \leq \sqrt{n/d}\), which gives 
    \begin{equation}
        \norm*{\vA_{k}(\beta^{-1}\vI + \vA_{k}^{\top}\vA_{k})^{-1} - \frac{1}{1
        + \beta^{-1}}\vA_{k}} \leq 2\frac{(C_{2} + C_{1}C_{2}C_{3}^{-1}) \sqrt{n/d}}{1 + \beta^{-1}},
    \end{equation}
    as desired.
\end{proof}

\begin{lemma}\label{lem:neumann-op-norm}
    There exist universal constants \(C_{1}, C_{2} > 0\) such that the following holds. Let \(d, p, n, K \in \N\) be such that \Cref{ass:parameter_config} holds. Let \(\vA_{k}\), \(k \in [K]\), be defined as above. Let \(\vD_{k}\) be defined as in \Cref{eq:ak-defn}. Let \(\vXi_{k}\) be defined as in \Cref{eq:xik-defn}. Let \(E^{\star}\) be the good event defined in \Cref{eq:good_event}. If \(E^{\star}\) occurs, then for \(k \in [K]\) we have
    \begin{equation}
        \norm{\vXi_{k}\vD_{k}}^{-1} \leq C_{1}\beta [\sigma^{2} + \sigma(C_{2} + \sqrt{n/d})].
    \end{equation}
\end{lemma}
\begin{proof}
    Since we have
    \begin{equation}
        \vD_{k} = \mat{\beta^{-1} \vI & & & & \\ & \ddots & & & \\ & & \beta^{-1}\vI + \vA_{k}^{\top}\vA_{k} & & \\ & & & \ddots & \\ & & & & \beta^{-1} \vI}
    \end{equation}
    it holds that 
    \begin{equation}
        \vD_{k}^{-1} = \mat{\beta \vI & & & & \\ & \ddots & & & \\ & & (\beta^{-1}\vI + \vA_{k}^{\top}\vA_{k})^{-1} & & \\ & & & \ddots & \\ & & & & \beta \vI}.
    \end{equation}
    We will use the straightforward estimate \(\norm{\vXi_{k}\vD_{k}^{-1}} \leq \norm{\vXi_{k}}\norm{\vD_{k}^{-1}}\) and bound the two matrices' operator norms individually. By the previous expression, 
    \begin{equation}
        \norm{\vD_{k}^{-1}} = \max\{\beta, \norm{(\beta^{-1}\vI + \vA_{k}^{\top}\vA_{k})^{-1}}\} \leq \beta,
    \end{equation}
    because \(\vA_{k}^{\top}\vA_{k} \succeq \vZero\), so we need only control the operator norm of \(\vXi_{k}\).
    To this end, note the convenient expression
    \begin{equation}
        \vXi_k
        =
        \vPi \vDelta\adj \vU_k \vU_k \adj \vDelta \vPi\adj
        +
        2 \Sym\left(
            (\vDelta \vPi\adj)\adj
            \begin{bmatrix}
                \vZero & \hdots & \vU_k \vA_k & \hdots & \vZero
            \end{bmatrix}
        \right),
    \end{equation}
    where $\Sym(\spcdot)$ denotes the symmetric part operator.
    By the triangle inequality, the operator norm of $\vXi_k$ is no larger
    than the sum of the operator norms of each term in the previous
    expression.
    The operator norm of the first term is no larger than
    $\norm{\vDelta}^2$, because $\vPi$ is a permutation matrix and $\vU
    \vU_k\adj$ is an orthogonal projection.
    Meanwhile, using that the symmetric part operator is the orthogonal
    projection onto the space of symmetric matrices, it follows
    \begin{equation}
        \scalebox{0.9}{\(\norm*{
            2 \Sym\left(
                (\vDelta \vPi\adj)\adj
                \begin{bmatrix}
                    \vZero & \hdots & \vU_k \vA_k & \hdots & \vZero
                \end{bmatrix}
            \right)
        }
        \leq
        2\norm*{
            (\vDelta \vPi\adj)\adj
            \begin{bmatrix}
                \vZero & \hdots & \vU_k \vA_k & \hdots & \vZero
            \end{bmatrix}
        }\)},
    \end{equation}
    and then we find as above that the RHS is no larger than $2
    \norm{\vDelta} \norm{\vA_k}$. Since the good event \(E^{\star}\) defined in \Cref{eq:good_event} holds by assumption, we have that there are constants \(C_{1}, C_{2} > 0\) such that 
    \begin{align}
        \norm{\vDelta} 
        &\leq \sigma \bp{C_{1} + \sqrt{\frac{n}{d}}} \\ 
        \norm{\vA_{k}}
        &\leq 1 + C_{2}\sqrt{\frac{n}{d}}.
    \end{align}
    By \Cref{ass:parameter_config} we have \(d \geq n\), so that \(\sqrt{n/d} \leq 1\). Therefore we have
    \begin{equation}
        \norm{\vDelta} \leq \sigma\bp{C_{1} + 1} = C_{3}\sigma
    \end{equation}
    for \(C_{3} \doteq C_{1} + 1\) another universal constant. Thus on this good event we have
    \begin{equation}
        2\norm{\vDelta}\norm{\vA_{k}} \leq C_{3}\sigma\bp{1 + C_{2}\sqrt{n/d}}.
    \end{equation}
    Therefore, we have
    \begin{align}
        \norm{\vXi_{k}}
        &\leq \norm{\vDelta}^{2} + 2\norm{\vDelta}\norm{\vA_{k}} \\
        &\leq C_{3}^{2}\sigma^{2} + C_{3}\sigma(1 + C_{2}\sqrt{n/d}) \\
        &\leq C_{4}[\sigma^{2} + \sigma(1 + C_{2}\sqrt{n/d})]
    \end{align}
    where \(C_{4} = \max\{C_{3}, C_{3}^{2}\}\) is another universal constant. Thus on \(E^{\star}\) we have
    \begin{equation}
        \norm{\vXi_{k}\vD_{k}^{-1}} \leq C_{4}\beta[\sigma^{2} + \sigma(1 + C_{2}\sqrt{n/d})] \leq C_{5}\beta [\sigma^{2} + \sigma(1 + \sqrt{n/d})]
    \end{equation}
    for \(C_{5} > 0\) another obvious universal constant.
\end{proof}

\subsubsection{Concentration Inequalities for Our Setting}

In this section we prove some simple concentration inequalities that are adapted to the problem setting. These results are used to prove the key lemmata above, and indeed are also invoked in the main theorem. They follow from even simpler concentration inequalities that are abstracted away from the problem setting, which we discuss in the following subsections.

\begin{proposition}\label{prop:delta_concentration_opnorm}
    There are universal constants \(C_{1}, C_{2}, C_{3} > 0\) such that the following holds. Let \(d, n \in \N\) be such that \Cref{ass:parameter_config} holds. Let \(\vDelta \in \R^{d \times n}\) be defined as above. Then
    \begin{equation}
        \Pr*{\norm*{\vDelta}_{\op} > \sigma\bp{C_{1} + \sqrt{\frac{n}{d}}}} \leq C_{2}e^{-C_{3}d}.
    \end{equation}
\end{proposition}
\begin{proof}
    We use \Cref{prop:gaussian_concentration_opnorm} with the parameters \(q = d\), \(n = n\), and \(x = \sigma/\sqrt{d}\), which obtains
    \begin{equation}
        \Pr{\norm{\vDelta}_{\op} > s} \leq C_{1}\exp\rp{-d\bc{\frac{s\sqrt{d}/\sigma - \sqrt{n}}{C_{2}\sqrt{d}} - 1}^{2}}, \qquad \forall s > \frac{\sigma}{\sqrt{d}}(\sqrt{n} + C_{2}\sqrt{d})
    \end{equation}
    notice that we have
    \begin{equation}
        \frac{s\sqrt{d}/\sigma - \sqrt{n}}{C_{2}\sqrt{d}} - 1 = \frac{1}{C_{2}}\bp{\frac{s}{\sigma} - \sqrt{\frac{n}{d}}} - 1, \qquad \frac{\sigma}{\sqrt{d}}(\sqrt{n} + C_{2}\sqrt{d}) = \sigma \bp{\sqrt{\frac{n}{d}} + C_{2}}.
    \end{equation}
    To make the squared term equal to \(1\), we pick
    \begin{equation}
        s = \sigma \bp{\sqrt{\frac{n}{d}} + 2C_{2}},
    \end{equation}
    which gives
    \begin{equation}
        \Pr*{\norm{\vDelta}_{\op} > \sigma\bp{2C_{2} + \sqrt{\frac{n}{d}}}} \leq C_{2}e^{-d}.
    \end{equation}
\end{proof}

\begin{proposition}\label{prop:xk_concentration_opnorm}
    There are universal constants \(C_{1}, C_{2}, C_{3}, C_{4} > 0\) such that the following holds. Let \(p, n, K \in \N\) be such that \Cref{ass:parameter_config} holds. Let \(\vA_{k}\), \(k \in [K]\), be defined as above. Then
    \begin{equation}
        \Pr*{\norm{\vA_{k}}_{\op} > 1 + C_{1}\sqrt{\frac{n}{d}}} \leq C_{2}\exp\rp{-C_{3}\frac{n}{K}} + \frac{C_{4}}{n^{2}}
    \end{equation}
\end{proposition}
\begin{proof}
    By \Cref{prop:binom_concentration,prop:gaussian_concentration_opnorm} with parameters \(n = n\), \(k = K\), \(q = p\), and \(x = 1/\sqrt{p}\), if we define 
    \begin{equation}
        n_{\min} \doteq \floor*{\frac{n}{K} - C_{1}\sqrt{n\log n}}, \qquad
        n_{\max} \doteq \ceil*{\frac{n}{K} + C_{1}\sqrt{n \log n}}
    \end{equation}
    then we have
    \begin{align}
        \Pr{\norm{\vA_{k}}_{\op} > s}
        &\leq \sum_{n = n_{\min}}^{n_{\max}}\Pr{\norm{\vA_{k}}_{\op} > s \mid K_{k} = n}\Pr{K_{k} = n} + \frac{C_{2}}{n^{2}} \\
        &\leq \sum_{n = n_{\min}}^{n_{\max}}C_{3}\exp\rp{-n\bc{\frac{s\sqrt{p} - \sqrt{p}}{C_{4}\sqrt{4}} - 1}^{2}}\Pr{K_{k} = n} + \frac{C_{2}}{n^{2}},
    \end{align}
    for all \(s\) obeying
    \begin{align}
        s 
        &\geq \frac{1}{\sqrt{p}}\bp{\sqrt{p} + C_{4}\sqrt{n_{\max}}} \\ 
        &= 1 + C_{4}\sqrt{\frac{n_{\max}}{p}}.
    \end{align}
    Thus we have that the concentration holds for all \(s\) obeying
    \begin{equation}
        s \geq 1 + C_{4}\sqrt{\frac{n_{\max}}{p}}.
    \end{equation}
    In order to cancel out the most interior terms, we choose
    \begin{equation}
        s = 1 + 2C_{4}\sqrt{\frac{n_{\max}}{p}}.
    \end{equation}
    This choice obtains
    \begin{align}
        \Pr{\norm{\vA_{k}}_{\op} > s}
        &\leq \sum_{n = n_{\min}}^{n_{\max}}C_{3}\exp\rp{-n\bc{\frac{s\sqrt{p} - \sqrt{p}}{C_{4}\sqrt{n}} - 1}^{2}}\Pr{K_{k} = n} + \frac{C_{2}}{n^{2}} \\
        &= \sum_{n = n_{\min}}^{n_{\max}}C_{3}\exp\rp{-n\underbrace{\bc{2\underbrace{\sqrt{\frac{n_{\max}}{n}}}_{\geq 1} - 1}^{2}}_{\geq 1}}\Pr{K_{k} = n} + \frac{C_{2}}{n^{2}} \\
        &\leq \sum_{n = n_{\min}}^{n_{\max}}C_{3}\exp\rp{-n}\Pr{K_{k} = n} + \frac{C_{2}}{n^{2}} \\
        &\leq \sum_{n = n_{\min}}^{n_{\max}}C_{3}\exp\rp{-n_{\min}}\Pr{K_{k} = n} + \frac{C_{2}}{n^{2}} \\
        &\leq C_{3}\exp(-n_{\min}) + \frac{C_{2}}{n^{2}} \\ 
        &= C_{3}\exp\rp{-\frac{n}{K} + C_{1}\sqrt{n \log n}} + \frac{C_{2}}{n^{2}} \\
        &\leq C_{3}\exp\rp{-\frac{n}{K} + \frac{1}{2}\sqrt{n \log n}} + \frac{C_{2}}{n^{2}} \\ 
        &\leq C_{3}\exp\rp{-\frac{1}{2}\cdot \frac{n}{K}} + \frac{C_{2}}{n^{2}}.
    \end{align}
    To obtain the conclusion of the theorem, note that any \(s\) such that
    \begin{equation}
        s \geq 1 + 2C_{4}\sqrt{\frac{n_{\max}}{p}} = 1 + 2C_{4}\sqrt{\frac{n/K}{p} + C_{1}\frac{\sqrt{n \log n}}{p}} = 1 + 2C_{4}\sqrt{\frac{n}{d} + C_{1}\frac{\sqrt{n \log n}}{p}}
    \end{equation}
    enjoys the same high-probability bound. By \Cref{ass:parameter_config}, we have
    \begin{align}
        &\;1 + 2C_{4}\sqrt{\frac{n}{d} + C_{1}\frac{\sqrt{n \log n}}{p}} \leq 1 +
          2C_{4}\sqrt{\frac{n}{d} + \frac{1}{2} \cdot \frac{n/K}{p}}\\
          = &\;1 +
          2C_{4}\sqrt{\frac{n}{d} + \frac{1}{2} \cdot \frac{n}{d}} = 1 + 2C_{4}\sqrt{\frac{3}{2}}\cdot \sqrt{\frac{n}{d}}
    \end{align}
    whence the ultimate conclusion is obtained by combining constants.
\end{proof}

\begin{proposition}\label{prop:xktxk_concentration_opnorm}
    There are universal constants \(C_{1}, C_{2}, C_{3}, C_{4} > 0\) such that the following holds. Let \(p, n, K \in \N\) be such that \Cref{ass:parameter_config} holds. Let \(\vA_{k}\), \(k \in [K]\), be defined as above. Then
    \begin{equation}
        \Pr*{\norm{\vA_{k}^{\top}\vA_{k} - \vI}_{\op} > C_{1}\sqrt{\frac{n}{d}}} \leq C_{2}\exp\rp{-C_{3}\frac{n}{K}} + \frac{C_{4}}{n^{2}}.
    \end{equation}
\end{proposition}
\begin{proof}
    By \Cref{prop:binom_concentration,prop:gaussian_covariance_concentration_opnorm} with parameters \(n = n\), \(k = K\), \(q = p\), and \(x = 1/\sqrt{p}\), if we define 
    \begin{equation}
        n_{\min} \doteq \floor*{\frac{n}{K} - C_{1}\sqrt{n\log n}}, \qquad 
        n_{\max} \doteq \ceil*{\frac{n}{K} + C_{1}\sqrt{n \log n}}
    \end{equation}
    then we have
    \begin{align}
        &\Pr{\norm{\vA_{k}^{\top}\vA_{k} - \vI}_{\op} > s} \\
        &\leq \sum_{n = n_{\min}}^{n_{\max}}\Pr{\norm{\vA_{k}^{\top}\vA_{k} - \vI}_{\op} > s \mid K_{k} = n}\Pr{K_{k} = n} + \frac{C_{2}}{n^{2}} \\
        &\leq \frac{C_{2}}{n^{2}} + \sum_{n = n_{\min}}^{n_{\max}}\Pr{K_{k} = n}\cdot \scalebox{0.8}{\(\casework{C_{3}\exp\rp{-n\bc{\frac{1}{C_{4}^{2}C_{5}\sqrt{n/p}}s - 1}^{2}}, & \text{if}\ C_{4}^{2}C_{5}\sqrt{n/p} \leq s \leq C_{4}^{2} \\ C_{3}\exp\rp{-n\bc{\frac{1}{C_{4}C_{5}\sqrt{n/p}}\sqrt{s} - 1}^{2}}, & \text{if}\ s \geq C_{4}^{2}.}\)}
    \end{align}
    In order to cancel the most terms, we choose
    \begin{equation}
        s = 2C_{4}^{2}C_{5}\sqrt{\frac{n_{\max}}{p}}.
    \end{equation}
    In order to assure ourselves that this choice still has \(s \leq C_{4}^{2}\) (so that we can use the first case for all \(n\)), we have
    \begin{align}
        s 
        &= 2C_{4}^{2}C_{5}\sqrt{\frac{n_{\max}}{p}} \\
        &= 2C_{4}^{2}C_{5}\sqrt{\frac{n/K + C_{1}\sqrt{n \log n}}{p}} \\ 
        &= 2C_{4}^{2}C_{5}\sqrt{\frac{n/K + \frac{1}{2}n/K}{p}} \\ 
        &= 2\sqrt{\frac{3}{2}}C_{4}^{2}C_{5}\cdot \sqrt{\frac{n/K}{p}} \\
        &= \sqrt{6}C_{4}^{2}C_{5}\cdot \sqrt{\frac{n}{d}} \\ 
        &\leq C_{4}^{2} \quad \text{when} \quad \sqrt{6}C_{5}\sqrt{\frac{n}{d}} \leq 1.
    \end{align}
    Of course, this condition is assured by \Cref{ass:parameter_config}. Now that we have this, we know \(s\) falls in the first, and so we have
    \begin{align}
        &\Pr*{\norm{\vA_{k}^{\top}\vA_{k} - \vI}_{\op} > C_{1}\sqrt{\frac{n}{d}}} \\
        &\leq \frac{C_{2}}{n^{2}} + \sum_{n = n_{\min}}^{n_{\max}}\Pr{K_{k} = n}\cdot \scalebox{0.8}{\(\casework{C_{3}\exp\rp{-n\bc{\frac{1}{C_{4}^{2}C_{5}\sqrt{n/p}}s - 1}^{2}}, & \text{if}\ C_{4}^{2}C_{5}\sqrt{n/p} \leq s \leq C_{4}^{2} \\ C_{3}\exp\rp{-n\bc{\frac{1}{C_{4}C_{5}\sqrt{n/p}}\sqrt{s} - 1}^{2}}, & \text{if}\ s \geq C_{4}^{2}}\)} \\
        &\leq \frac{C_{2}}{n^{2}} + \sum_{n = n_{\min}}^{n_{\max}}C_{3}\exp\rp{-n\bc{\frac{2C_{4}^{2}C_{5}\sqrt{n_{\max}/p}}{C_{4}^{2}C_{5}\sqrt{n/p}} - 1}^{2}}\Pr{K_{k} = n} \\ 
        &= \frac{C_{2}}{n^{2}} + \sum_{n = n_{\min}}^{n_{\max}}C_{3}\exp\rp{-n\bc{2\sqrt{\frac{n_{\max}}{n}} - 1}^{2}}\Pr{K_{k} = n} \\ 
        &\leq C_{3}\exp\rp{-\frac{1}{2} \cdot \frac{n}{K}} + \frac{C_{2}}{n^{2}}
    \end{align}
    where the last inequality follows from the exact same argument as in \Cref{prop:xk_concentration_opnorm}.
\end{proof}

\subsubsection{Generic Concentration Inequalities}

In this subsection we prove the base-level concentration inequalities used throughout the proofs in this paper.

\paragraph{Binomial concentration.}

\begin{proposition}\label{prop:binom_concentration}
    There exist universal constants \(C_{1}, C_{2} > 0\) such that the following holds. Let \(n, k \in \Z\). For each \(i \in [k]\), let \(B_{i} \sim \dBin(n, 1/k)\), such that the \(B_{i}\) are identically (marginally) distributed but not necessarily independent binomial random variables. Let \(E\) be an event. Then for any \(i \in [k]\), we have
    \begin{equation}
        \Pr{E} \leq \sum_{b = \floor*{n/k - C_{1}\sqrt{n \log n}}}^{\ceil*{n/k + C_{1}\sqrt{n \log n}}}\Pr{E \given B_{i} = b}\Pr{B_{i} = b} + \frac{C_{2}}{n^{2}}.
    \end{equation}
\end{proposition}
\begin{proof}
    We have
    \begin{equation}
        \Pr{E} = \E{\E{E \mid B_{i}}} = \sum_{b = 0}^{n}\Pr{E \mid B_{i} = b}\Pr{B_{i} = b}.
    \end{equation}
    Each \(B_{i}\) is unconditionally distributed as \(\dBin(n, 1/k)\). By union bound and Hoeffding's inequality \citep[Theorem 2.2.6]{vershynin2018high}, we have
    \begin{equation}
        \Pr*{\abs{B_{i} - n/k} \geq t} \leq 2\exp\rp{-\frac{2t^{2}}{n}}.
    \end{equation}
    Inverting this inequality obtains that there exists some (simple) universal constants \(C_{1}, C_{2} > 0\) such that 
    \begin{equation}
        \Pr*{\abs{B_{i} - n/k} \geq C_{1}\sqrt{n\log n}} \leq \frac{C_{2}}{n^{3}}.
    \end{equation}
    Thus, if we define
    \begin{equation}
        b_{\min} \doteq \floor*{\frac{n}{k} - C_{1}\sqrt{n \log n}}, \qquad b_{\max} \doteq \ceil*{\frac{n}{k} + C_{1}\sqrt{n \log n}},
    \end{equation}
    then we have
    \begin{align}
        \Pr{E}
        &= \sum_{b = 0}^{n}\Pr{E \mid B_{i} = b}\Pr{B_{i} = b} \\ 
        &= \sum_{b = 0}^{b_{\min} - 1}\underbrace{\Pr{E \mid B_{i} = b}}_{\leq 1}\underbrace{\Pr{B_{i} = b}}_{\leq C_{2}/n^{3}} + \sum_{b = b_{\max} + 1}^{n}\underbrace{\Pr{E \mid B_{i} = b_{1}}}_{\leq 1}\underbrace{\Pr{B_{i} = b}}_{\leq C_{2}/n^{3}} \\ 
        &\qquad + \sum_{b = b_{\min}}^{b_{\max}}\Pr{E \mid B_{i} = b}\Pr{B_{i} = b} \\ 
        &\leq \sum_{b = 0}^{b_{\min} - 1}\frac{C_{2}}{n^{3}} + \sum_{b = b_{\max} + 1}^{n}\frac{C_{2}}{n^{3}} + \sum_{b = b_{\min}}^{b_{\max}}\Pr{E \mid B_{i} = b}\Pr{B_{i} = b} \\
        &\leq \sum_{b = 0}^{n}\frac{C_{2}}{n^{3}} + \sum_{b = b_{\min}}^{b_{\max}}\Pr{E \mid B_{i} = b}\Pr{B_{i} = b} \\
        &= \frac{C_{2}}{n^{2}} + \frac{C_{2}}{n^{3}} + \sum_{b = b_{\min}}^{b_{\max}}\Pr{E \mid B_{i} = b}\Pr{B_{i} = b} \\
        &\leq \frac{2C_{2}}{n^{2}} + \sum_{b = b_{\min}}^{b_{\max}}\Pr{E \mid B_{i} = b}\Pr{B_{i} = b}.
    \end{align}
\end{proof}
\begin{remark}
    Notice that a simple adaptation of this argument can turn the additive probability \(C_{3}/n^{2}\) into \(C_{3}^{\prime}/n^{z}\) for any positive integer \(z \in \N\) (where \(C_{3}^{\prime}\) depends on \(z\)). However, trying to replace it with \(C_{3}^{\prime}e^{-C^{\prime}n}\) is more difficult.
\end{remark}

\paragraph{Gaussian concentration.}

\begin{proposition}\label{prop:gaussian_concentration_opnorm}
    There are universal constants \(C_{1}, C_{2}, C_{3} > 0\) such that the following holds. Let \(n, q \in \N\), and let \(\vM \in \R^{q \times n}\) be such that \(M_{ij} \simiid \dNorm(0, x^{2})\). Then 
        \begin{align}
            \Pr{\norm{\vM}_{\op} > s}  
            &\leq C_{1}\exp\rp{-n\bc{\frac{s/x - \sqrt{q}}{C_{2}\sqrt{n}} - 1}^{2}}, \qquad \forall s > x\bc{\sqrt{q} + C_{2}\sqrt{n}} \\
            \Pr{\norm{\vM}_{\op} > s}  
            &\leq C_{1}\exp\rp{-q\bc{\frac{s/x - \sqrt{n}}{C_{3}\sqrt{q}} - 1}^{2}}, \qquad \forall s > x\bc{\sqrt{n} + C_{3}\sqrt{q}}.
        \end{align}
\end{proposition}
\begin{proof}
    Define \(\ol{\vM} \doteq \frac{1}{x}\vM\), so that \(M_{ij} \simiid \dNorm(0, 1)\). By \citet[Example 2.5.8, Lemma 3.4.2]{vershynin2018high}, we see that each row \(\ol{\vM}_{i}\) has Orlicz norm \(\norm{\ol{\vM}_{i}}_{\psi_{2}} \leq C_{1}\) for some universal constant \(C_{1} > 0\). 

    By \citet[Theorem 4.6.1]{vershynin2018high} we have for some other universal constant \(C_{2} > 0\) that for all \(t > 0\),
    \begin{equation}
        \sqrt{q} - C_{1}^{2}C_{2}(\sqrt{n} + t) \leq \sigma_{\min(n, q)}(\ol{\vM}) \leq \sigma_{1}(\ol{\vM}) \leq \sqrt{q} + C_{1}^{2}C_{2}(\sqrt{n} + t)
    \end{equation}
    with probability at least \(1 - 2e^{-t^{2}}\). Defining \(C_{3} \doteq C_{1}^{2}C_{2}\) and noting that \(\norm{\cdot}_{\op} = \sigma_{1}(\cdot)\), we have with the same probability that
    \begin{equation}
        \norm{\ol{\vM}}_{\op} - \sqrt{q} \leq C_{3}\bp{\sqrt{n} + t}.
    \end{equation}
    Simplifying, we obtain
    \begin{align}
        \norm{\ol{\vM}}_{\op} - \sqrt{q}
        &\leq C_{3}\bp{\sqrt{n} + t} \\
        \frac{1}{x}\norm{\vM}_{\op} - \sqrt{q} 
        &\leq C_{3}\bp{\sqrt{n} + t} \\
        \norm{\vM}_{\op} - x\sqrt{q} 
        &\leq C_{3}x\bp{\sqrt{n} + t} \\
        \norm{\vM}_{\op} 
        &\leq x \bc{\sqrt{q} + C_{3}\bp{\sqrt{n} + t}}.
    \end{align}
    Define \(s > 0\) by 
    \begin{equation}
        s \doteq x\bc{\sqrt{q} + C_{3}\bp{\sqrt{n} + t}} \iff t = \frac{s/x - \sqrt{q}}{C_{3}} - \sqrt{n}.
    \end{equation}
    note that the range of validity is
    \begin{equation}
        t > 0 \iff s > x\bc{\sqrt{q} + C_{3}\sqrt{n}}.
    \end{equation}
    For \(s\) in this range, we have
    \begin{equation}
        \Pr{\norm{\vM}_{\op} > s} \leq 2\exp\rp{-\bc{\frac{s/x - \sqrt{q}}{C_{3}} - \sqrt{n}}^{2}} = 2\exp\rp{-n\bc{\frac{s/x - \sqrt{q}}{C_{3}\sqrt{n}} - 1}^{2}}.
    \end{equation}
    The other inequality follows from applying this inequality to \(\vM^{\top}\).
\end{proof}

\begin{proposition}\label{prop:gaussian_covariance_concentration_opnorm}
    There are universal constants \(C_{1}, C_{2}, C_{3} > 0\) such that the following holds. Let \(n, q \in \N\), and let \(\vM \in \R^{q \times n}\) be such that \(M_{ij} \simiid \dNorm(0, x^{2})\). Then 
    \begin{align}
        &\Pr*{\norm*{\vM^{\top}\vM - qx^{2}\vI}_{\op} > s} \\
        &\leq \casework{C_{1}\exp\rp{-n\bc{\frac{1}{C_{2}^{2}C_{3}\sqrt{nq}x^{2}}s - 1}^{2}}, & \text{if}\ C_{2}^{2}C_{3}\sqrt{nq}x^{2} \leq s \leq C_{2}^{2}qx^{2} \\ C_{1}\exp\rp{-n\bc{\frac{1}{C_{2}C_{3}\sqrt{n}x}\sqrt{s} - 1}^{2}}, & \text{if}\ s \geq C_{2}^{2}qx^{2}.}
    \end{align}
\end{proposition}
\begin{proof}
    Define \(\ol{\vM} \doteq \frac{1}{x}\vM\), so that \(\ol{M}_{ij} \simiid \dNorm(0, 1)\). By \citet[Example 2.5.8, Lemma 3.4.2]{vershynin2018high}, we see that each row has Orlicz norm \(\norm{\ol{\vM}_{i}}_{\psi_{2}} \leq C_{1}\) for some universal constant \(C_{1} > 0\). 

    By \citet[Eq.~4.22]{vershynin2018high} we have for some other universal constant \(C_{2} > 0\) that for all \(t > 0\),
    \begin{equation}
        \norm*{\frac{1}{q}\ol{\vM}^{\top}\ol{\vM} - \vI}_{\op} \leq C_{1}^{2}\max\{\delta, \delta^{2}\} \quad \text{where} \quad \delta \doteq C_{2}\frac{\sqrt{n} + t}{\sqrt{q}}.
    \end{equation}
    with probability at least \(1 - 2e^{-t^{2}}\). Simplifying, we obtain
    \begin{align}
        \norm*{\frac{1}{q}\ol{\vM}^{\top}\ol{\vM} - \vI}_{\op} 
        &\leq C_{1}^{2}\max\{\delta, \delta^{2}\} \\
        \norm*{\ol{\vM}^{\top}\ol{\vM} - q\vI}_{\op} 
        &\leq C_{1}^{2}q\max\{\delta, \delta^{2}\} \\
        \norm*{(x^{-1}\vM)^{\top}(x^{-1}\vM) - q\vI}_{\op} 
        &\leq C_{1}^{2}q\max\{\delta, \delta^{2}\} \\
        \norm*{x^{-2}\vM^{\top}\vM - q\vI}_{\op} 
        &\leq C_{1}^{2}q\max\{\delta, \delta^{2}\} \\
        \norm*{\vM^{\top}\vM - qx^{2}\vI}_{\op} 
        &\leq C_{1}^{2}qx^{2}\cdot\max\{\delta, \delta^{2}\}.
    \end{align}
    Now from simple algebra and the fact that \(n \geq 1\), we have
    \begin{align}
        \max\bc{\delta, \delta^{2}} = \delta 
        &\iff 0 \leq t \leq C_{2}^{-1}\sqrt{q} - \sqrt{n} \\
        \max\bc{\delta, \delta^{2}} = \delta^{2}
        &\iff t \geq C_{2}^{-1}\sqrt{q} - \sqrt{n}.
    \end{align}
    Now define \(s \geq 0\) by
    \begin{equation}
        s \doteq C_{1}^{2}qx^{2} \cdot \max\{\delta, \delta^{2}\}.
    \end{equation}
    Thus in the first case we have
    \begin{equation}
        s \doteq C_{1}^{2}C_{2}\sqrt{q}x^{2}(\sqrt{n} + t) \iff t = \frac{1}{C_{1}^{2}C_{2}\sqrt{q}x^{2}}s - \sqrt{n},
    \end{equation}
    and in particular the first case holds when
    \begin{equation}
        \max\{\delta, \delta^{2}\} = \delta  \iff 0 \leq t \leq C_{2}^{-1}\sqrt{q} - \sqrt{n} \iff C_{1}^{2}C_{2}\sqrt{nq}x^{2} \leq s \leq C_{1}^{2}qx^{2}.
    \end{equation}
    Meanwhile, in the second case, we have
    \begin{equation}
        s \doteq C_{1}^{2}C_{2}^{2}(\sqrt{n} + t)^{2}x^{2} \iff t = \frac{1}{C_{1}C_{2}x}\sqrt{s} - \sqrt{n}
    \end{equation}
    where we obtain only one solution to the quadratic equation by requiring \(t \geq 0\), and in particular the second case holds when 
    \begin{equation}
        \max\{\delta, \delta^{2}\} = \delta \iff t \geq C_{2}^{-1}\sqrt{q} - \sqrt{n} \iff s \geq C_{1}^{2}qx^{2}.
    \end{equation}
    Thus we have
    \begin{align}
        &\Pr{\norm{\vM^{\top}\vM - qx^{2}\vI}_{\op} > s} \\
        &\leq \casework{2\exp\rp{-\bc{\frac{1}{C_{1}^{2}C_{2}\sqrt{q}x^{2}}s - \sqrt{n}}^{2}}, & \text{if}\ C_{1}^{2}C_{2}\sqrt{nq}x^{2} \leq s \leq C_{1}^{2}qx^{2} \\ 2\exp\rp{-\bc{\frac{1}{C_{1}C_{2}x}\sqrt{s} - \sqrt{n}}^{2}}, & \text{if}\ s \geq C_{1}^{2}qx^{2}} \\
        &= \casework{2\exp\rp{-n\bc{\frac{1}{C_{1}^{2}C_{2}\sqrt{nq}x^{2}}s - 1}^{2}}, & \text{if}\ C_{1}^{2}C_{2}\sqrt{nq}x^{2} \leq s \leq C_{1}^{2}qx^{2} \\ 2\exp\rp{-n\bc{\frac{1}{C_{1}C_{2}\sqrt{n}x}\sqrt{s} - 1}^{2}}, & \text{if}\ s \geq C_{1}^{2}qx^{2}.}
    \end{align}
\end{proof}

%% file: app_technical_RR_score.tex
\subsection{Companion to \Cref{sub:structured_diffusion_impl}}
\label{app:structured-diffusion-calcs}

In this section, we justify the scaling applied to \(\nabla R^{c}\) in \Cref{sub:structured_diffusion}, and supply the discretization scheme used in \Cref{sub:structured_diffusion_impl}.

First, suppose that \(\vZ_{\natural}^{\ell}\) satisfies \Cref{model:gaussian_tokens}, and \(\vZ_{t} \doteq \vZ_{\natural}^{\ell} + \sigma_{t}\vW\), where \(\vW\) is a standard Gaussian matrix, so that \(\vZ_{t}\) satisfies \Cref{model:gaussian_tokens_noise} with noise level \(\sigma_{t} > 0\). Let \(q_{t}\) be the density of \(\vZ_{t}\). Theoretical analysis from \citep{lu2023understanding} and empirical analysis from \citep{Song2019-ww} demonstrates that under generic conditions, we have that 
\begin{equation}
    \norm{\nabla q_{t}(\vZ_{t})}_{2} \propto \frac{1}{\sigma_{t}^{2}},
\end{equation}
ignoring all terms in the right-hand side except for those involving \(\sigma_{t}\). On the other hand, from the proof of \Cref{lem:inverse-term}, we obtain that \(-\nabla R^{c}(\vZ_{t})\) has constant (in \(\sigma_{t}\)) magnitude with high probability. Thus, in order to have them be the same magnitude, we need to divide \(-\nabla R^{c}(\vZ_{t})\) by \(\sigma_{t}^{2}\) to have it be a drop-in replacement for the score function, as alluded to in \Cref{sub:structured_diffusion}.

Second, we wish to explicitly state our discretization scheme given in \Cref{sub:structured_diffusion_impl}. To wit, we provide a discretization scheme that turns the structured diffusion ODE \Cref{eq:structured_diffusion_ode_prediscretized} into its gradient descent analogue \Cref{eq:structured_diffusion_discretized}; the other discretization, namely of the structured denoising ODE, from \Cref{eq:structured_denoising_ode_prediscretized} to \Cref{eq:structured_denoising_ode_discretized} occurs similarly. To begin with, define the shorthand notation 
\begin{equation}
    f(t, \wt{\vZ}(t)) \doteq \nabla R^{c}(\wt{\vZ}(t) \mid \vU_{[K]}(T - t)),
\end{equation}
so that we have
\begin{equation*}
    \odif{\wt{\vZ}(t)} = \frac{1}{2t}f(t, \wt{\vZ}(t))\odif{t}. \tag{\ref{eq:structured_diffusion_ode_prediscretized}}
\end{equation*}
Fix \(L\), and let \(0 < t_{1} < t_{2} < \cdots < t_{L} = T\), such that \(t_{1}\) is small. (These will be specified shortly in order to supply the discretization scheme.) A suitable first-order discretization is given by
\begin{equation}
    \wt{\vZ}^{\ell + 1} \approx \wt{\vZ}^{\ell} + \frac{t_{\ell + 1} - t_{\ell}}{2t_{\ell}}f(t_{\ell}, \wt{\vZ}^{\ell}).
\end{equation}
Thus it remains to set \(t_{1}, \dots, t_{L}\) such that
\begin{equation}
    \frac{t_{\ell + 1} - t_{\ell}}{2t_{\ell}} = \kappa
\end{equation}
for some constant \(\kappa\), we observe that we must set 
\begin{equation}
    t_{\ell + 1} = (1 + 2\kappa)t_{\ell},
\end{equation}
so that the time grows exponentially in the index. The reverse process time decays exponentially in the index, which matches practical discretization schemes for ordinary diffusion models \citep{Song2019-ww}. Finally, we have \(T = t_{L} = (1 + 2\kappa)^{L}t_{1}\), so that \(t_{1} = \frac{T}{(1 + 2\kappa)^{L}}\).

%% file: app_experiment.tex
In this section, we provide details about our experiments, and report the results of additional experiments that were not covered in the main text. \ours{} takes arguably the most basic design choices possible, and so we do \textit{not} attempt to directly compete with state-of-the-art performance from heavily engineered and empirically designed transformers. 
The results of our experiments are meant to convey a few core messages:
\begin{itemize}
    \item \textit{Despite not being engineered to compete with the state-of-the-art, \ours{} performs strongly on large-scale real-world datasets}, including classification on ImageNet-1K. \ours{} also achieves strong transfer learning performance.
    \item \textit{Because our model is designed through unrolled optimization of a well-understood objective, each layer is interpretable}. In particular, we can analyze the performance of \ours{}, as well as design network modifications, on a \textit{layer-wise basis}. This is powered by an arguably unparalleled level of insight into the role of each operator in our network.
    \item \textit{We make the simplest possible choices during the design of \ours{}, but these can be changed easily while keeping the same framework}. We study a few modifications later in this section (\Cref{subsec:appendix-crate-ablation}) and show that they do not significantly hurt empirical performance, but emphasize here that there is significant potential for improvement with different architecture choices (and in particular a different theoretical analysis).
\end{itemize}

\subsection{Details about \ourscaps{} for Image Classification} \label{app:subsec-experiment-supervise-classification-details}
In this subsection, we provide more details for implementing \ours{} on vision tasks.

\paragraph{Training setup.} 
We fine-tune our pre-trained \ours{} on the following target datasets: CIFAR10/CIFAR100~\citep{krizhevsky2009learning}, Oxford Flowers-102~\citep{nilsback2008automated}, Oxford-IIIT-Pets~\citep{parkhi2012cats}. 
For each fine-tuning task, we employ the AdamW optimizer~\citep{loshchilov2017decoupled}. 
We configure the learning rate as $5 \times 10^{-5}$, weight decay as $0.01$, and batch size as $256$. 
To allow transfer learning, in all training and evaluations setups we first resize our input data to 224 height and width. For data augmentations during pre-training and fine-tuning, we also adopt several standard techniques: random cropping, random horizontal flipping, and random augmentation (with number of transformations $n=2$ and magnitude of transformations $m=14$).

\paragraph{Pre-training on ImageNet-1K.} 
We apply the Lion optimizer~\citep{chen2023symbolic} for pre-training both \ours{} and ViT models. 
We configure the learning rate as $2.4 \times 10^{-4}$, weight decay as 0.5, and batch size as 2,048. 
We incorporate a warm-up strategy with a linear increase over 5 epochs, followed by training the models for a total of 150 epochs with cosine decay. 
For data augmentation, we only apply the standard techniques, random cropping and random horizontal flipping, on the ImageNet-1K dataset. 
We apply label smoothing with smoothing parameter $0.1$. 
One training epoch of \ours{-Base} takes around 240 seconds using 16 A100 40GB GPUs.

\paragraph{Fine-tuning on image classification tasks.} 
For each fine-tuning task, we use the AdamW optimizer~\citep{loshchilov2017decoupled}. 
We configure the learning rate as $5 \times 10^{-5}$, weight decay as 0.01, and batch size as 512. 
To allow transfer learning, we first resize our input data to 224. For data augmentations, we also adopt several standard techniques: random cropping, random horizontal flipping, and random augmentation (with number of transformations $n=2$ and magnitude of transformations $m=14$).\footnote{\url{https://github.com/huggingface/pytorch-image-models/blob/main/timm/data/auto_augment.py}}

\subsection{Details about \ourscaps{-MAE} for Image Completion}\label{app:subsec-experiment-MAE}

\begin{table}[t!]
    \centering
    \caption{\small Image classification model configurations for different sizes of \ours{}, parameter counts, and comparisons to ViT models.}
    \label{tab:model_configs_classification}
    \footnotesize
    \setlength{\tabcolsep}{12.6pt}
        \begin{tabular}{@{}lcccccc@{}}
            \toprule
            \textbf{Model Size} & \(L\) & \(d\) & \(K\) & \(N\) & \ours{} \# Parameters & ViT \# Parameters
            \\ 
            \midrule
            \midrule
            Tiny & 12 & 384 & 6 & 196 & 6.1M & 5.72M \\
            \midrule
            Small & 12 & 576 & 12 & 196 &  13.1M & 22.05M \\
            \midrule
            Base & 12 & 768 & 12 & 196 & 22.8M & 86.54M \\
            \midrule
            Large & 24 & 1024 & 16 & 196 & 77.6M & 307M \\
            \bottomrule
        \end{tabular}%
\end{table}

\begin{table}[t!]
    \centering
    \caption{\small Model configurations for different sizes of \ours{-MAE}, parameter counts, and comparisons to MAE models. Note that MAE does not provide model configurations smaller than Base \citep{he2022masked}. We observe that \ours{-MAE}-Base uses around 30\% of the parameters of MAE-Base. $^{\dagger}$Note that MAE-Large has more than 12 encoder layers --- we use the model configuration from \citet{he2022masked} for the MAE models --- but it is the next-smallest MAE model after Base, so it is suitable for a comparison with \ours{-MAE}-Large.}
    \label{tab:model_configs}
    \footnotesize
    \setlength{\tabcolsep}{12.6pt}
        \begin{tabular}{@{}lcccccc@{}}
            \toprule
            \textbf{Model Size} & \(L\) & \(d\) & \(K\) & \(N\) & \ours{} \# Parameters & MAE \# Parameters
            \\ 
            \midrule
            \midrule
            Small & 12 & 576 & 12 & 196 & 25.4M & 47.5M \\
            \midrule
            Base & 12 & 768 & 12 & 196 & 44.6M & 143.8M \\
            \midrule
            Large & 12 & 1024 & 16 & 196 & 78.5M & 406.0M$^{\dagger}$ \\
            \bottomrule
        \end{tabular}%
\end{table}

\paragraph{Pre-training on ImageNet-1K.} 
We apply the AdamW optimizer~\citep{loshchilov2017decoupled} for pre-training both \ours{-MAE} models on ImageNet-1K. 
We configure the learning rate as $3 \times 10^{-5}$, weight decay as $0.1$, and batch size as \(4,096\).
We incorporate a warm-up strategy with a linear increase over 40 epochs, followed by training the models for a total of 800 epochs with cosine decay. 
For data augmentation, we apply the standard augmentations as used in \citet{he2022masked}, random cropping and random horizontal flipping.

\paragraph{Fine-tuning \ourscaps{-MAE} models.} 
Recall that in \Cref{subsubsec:image-completion-task}, we described two methods of fine-tuning our model: full fine-tuning, which updates all the weights, and linear probing via logistic regression, which only updates the classification head. For full fine-tuning, we apply the same fine-tuning configuration (data augmentation, training epochs, optimization, etc.) as described in \Cref{app:subsec-experiment-supervise-classification-details}. For linear probing, we apply the optimization solver in \texttt{scikit-learn}, i.e., \texttt{linear\_model.LogisticRegression}, to learn a logistic regression model with $\ell^2$ regularization, and use cross-validation to select the $\ell^2$ regularization parameter for each model-dataset pair. In our experiments, we consider $\ell^2$ regularization parameters from the set $\{10^{-3}, 10^{-2}, 10^{-1}, 1, 10, 10^{2}, 10^{3}\}$.

\subsection{Details about \ourscaps{-DINO} for Self-Supervised Learning}\label{app:subsec-experiment-dino}

\paragraph{DINO pre-training.} 
To train \ours{} with the DINO self-supervised learning objective, we use the ImageNet-1K dataset and match our training recipe closely to the official implementation with some minor variations. Specifically, we choose the base learning rate to be $0.001$ and the weight decay parameter to be $0.005$ in our experiments.
Other hyper-parameters such as the temperature parameters $\tau_t$ and $\tau_s$ are chosen to exactly match with the publicly available implementation of DINO\footnote{\url{https://github.com/facebookresearch/dino}}. Due to computational limitations, we only train our models for 100 epochs, compared to 400 to 800 epochs of training for the official implementation.

\paragraph{Linear probing models trained via DINO.}
As introduced in \Cref{sec:experiment-DINO}, we use both weighted nearest neighbor ($k-$NN) and logistic regression on top of the features extracted from the teacher networks after pre-training with DINO. For $k-$NN, we search the number of neighbors $k$ from the candidate set of $\{10,20,100,200\}$ and choose the model that achieves the highest training accuracy. For logistic regression, we use the same procedure as introduced earlier for \ours{-MAE} in \Cref{app:subsec-experiment-MAE}.

\subsection{Details about \ourscaps{-BERT} and \ourscaps{-GPT} on Natural Language}\label{app:subsec-experiment-text}

\begin{table*}[t!]
    \centering
    \caption{\small (\textit{upper}) Configuration of the \ours{} BERT-family models. (\textit{lower}) Configuration of the \ours{} GPT-family models. $d$ represents the dimension of the token, $K$ represents the number of attentions head in the attention block, $L$ represents the model depth, and $h$ represents the hidden dimension of the MLP block in standard transformers. For \ours{} models, there is no hidden dimension.}
    \label{tab:configs-crate-text}
    \footnotesize
    \setlength{\tabcolsep}{9.6pt}
        \begin{tabular}{llcccccc}
            \toprule
            {Model Size} & Pre-training & $d$ & $K$ & $L$ & $h$ & \ours{} & BERT 
            \\ 
            \midrule
            Medium & Masked LM & 512 & 8 & 8 & 2,048 &  33M & 41M 
            \\
            Base & Masked LM & 768 & 12 & 12 & 3,072 &  61M & 110M 
            \\
            Large & Masked LM & 1,024 & 16 & 24 & 4,096 &  129M & 335M 
            \\
            \midrule
            \midrule
            {Model Size} & Pre-training & $d$ & $K$ & $L$ & $h$ & \ours{} & GPT-2 
            \\
            \midrule
            {\color{revision}Small} & {\color{revision}Causal LM} & {\color{revision}512} & {\color{revision}8} & {\color{revision}12} & {\color{revision}2,048} & {\color{revision}-}  & {\color{revision}64M}
            \\
            Base & Causal LM & 768 & 12 & 12 & 3,072 & 60M  & 124M
            \\
            Large & Causal LM & 1,280 & 20 & 36 & 5,120 & 242M & 774M
            \\
            \bottomrule
        \end{tabular}%
\end{table*}

\paragraph{Details about pre-training for \ourscaps{}-BERT.} 
We use a batch size of 8,096 and train for 30,000 steps with the Adam optimizer~\citep{kingma2014adam}. 
For the Adam optimizer, we use $(\beta_1, \beta_2)=(0.9, 0.98)$, and a weight decay of $0.01$. 
For the learning rate scheduler, we apply the linear warm-up and linear decay, with the peak learning rate  at the iteration of 1,800 steps with value $\eta=10^{-3}$.

\paragraph{Details about pre-training for \ourscaps{-GPT}.} 
We use a batch size of 384 and train for 600,000 steps with the Adam optimizer~\citep{kingma2014adam}. 
For the Adam optimizer, we use $(\beta_1, \beta_2)=(0.9, 0.95)$, and weight decay of $0.1$. 
For the learning rate scheduler, we apply the linear warm-up and cosine decay, with a peak value of $\eta=6\times 10^{-4}$ at the $2,000$ iteration, and minimum value $6\times 10^{-5}$. 
The training and validation losses over iterations are shown in Figure~\ref{fig:crate-text-evals}~(right).
The training/validation loss converges around $3.37$ after training with a batch size of 384 and 600,000 iterations. 
In comparison, the open GPT-2 implementation is pre-trained on OpenWebText with a batch size of 512 and 600,000 steps, and converges to a validation loss of $2.85$ \citep{nanogpt}.

\paragraph{Details about fine-tuning \ourscaps{}-BERT on GLUE.} 
For all the tasks in GLUE, we use a learning rate of $2\times 10^{-5}$ without any hyperparameter sweep. We fine-tune the three models (BERT-Base, BERT-Medium and \ours{}-BERT-Base) using the same fine-tuning configuration.
We fine-tune for 8 epochs on MNLI, for 5 epochs on WNLI and MRPC (because these two datasets are tiny), and for 3 epochs on all other tasks. For all tasks, we use a batch size of 32. 
The maximum sequence length is set to 128 for WNLI and 512 for all other tasks. We find that the performance of BERT on WNLI is very sensitive to the sequence length, so we picked the sequence length to be the same as~\citet{huggingface_transformers} (while \ours{}-BERT is very stable to the sequence length). We do not observe a significant difference of performance using different sequence length for other tasks.
Evaluations shown in Figure \ref{fig:crate-text-evals} are meant to demonstrate the effectiveness of learning, so we finetune 3 epochs for all tasks (except MNLI in the 24000 and 30000 step), and we use the sequence length of 512 for all tasks.

\paragraph{Details about \ourscaps{}-GPT evaluation.}

We use the test split as the validation set for all tasks. The evaluation datasets WikiText2 and WikiText103 has the same test split so we merge them to WikiText.  For both models (GPT2-Base and \ours{}-GPT2-Base) on all tasks, we use a block size of $1,024$, evaluation batch size of $4$, and evaluate for $1,000$ iterations to get a more accurate estimation.

\subsection{Ablation Study of \ourscaps{} on Image Classification}\label{subsec:appendix-crate-ablation}

\paragraph{Training hyperparameters for \ourscaps{} on image classification.}
In \Cref{tab:ablation-parameters}, we present evaluation of \ours{} trained with various parameters.
More specifically, we investigate the effect of number of epochs, weight decay, learning rate, step size $(\eta)$ and the regularization term $(\lambda)$ in \texttt{ISTA} block. 
As shown in \Cref{tab:ablation-parameters}, \ours{} demonstrates consistently satisfactory performance across a diverse range of hyperparameters.

\begin{table*}[ht]
\centering
\caption{\small Top 1 accuracy of \ours{} on various datasets with different architecture design variants 
when trained on ImageNet. }
\label{tab:ablation-parameters}
\small
    \setlength{\tabcolsep}{13.6pt}
\resizebox{0.98\textwidth}{!}{%
\begin{tabular}{@{}l|ccc|cc|cc@{}}
\toprule
\textbf{Model} & \texttt{epoch} & \texttt{weight decay}  &   \texttt{lr} &   $\eta$ (\texttt{ISTA}) &   $\lambda$ (\texttt{ISTA}) & ImageNet  \\ 
\midrule
\midrule
 \ours{-B} & 150 (default) & 0.5 (default) & $2.4\times 10^{-4}$ & 0.1 & 0.1 & 70.8 \\
 \midrule
 \midrule
 \ours{-B} & 150 & 0.5 & $2.4\times 10^{-4}$ & \textit{\color{gray} 0.02} & 0.1 & 70.7 \\
 \midrule
 \ours{-B} & 150 & 0.5 & $2.4\times 10^{-4}$ & \textit{\color{gray}0.5} & 0.1 & 66.7 \\
 \midrule
 \ours{-B} & 150 & 0.5 & $2.4\times 10^{-4}$ & 0.1 & \textit{\color{gray}0.02} & 70.8 \\
 \midrule
 \ours{-B} & 150 & 0.5 & $2.4\times 10^{-4}$ & 0.1 & \textit{\color{gray}0.5} & 70.5 \\
 \midrule
 \midrule
 \ours{-B} & \textit{\color{gray}90} & 0.5 & $2.4\times 10^{-4}$ & 0.1 & 0.1 & 69.5 \\
 \midrule
 \ours{-B} & \textit{\color{gray}300} & 0.5 & $2.4\times 10^{-4}$ & 0.1 & 0.1 & 70.9 \\
 \midrule
 \ours{-B} & 150 & \textit{\color{gray}1.0} & $2.4\times 10^{-4}$ & 0.1 & 0.1 & 70.3 \\
 \midrule
 \ours{-B} & 150 & \textit{\color{gray}0.05} & $2.4\times 10^{-4}$ & 0.1 & 0.1 & 70.2 \\
 \midrule
 \ours{-B} & 150 & 0.5 & \textit{\color{gray}$4.8\times 10^{-4}$} & 0.1 & 0.1 & 70.2 \\
 \midrule
 \ours{-B} & 150 & 0.5 & \textit{\color{gray}$1.2\times 10^{-4}$} & 0.1 & 0.1 & 70.3 \\
 \bottomrule
\end{tabular}%
}
\end{table*}

\paragraph{Exploring architecture variants of \ourscaps{}.}
In this section, we explore the two following alternative architectures. One architecture involves a modification to the attention mechanism, while the other involves a modification to the sparsification mechanism. Again, we re-emphasize that these choices, although principled, are entirely modular and the choices we make here still lead to very simple architectures. A more sophisticated analysis may lead to different, more complicated architectures that perform better in practice. The architectures we experiment with are:
\begin{itemize}
    \item Compression-inspired attention mechanism: revert the change in \Cref{eq:mssa_trainable_w_main_paper}. That is, the attention mechanism implements \Cref{eq:SSA,eq:Multi-Head-SSA} directly.
    \item Majorization-minimization proximal step sparsification: instead of \Cref{eq:ista-block}, implement \Cref{eq:prox_maj_min_iteration}.
\end{itemize}

We obtain the following classification results in \Cref{tab:ablation-arch-variants}. 
After conducting additional simplifications to the network architecture (i.e., imposing additional constraints to the network architecture design), we discover that \ours{} maintains reasonable performance on ImageNet-1K.

\begin{table*}[ht]
\centering
\caption{\small Top 1 accuracy of \ours{} on various datasets with different architecture design variants 
when trained on ImageNet. }
\label{tab:ablation-arch-variants}
\small
    \setlength{\tabcolsep}{13.6pt}
\resizebox{0.7\textwidth}{!}{%
\begin{tabular}{@{}lcc|cc|cc@{}}
\toprule
\textbf{Model} & \texttt{MSSA}-block  &   \texttt{ISTA}-block& ImageNet  \\ 
\midrule
\midrule
 \ours{-B} & default & default & 70.8 \\
 \midrule
 \ours{-B} & Eq.~\Cref{eq:SSA,eq:Multi-Head-SSA} & default & 63.3 \\
 \ours{-B} & default & Eq.~\Cref{eq:prox_maj_min_iteration} & 68.6 \\
 \bottomrule
\end{tabular}%
}
\end{table*}

{\color{revision}
\paragraph{Empirical evaluation of soft-thresholding-based architecture.} We summarize the results of CRATE with soft-thresholding activation in Table~\ref{tab:ablation-arch-variants-soft-thresholding}. We use $\lambda=10$ in $S_{\lambda \eta}$ and set all other hyperparameters the same as in the original CRATE-Base evaluation on ImageNet-1K. We find that such a soft-thresholding model achieves slightly worse performance–a drop of 3.2\% top-1 accuracy—compared to the default CRATE-Base (with ReLU activation).
}

\begin{table*}[ht]
\centering
\caption{\small \color{revision} Top 1 accuracy of \ours{} on ImageNet-1k with different architecture design variants. The {\color{blue!70!black}\texttt{SoftShrink} activation $S_{\lambda \eta}$} is defined as $S_{\lambda\eta}(x) = \text{sgn}(x)\cdot (|x| - \lambda \eta)_{+}$, where $\text{sgn}$ is the sign function and $(\cdot)_{+} = \max(\cdot, 0)$.}
\label{tab:ablation-arch-variants-soft-thresholding}
\vspace{-1mm}
\small
    \setlength{\tabcolsep}{13.6pt}
\resizebox{0.98\textwidth}{!}{%
\begin{tabular}{@{}lcc|cccc@{}}
\toprule
{\color{revision}\textbf{Model}} & {\color{revision}\texttt{MSSA}-block}  &   {\color{revision}\texttt{ISTA}-block} & {\color{revision}ImageNet} & {\color{revision}CIFAR 10*} & {\color{revision}CIFAR 100*} \\ 
\midrule
\midrule
 {\color{revision}\ours{-B}} & {\color{revision}default} & {\color{revision}\texttt{ReLU} activation (default)} & {\color{revision}70.8\%} & {\color{revision}96.8\%} & {\color{revision}82.7\%} \\
 \midrule
 {\color{revision}\ours{-B}} & {\color{revision}default} & {\color{revision}{\color{blue!70!black}\texttt{SoftShrink} activation}} & {\color{revision}67.6\%} & {\color{revision}96.0\%} & {\color{revision}76.8\%} \\
 \bottomrule
\end{tabular}%
}
\vspace{-0.1in}
\end{table*}

\begin{table*}[ht]
\caption{\small Top 1 accuracy of \ours{} on various datasets with different model scales 
when pre-trained on ImageNet-21K and fine-tuned on the downstream datasets. }
\centering
\small
    \setlength{\tabcolsep}{13.6pt}
\resizebox{0.9\textwidth}{!}{%
\begin{tabular}{@{}lccc|cc@{}}
\toprule
 & \ours{-T}  &  \ours{-S} & \ours{-B} & { \color{gray} ViT-T} &  { \color{gray}ViT-B } \\ 
\midrule
\midrule
 \# parameters & 5.74M & 14.12M & 38.83M & { \color{gray} 10.36M} & { \color{gray} 102.61M} \\
\midrule
 ImageNet-1K & 62.7 & 74.2 & 79.5 & { \color{gray} 71.8} & { \color{gray} 85.8} \\
 CIFAR10 & 94.1 & 97.2 & 98.1 & { \color{gray} 97.2} & { \color{gray} 98.9} \\
 CIFAR100 & 76.7 & 84.1 & 87.9 & { \color{gray} 84.4} & { \color{gray} 90.1}\\
 Oxford Flowers-102 & 82.2 & 92.2 & 96.7 & { \color{gray} 92.1} & { \color{gray} 99.5}\\
 Oxford-IIIT-Pets & 77.0 & 86.4 & 90.7 & { \color{gray} 86.2} & { \color{gray} 91.8} \\
 \bottomrule
\end{tabular}%
}
\label{tab:Comparison_with_VIT_IN21k}
\end{table*}

\paragraph{Pre-training on ImageNet-21K.} 
We study \ours{} when pre-training with supervised learning on a larger pre-training dataset, ImageNet-21K~\citep{deng2009imagenet}. 
This dataset comprises 14,197,122 images distributed across 21,841 classes. 
For training, each RGB image was resized to dimensions $3\times 224\times 224$, normalized using means of $(0.485, 0.456, 0.406)$ and standard deviations of $(0.229, 0.224, 0.225)$, and then subjected to center cropping and random flipping. 
We set the mini-batch size as 4,096 and apply the Lion optimizer~\citep{chen2023symbolic} with learning rate $9.6\times 10^{-5}$ and weight decay $0.05$. 
All the models, including \ours{s} and ViTs are pre-trained with 90 epochs on ImageNet-21K. 
For fine-tuning on downstream tasks, we use the AdamW optimizer~\cite{loshchilov2017decoupled} and configure the learning rate to $5 \times 10^{-5}$, weight decay as 0.01. Due to memory constraints, we set the batch size to be 128 for all experiments conducted for the base models and set it to be 256 for the other smaller models. 

As summarized in Table \ref{tab:Comparison_with_VIT_IN21k}, we evaluate transfer learning performance of \ours{} by fine-tuning models that are pre-trained on ImageNet-21k for the following downstream vision classification tasks: ImageNet-1k~\citep{deng2009imagenet}, CIFAR10/CIFAR100~\citep{krizhevsky2009learning}, Oxford Flowers-102~\citep{nilsback2008automated}, Oxford-IIIT-Pets~\citep{parkhi2012cats}. We also evaluate the transfer learning fine-tuning performance of two ViT models (-T/8 and -B/8) pre-trained on ImageNet-21K for reference. For both \ours{} and ViT models, we consider the image patch with size $8\times8$.

\newpage
\subsection{\color{revision}Ablation Study of ISTA Layer in CRATE}\label{subsec:appendix-ablation-ista}

{\color{revision}
We study the role of the ISTA block in our white-box architecture design. Specifically, we remove the ISTA block in \textsc{crate} and compare its performance with the default \textsc{crate} in both vision and language tasks. 
The results are summarized in Table~\ref{tab:ablation-arch-variants-ista-vision} and \ref{tab:gpt-eval-ablation-ista}. 
We can find that the constructed white-box transformers without the ISTA block is able to achieve non-trivial performance on ImageNet. 
On the other hand, the white-box transformer \textsc{crate} (with the ISTA sparsity block) achieves much better performance than \textsc{crate} without the ISTA block, demonstrating the effectiveness of the sparsity block in our architecture design.
}
{\color{revision}
\begin{table*}[ht]
\centering
\caption{\small \color{revision} Top 1 accuracy of \ours{}  with different architecture design variants 
when trained on ImageNet. }
\label{tab:ablation-arch-variants-ista-vision}
\small
\setlength{\tabcolsep}{13.6pt}
\resizebox{0.7\textwidth}{!}{%
\begin{tabular}{@{}lcc|cc|cc@{}}
\toprule
\textbf{\color{revision}Model} & {\color{revision}\texttt{MSSA}-block}  &   {\color{revision}\texttt{ISTA}-block} & {\color{revision}ImageNet}  \\ 
\midrule
\midrule
{\color{revision}\ours{-B}} & {\color{revision}default} & {\color{revision}default} & {\color{revision}70.8} \\
\midrule
{\color{revision}\ours{-B}} & {\color{revision}default} & {\color{revision}\xmark} & {\color{revision}61.9} \\
\bottomrule
\end{tabular}%
}

\end{table*}
}

{
\begin{table}[ht]
\def\arraystretch{1.1}
    \small
    \caption{\small \color{revision} Zero-shot cross-entropy loss of the \ours{}-GPT2-Base model \textit{with and without the ISTA block} (lower is better). 
    }
    \centering
    \begin{tabular}{cccc}
    \hline
      & {\color{revision}\texttt{MSSA}-block}  &   {\color{revision}\texttt{ISTA}-block} & {\color{revision}\textbf{OWT}} \\
     \hline
     {\color{revision}\ours{}-GPT2-Base} & {\color{revision}default} & {\color{revision}default}  & {\color{revision}3.37}  \\
     {\color{revision}\ours{}-GPT2-Base} & {\color{revision}default} & {\color{revision}\xmark}  & {\color{revision}3.68} \\
     \hline
    \end{tabular}
    \label{tab:gpt-eval-ablation-ista}
\end{table}

}

\subsection{\color{revision}Ablation Study of MSSA Layer and ISTA Layer in CRATE and Comparison with ViT}

\begin{table*}[ht]
\centering
\vspace{0.1in}
\small
\setlength{\tabcolsep}{8pt}
\resizebox{0.9\textwidth}{!}{\begin{tabular}{@{}llcccc|cccc|c@{}}
\toprule
 &  & & \multicolumn{3}{c}{{\color{revision}COCO Detection}}  & {\color{revision}VOC Seg.} & {\color{revision}ImageNet-1K}\\ 
{\color{revision}Model} & {\color{revision}Attention} & {\color{revision}Nonlinearity}  & {\color{revision}AP$_{50}$} & {\color{revision}AP$_{75}$} & {\color{revision}AP}  & {\color{revision}mIoU} &  {\color{revision}Accuracy} \\ 
\midrule
{\color{revision}\ours{}} & {\color{blue!50!black}\texttt{MSSA}} & {\color{blue!50!black}\texttt{ISTA}} & {\color{revision}2.1} & {\color{revision}0.7} & {\color{revision}0.8} & {\color{revision}23.9} & {\color{revision}74.2} \\
{\color{revision}\ours{-\texttt{MLP}}} & {\color{blue!50!black}\texttt{MSSA}} & {\color{red!60!black}\texttt{MLP}} & {\color{revision}0.2} & {\color{revision}0.2} & {\color{revision}0.2} & {\color{revision}22.0} & {\color{revision}79.2}\\
{\color{revision}\ours{-\texttt{MHSA}}} & {\color{red!60!black}\texttt{MHSA}} & {\color{blue!50!black}\texttt{ISTA}} & {\color{revision}0.1} & {\color{revision}0.1} & {\color{revision}0.0} & {\color{revision}18.4} & {\color{revision}77.8} \\
{\color{revision}ViT} & {\color{red!60!black}\texttt{MHSA}} & {\color{red!60!black}\texttt{MLP}} & {\color{revision}0.1} & {\color{revision}0.1} & {\color{revision}0.0} & {\color{revision}14.1} & {\color{revision}80.9} \\
\bottomrule
\end{tabular}}
\caption{\color{revision}\textbf{Ablation study of different \ours{} variants.} We use the \texttt{Small-Patch8} (\texttt{S-8}) model configuration across all experiments in this table.}
\label{tab:ablationstudy_mhsa_ista}
\vspace{-0.2in}
\end{table*}

{\color{revision}
\vspace{0.2in}
\noindent
Both the attention block (\texttt{MSSA}) and the MLP block (\texttt{ISTA}) in \ours{} are different from the ones in the ViT.  
In order to understand the effect of each component of \ours{}, we study three different variants of \ours{}: \ours{}, \ours{-\texttt{MHSA}}, \ours{-\texttt{MLP}}, where we denote the attention block and MLP block in ViT as \texttt{MHSA} and \texttt{MLP} respectively. We summarize different model architectures in \Cref{tab:ablationstudy_mhsa_ista}. 
}

{\color{revision}For all models in \Cref{tab:ablationstudy_mhsa_ista}, we first apply the same pre-training setup on the ImageNet-21k dataset. The detection and segmentation results are evaluated on models pre-trained on ImageNet-21k. Secondly, we further fine-tune the ImageNet-21k pre-trained models on ImageNet-1k with 50 epochs and summarize the accuracy of different models in Table \Cref{tab:ablationstudy_mhsa_ista}.
We then apply the coarse segmentation evaluation and MaskCut evaluation (described in \Cref{subsubsec:segmentation_measurement}) to quantitatively compare the performance of different models. 
As shown in \Cref{tab:ablationstudy_mhsa_ista}, \ours{} significantly outperforms other model architectures across all tasks. 
Interestingly, we find that the coarse segmentation performance (i.e., VOC Seg) of the ViT can be significantly improved by simply replacing the {\color{red!60!black}\texttt{MHSA}} in ViT with the {\color{blue!50!black}\texttt{MSSA}} in \ours{}, despite the architectural differences between {\color{red!60!black}\texttt{MHSA}} and {\color{blue!50!black}\texttt{MSSA}} being small. 
This demonstrates the effectiveness of the white-box design.}

\newpage
\subsection{Additional Experimental Results of Layer-Wise Analysis}\label{subsec:appendix-exp-results}
{In this subsection, we provide additional experimental results on \ours{}, including layer-wise measurements and visualizations.

\paragraph{Layer-wise evaluation of compression and sparsity.} 
Similar to \Cref{fig:exp-rc-sparisty-small}, we conduct the layer-wise evaluation of compression term and sparsity for \ours{-Tiny}, \ours{-Base}, and \ours{-Large} in \Cref{fig:appendix-exp-rc-sparisty-all-model-size}.
We observe similar behavior as mentioned in \Cref{subsec:exp-in-depth-analysis}: both the compression term and the sparsity term improves as the layer index increases.

}

\begin{figure}[ht]
     \centering
     \begin{subfigure}[b]{0.47\textwidth}
         \centering
    \includegraphics[width=\textwidth]{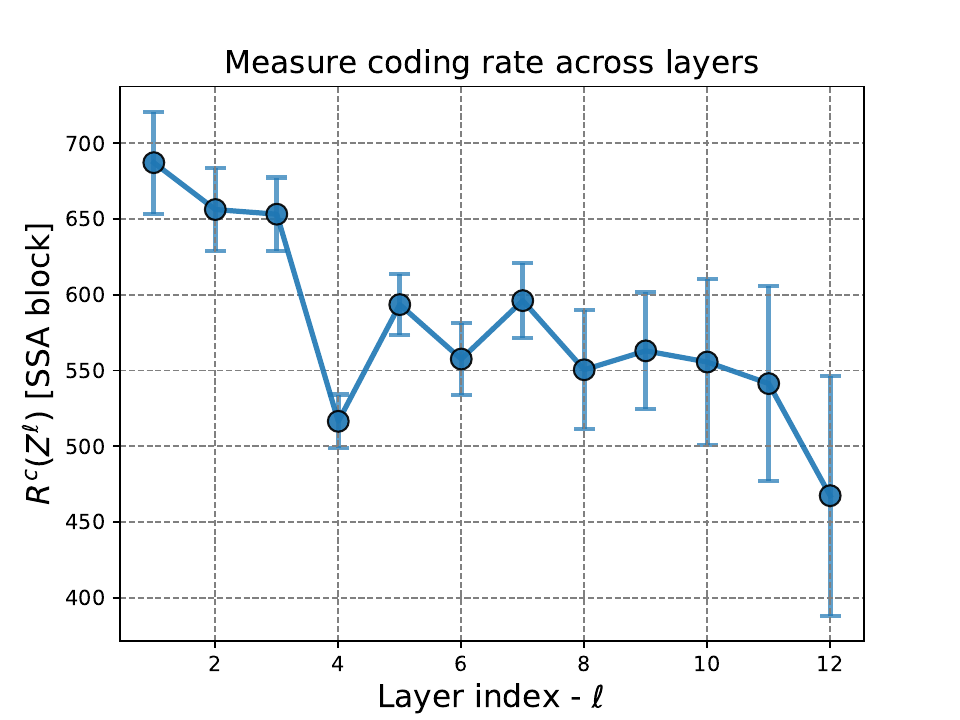}
         \caption{Compression (Model: \ours{-Tiny}).}
     \end{subfigure}
     \begin{subfigure}[b]{0.482\textwidth}
         \centering
    \includegraphics[width=\textwidth]{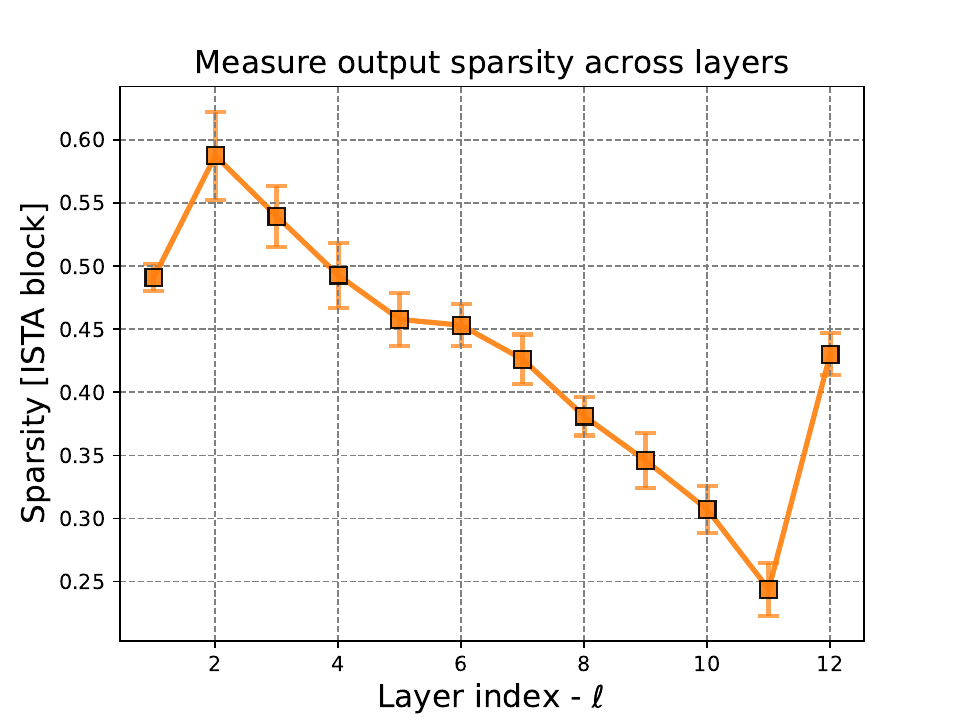}
         \caption{Sparsity (Model: \ours{-Tiny}).}
     \end{subfigure}
     \begin{subfigure}[b]{0.47\textwidth}
         \centering
    \includegraphics[width=\textwidth]{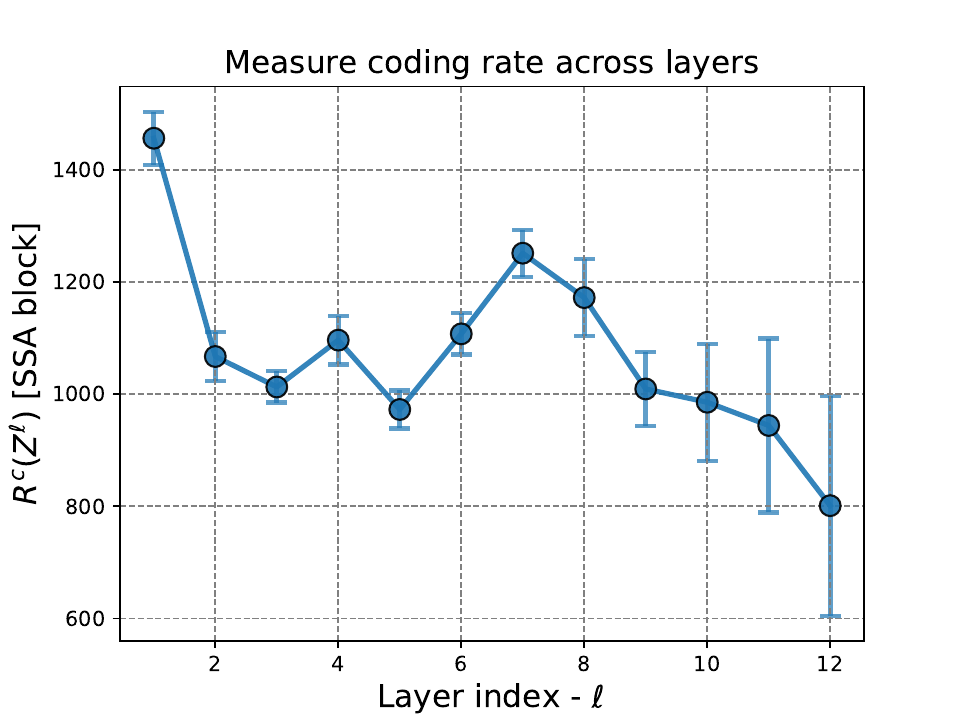}
         \caption{Compression (Model: \ours{-Base}).}
     \end{subfigure}
     \begin{subfigure}[b]{0.482\textwidth}
         \centering
    \includegraphics[width=\textwidth]{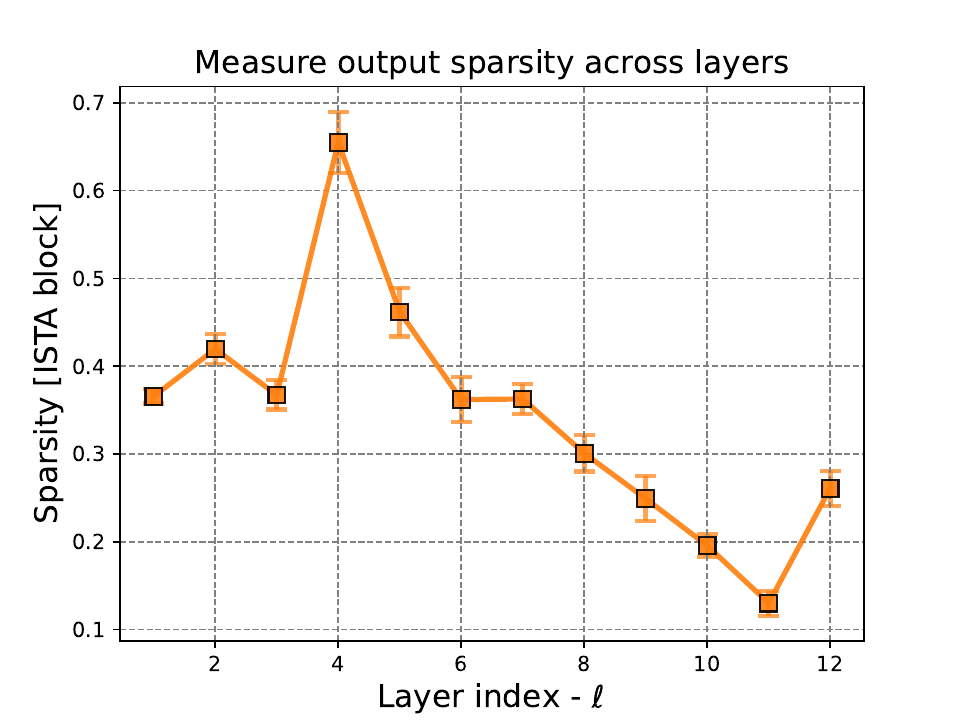}
         \caption{Sparsity (Model: \ours{-Base}).}
     \end{subfigure}
     \begin{subfigure}[b]{0.47\textwidth}
         \centering
    \includegraphics[width=\textwidth]{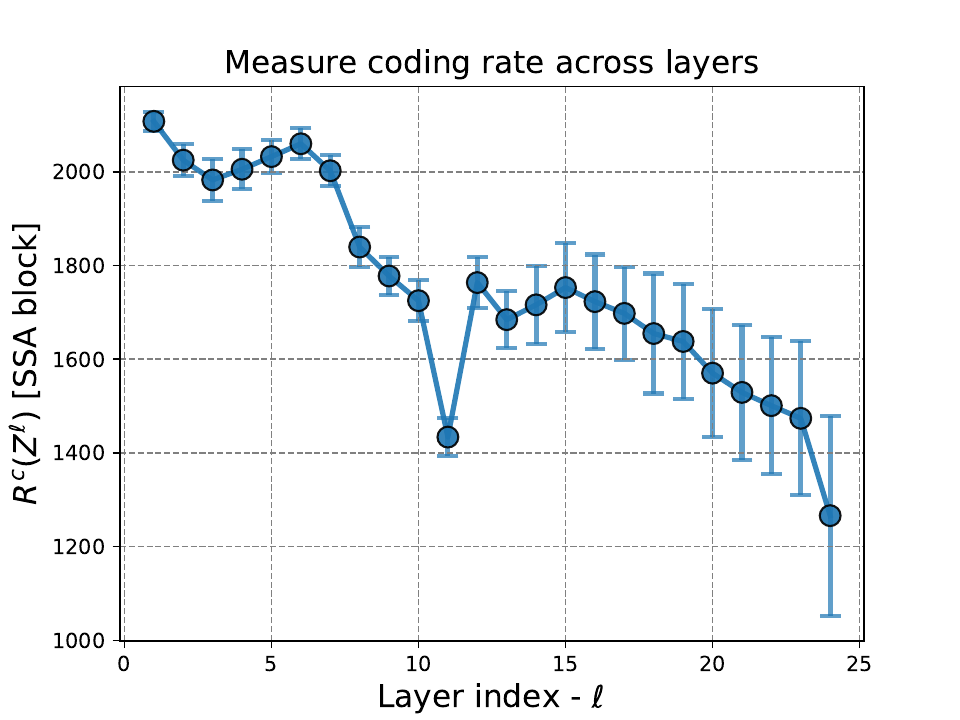}
         \caption{Compression (Model: \ours{-Large}).}
     \end{subfigure}
     \begin{subfigure}[b]{0.482\textwidth}
         \centering
    \includegraphics[width=\textwidth]{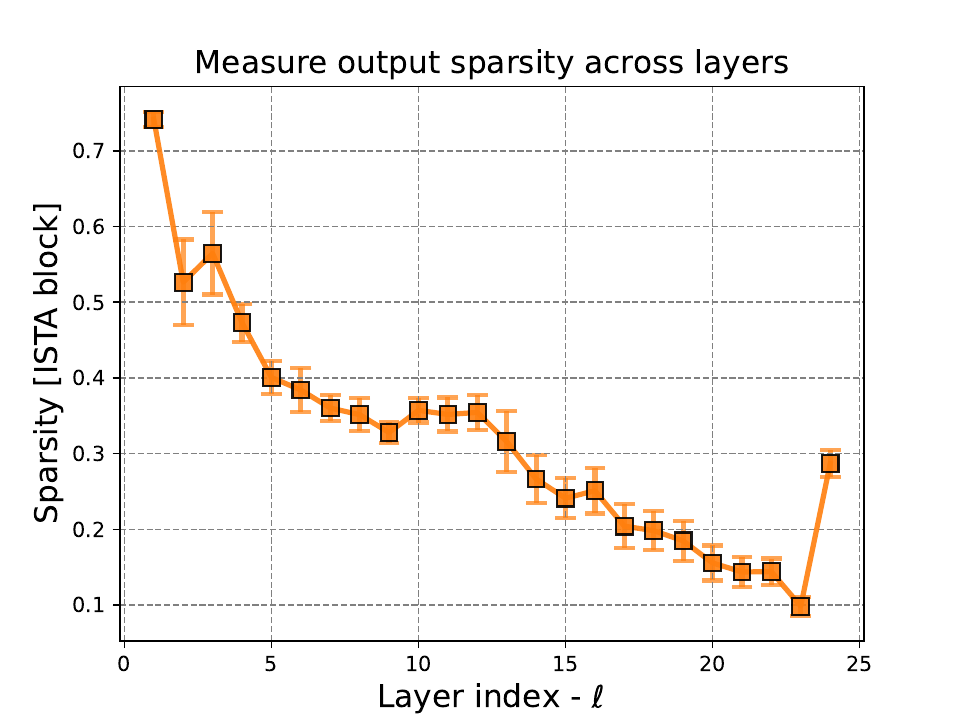}
         \caption{Sparsity (Model: \ours{-Large}).}
     \end{subfigure}
        \caption{\small \textit{Left}: The compression term $R^{c}(\vZ^{\ell+1/2})$ of the \texttt{MSSA} outputs at different layers. \textit{Right}: the sparsity of the \texttt{ISTA} output block, $\norm{\vZ^{\ell+1}}_0 / (d\cdot N)$, at different layers.}
        \label{fig:appendix-exp-rc-sparisty-all-model-size}
\end{figure}

\paragraph{Visualizing layer-wise token representations.} 
In \Cref{fig:appendix-exp-ista-sparsity-heatmap}, we visualize the token representations $\vZ^{\ell}$ at different layers $\ell \in \{1, \dots, 12\}$. We provide more results of the layer-wise token representation visualization on other samples in \Cref{fig:appendix-exp-ista-sparsity-heatmap-sample1}, \Cref{fig:appendix-exp-ista-sparsity-heatmap-sample2}, \Cref{fig:appendix-exp-ista-sparsity-heatmap-sample3},  and \Cref{fig:appendix-exp-ista-sparsity-heatmap-sample4} (Model: \ours{-Base}).

\begin{figure}[ht]
     \centering
     \begin{subfigure}[b]{0.2\textwidth}
         \centering
    \includegraphics[width=\textwidth]{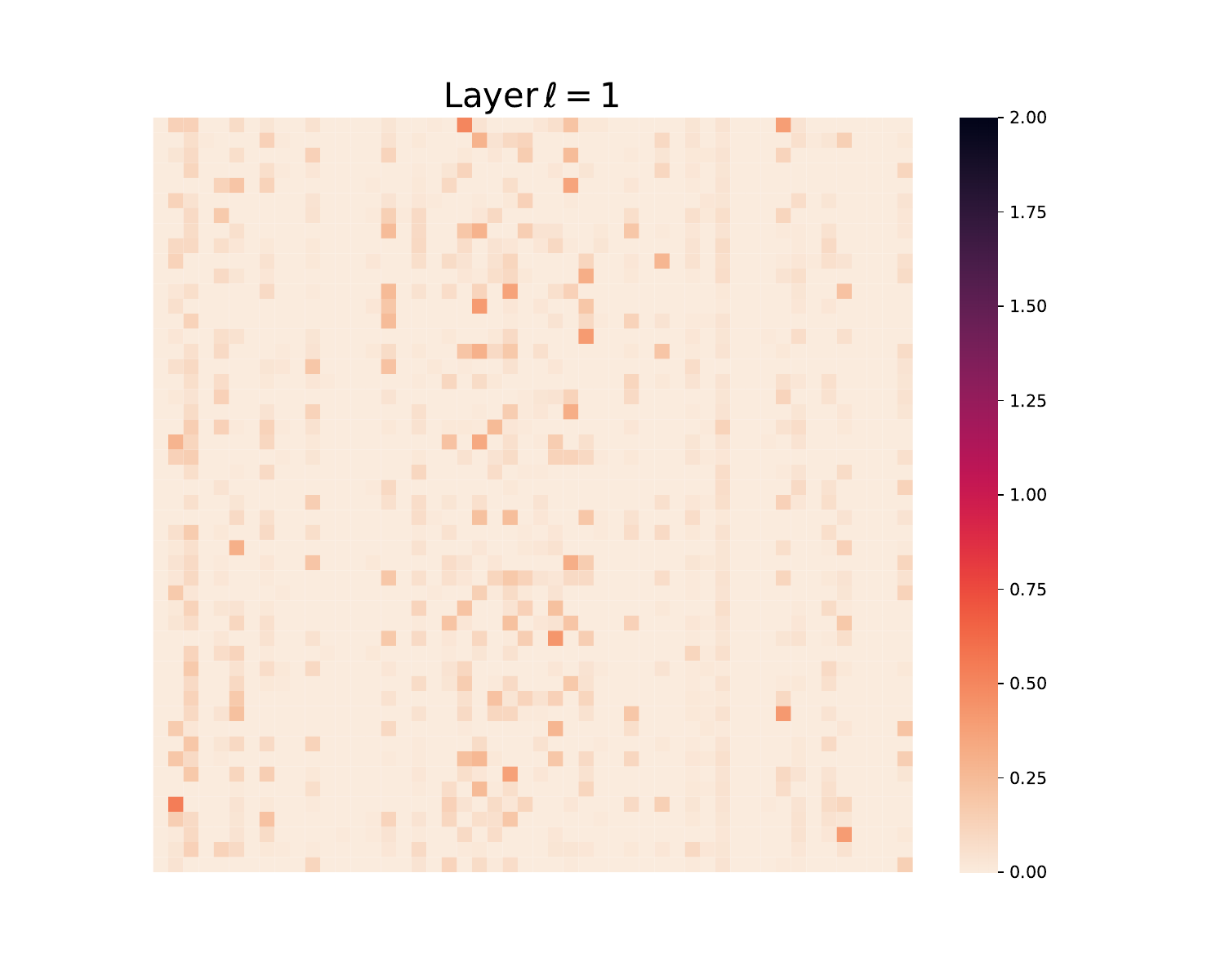}
     \end{subfigure}
     \begin{subfigure}[b]{0.2\textwidth}
         \centering
    \includegraphics[width=\textwidth]{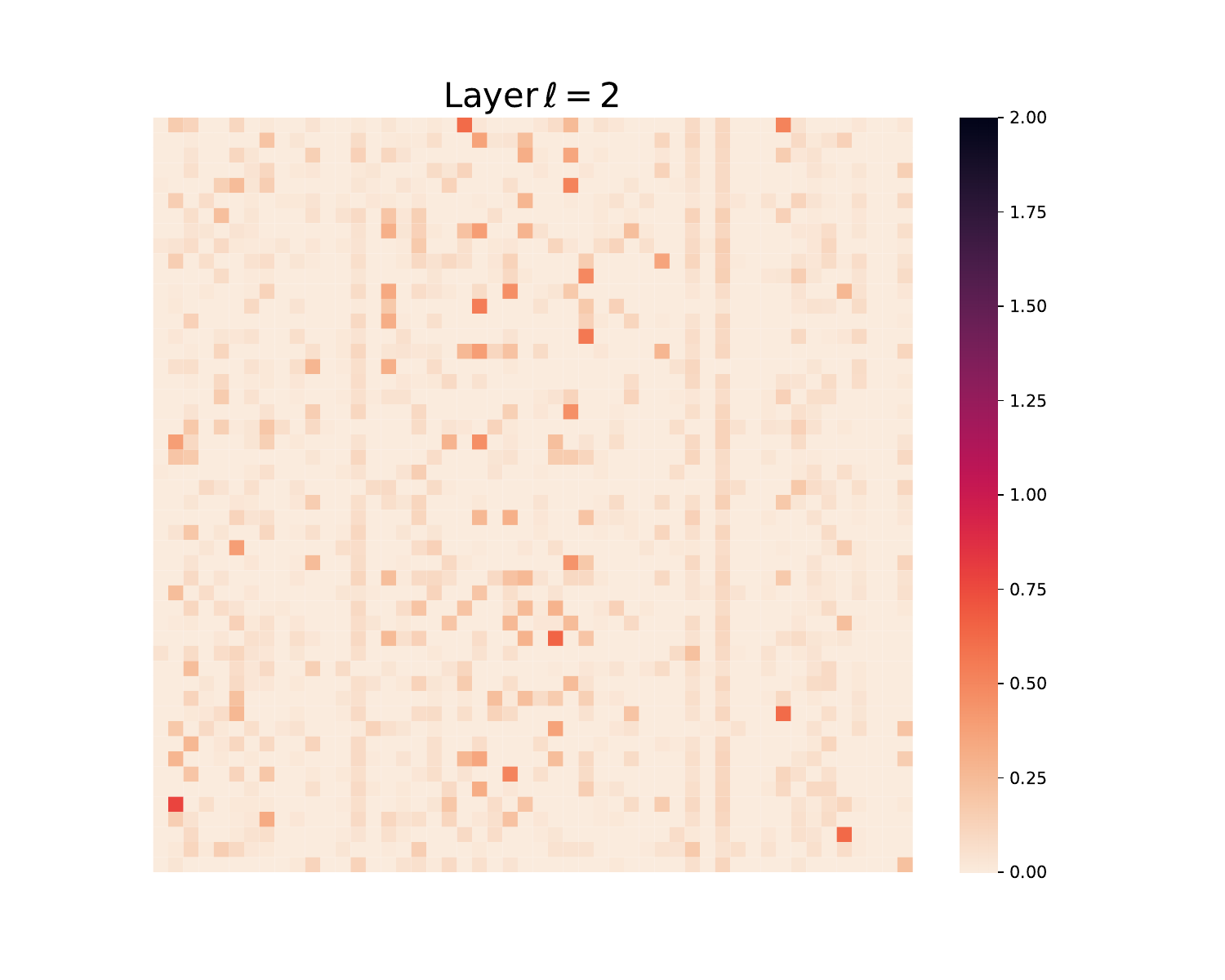}
     \end{subfigure}
     \begin{subfigure}[b]{0.2\textwidth}
         \centering
    \includegraphics[width=\textwidth]{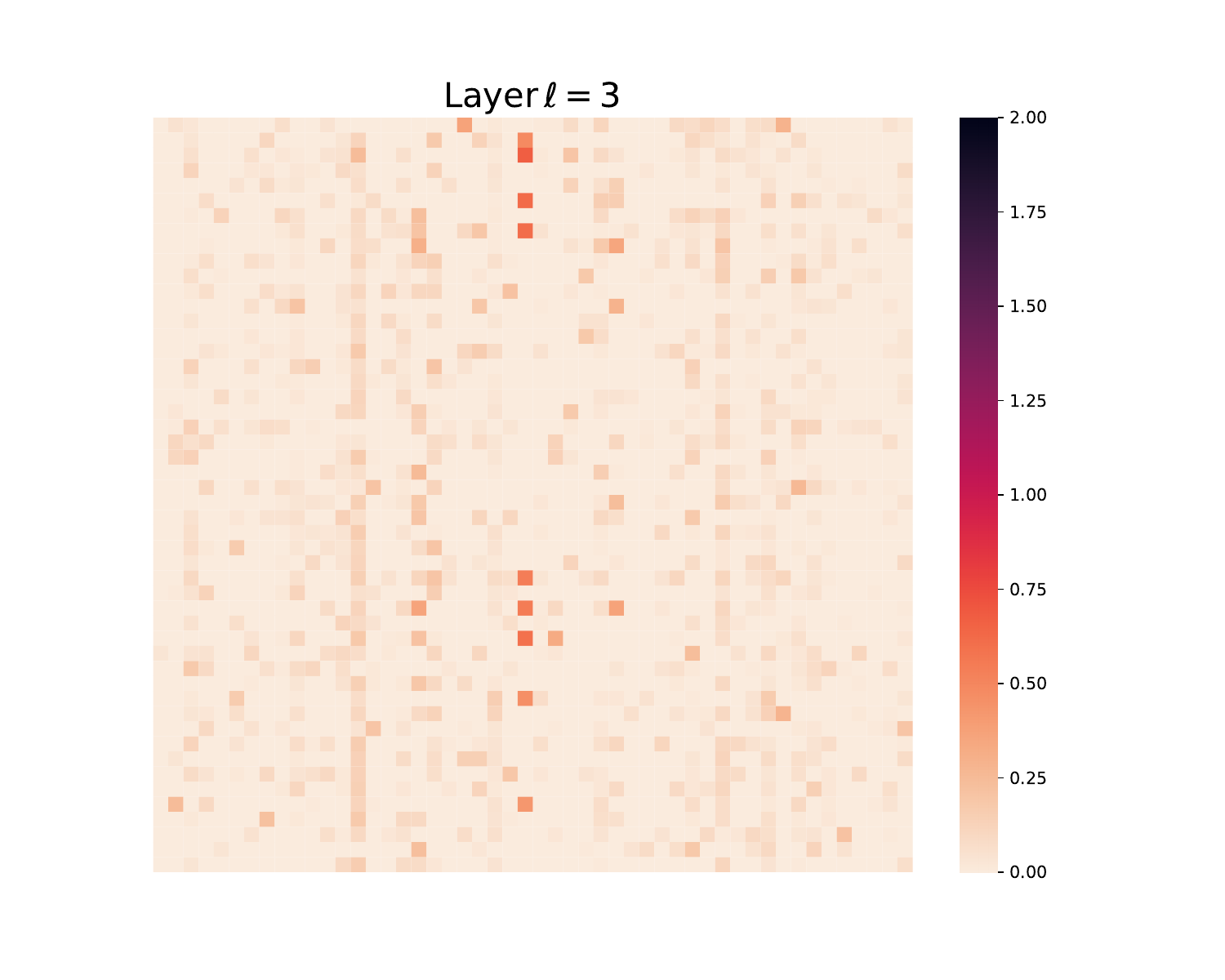}
     \end{subfigure}
     \begin{subfigure}[b]{0.2\textwidth}
         \centering
    \includegraphics[width=\textwidth]{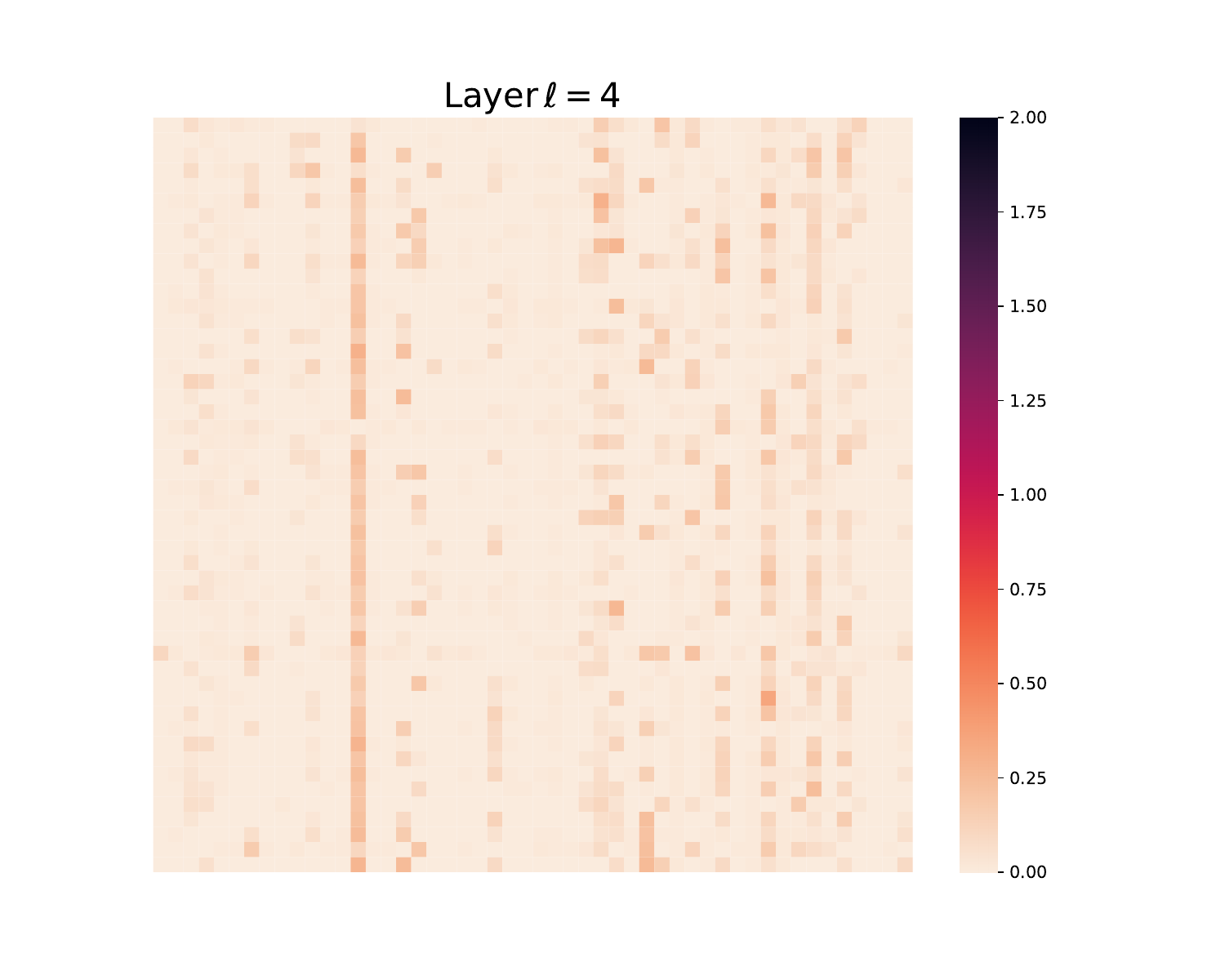}
     \end{subfigure}
     \begin{subfigure}[b]{0.2\textwidth}
         \centering
    \includegraphics[width=\textwidth]{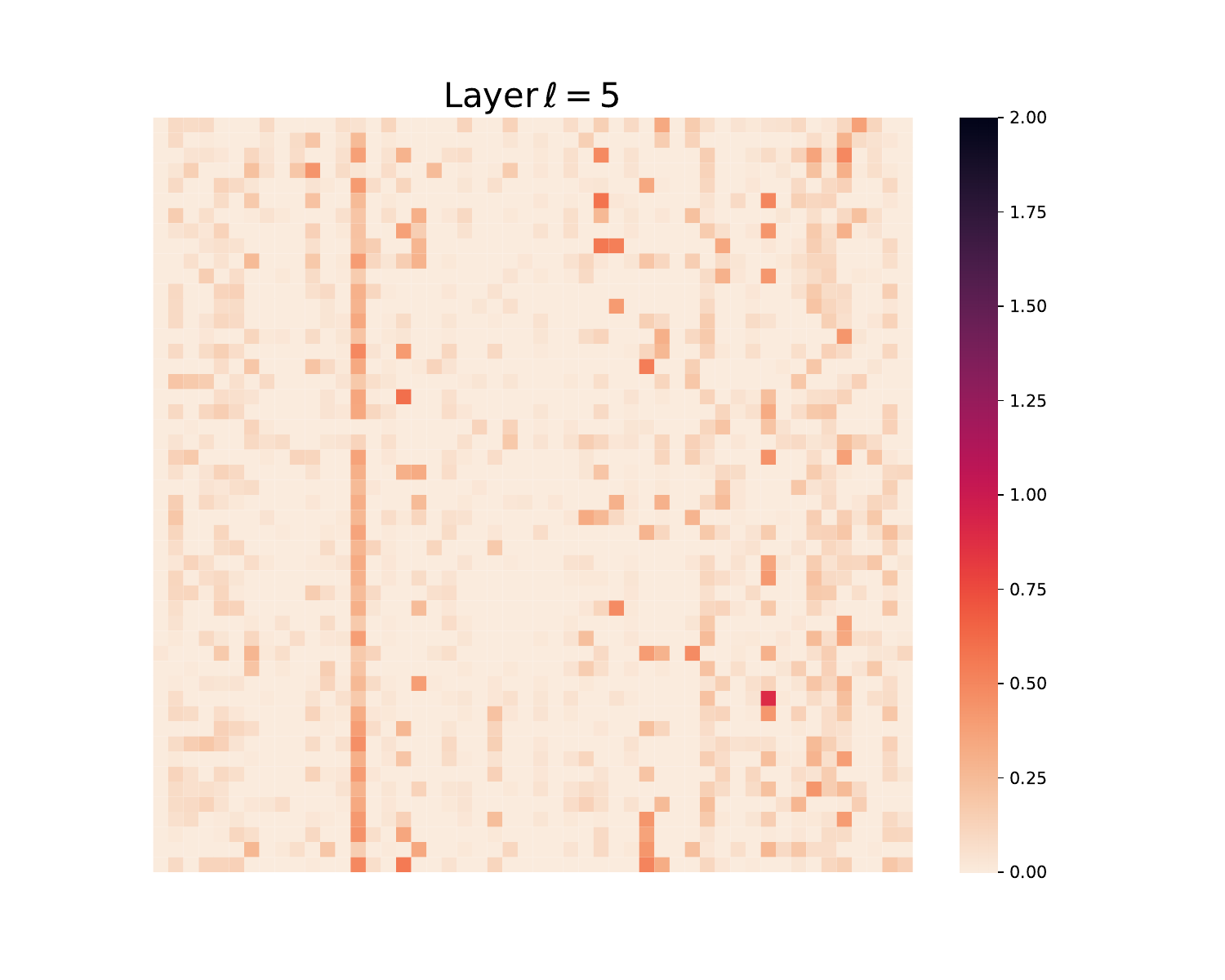}
     \end{subfigure}
     \begin{subfigure}[b]{0.2\textwidth}
         \centering
    \includegraphics[width=\textwidth]{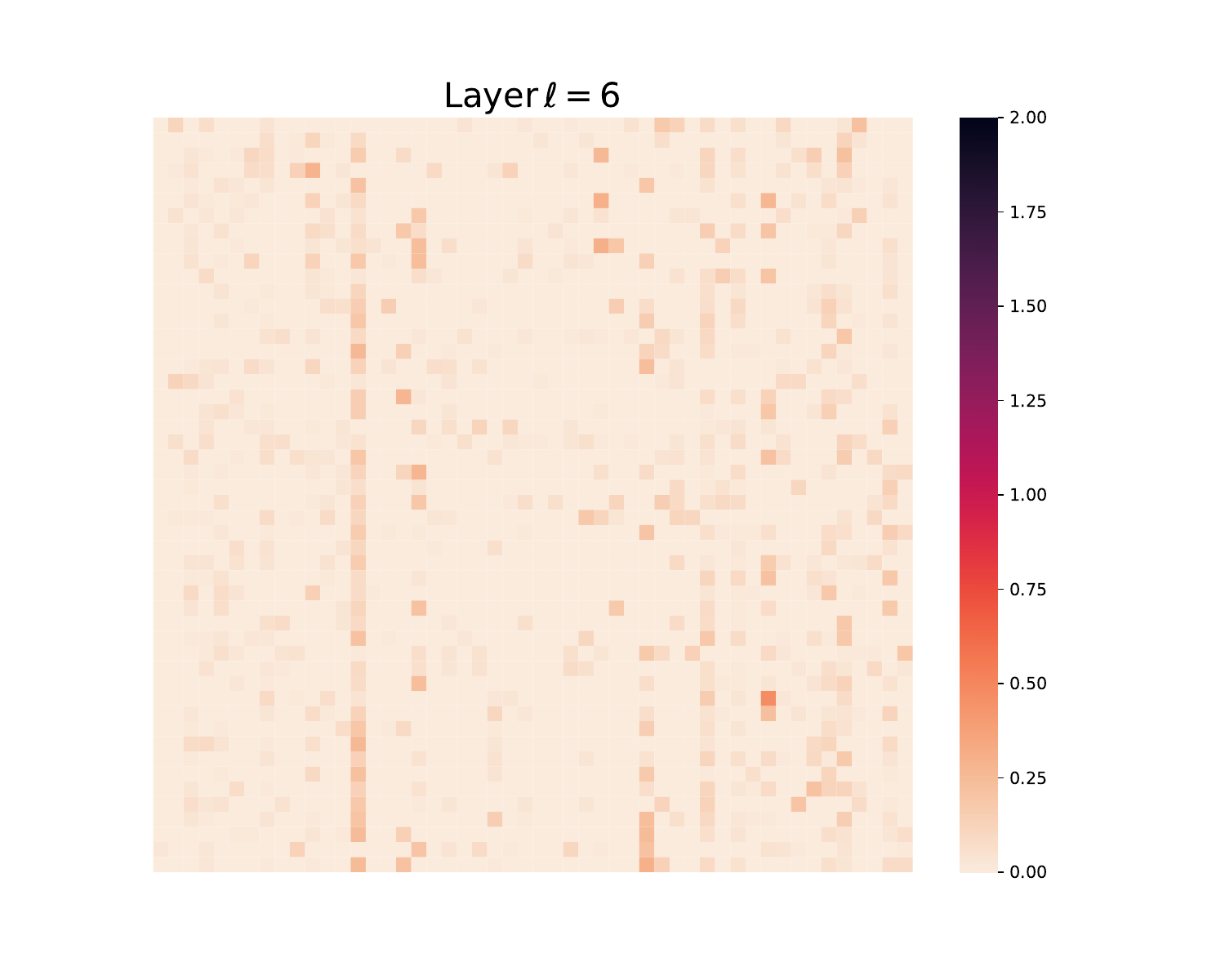}
     \end{subfigure}
     \begin{subfigure}[b]{0.2\textwidth}
         \centering
    \includegraphics[width=\textwidth]{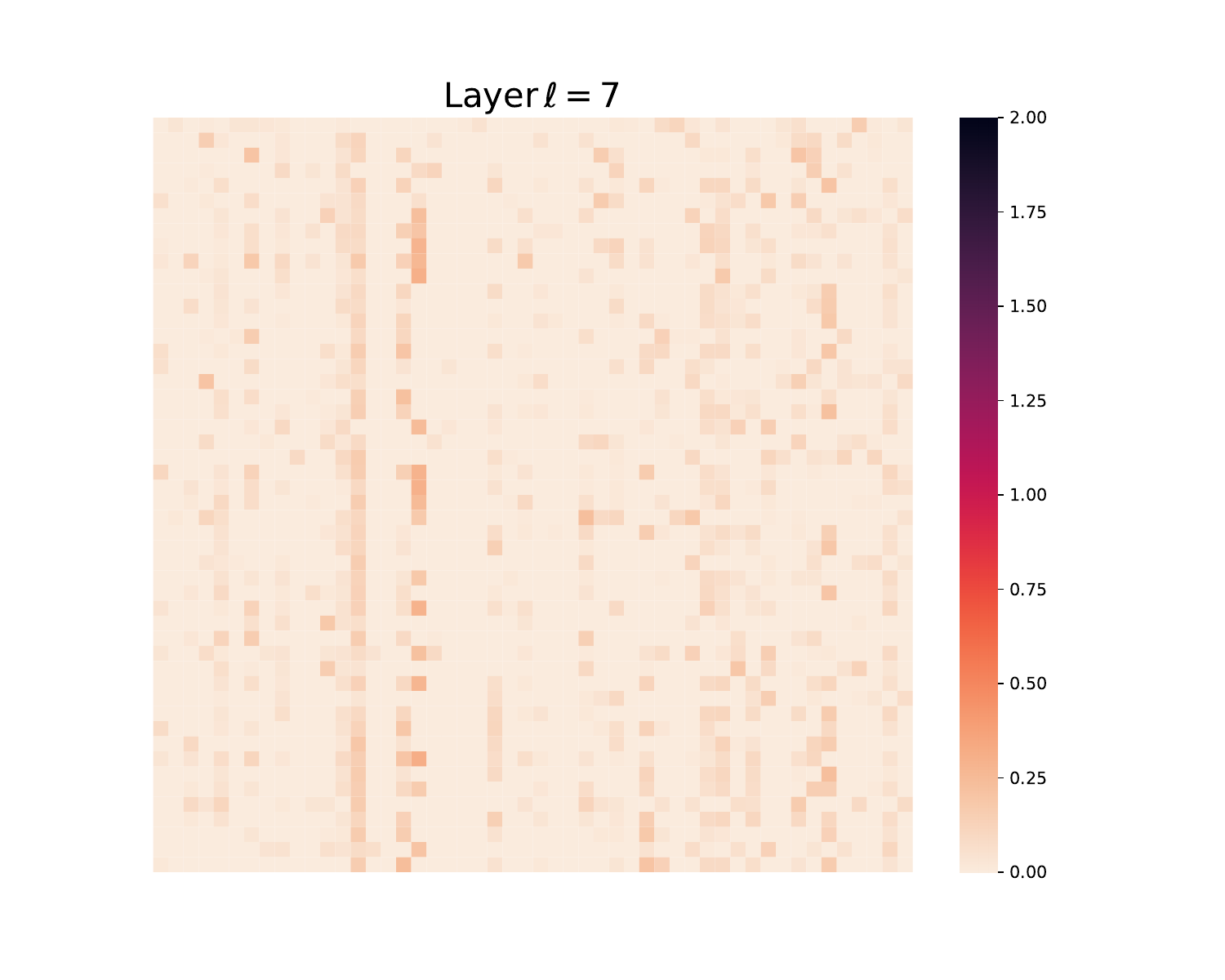}
     \end{subfigure}
     \begin{subfigure}[b]{0.2\textwidth}
         \centering
    \includegraphics[width=\textwidth]{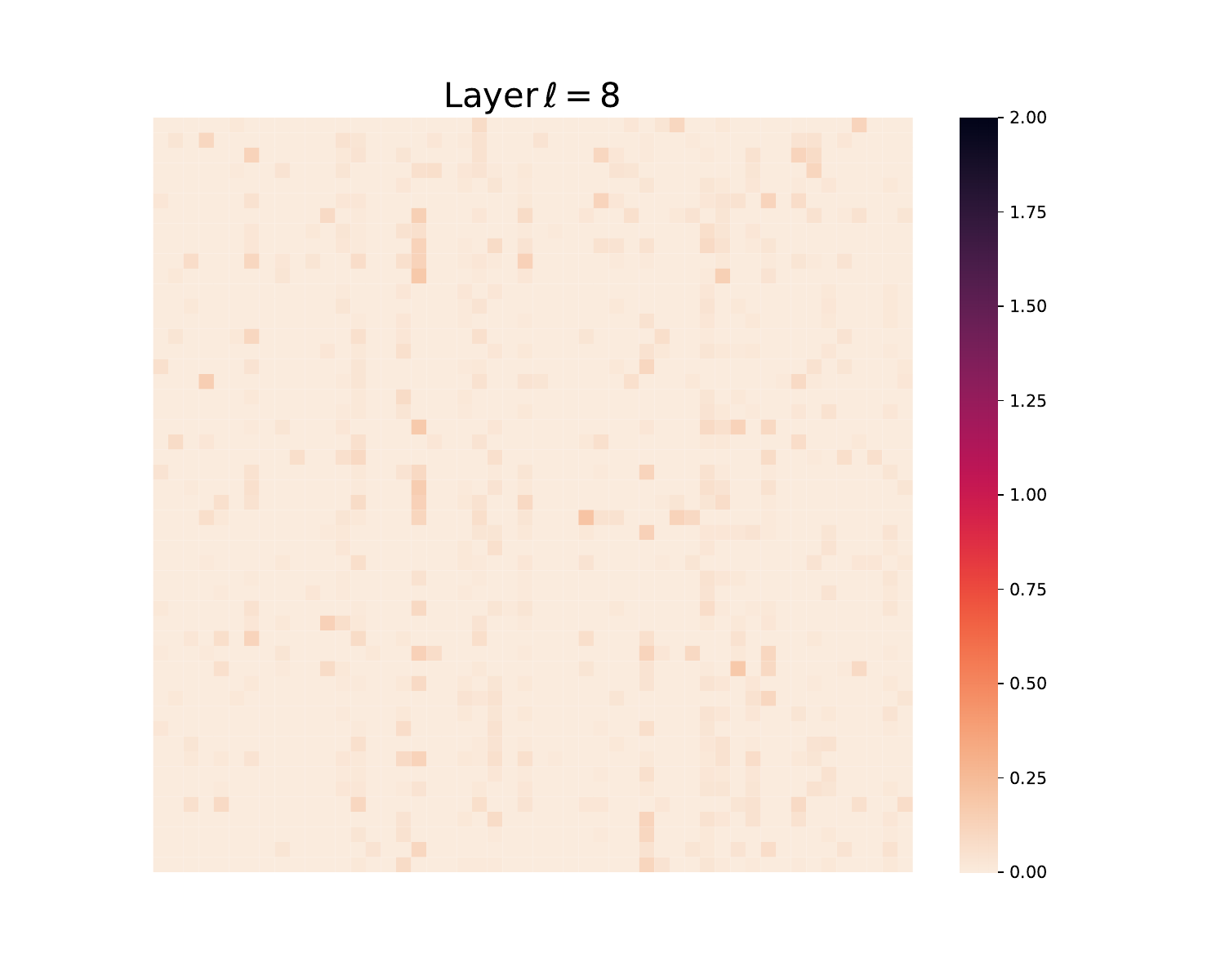}
     \end{subfigure}
     \begin{subfigure}[b]{0.2\textwidth}
         \centering
    \includegraphics[width=\textwidth]{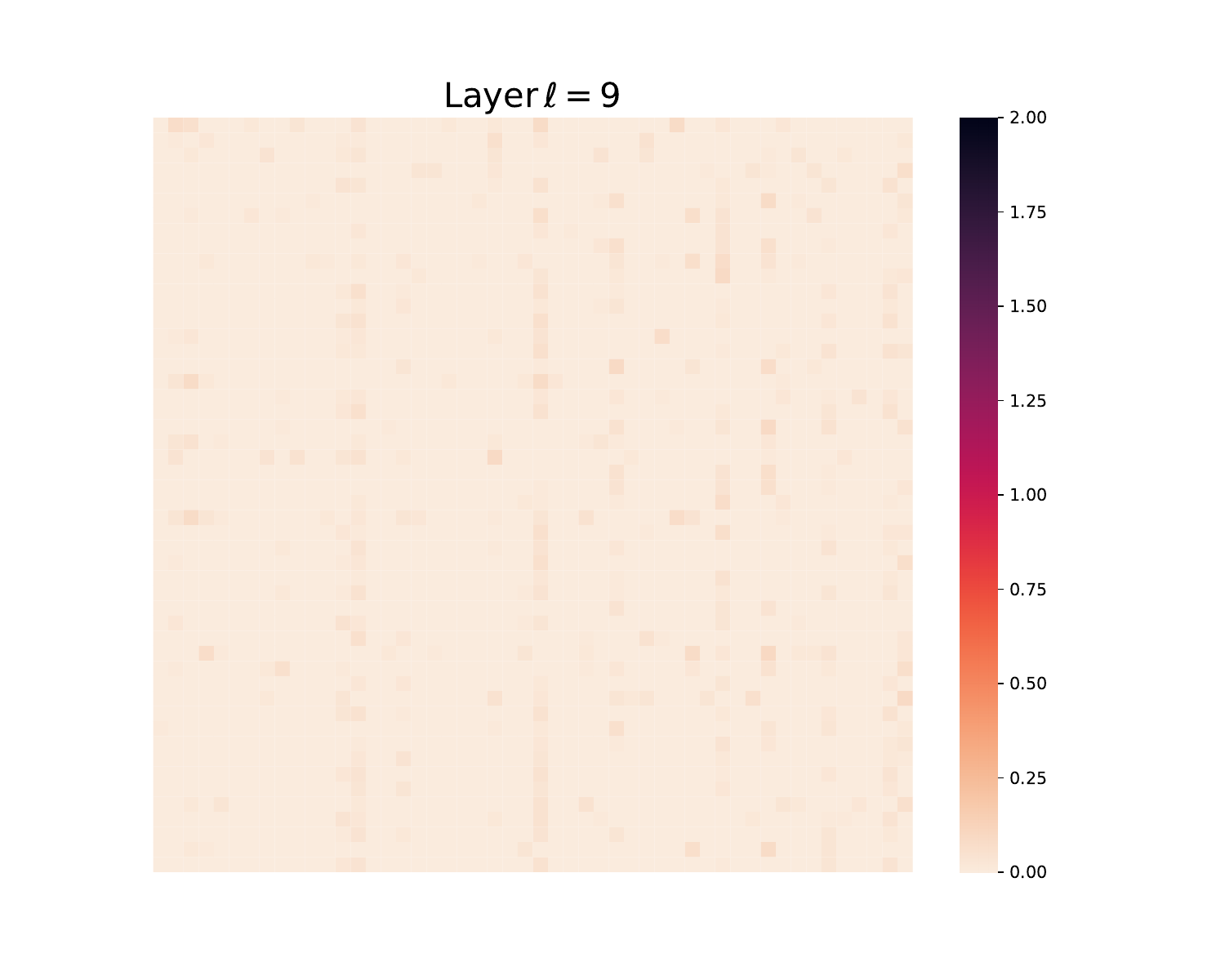}
     \end{subfigure}
     \begin{subfigure}[b]{0.2\textwidth}
         \centering
    \includegraphics[width=\textwidth]{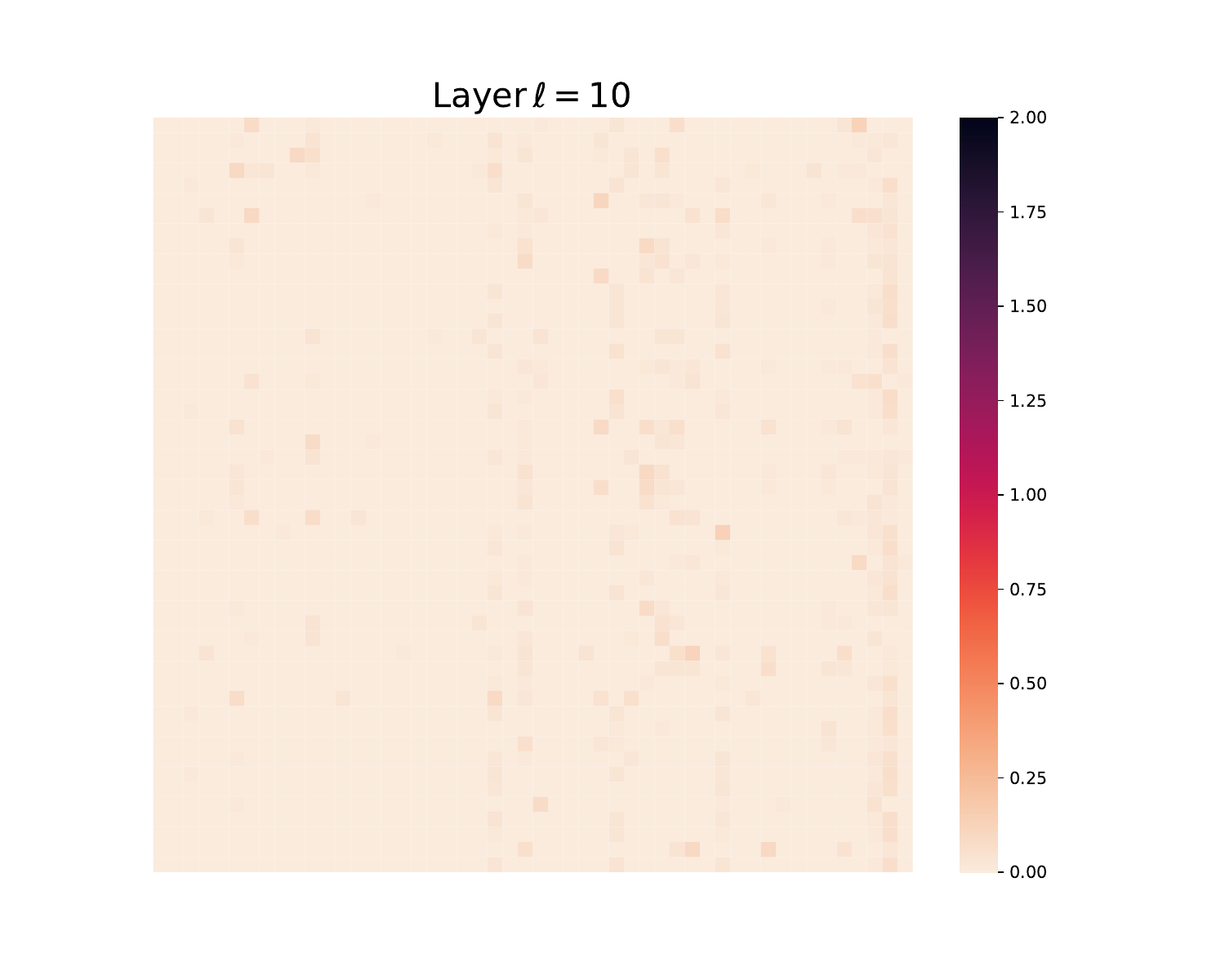}
     \end{subfigure}
     \begin{subfigure}[b]{0.2\textwidth}
         \centering
    \includegraphics[width=\textwidth]{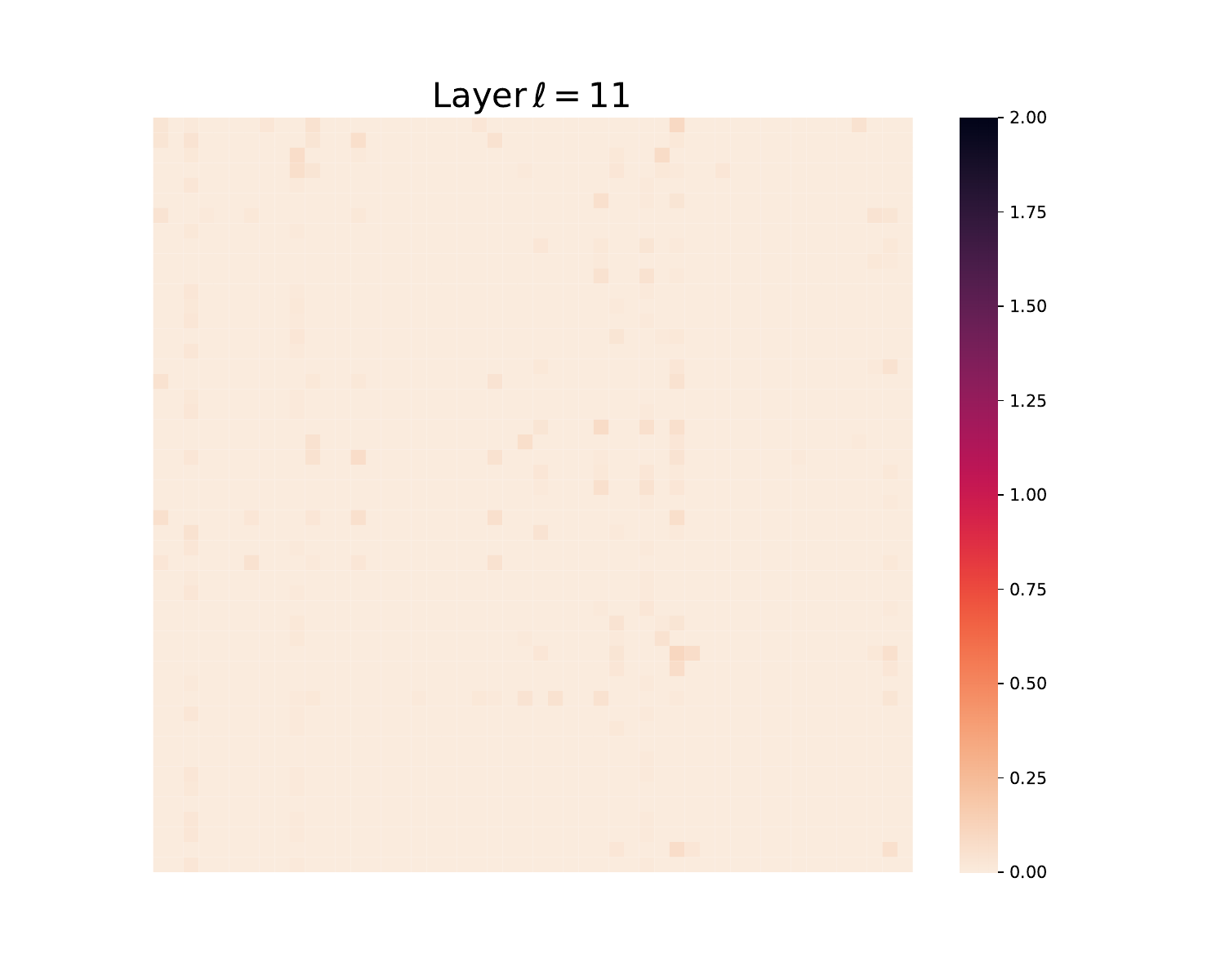}
     \end{subfigure}
     \begin{subfigure}[b]{0.2\textwidth}
         \centering
    \includegraphics[width=\textwidth]{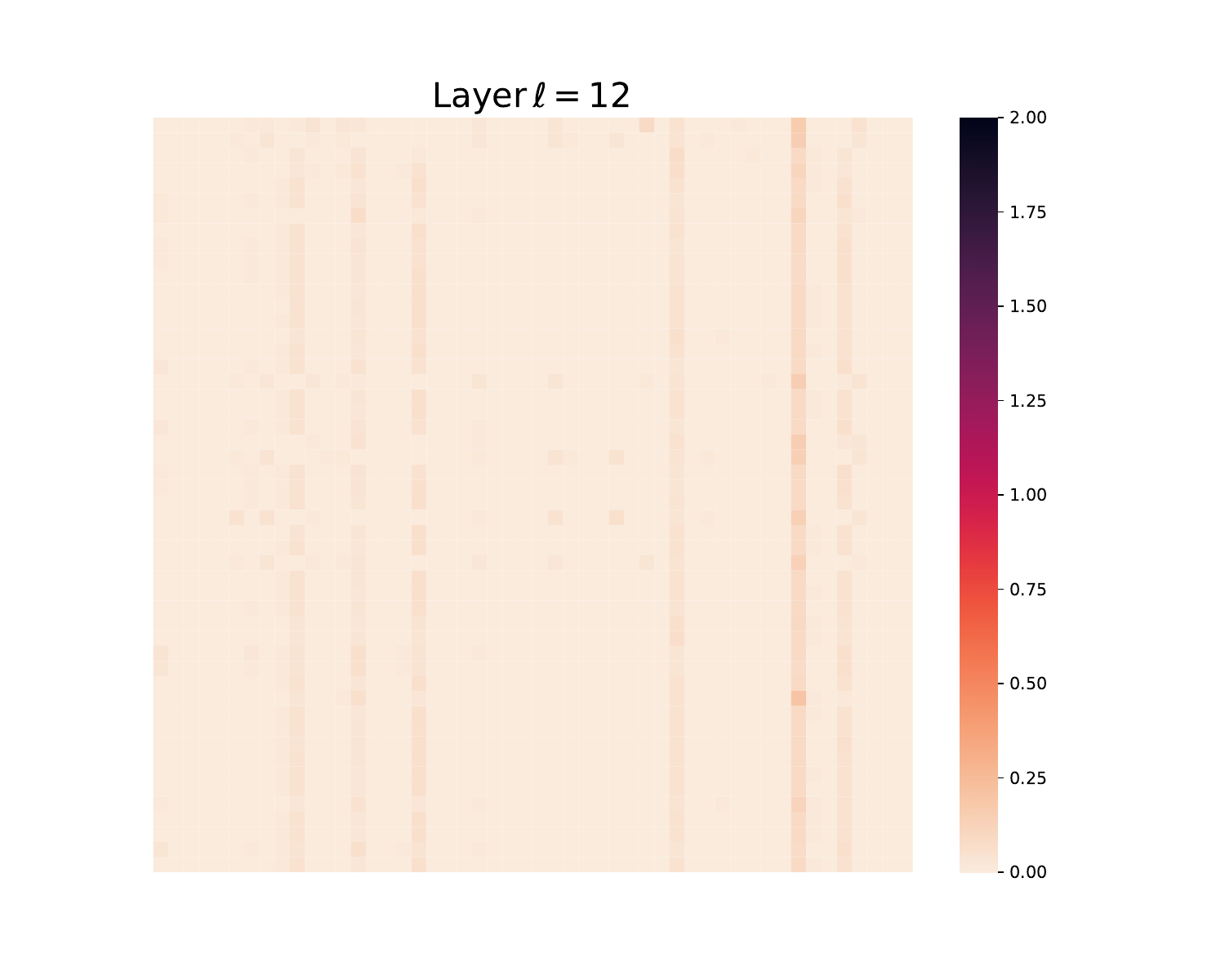}
     \end{subfigure}
        \caption{\small Visualizing layer-wise token $\vZ^{\ell}$ representations at each layer $\ell$. To enhance the visual clarity, we randomly extract a 50$\times$50 sub-matrix from $\vZ^{\ell}$ for display purposes. (\textit{Sample 1})}
        \label{fig:appendix-exp-ista-sparsity-heatmap-sample1}
\end{figure}

\begin{figure}[ht]
     \centering
     \begin{subfigure}[b]{0.2\textwidth}
         \centering
    \includegraphics[width=\textwidth]{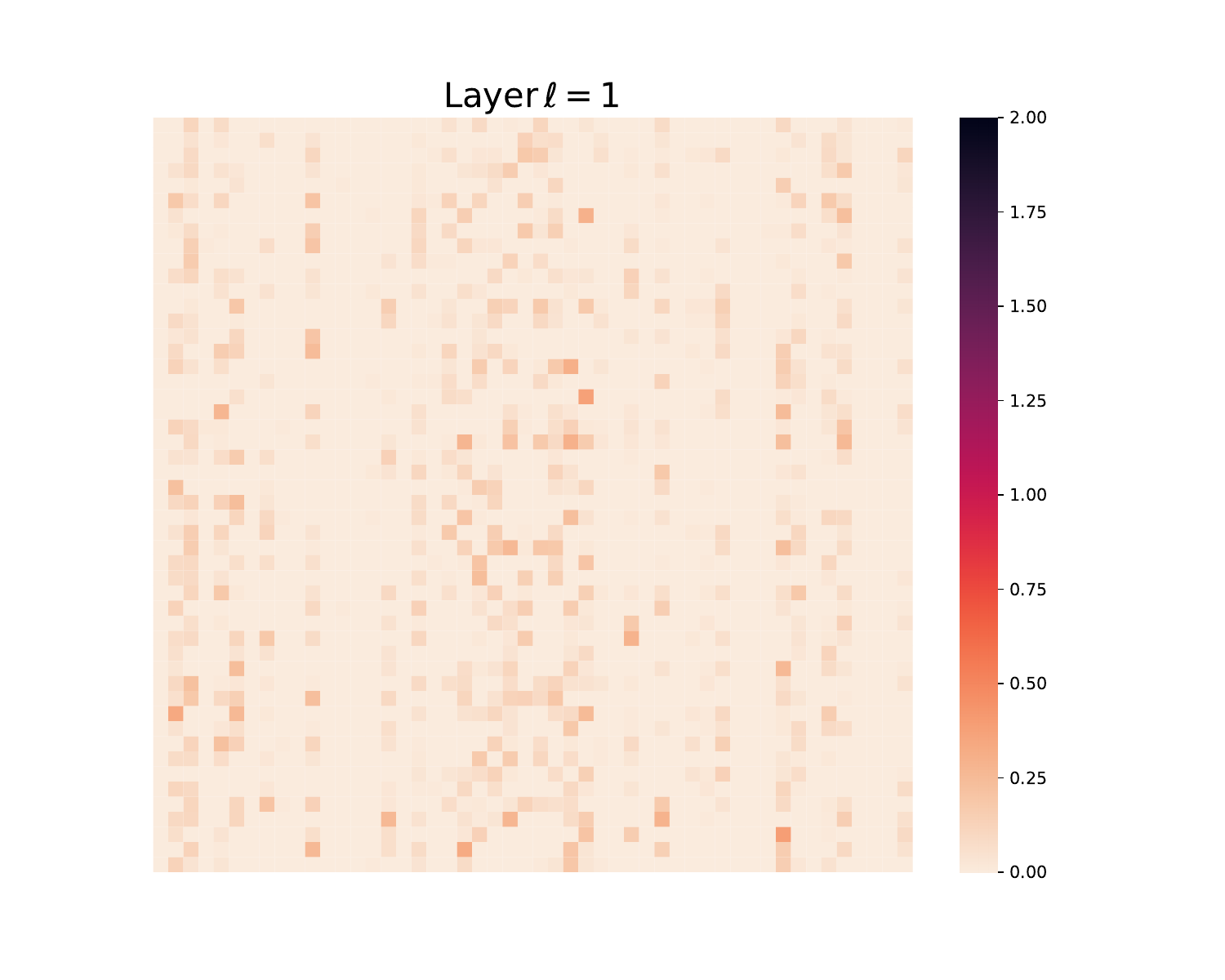}
     \end{subfigure}
     \begin{subfigure}[b]{0.2\textwidth}
         \centering
    \includegraphics[width=\textwidth]{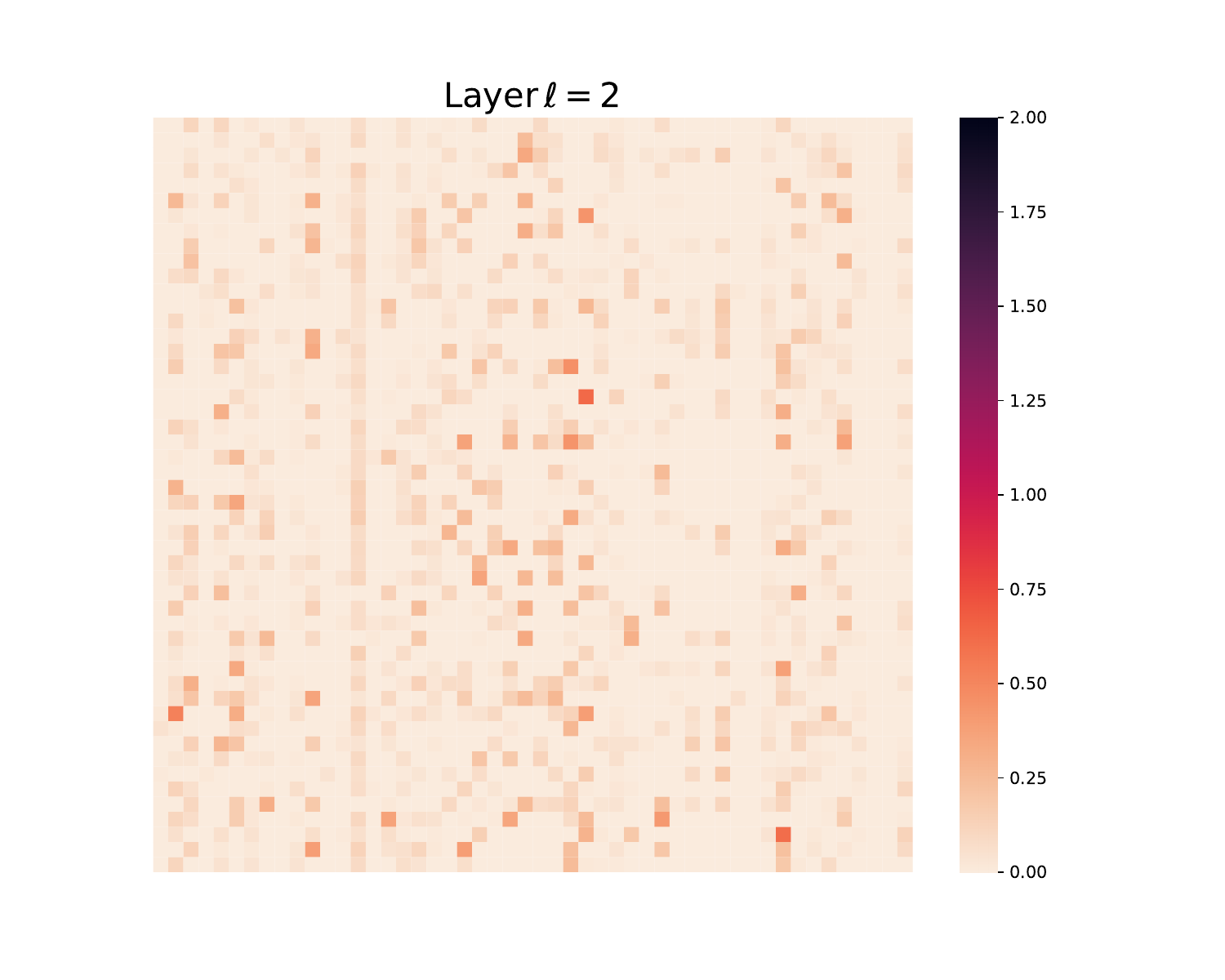}
     \end{subfigure}
     \begin{subfigure}[b]{0.2\textwidth}
         \centering
    \includegraphics[width=\textwidth]{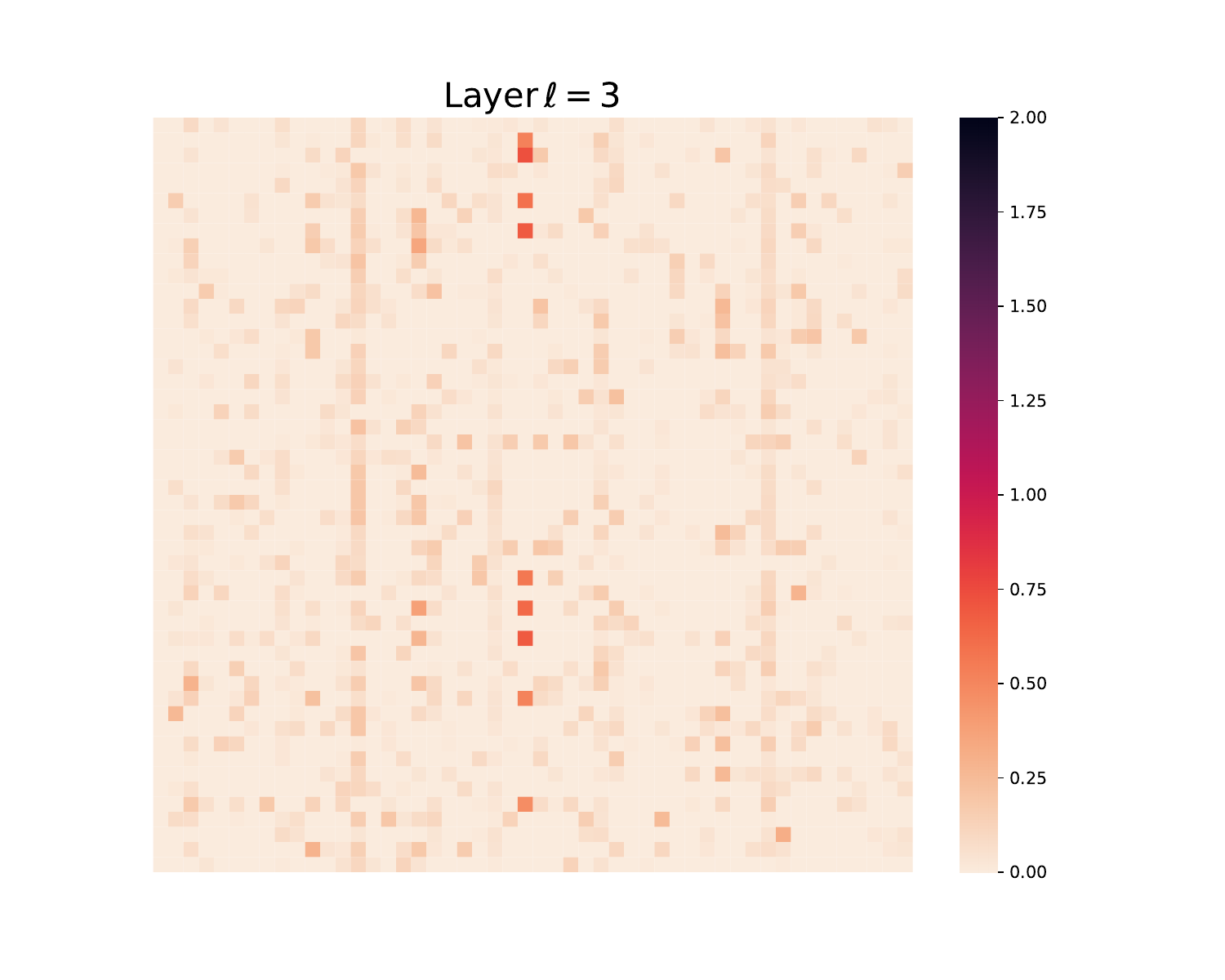}
     \end{subfigure}
     \begin{subfigure}[b]{0.2\textwidth}
         \centering
    \includegraphics[width=\textwidth]{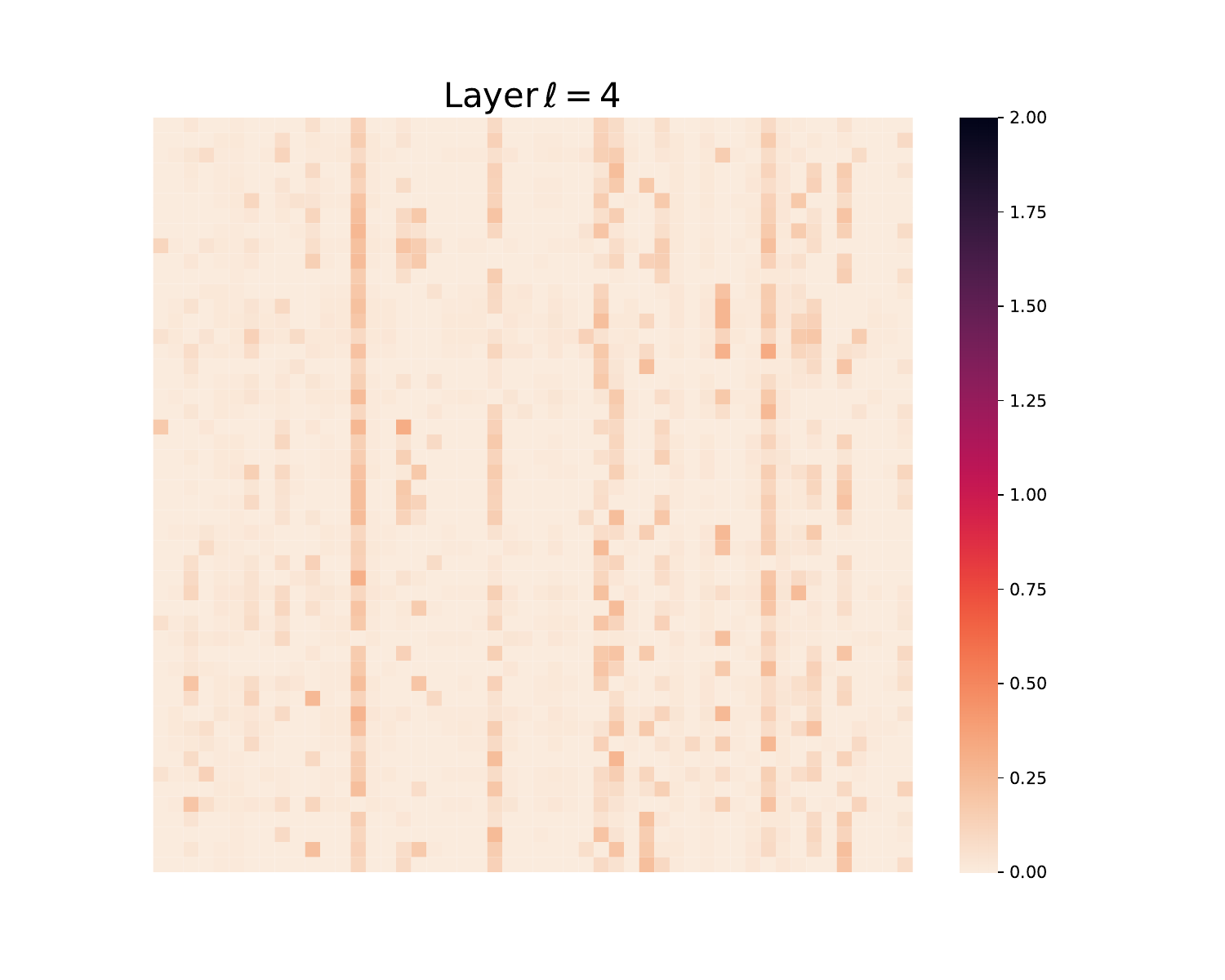}
     \end{subfigure}
     \begin{subfigure}[b]{0.2\textwidth}
         \centering
    \includegraphics[width=\textwidth]{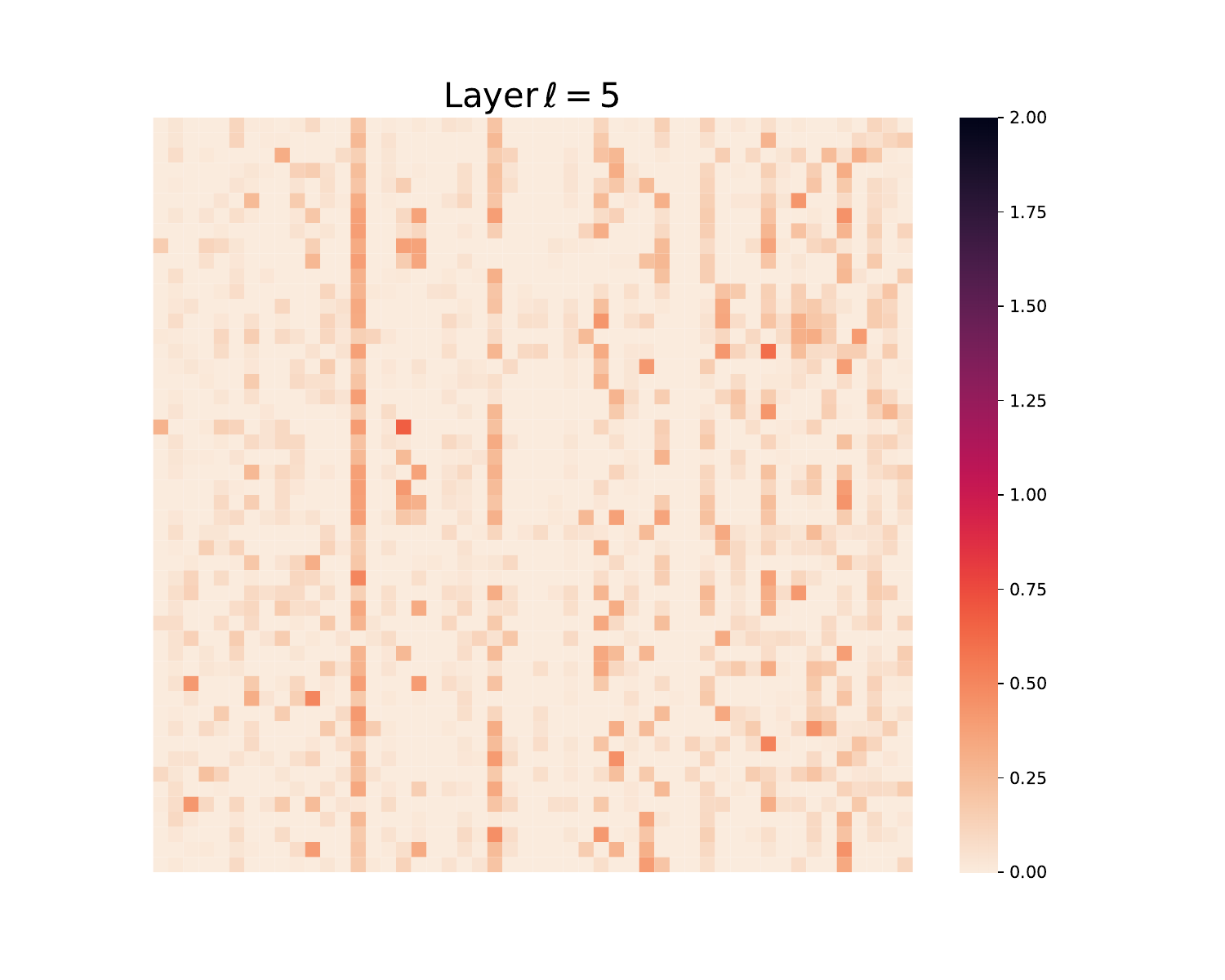}
     \end{subfigure}
     \begin{subfigure}[b]{0.2\textwidth}
         \centering
    \includegraphics[width=\textwidth]{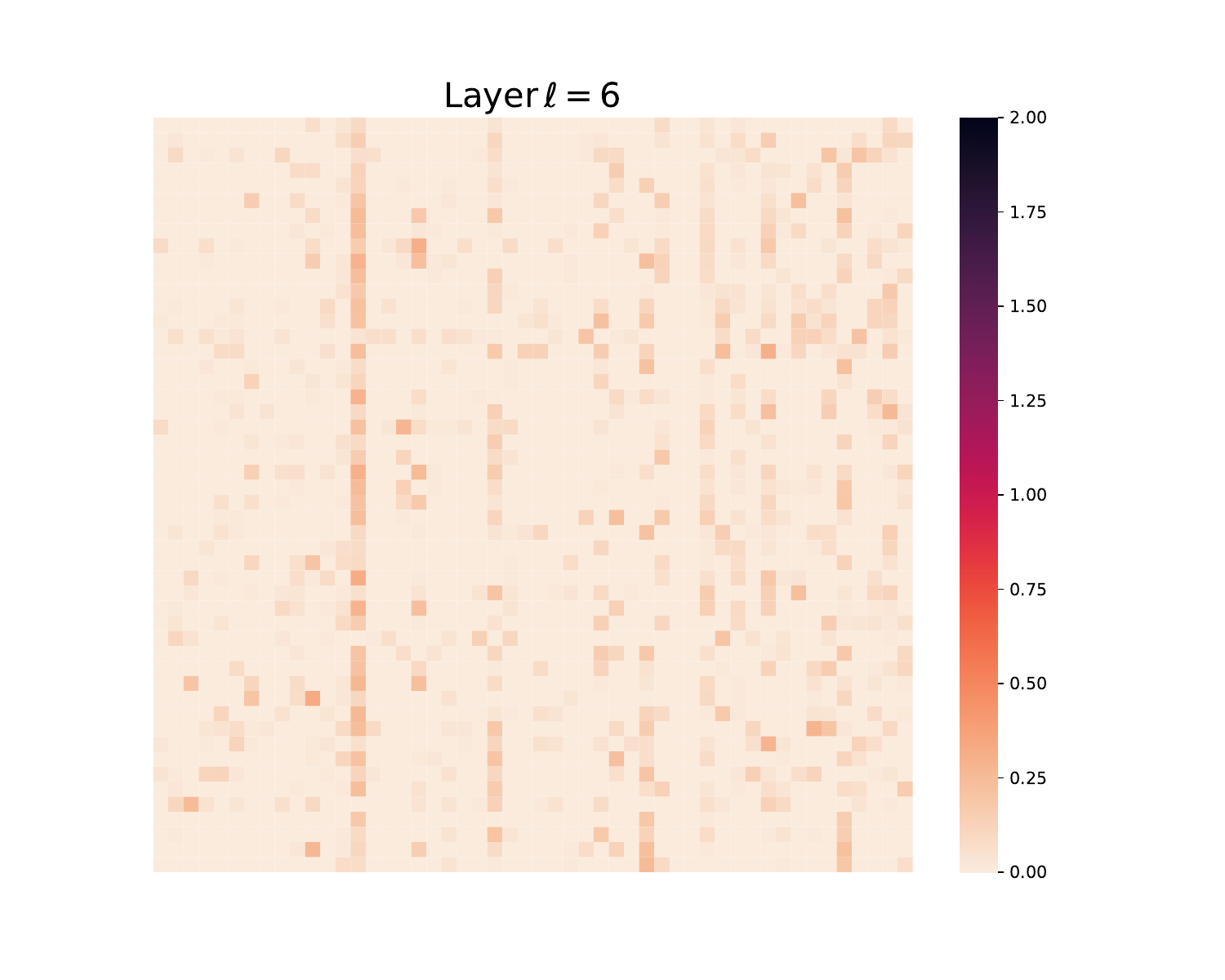}
     \end{subfigure}
     \begin{subfigure}[b]{0.2\textwidth}
         \centering
    \includegraphics[width=\textwidth]{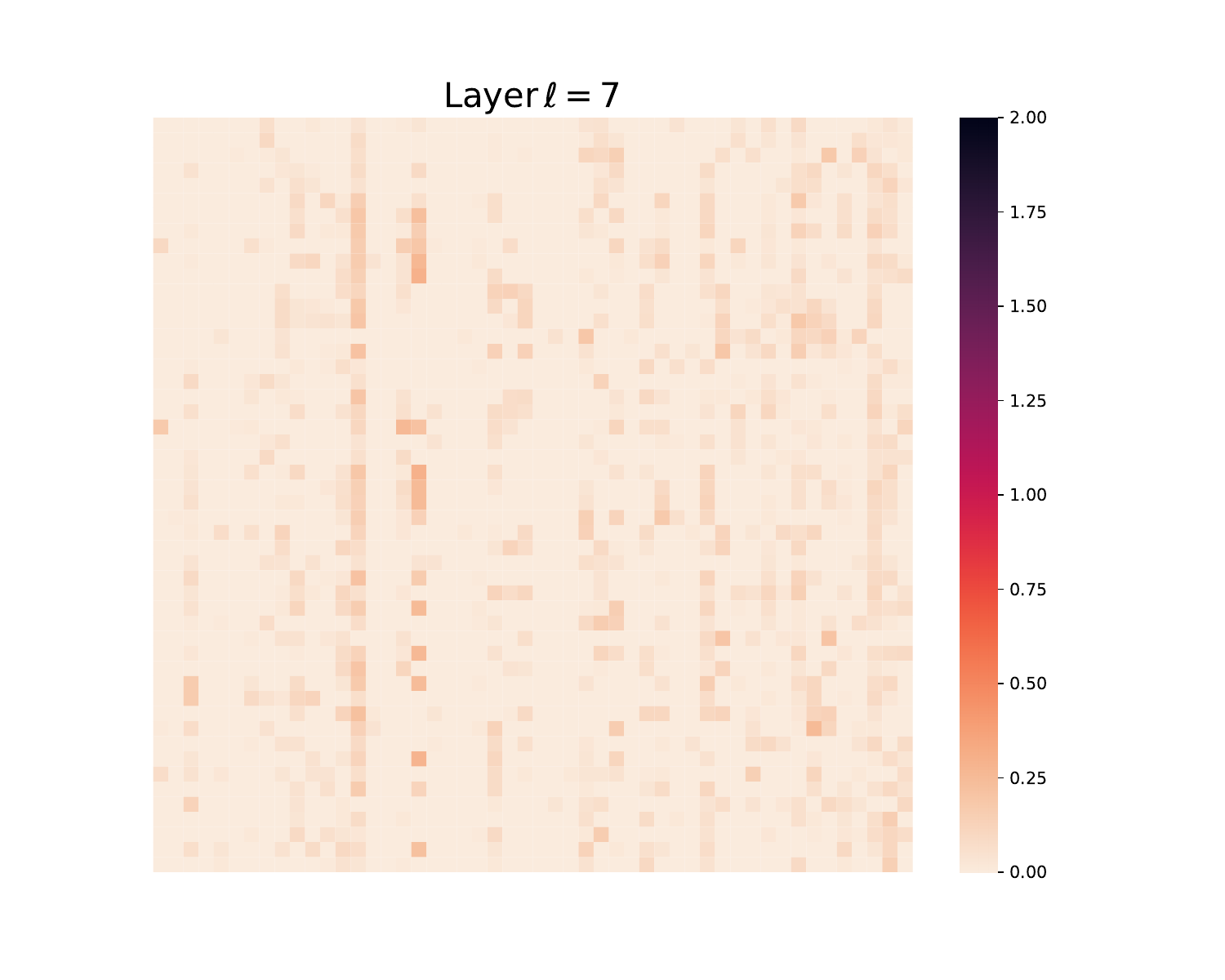}
     \end{subfigure}
     \begin{subfigure}[b]{0.2\textwidth}
         \centering
    \includegraphics[width=\textwidth]{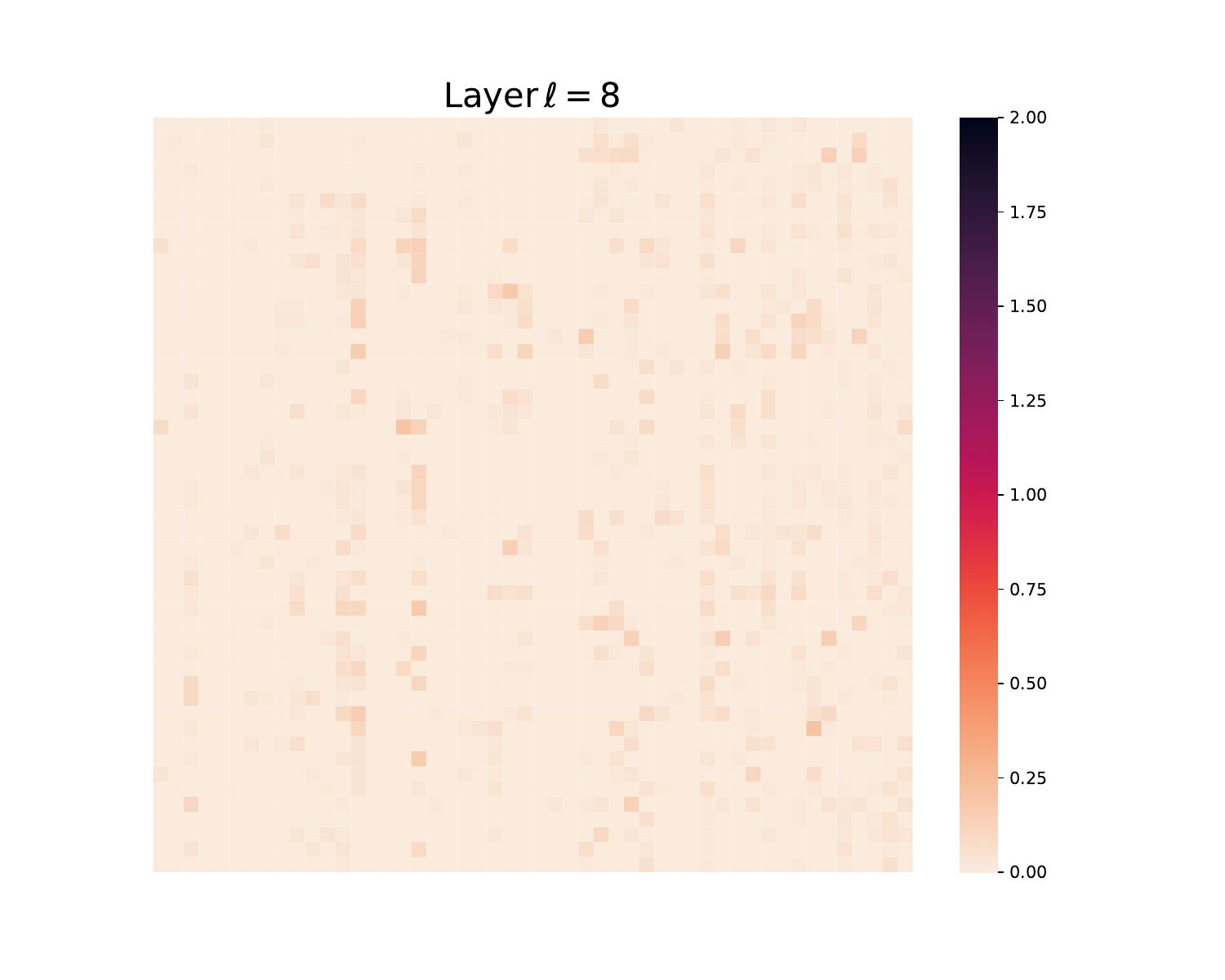}
     \end{subfigure}
     \begin{subfigure}[b]{0.2\textwidth}
         \centering
    \includegraphics[width=\textwidth]{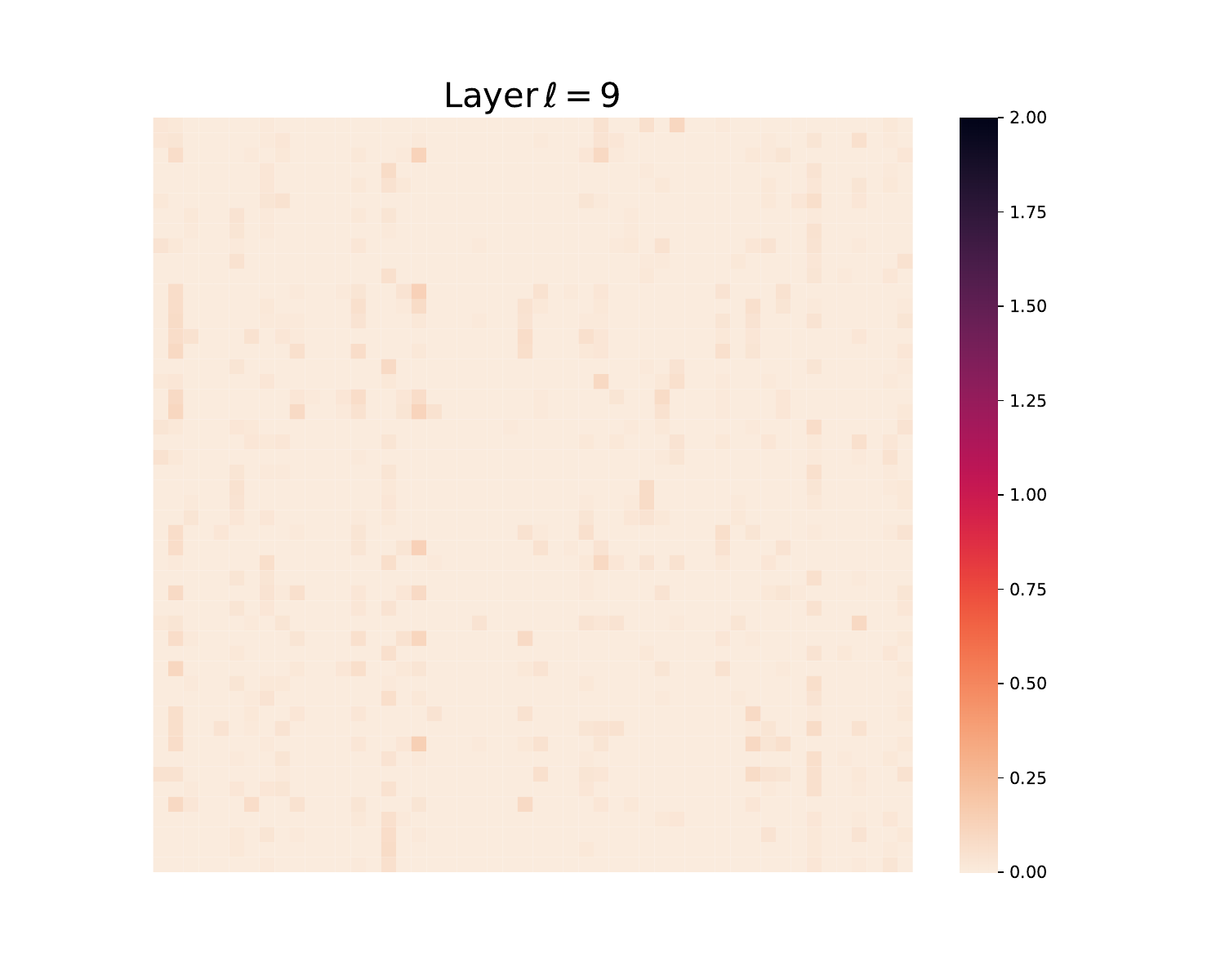}
     \end{subfigure}
     \begin{subfigure}[b]{0.2\textwidth}
         \centering
    \includegraphics[width=\textwidth]{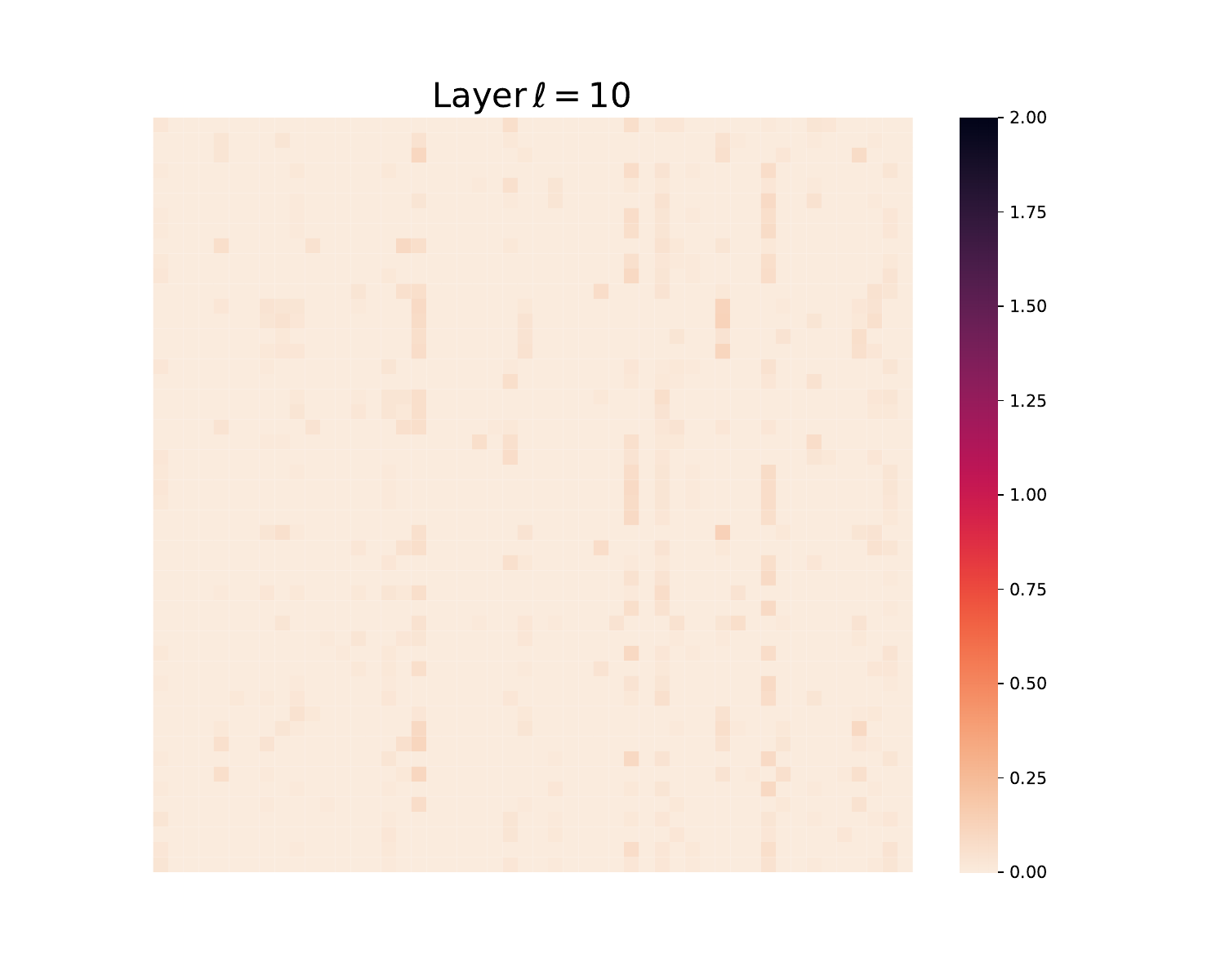}
     \end{subfigure}
     \begin{subfigure}[b]{0.2\textwidth}
         \centering
    \includegraphics[width=\textwidth]{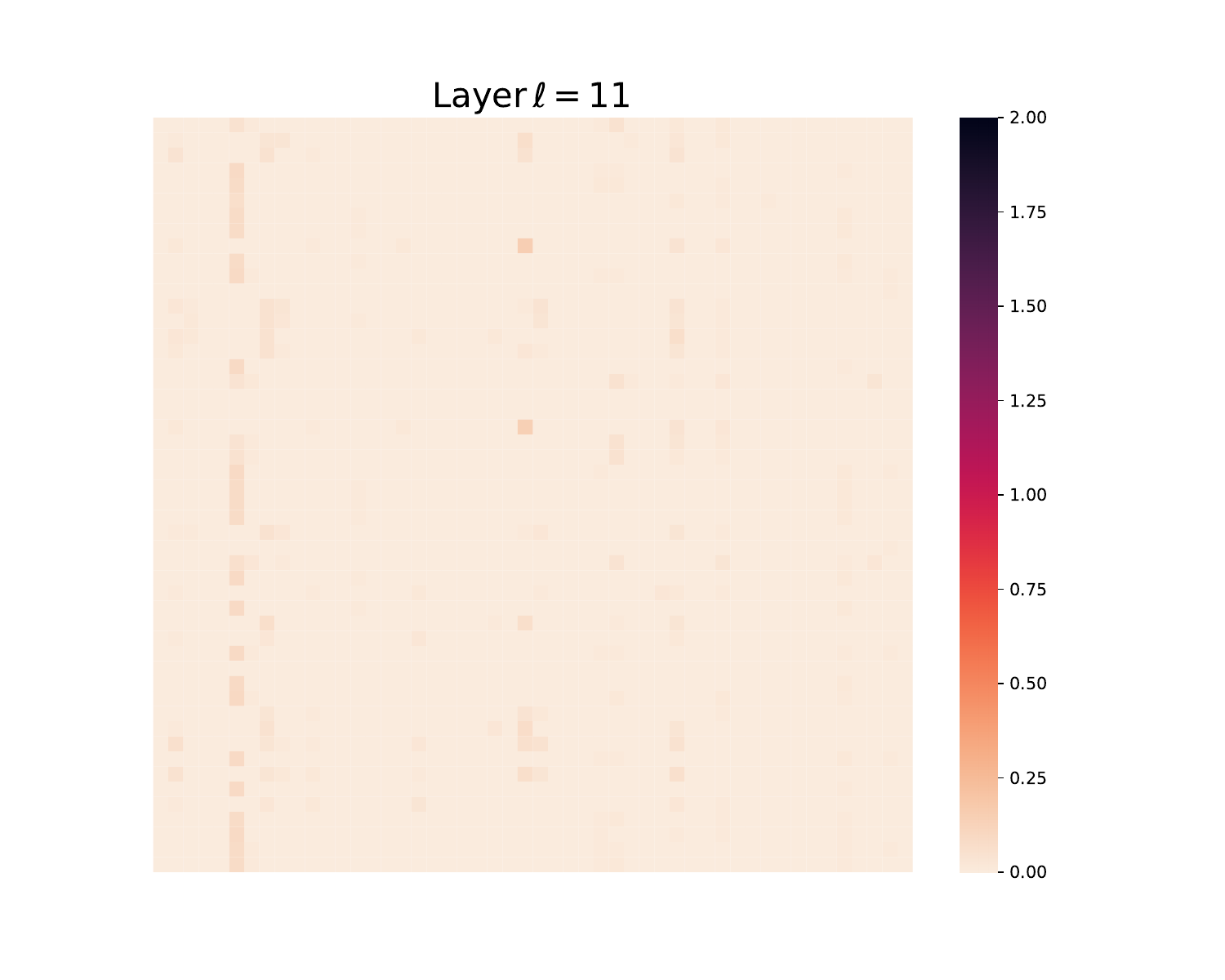}
     \end{subfigure}
     \begin{subfigure}[b]{0.2\textwidth}
         \centering
    \includegraphics[width=\textwidth]{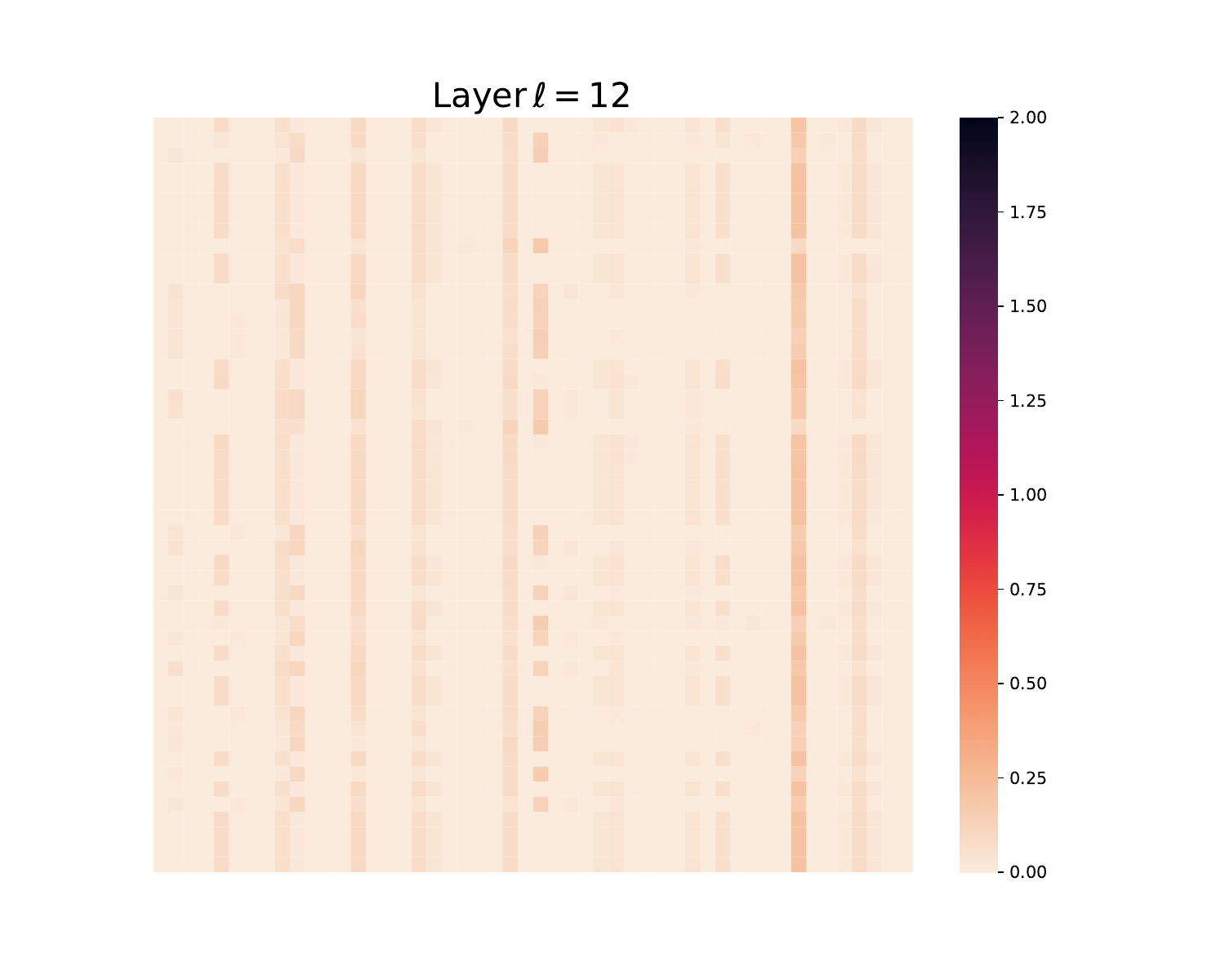}
     \end{subfigure}
        \caption{\small Visualizing layer-wise token $\vZ^{\ell}$ representations at each layer $\ell$. To enhance the visual clarity, we randomly extract a 50$\times$50 sub-matrix from $\vZ^{\ell}$ for display purposes. (\textit{Sample 2})}
        \label{fig:appendix-exp-ista-sparsity-heatmap-sample2}
\end{figure}

\begin{figure}[ht]
     \centering
     \begin{subfigure}[b]{0.2\textwidth}
         \centering
    \includegraphics[width=\textwidth]{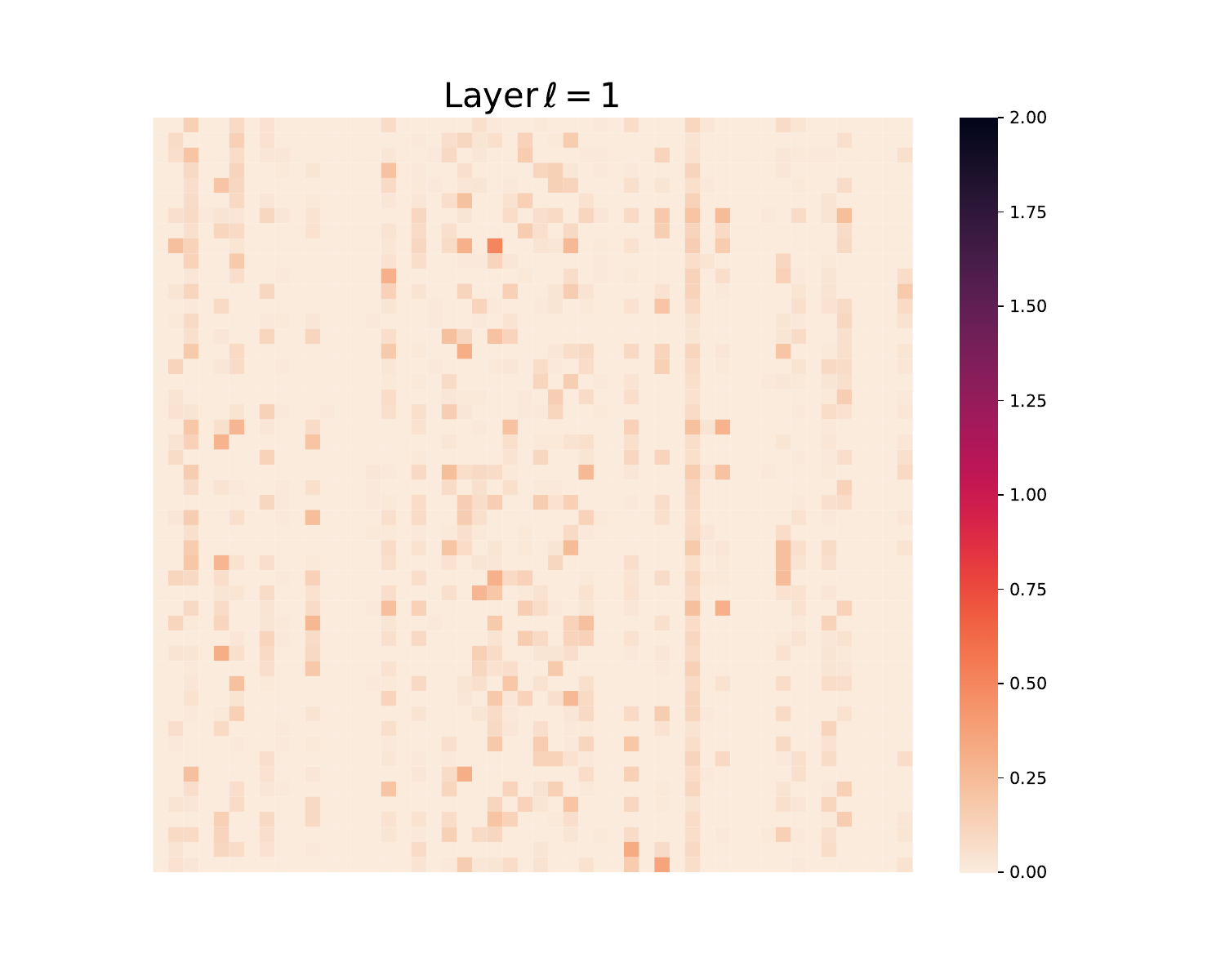}
     \end{subfigure}
     \begin{subfigure}[b]{0.2\textwidth}
         \centering
    \includegraphics[width=\textwidth]{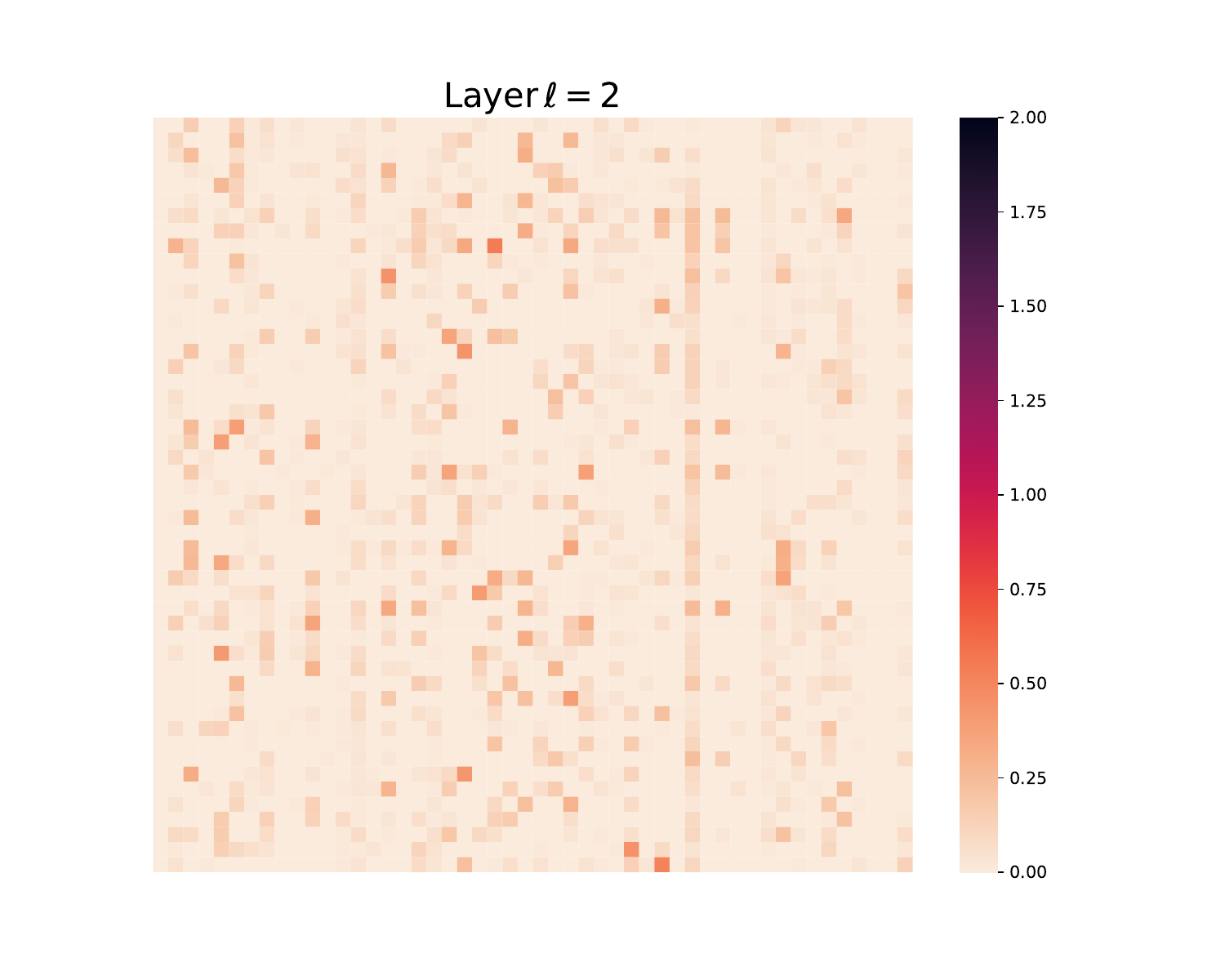}
     \end{subfigure}
     \begin{subfigure}[b]{0.2\textwidth}
         \centering
    \includegraphics[width=\textwidth]{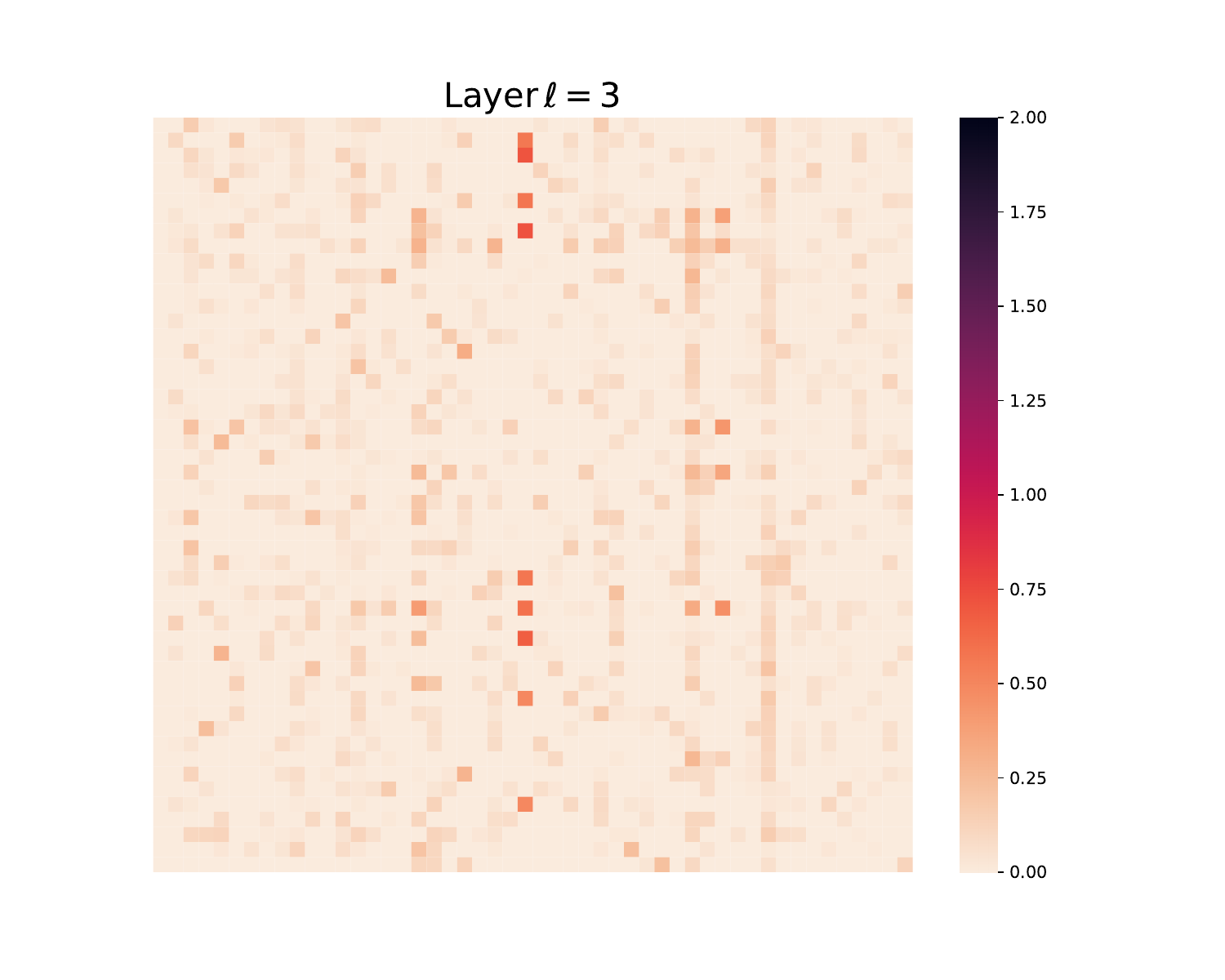}
     \end{subfigure}
     \begin{subfigure}[b]{0.2\textwidth}
         \centering
    \includegraphics[width=\textwidth]{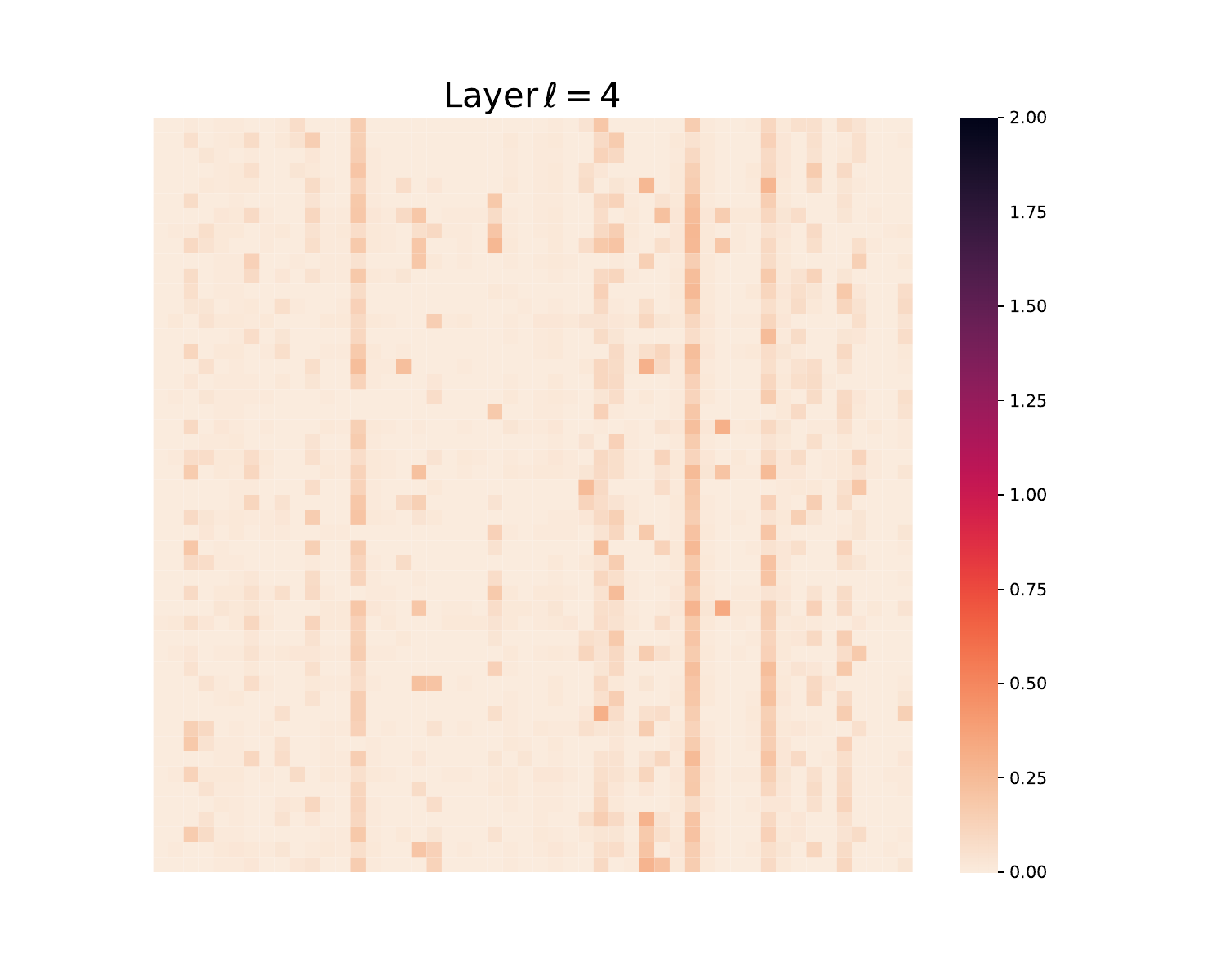}
     \end{subfigure}
     \begin{subfigure}[b]{0.2\textwidth}
         \centering
    \includegraphics[width=\textwidth]{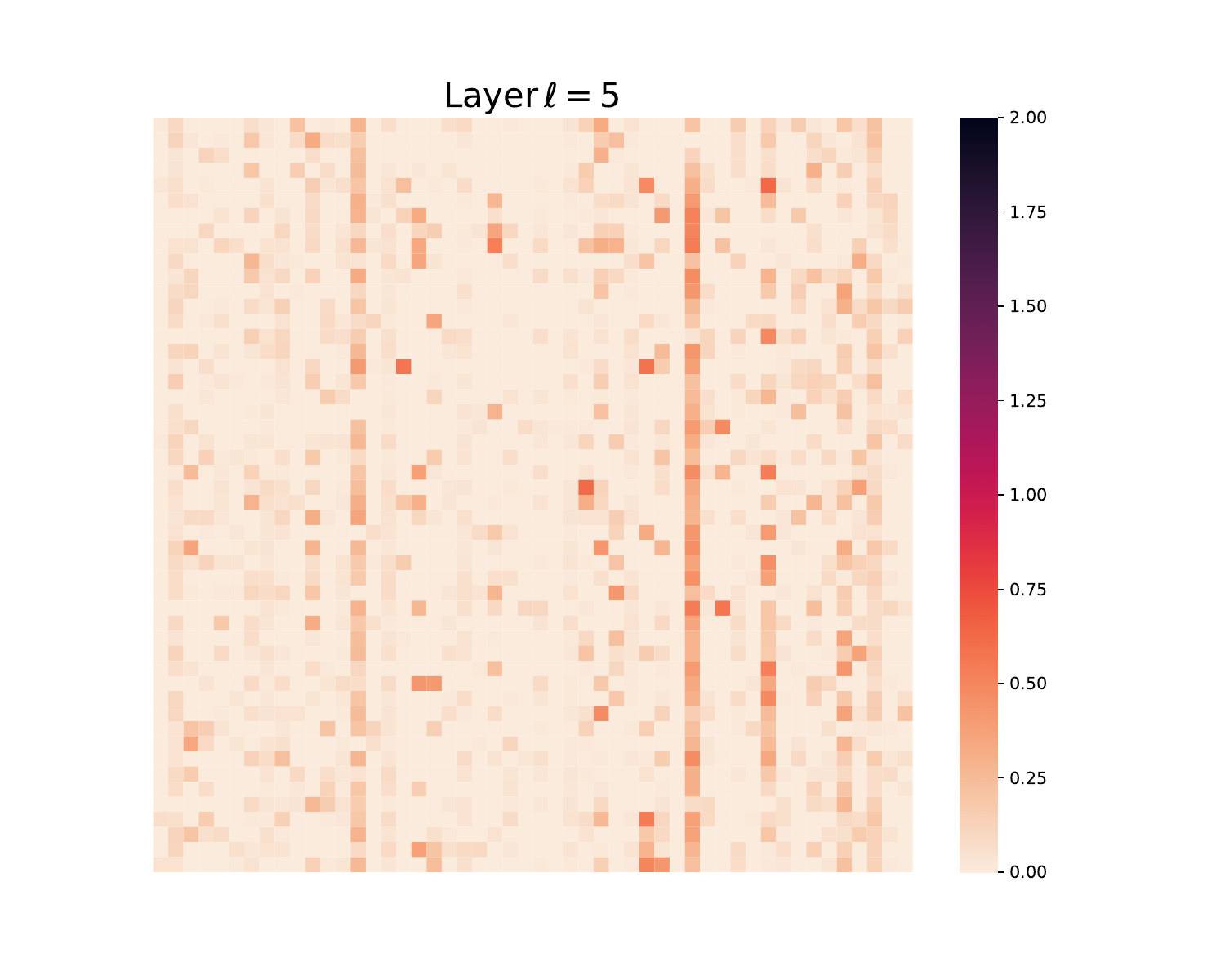}
     \end{subfigure}
     \begin{subfigure}[b]{0.2\textwidth}
         \centering
    \includegraphics[width=\textwidth]{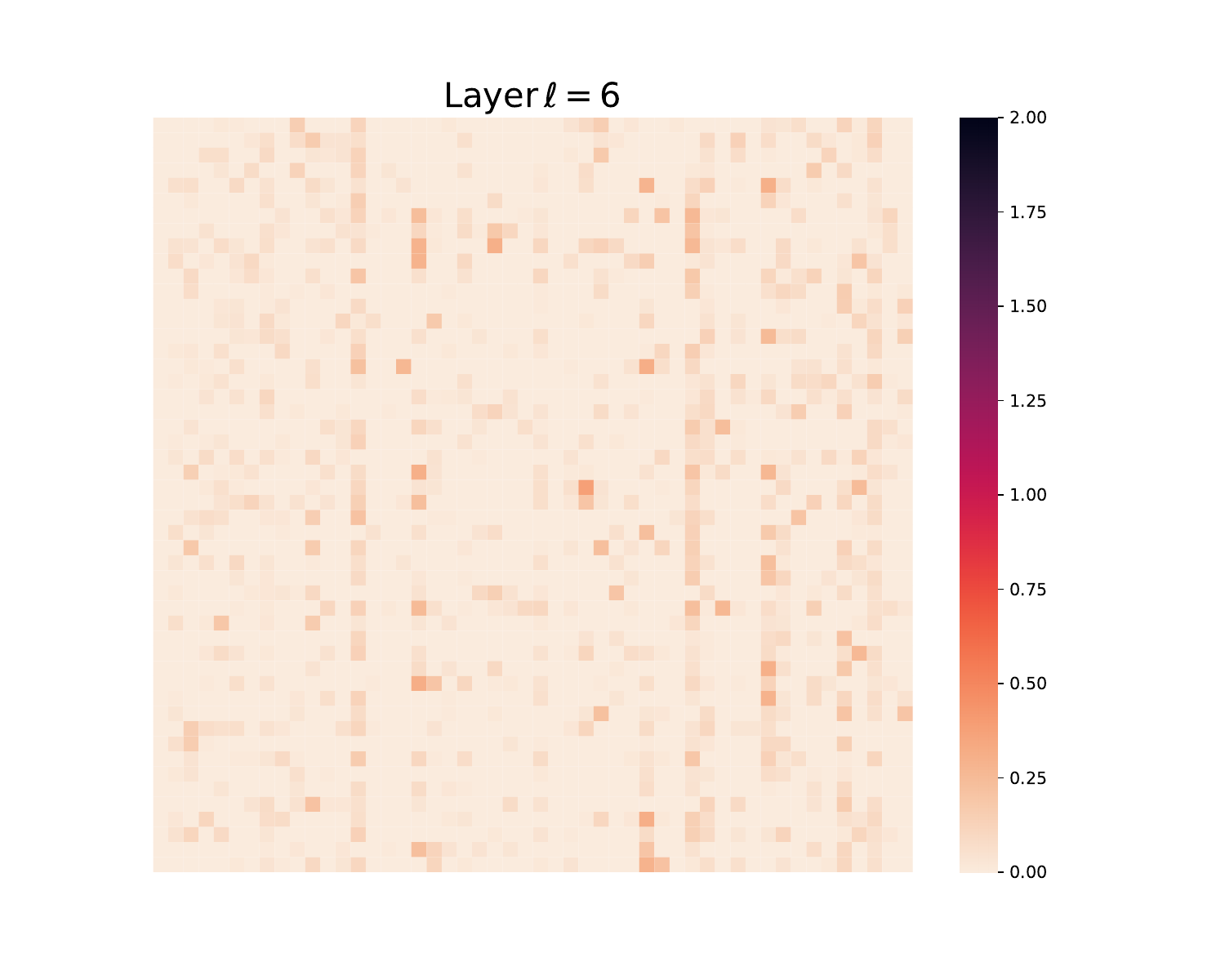}
     \end{subfigure}
     \begin{subfigure}[b]{0.2\textwidth}
         \centering
    \includegraphics[width=\textwidth]{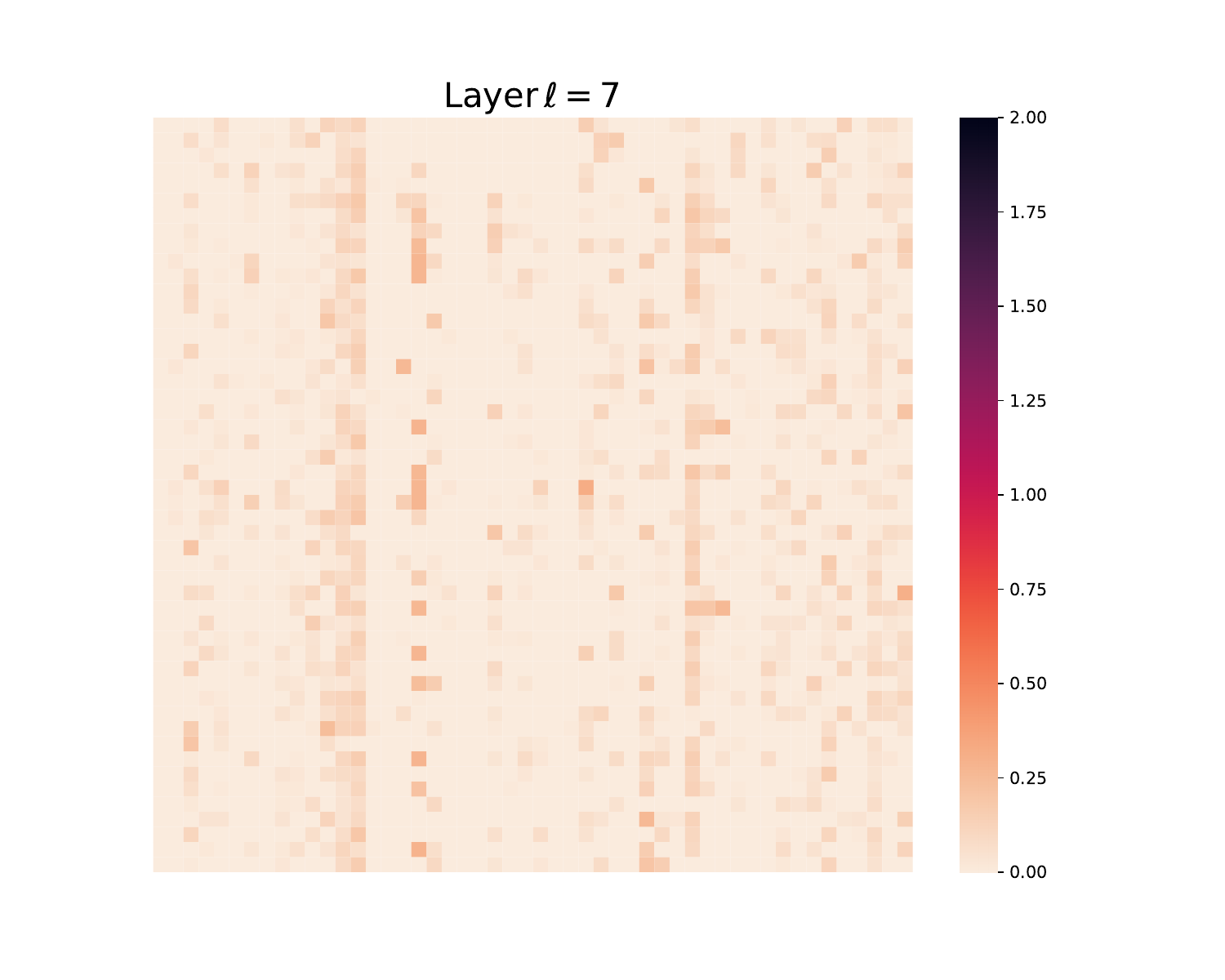}
     \end{subfigure}
     \begin{subfigure}[b]{0.2\textwidth}
         \centering
    \includegraphics[width=\textwidth]{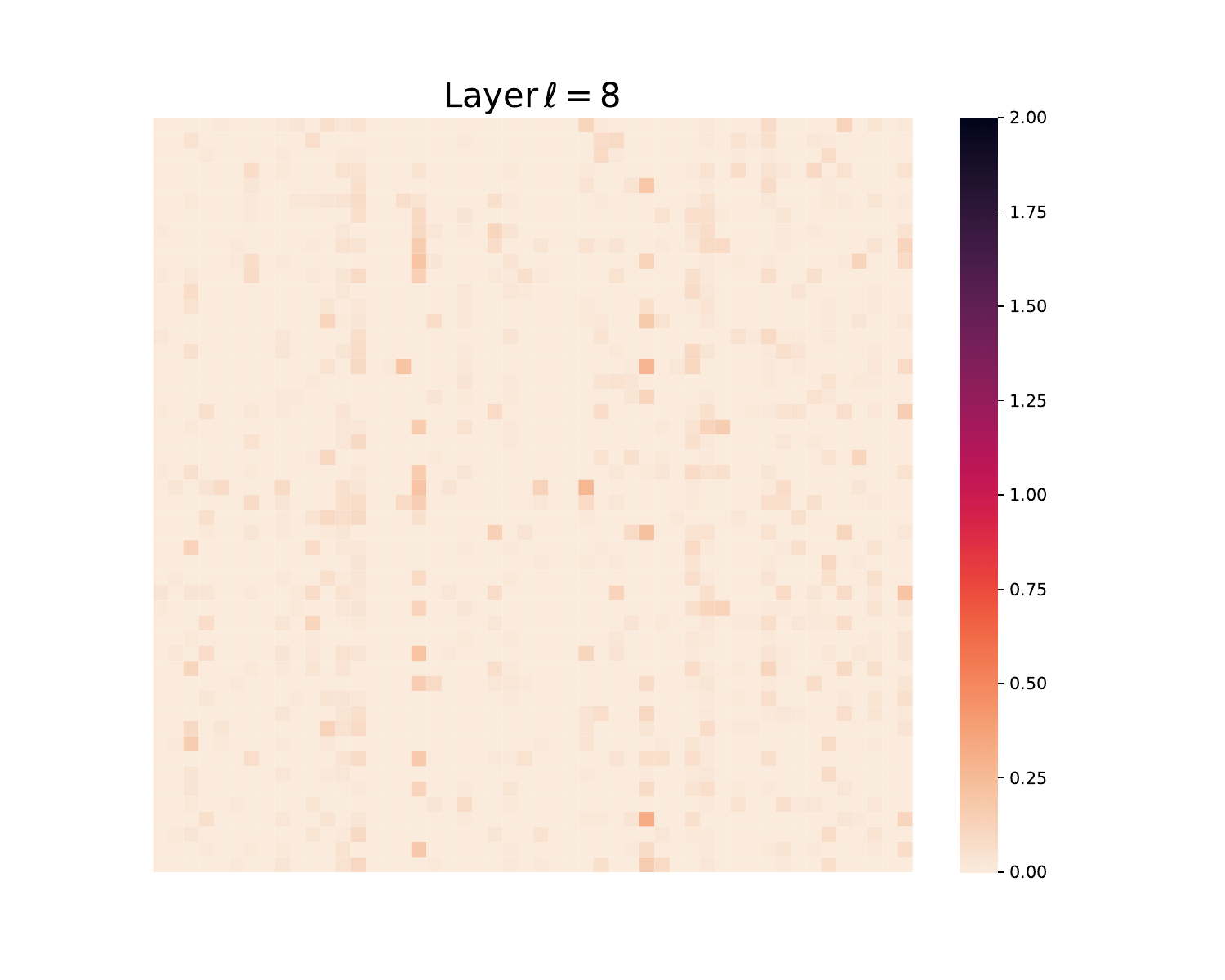}
     \end{subfigure}
     \begin{subfigure}[b]{0.2\textwidth}
         \centering
    \includegraphics[width=\textwidth]{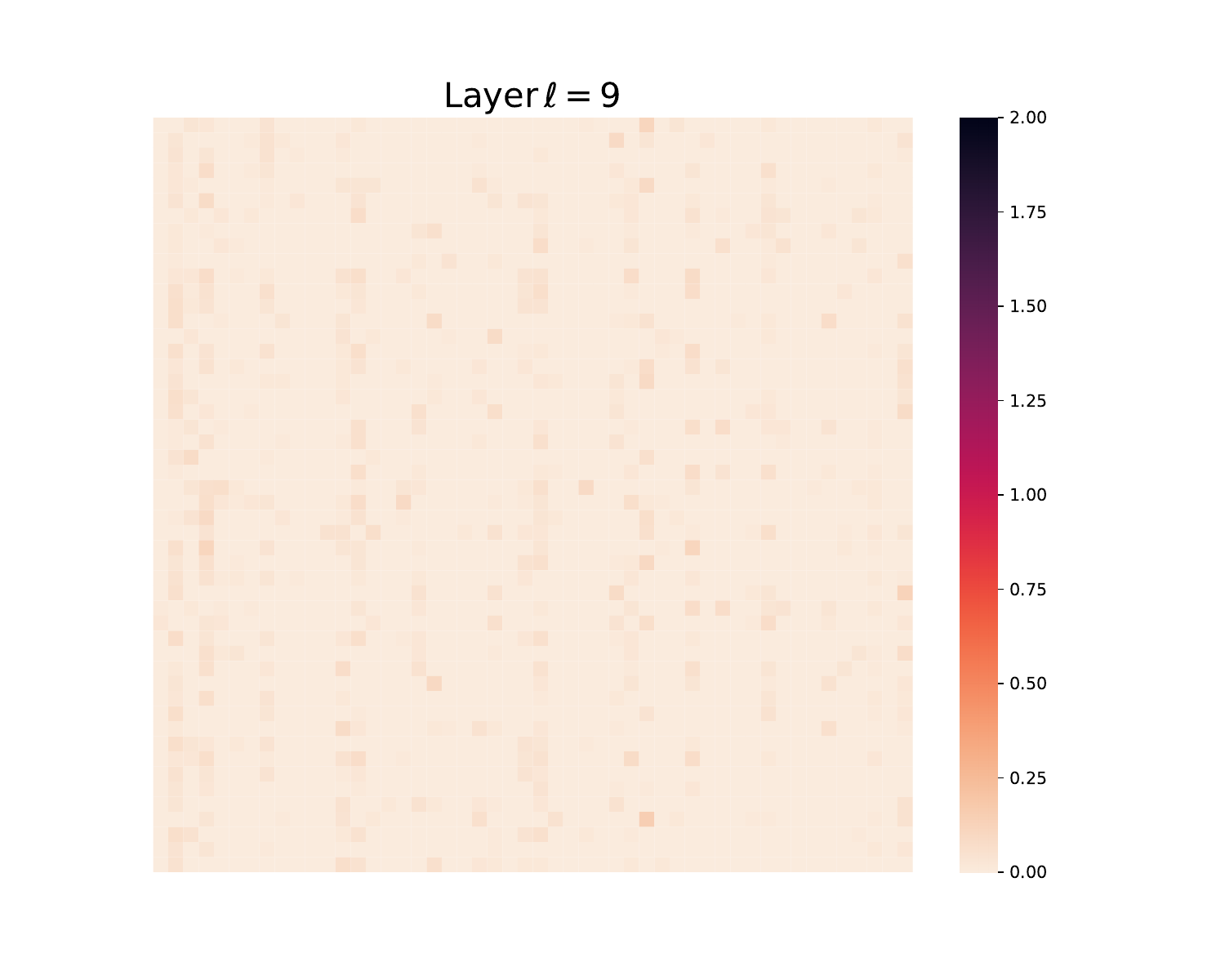}
     \end{subfigure}
     \begin{subfigure}[b]{0.2\textwidth}
         \centering
    \includegraphics[width=\textwidth]{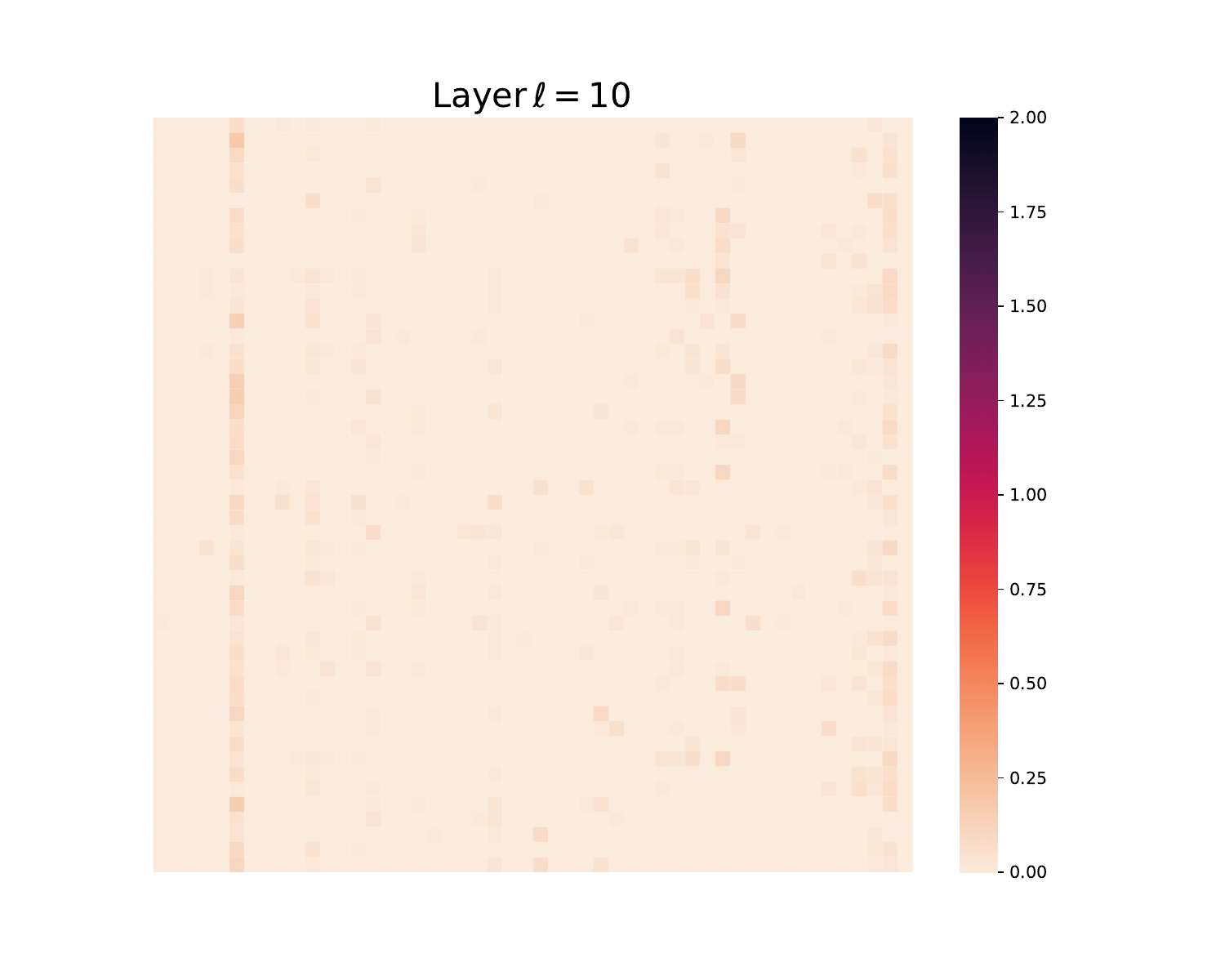}
     \end{subfigure}
     \begin{subfigure}[b]{0.2\textwidth}
         \centering
    \includegraphics[width=\textwidth]{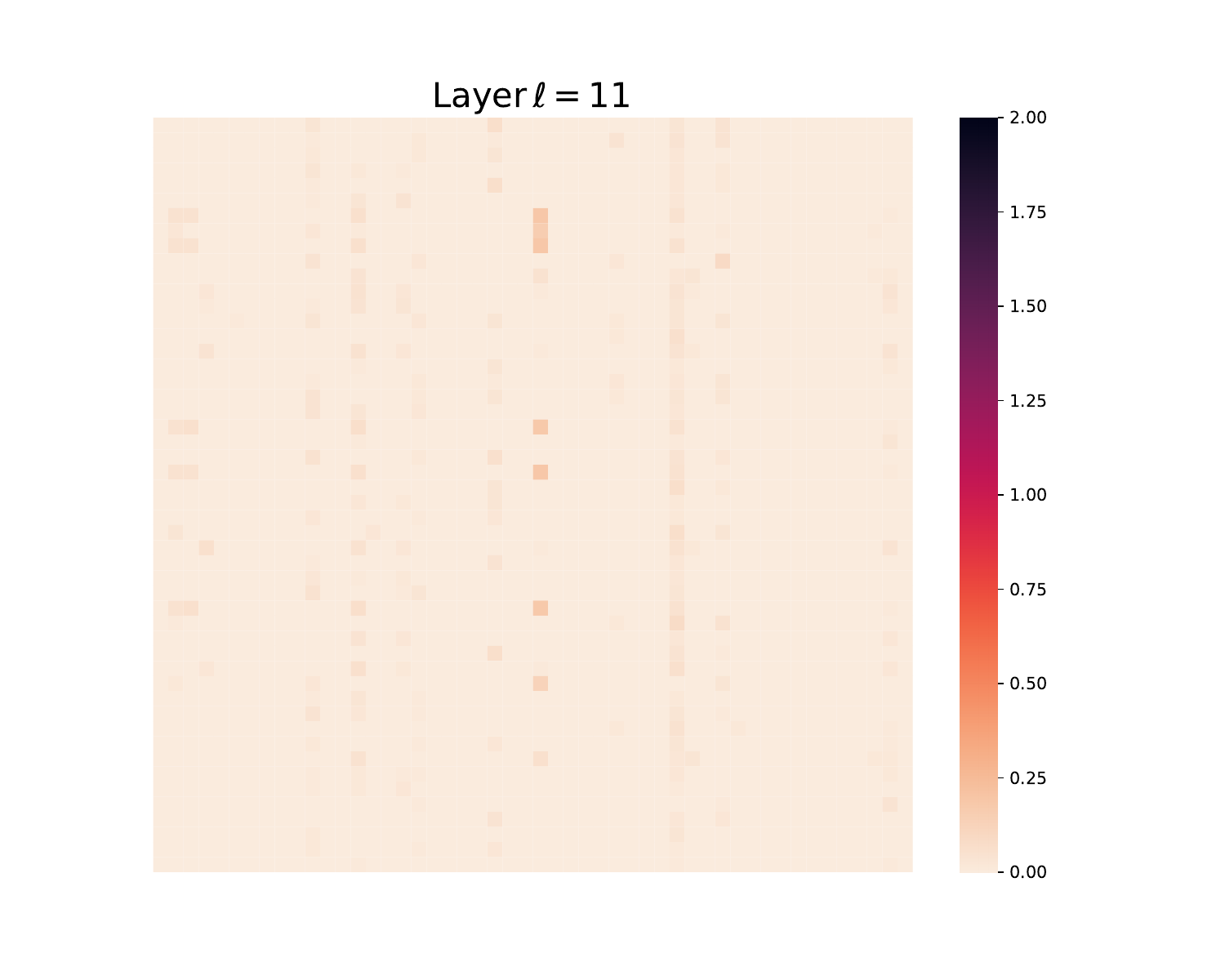}
     \end{subfigure}
     \begin{subfigure}[b]{0.2\textwidth}
         \centering
    \includegraphics[width=\textwidth]{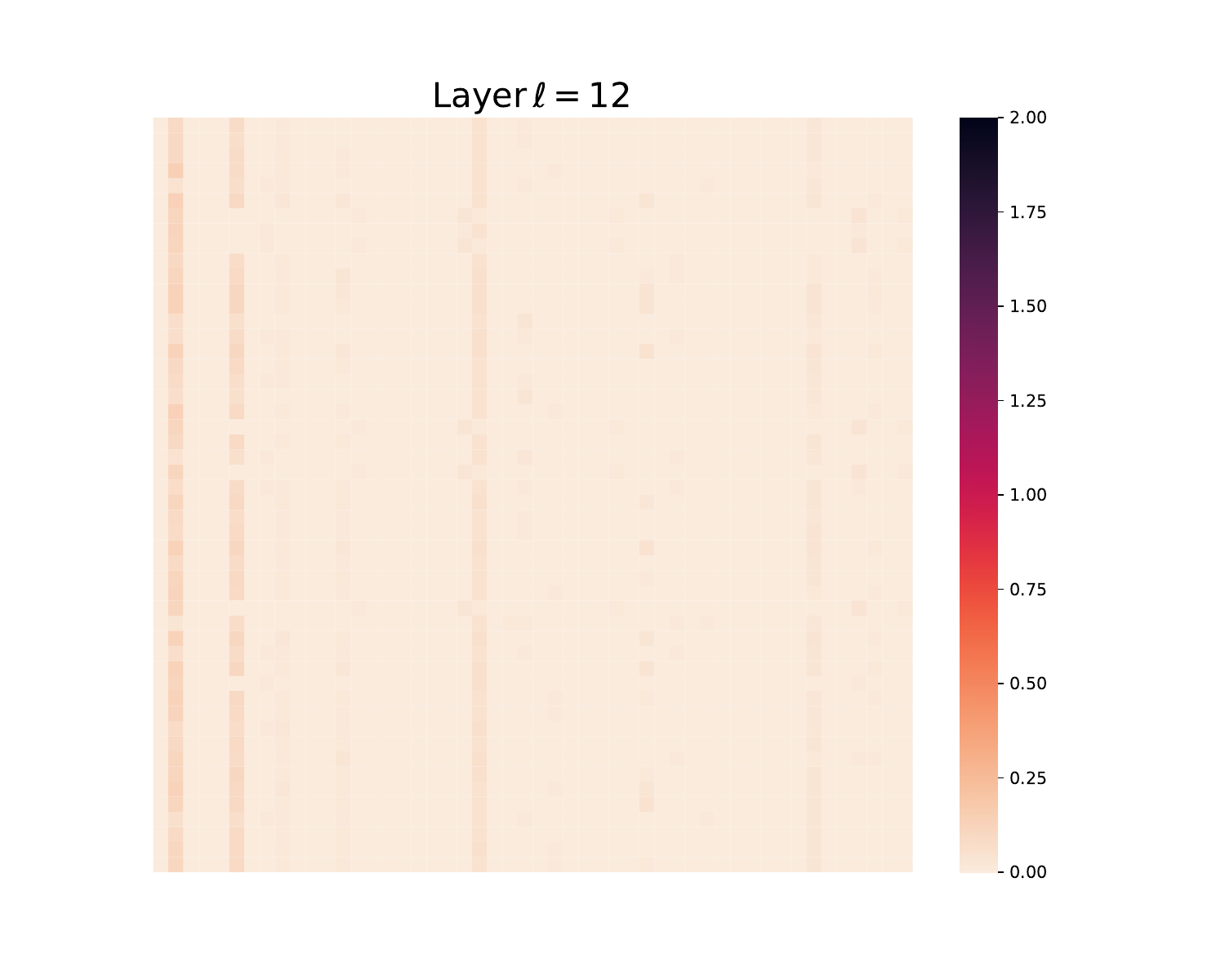}
     \end{subfigure}
        \caption{\small Visualizing layer-wise token $\vZ^{\ell}$ representations at each layer $\ell$. To enhance the visual clarity, we randomly extract a 50$\times$50 sub-matrix from $\vZ^{\ell}$ for display purposes. (\textit{Sample 3})}
        \label{fig:appendix-exp-ista-sparsity-heatmap-sample3}
\end{figure}

\begin{figure}[ht]
     \centering
     \begin{subfigure}[b]{0.2\textwidth}
         \centering
    \includegraphics[width=\textwidth]{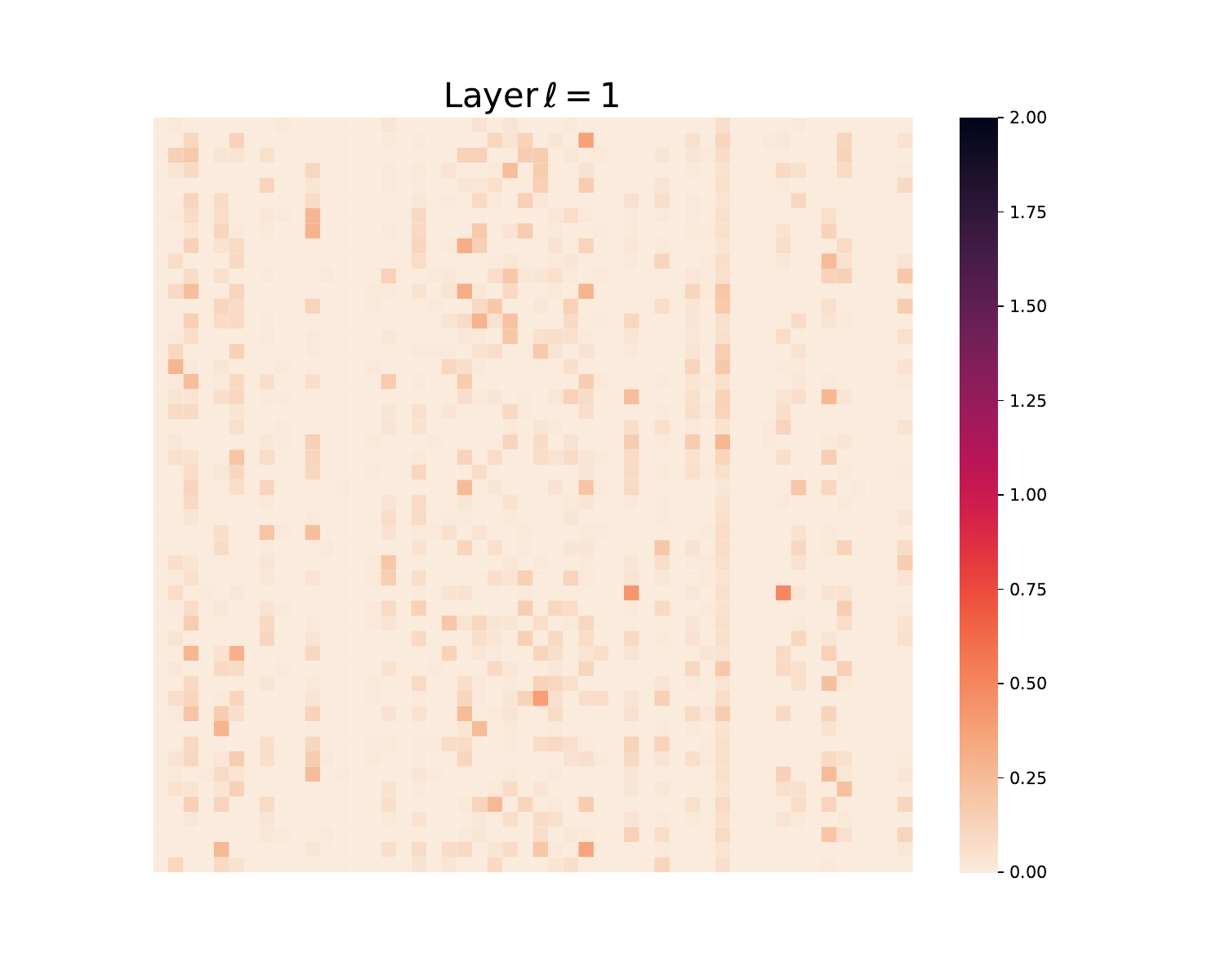}
     \end{subfigure}
     \begin{subfigure}[b]{0.2\textwidth}
         \centering
    \includegraphics[width=\textwidth]{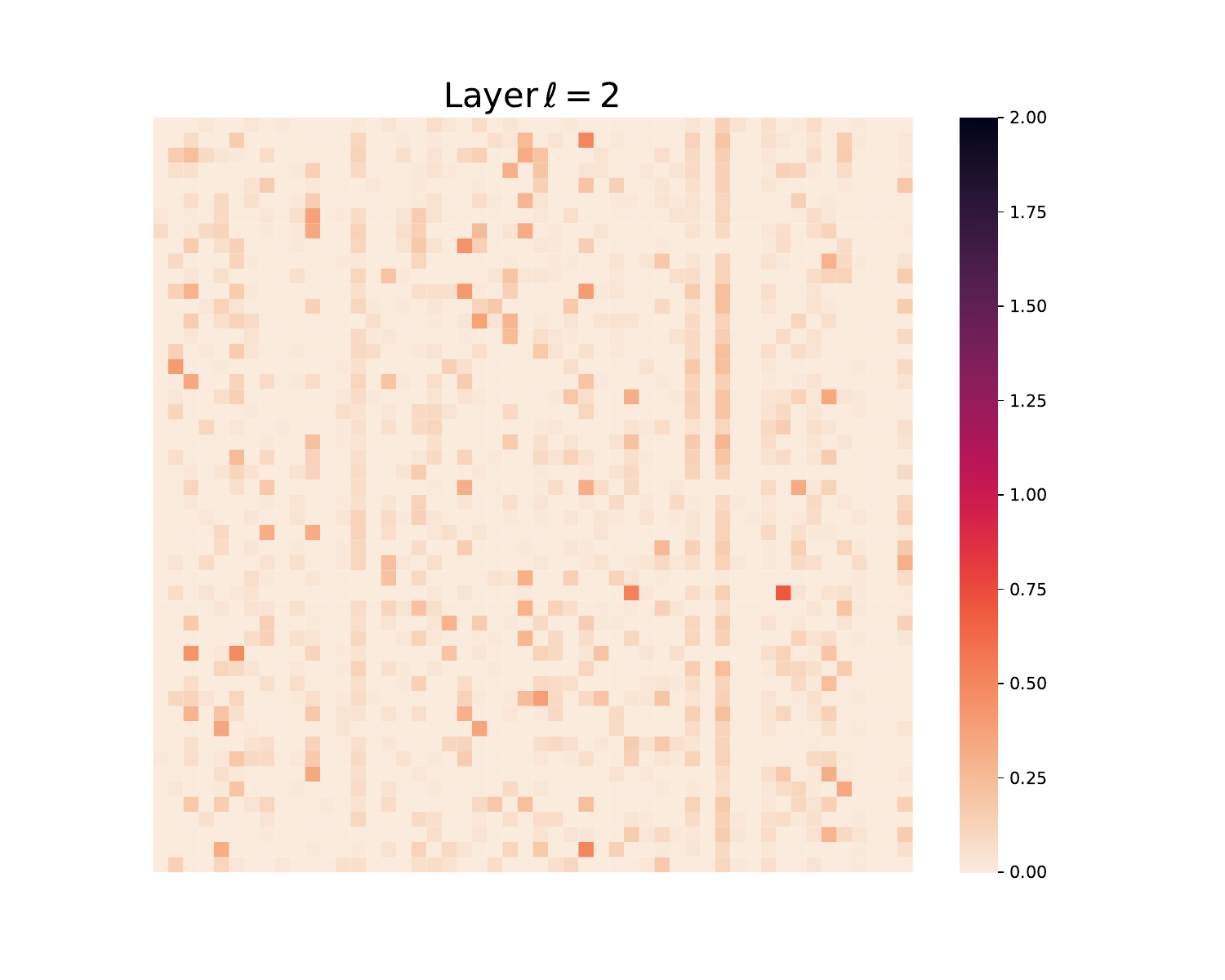}
     \end{subfigure}
     \begin{subfigure}[b]{0.2\textwidth}
         \centering
    \includegraphics[width=\textwidth]{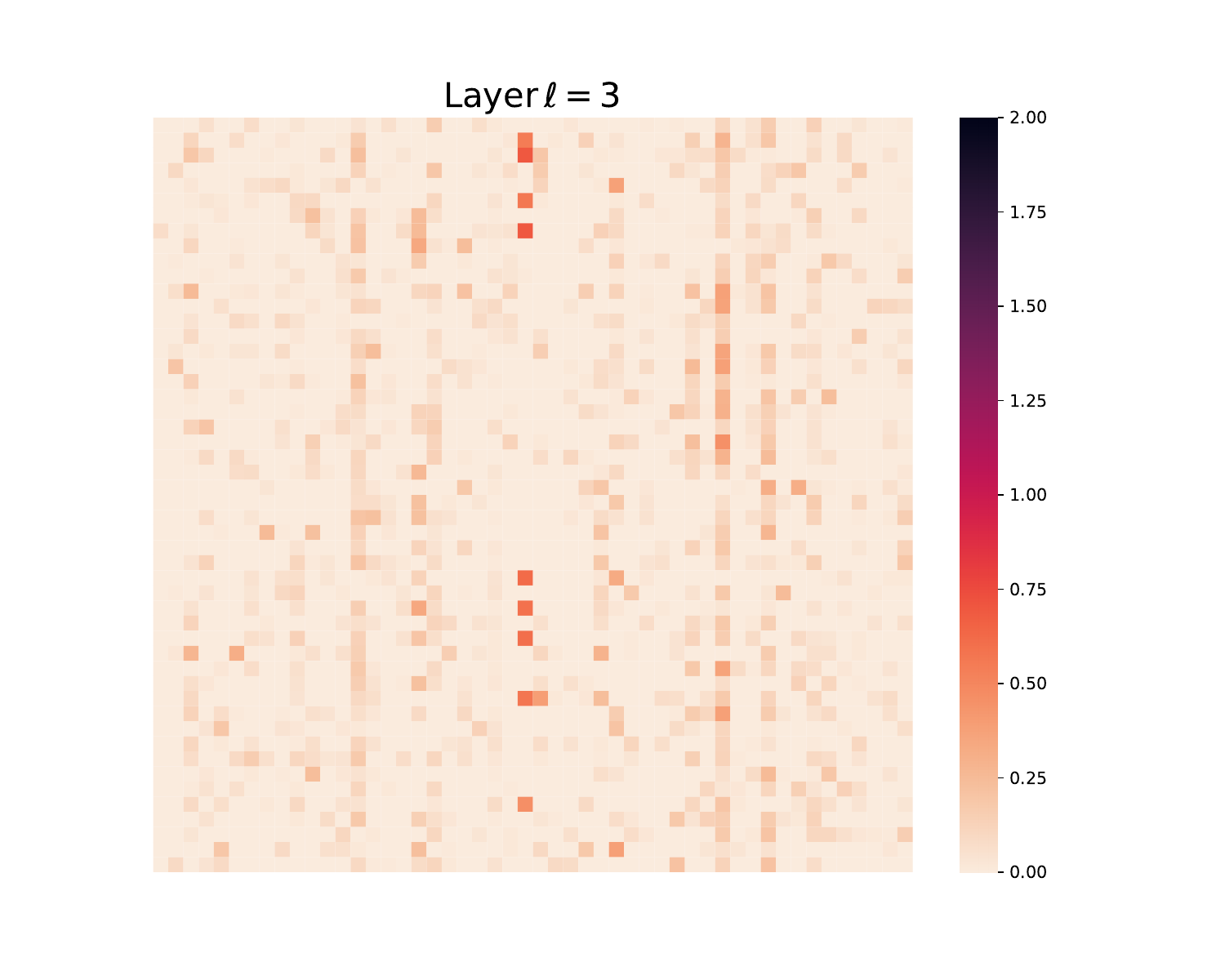}
     \end{subfigure}
     \begin{subfigure}[b]{0.2\textwidth}
         \centering
    \includegraphics[width=\textwidth]{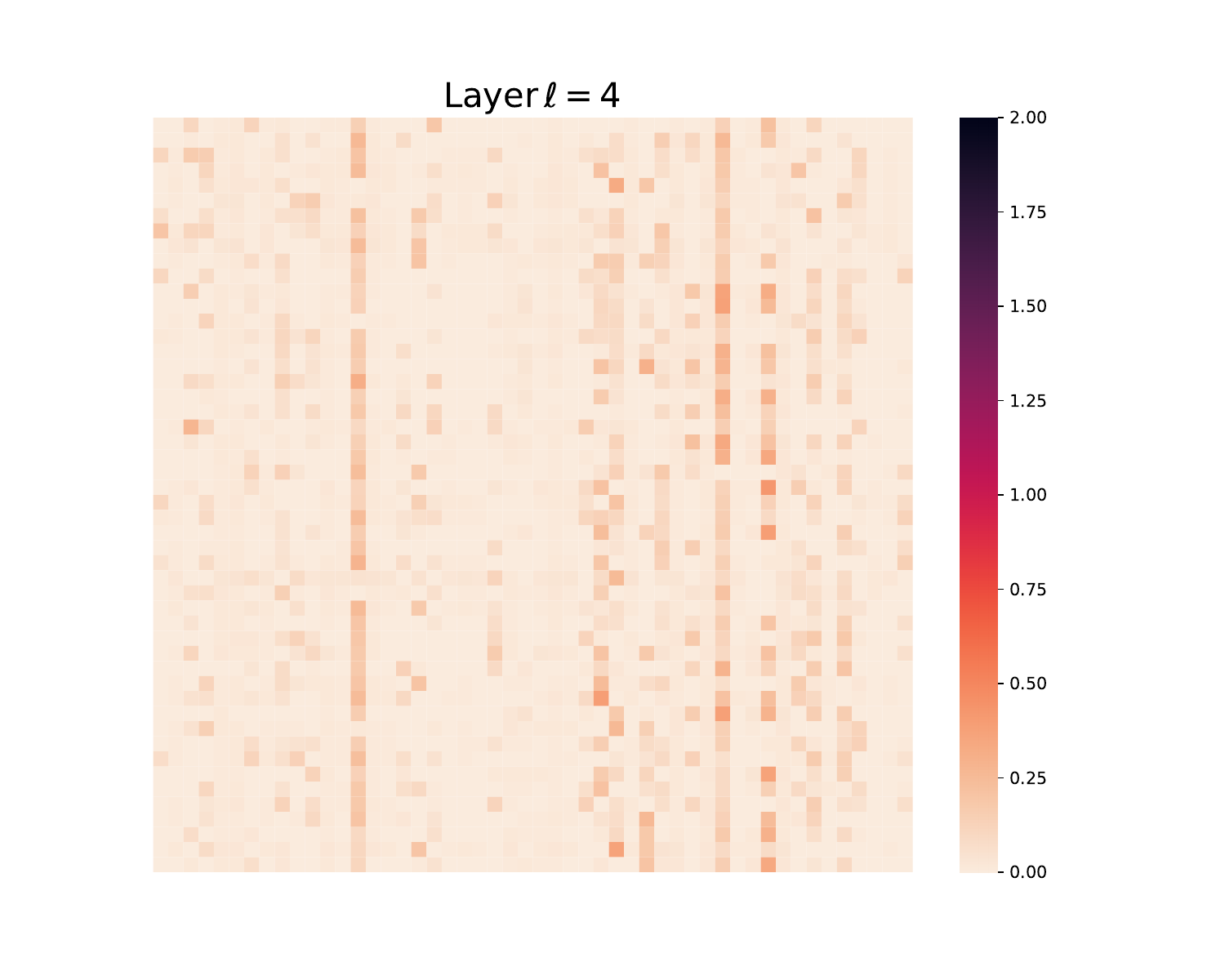}
     \end{subfigure}
     \begin{subfigure}[b]{0.2\textwidth}
         \centering
    \includegraphics[width=\textwidth]{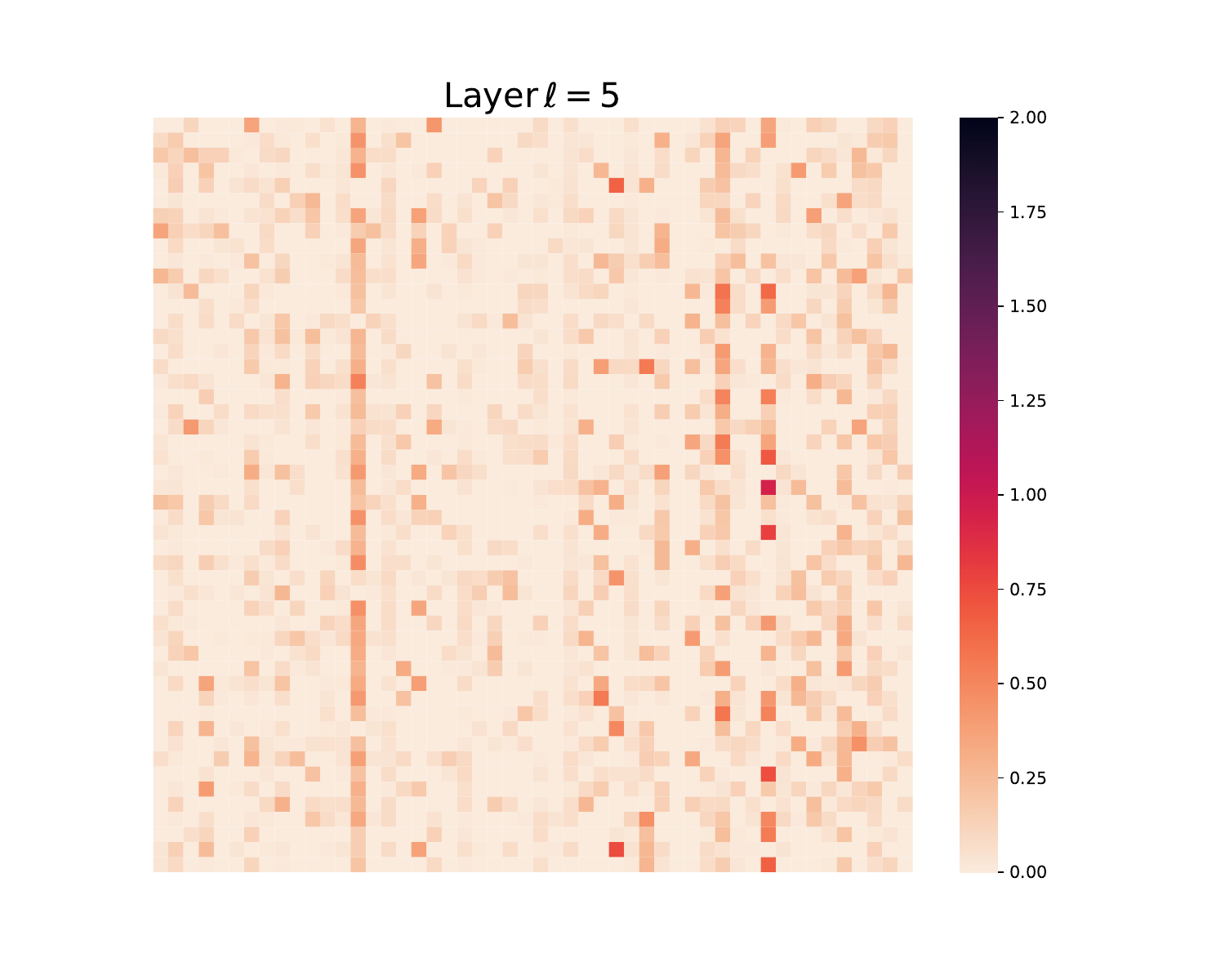}
     \end{subfigure}
     \begin{subfigure}[b]{0.2\textwidth}
         \centering
    \includegraphics[width=\textwidth]{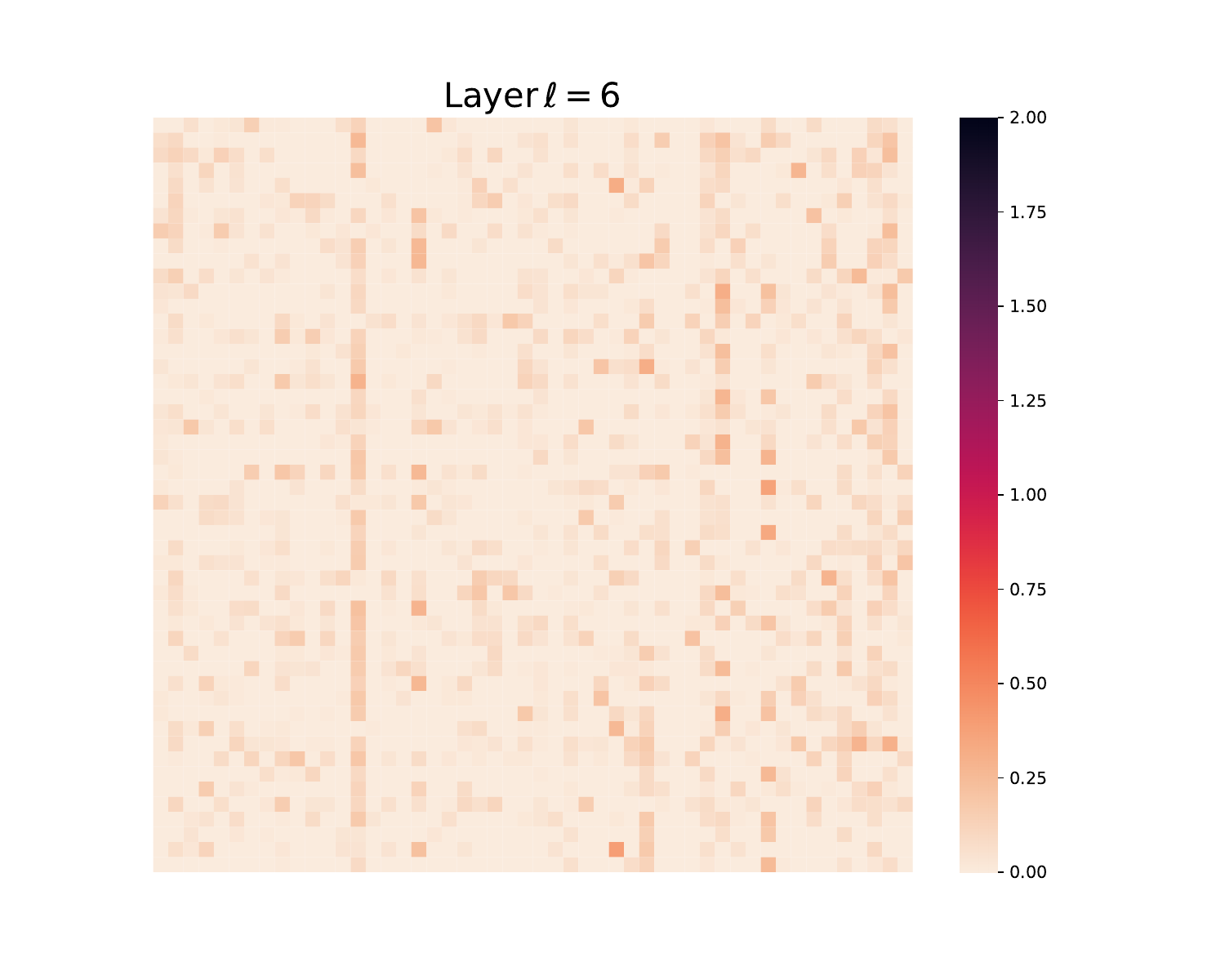}
     \end{subfigure}
     \begin{subfigure}[b]{0.2\textwidth}
         \centering
    \includegraphics[width=\textwidth]{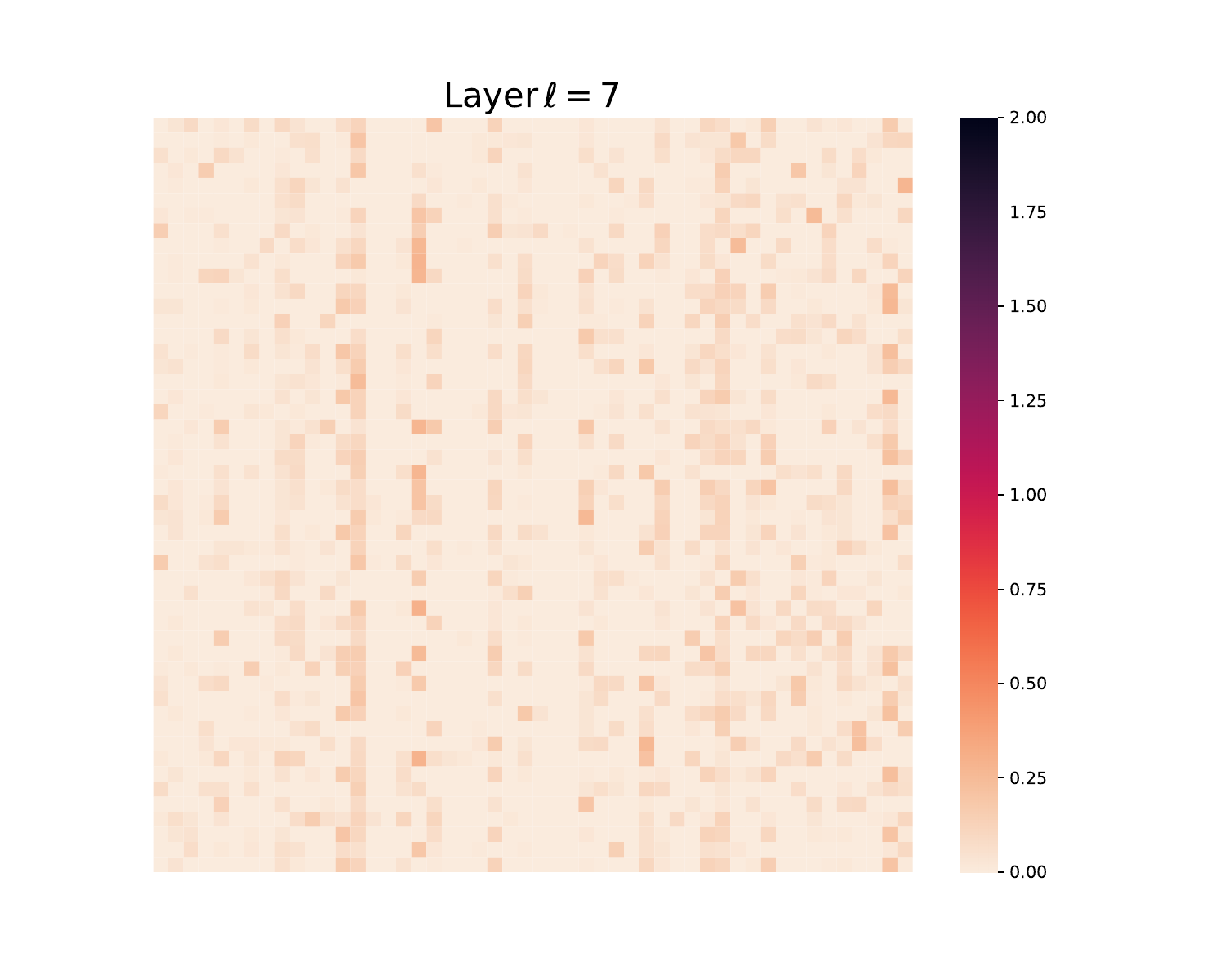}
     \end{subfigure}
     \begin{subfigure}[b]{0.2\textwidth}
         \centering
    \includegraphics[width=\textwidth]{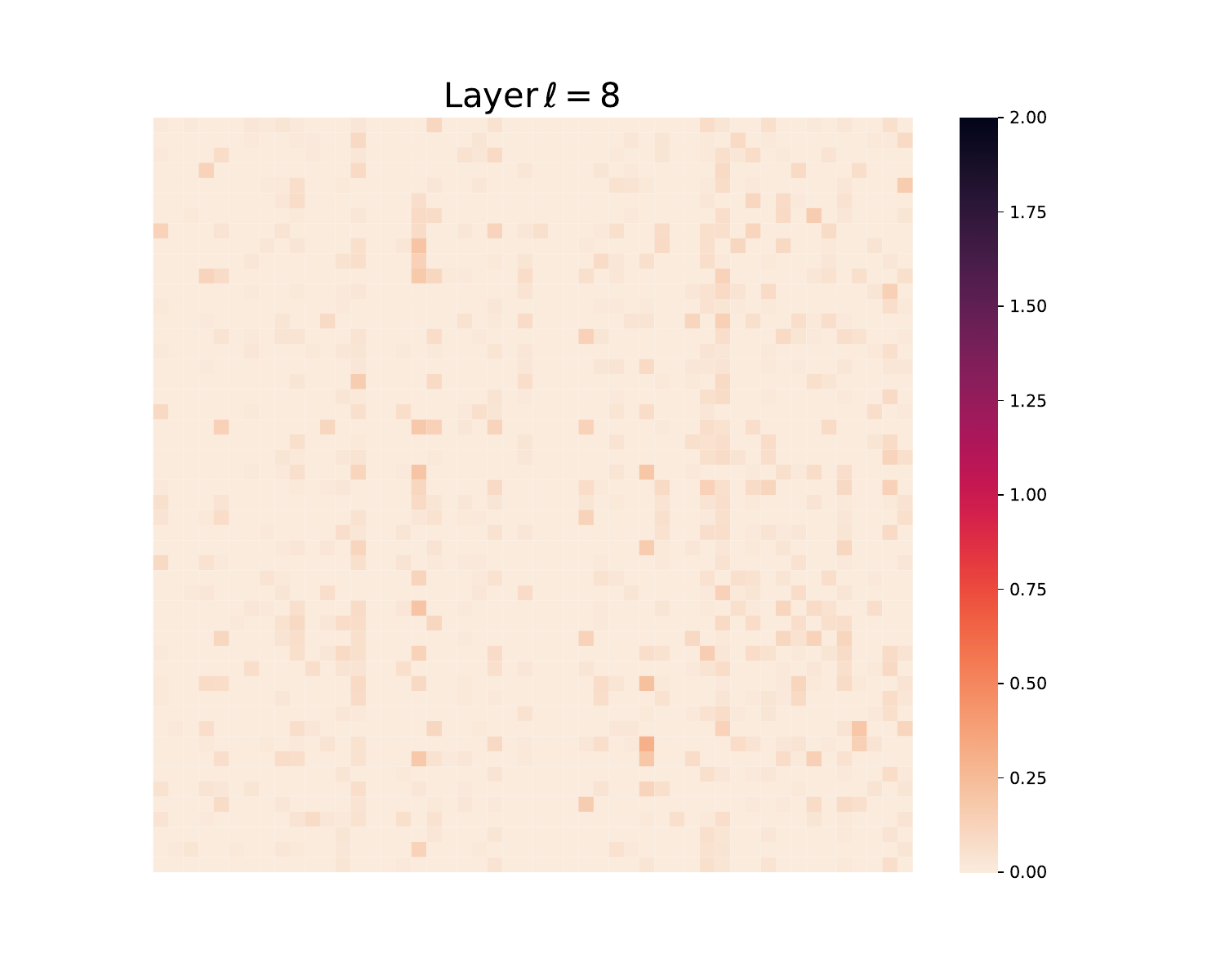}
     \end{subfigure}
     \begin{subfigure}[b]{0.2\textwidth}
         \centering
    \includegraphics[width=\textwidth]{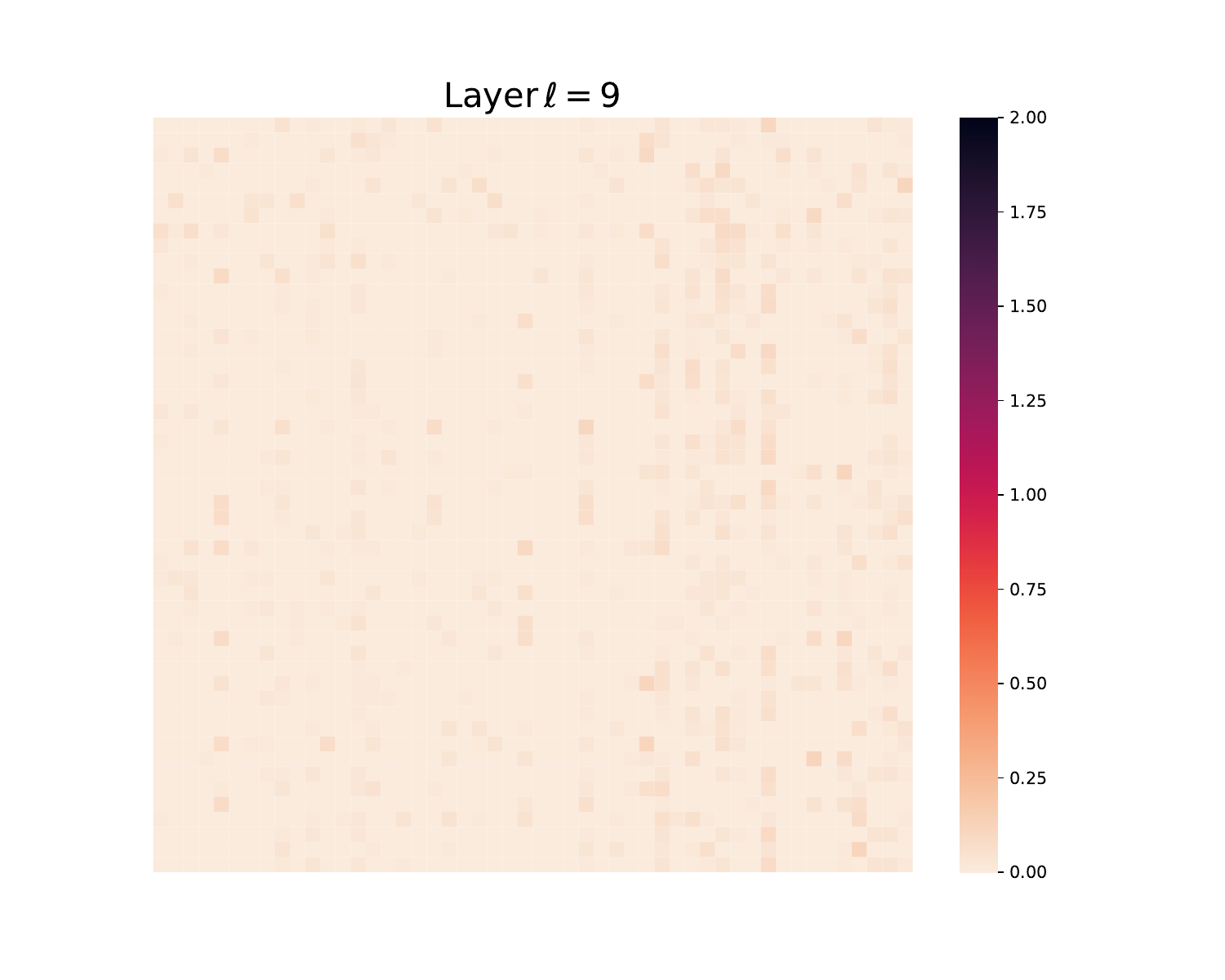}
     \end{subfigure}
     \begin{subfigure}[b]{0.2\textwidth}
         \centering
    \includegraphics[width=\textwidth]{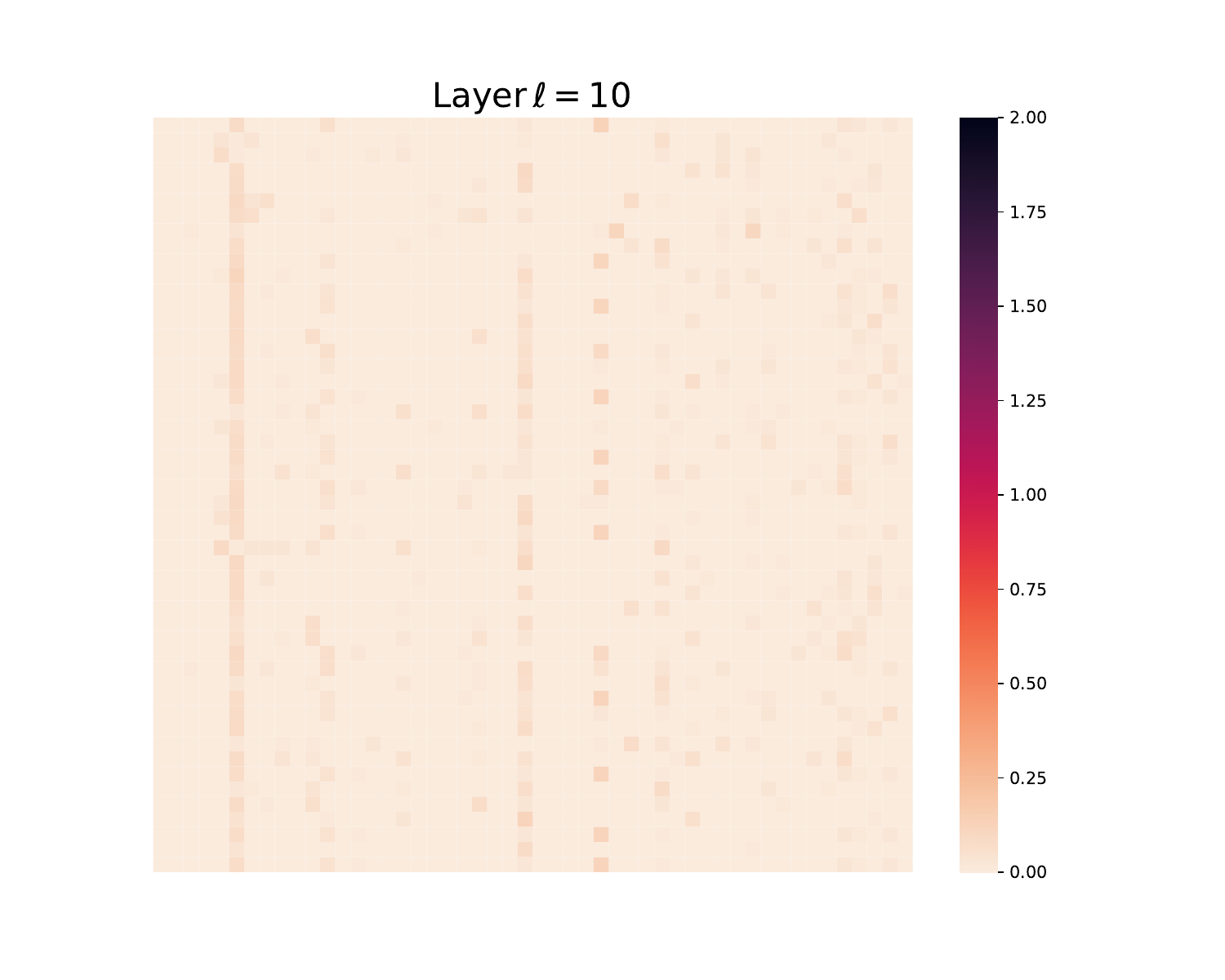}
     \end{subfigure}
     \begin{subfigure}[b]{0.2\textwidth}
         \centering
    \includegraphics[width=\textwidth]{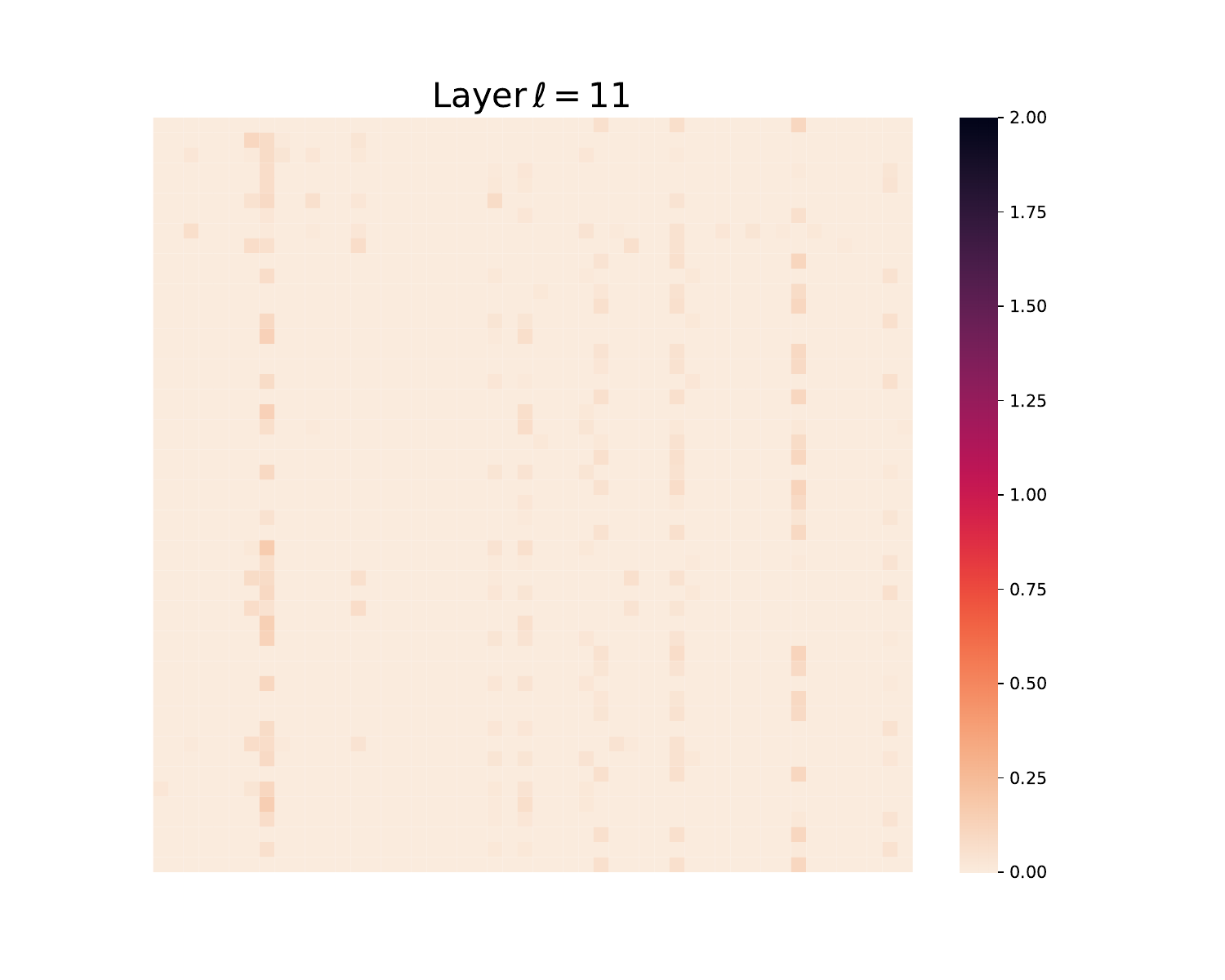}
     \end{subfigure}
     \begin{subfigure}[b]{0.2\textwidth}
         \centering
    \includegraphics[width=\textwidth]{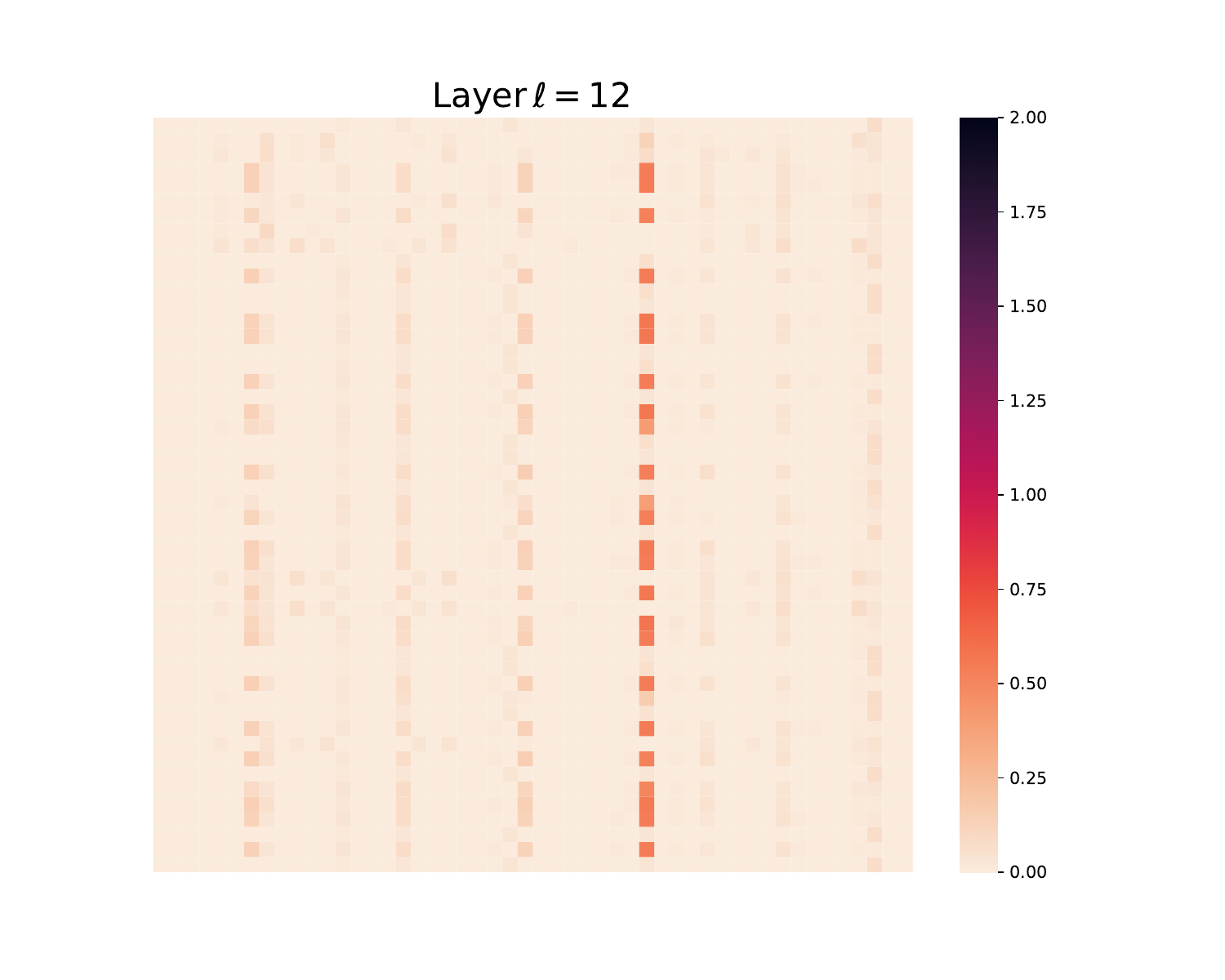}
     \end{subfigure}
        \caption{\small Visualizing layer-wise token $\vZ^{\ell}$ representations at each layer $\ell$. To enhance the visual clarity, we randomly extract a 50$\times$50 sub-matrix from $\vZ^{\ell}$ for display purposes. (\textit{Sample 4})}
        \label{fig:appendix-exp-ista-sparsity-heatmap-sample4}
\end{figure}

\clearpage
\subsection{\color{revision} Additional Experimental Results of Evaluating Compression and Sparsity for ViT}
{\color{revision}
We conduct  experiments to evaluate the compression ($R^{c}$) and sparsity ($\ell^{0}$) of token representations from each layer of a pre-trained ViT-Base (downloaded from \url{https://github.com/huggingface/pytorch-image-models}). 
We summarize the results in Figure~\ref{fig:exp-rc-sparisty-small-new}. 
We find that without our white-box design, the vanilla ViT does not optimize our proposed sparse rate reduction objective. 
This contrasts with the results shown in Figures~\ref{fig:exp-rc-sparisty-small} and \ref{fig:exp-rc-sparisty-small-epochs} of the work, wherein we can observe that the compression term $R^{c}$ and sparsity value decrease layerwise for \ours{}, in accordance with our theory. 
}

\begin{figure*}[htb!]
     \centering
     \begin{subfigure}[b]{0.3\textwidth}
         \centering
    \includegraphics[width=\textwidth]{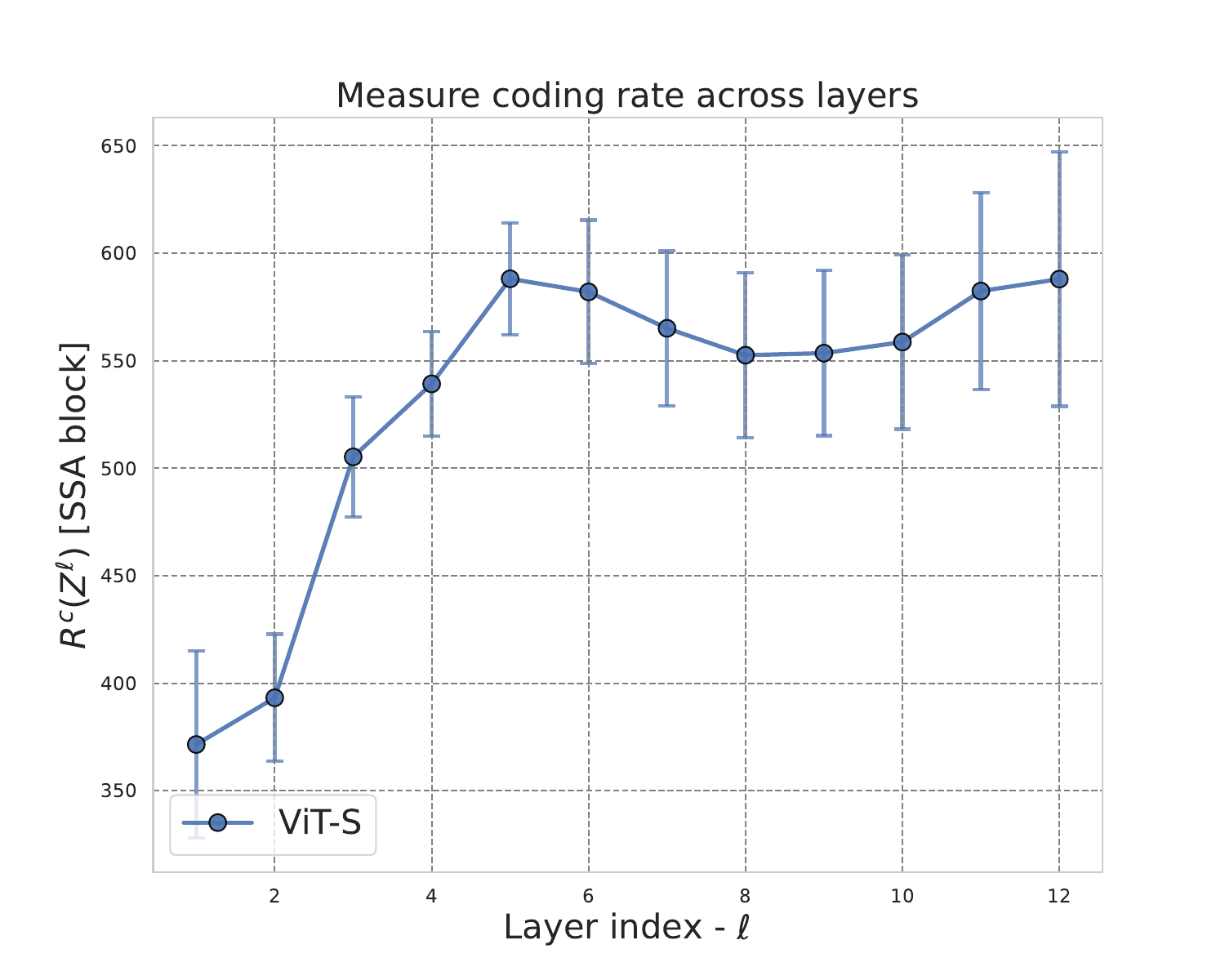}
     \end{subfigure}
     \begin{subfigure}[b]{0.3\textwidth}
         \centering
    \includegraphics[width=\textwidth]{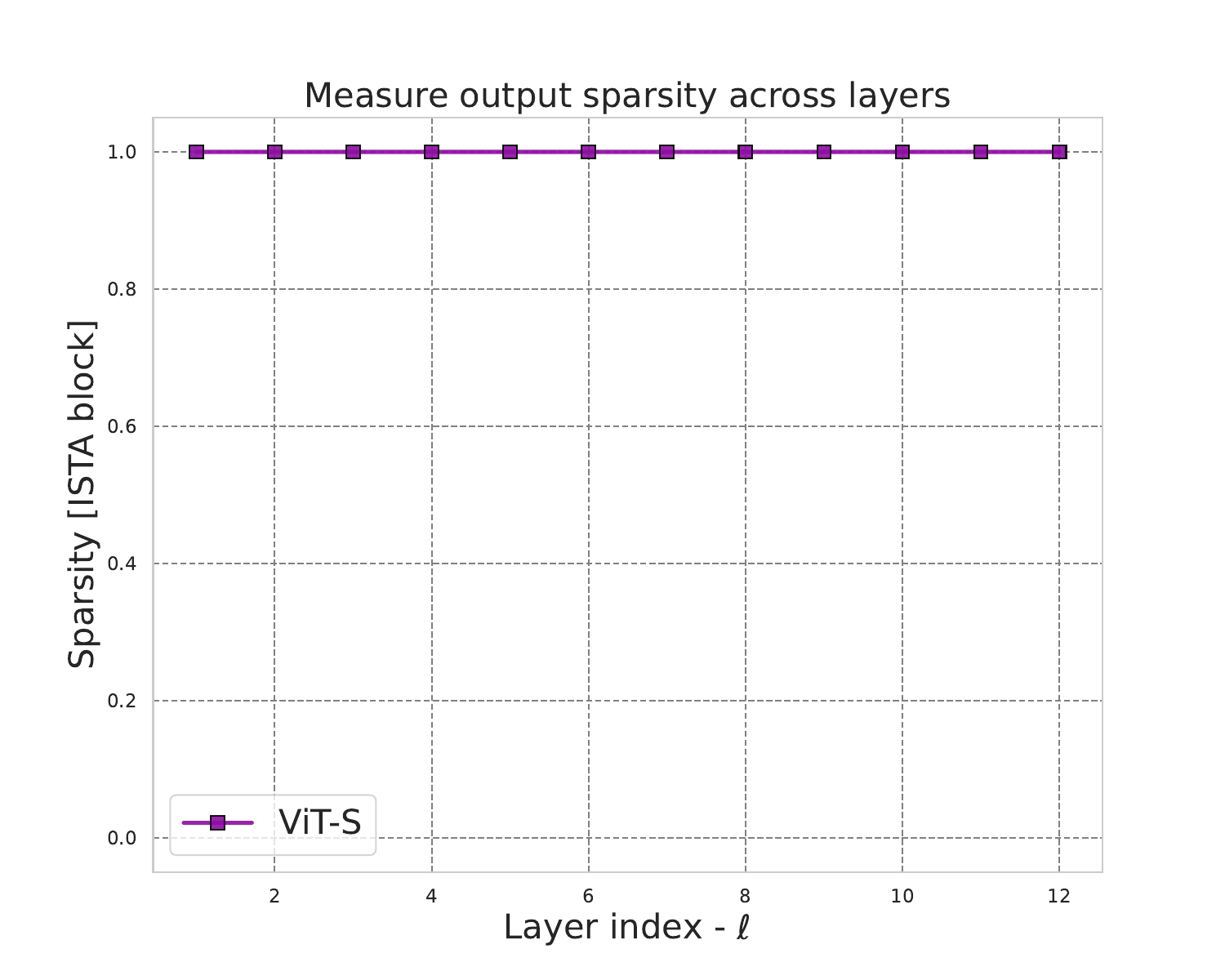}
     \end{subfigure}
     \begin{subfigure}[b]{0.3\textwidth}
         \centering
    \includegraphics[width=\textwidth]{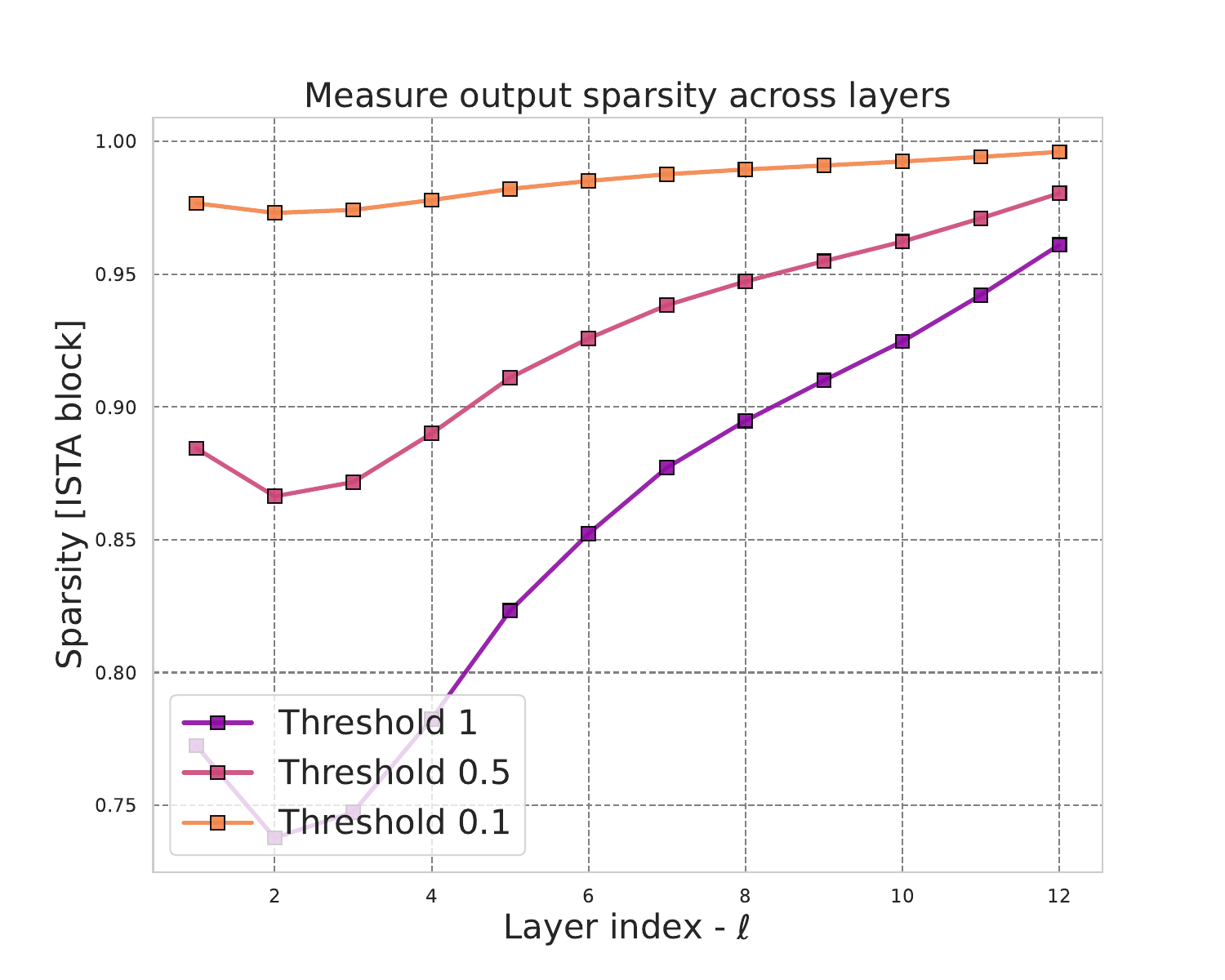}
     \end{subfigure}
     \begin{subfigure}[b]{0.3\textwidth}
         \centering
    \includegraphics[width=\textwidth]{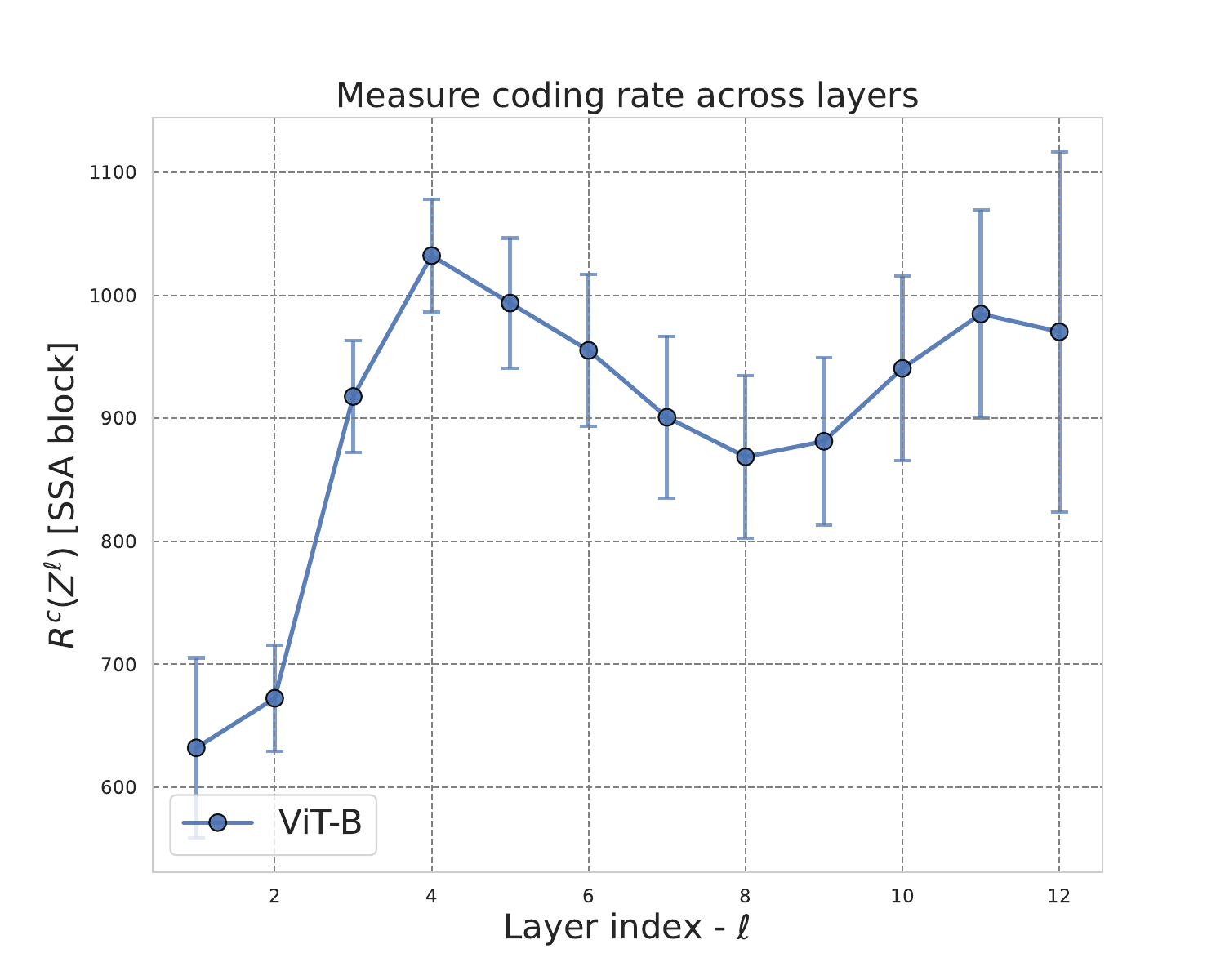}
     \end{subfigure}
     \begin{subfigure}[b]{0.3\textwidth}
         \centering
    \includegraphics[width=\textwidth]{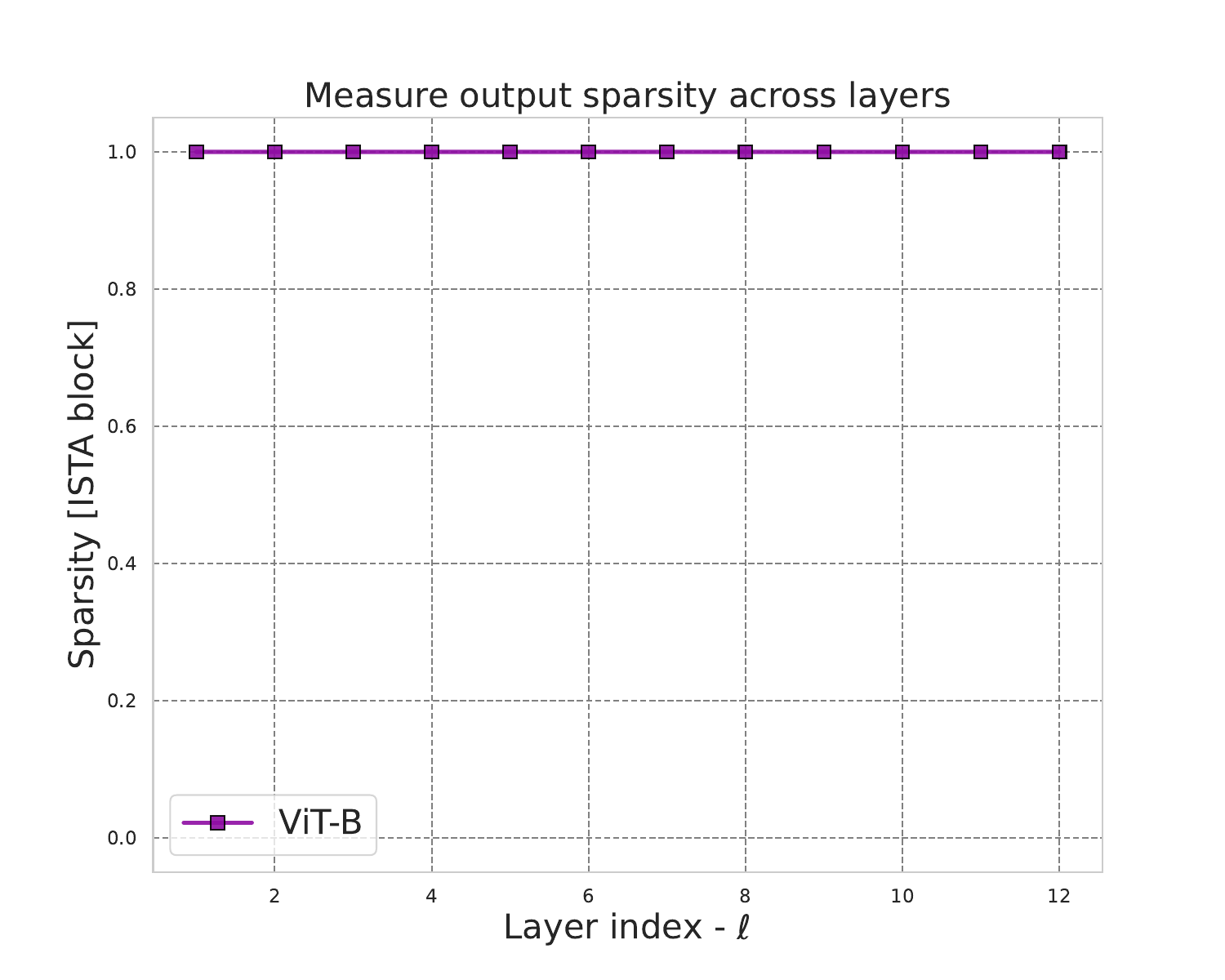}
     \end{subfigure}
     \begin{subfigure}[b]{0.3\textwidth}
         \centering
    \includegraphics[width=\textwidth]{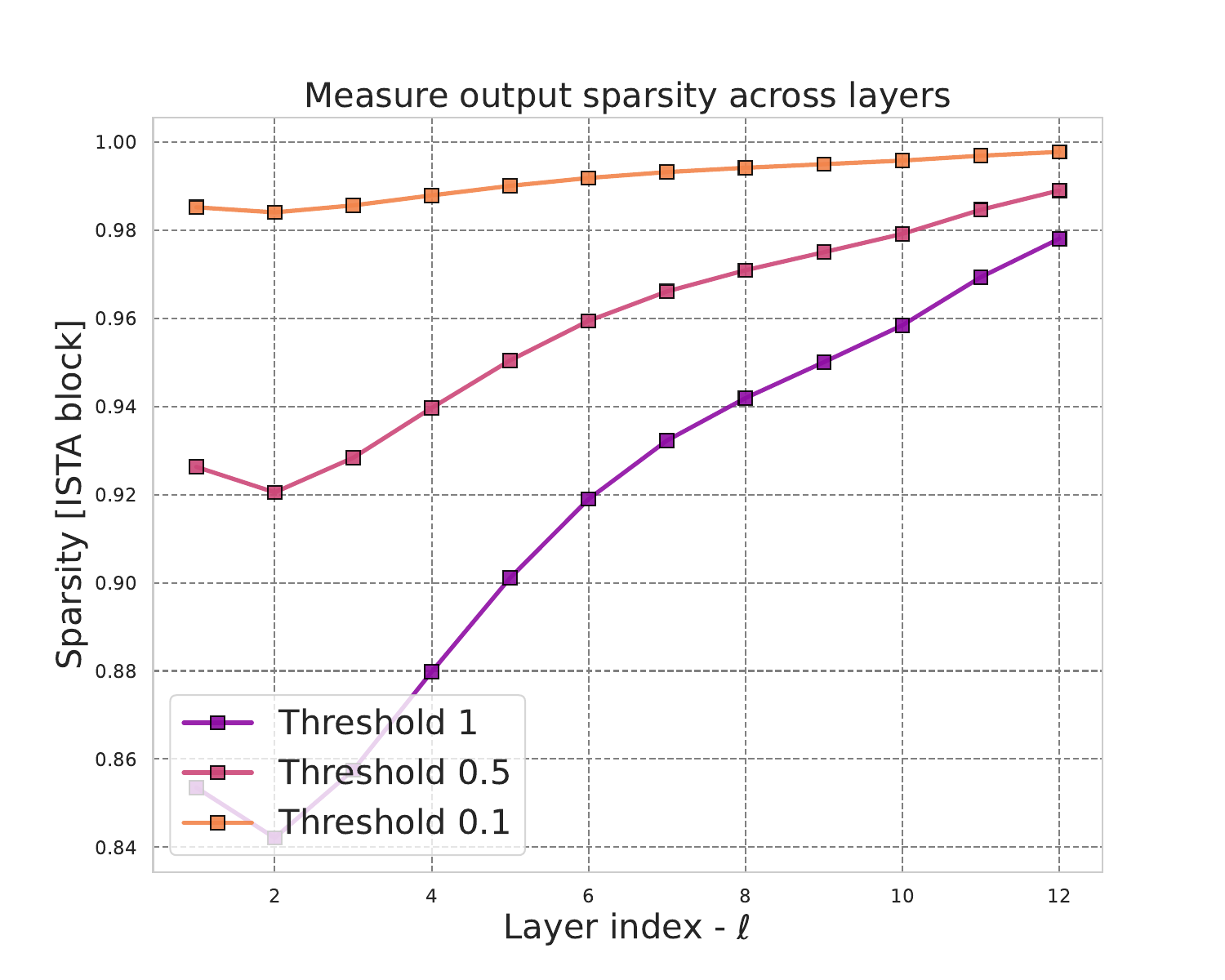}
     \end{subfigure}
     \begin{subfigure}[b]{0.3\textwidth}
         \centering
    \includegraphics[width=\textwidth]{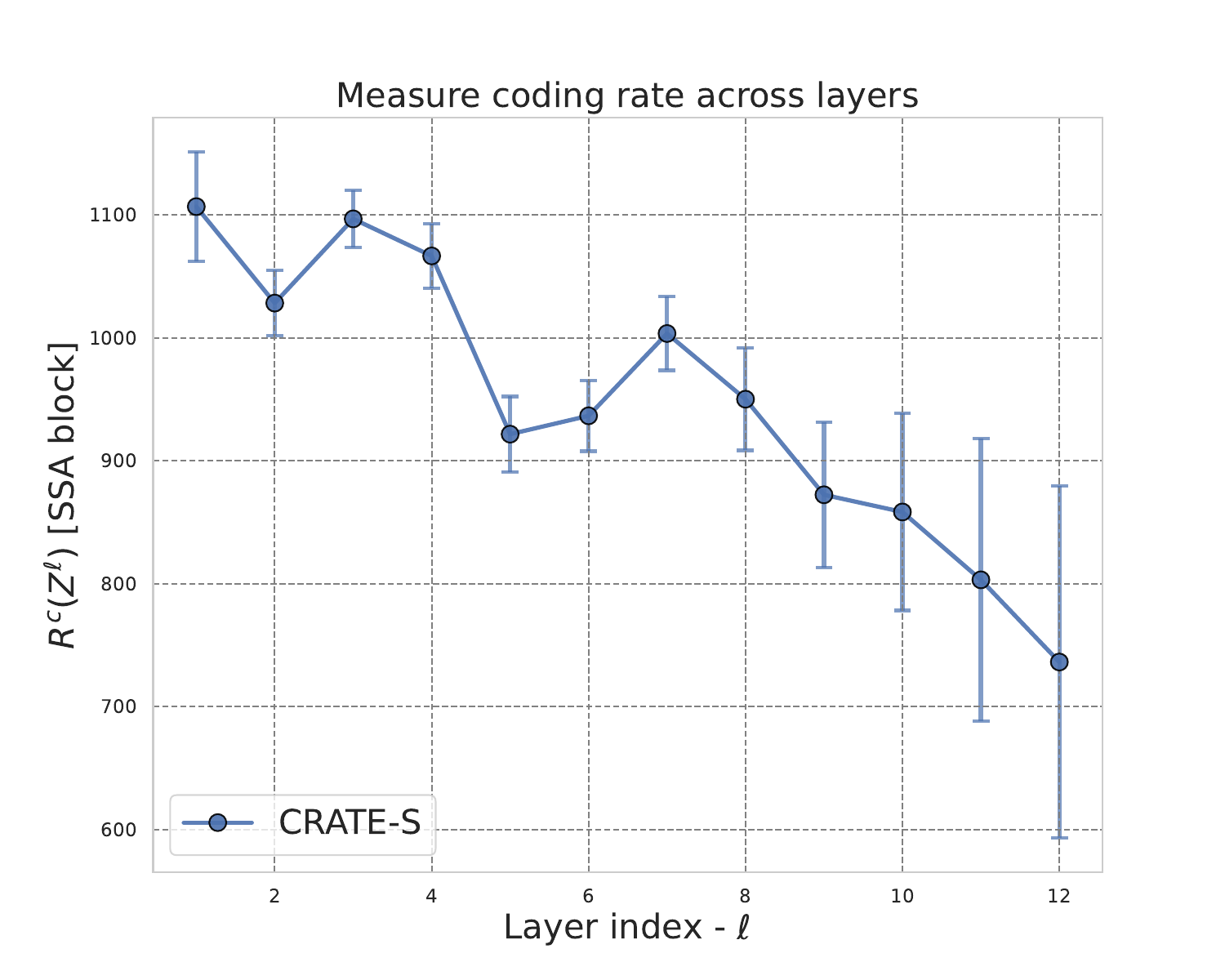}
     \end{subfigure}
     \begin{subfigure}[b]{0.3\textwidth}
         \centering
    \includegraphics[width=\textwidth]{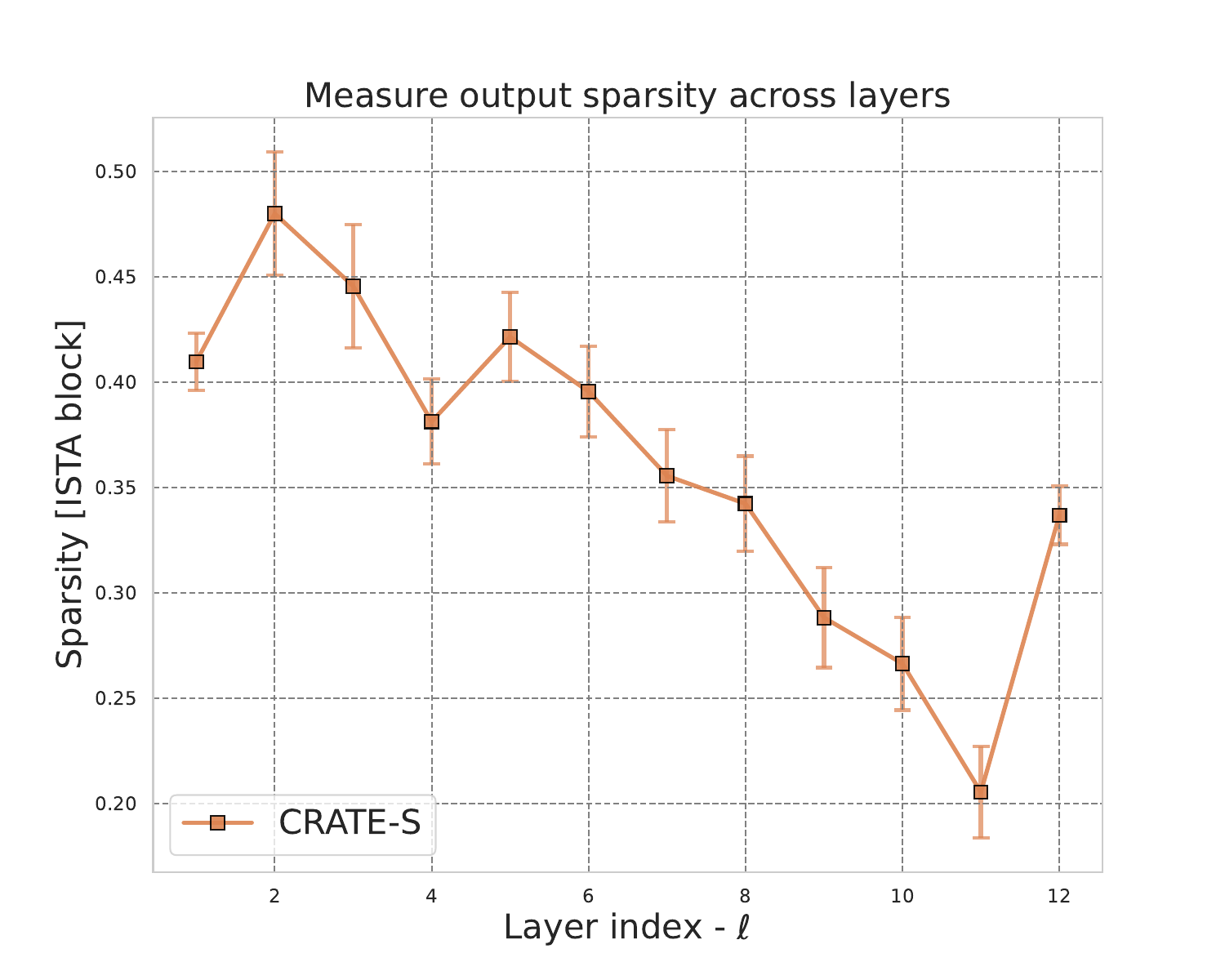}
     \end{subfigure}
     \begin{subfigure}[b]{0.3\textwidth}
         \centering
    \includegraphics[width=\textwidth]{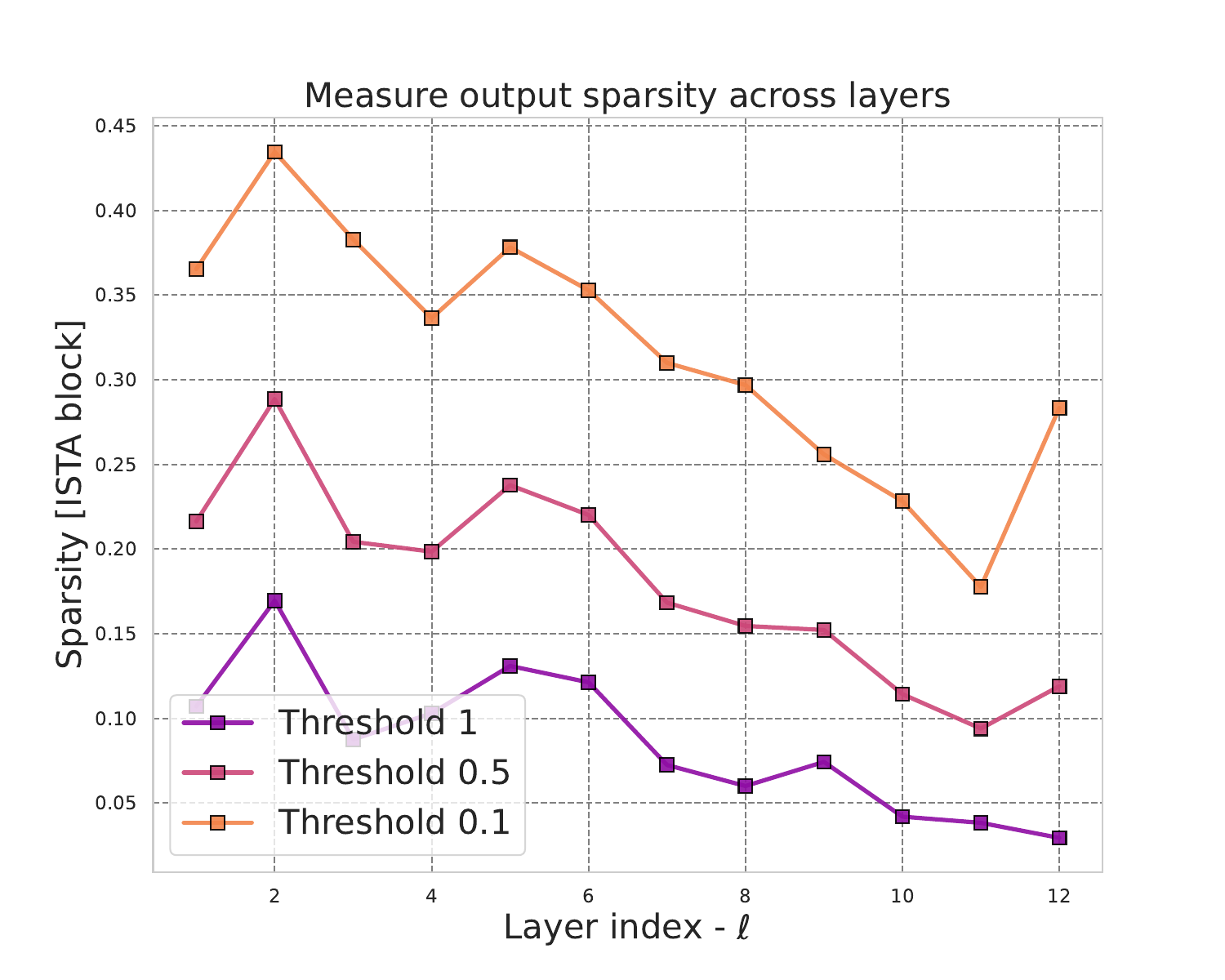}
     \end{subfigure}
     \vspace{-2mm}
        \caption{\small  \color{revision} \textit{Left}: The compression term $R^{c}(\bm{Z}^{\ell+1/2})$ of the multi-head self-attention outputs at different layers. \textit{Middle}: The sparsity of outputs of the \texttt{MLP} block, $\|\bm{Z}^{\ell+1}\|_0 / (d\cdot N)$, at different layers. \textit{Right}: To get a more fine-grained understanding of the sparsity of \texttt{MLP} block outputs of ViT, we use three different thresholds $\tau \in \{1.0, 0.5, 0.1\}$ and measure $\sum_{i,j} \mathbf{1}\{|\bm{Z}^{\ell+1}_{i, j}| < \tau\} / (d\cdot N)$, where $\bm{Z}^{\ell+1}_{i, j}$ represents the $j$-th element in the $i$-th token representation. 
        ({First row model: ViT{-Small}; second row model: ViT{-Base};} third row model: our proposed \ours{-Small}).}
        \label{fig:exp-rc-sparisty-small-new}
        \vspace{-0.2in}
\end{figure*}

%% file: app_segmentation.tex
\clearpage
\subsection{Details and Experimental Results of Attention Map Visualization}\label{app:attn_maps}

We recapitulate the method to visualize attention maps in \citet{abnar2020quantifying, caron2021emerging}, at first specializing their use to instances of the \ours{} model before generalizing to the ViT.

\noindent
For the $k^{\mathrm{th}}$ head at the $\ell^{\mathrm{th}}$ layer of \ours{}, we compute the \textit{self-attention matrix} $\A_{k}^{\ell}\in\bR^{n}$ defined as follows:
\begin{equation}\label{eq:def_attention_map_crate}
   \A_{k}^{\ell}= \mat{A_{k, 1}^{\ell} \\ \vdots \\ A_{k, N}^{\ell}} \in \bR^{N}, \quad \text{where} \quad A_{k, i}^{\ell} = \frac{\exp(\langle \vU^{\ell*}_k\z^{\ell}_{i}, \vU^{\ell*}_k\z^{\ell}_{\cls} \rangle)}{\sum_{j=1}^{N}\exp(\langle \vU^{\ell*}_k\z^{\ell}_{j}, \vU^{\ell*}_k\z^{\ell}_{\cls} \rangle)}.   
\end{equation}
We then reshape the attention matrix $\A_{k}^{\ell}$ into a ${\sqrt{n}\times\sqrt{n}}$ matrix and visualize the heatmap as shown in \Cref{fig:semantic_heads}. 
For example, the $i^{\mathrm{th}}$ row and the $j^{\mathrm{th}}$ column element of the heatmap in \Cref{fig:semantic_heads} corresponds to the $m^{\mathrm{th}}$ component of $\A_{k}^{\ell}$ if $m=(i - 1)\cdot \sqrt{n} + j$. 
In \Cref{fig:semantic_heads}, we select 4 attention heads of \ours{} and visualize the attention matrix $\A_{k}^{\ell}$ for each image. 

For the ViT, the entire methodology remains the same, except that the attention map is defined in the following reasonable way:
\begin{equation}\label{eq:def_attention_map_vit}
   \A_{k}^{\ell} = \mat{A_{k, 1}^{\ell} \\ \vdots \\ A_{k, N}^{\ell}} \in \bR^{N}, \quad \text{where} \quad A_{k, i}^{\ell} = \frac{\exp(\langle \vK^{\ell*}_k\z^{\ell}_{i}, \vQ^{\ell*}_k\z^{\ell}_{\cls} \rangle)}{\sum_{j=1}^{N}\exp(\langle \vK^{\ell*}_k\z^{\ell}_{j}, \vQ^{\ell*}_k\z^{\ell}_{\cls} \rangle)}.   
\end{equation}
where the ``query'' and ``key'' parameters of the standard transformer at head \(k\) and layer \(\ell\) are denoted \(\vK_{k}^{\ell}\) and \(\vQ_{k}^{\ell}\) respectively.

\paragraph{Details about MaskCut.}
We apply the MaskCut pipeline (Algorithm~\ref{alg:maskcut}) to generate segmentation masks and detection bounding box (discussed in \Cref{subsubsec:segmentation_measurement}). 
As described by \citet{wang2023cut}, we iteratively apply Normalized Cuts~\citep{shi2000normalized} on the patch-wise affinity matrix $\vM^{\ell}$, where $\vM_{ij}^{\ell} = \sum_{k=1}^{K}\langle \vU^{\ell*}_{k} \vz_{i}^{\ell}, \vU^{\ell*}_{k} \vz_{j}^{\ell}\rangle$. 
At each iterative step, we mask out the identified patch-wise entries on $\vM^{\ell}$. To obtain more fine-grained segmentation masks, MaskCut employs Conditional Random Fields (CRF)~\citep{krahenbuhl2011efficient} to post-process the masks, which smooths the edges and filters out unreasonable masks. Correspondingly, the detection bounding box is defined by the rectangular region that tightly encloses a segmentation mask.

\begin{algorithm}[H]
\caption{MaskCut}
\label{alg:maskcut}

\SetKwInput{Hyperparameter}{Hyperparameter}
\SetKwFunction{FMain}{MaskCut}
\SetKwProg{Fn}{Function}{:}{\textbf{return} \texttt{masks}}
\SetKwFor{For}{for}{do}{endfor}
\Hyperparameter{\(S\), the number of objects to segment.}

\Fn{\FMain{$\vM$}}{
    \For{\(i \in \{1, \dots, S\}\)}{
        \texttt{mask} $\gets$ \textsc{NCut}(\vM) \tcp*{\(\texttt{mask}\) is a boolean array}
        $\vM \gets \vM \odot \texttt{mask}$ \tcp*{Equivalent to applying the mask to \(\vM\)}
        \texttt{masks}[i] $\gets$ \texttt{mask}\;
    }

}
\end{algorithm}

\paragraph{Additional visualization of attention map.} We provide additional experiments on comparing the attention maps of \ours{} and ViT in \Cref{appendix_fig:crate_vs_vit}.
\begin{figure}
    \centering
    \includegraphics[width=1.0\linewidth]{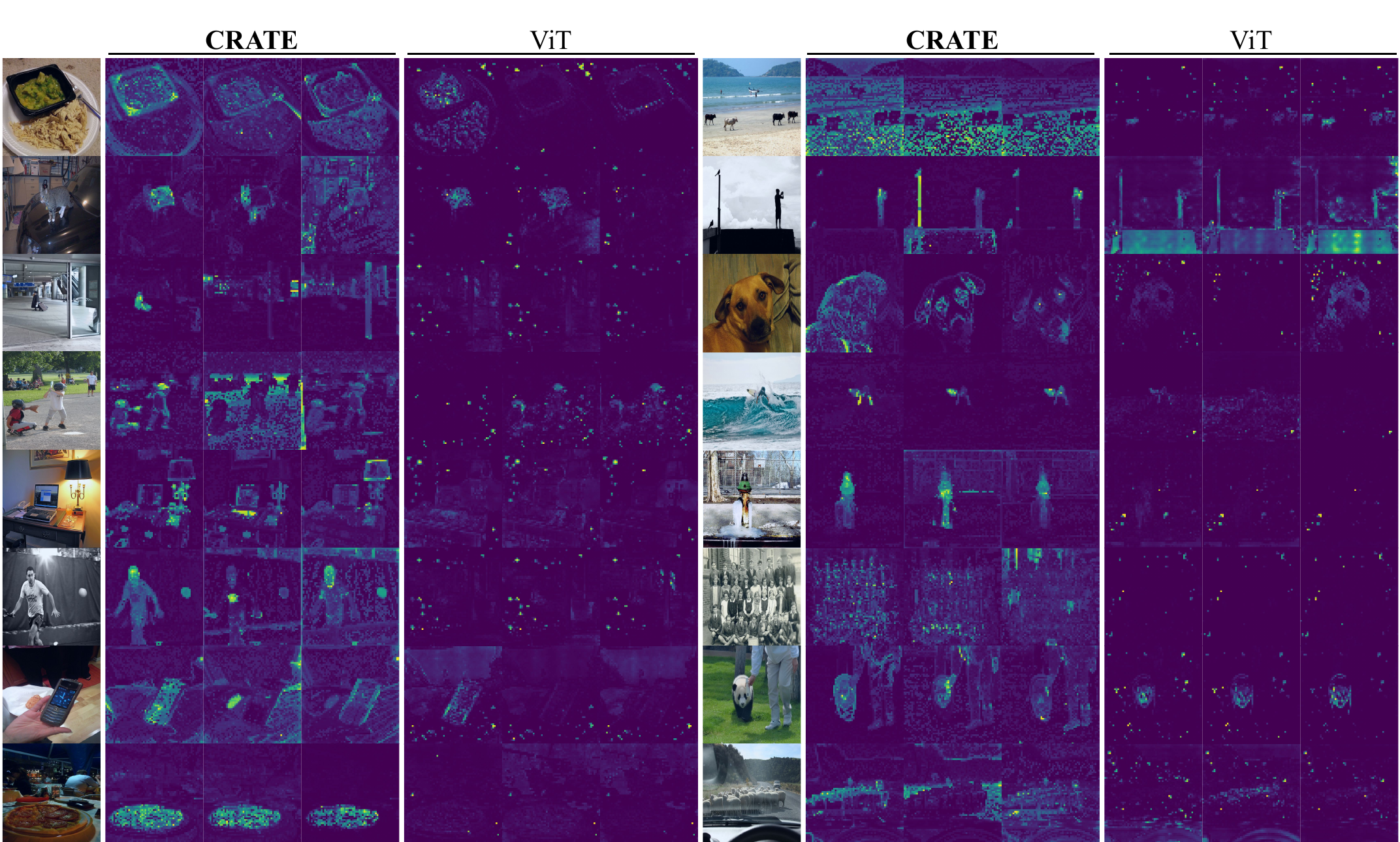}
    \caption{\textbf{More attention maps of supervised \ours{} and ViT } on images from COCO {val2017}. We select the second-to-last layer attention maps to visualize for \ours{} and the last layer for ViT.}
    \label{appendix_fig:crate_vs_vit}
\end{figure}

%% file: app_torchcode.tex
\clearpage
\section{\texttt{PyTorch} code for \ourscaps{}} 
\label{app:code}
We provide \texttt{PyTorch}-style code for implementing our proposed network architecture \ours{}, including the encoder architecture in \Cref{sec:encoding} and the decoder architecture described in \Cref{sec:autoencoding}.

\begin{itemize}
    \item In \Cref{app:experiment_pseudocode_primitives}, we provide the \texttt{PyTorch}-style pseudocode for the \texttt{MSSA} and \texttt{ISTA} blocks.
    \item In \Cref{app:experiment_pseudocode_encoder}, we provide the \texttt{PyTorch}-style pseudocode for an encoder layer \(f^{\ell}\) of the \ours{} transformer.
    \item In \Cref{app:experiment_pseudocode_decoder}, we provide the \texttt{PyTorch} style pseudocode for a decoder layer \(g^{\ell}\) of the \ours{} transformer.
    \item In \Cref{app:experiment_pseudocode_classifier}, as an example of all the top-level architectures we implement for experiments in \Cref{sec:exp}, we implement the \ours{} image classifier.
\end{itemize}

\subsection{PyTorch-Like Pseudocode for MSSA and ISTA Blocks}\label{app:experiment_pseudocode_primitives}

\begin{algorithm}
\caption{\texttt{PyTorch}-style pseudocode for \texttt{MSSA} block in transformer.}
    \begin{lstlisting}[language=Python]
class MSSA:
    # initialization
    def __init__(self, dim, heads = 8, dim_head = 64, dropout = 0.):
        inner_dim = dim_head * heads
        project_out = not (heads == 1 and dim_head == dim)
        self.heads = heads
        self.scale = dim_head ** -0.5
        self.attend = Softmax(dim = -1)
        self.dropout = Dropout(dropout)
        self.qkv = Linear(dim, inner_dim, bias=False)
        self.to_out = Sequential(Linear(inner_dim, dim), Dropout(dropout)) if project_out else nn.Identity()
        
    # forward pass
    def forward(self, x):
        w = rearrange(self.qkv(x), 'b n (h d) -> b h n d', h = self.heads)
        dots = matmul(w, w.transpose(-1, -2)) * self.scale
        attn = self.attend(dots)
        attn = self.dropout(attn)
        out = matmul(attn, w)
        out = rearrange(out, 'b h n d -> b n (h d)')
        return self.to_out(out)
    \end{lstlisting}
\end{algorithm}

\begin{algorithm} \caption{\texttt{PyTorch}-style pseudocode for \texttt{ISTA} block in transformer.}
    \begin{lstlisting}[language=Python]
class ISTA:
    # initialization
    def __init__(self, dim, dropout = 0., step_size=0.1, lambd=0.1):
        super().__init__()
        self.weight = nn.Parameter(torch.Tensor(dim, dim))
        with torch.no_grad():
            init.kaiming_uniform_(self.weight)
        self.step_size = step_size
        self.lambd = lambd
        
    # forward pass
    def forward(self, x):
        x1 = F.linear(x, self.weight, bias=None)
        grad_update = self.step_size * F.linear(x1 - x, self.weight.t(), bias=None)
        output = F.relu(x - grad_update - self.step_size * self.lambd)
        return output
    \end{lstlisting}
\end{algorithm}

\subsection{PyTorch-Like Pseudocode for \ourscaps{} Encoder}\label{app:experiment_pseudocode_encoder}

\begin{algorithm}
    \caption{\texttt{PyTorch}-style pseudocode for \ours{} encoder.}
    \begin{lstlisting}[language=Python]
class CRATEEncoder(Module):
    # initialization
    def __init__(self, dim, depth, heads, dim_head, dropout = 0.):
        self.layers = []
        self.depth = depth
        for _ in range(depth):
            self.layers.extend([LayerNorm(dim), MSSA(dim, heads, dim_head, dropout)])
            self.layers.extend([LayerNorm(dim), ISTA(dim, dropout)])
    
    # forward pass
    def forward(self, x):
        for ln1, attn, ln2, ff in self.layers:
            x_ = attn(ln1(x)) + ln1(x)
            x = ff(ln2(x_))
        return x
    \end{lstlisting}
\end{algorithm}

\subsection{PyTorch-Like Pseudocode for \ourscaps{} Decoder}\label{app:experiment_pseudocode_decoder}

\begin{algorithm}
    \caption{\texttt{PyTorch}-style pseudocode for \ours{} decoder.}
    \begin{lstlisting}[language=Python]
class CRATEDecoder(Module):
    # initialization
    def __init__(self, dim, depth, heads, dim_head, dropout = 0.):
        self.layers = []
        self.depth = depth
        for _ in range(depth):
            self.layers.extend([LayerNorm(dim), Linear(dim)])
            self.layers.extend([LayerNorm(dim), MSSA(dim, heads, dim_head, dropout)])
    
    # forward pass
    def forward(self, x):
        for ln1, lin, ln2, attn in self.layers:
            x_ = lin(ln1(x))
            x = ln2(x_) - attn(ln2(x_))
        return x
    \end{lstlisting}
\end{algorithm}

\subsection{PyTorch-Like Pseudocode for \ourscaps{} Image Classifier}\label{app:experiment_pseudocode_classifier}

\begin{algorithm}
    \caption{\texttt{PyTorch}-style pseudocode for \ours{} image classifier.}
    \begin{lstlisting}[language=Python]
class CRATEClassifier(Module):
    # initialization
    def init(self, image_size, patch_size, num_classes, dim, depth, heads, channels = 3, dim_head = 64, dropout = 0., emb_dropout = 0.):
        # define patch, image dimensions and number of patches
        image_height, image_width = pair(image_size)
        patch_height, patch_width = pair(patch_size)
        num_patches = (image_height // patch_height) * (image_width // patch_width)
        patch_dim = channels * patch_height * patch_width

        # define patch embedding, positional embedding, dropout, and transformer
        self.to_patch_embedding = Sequential(Rearrange, LayerNorm(patch_dim), Linear(patch_dim, dim), LayerNorm(dim))
        self.pos_embedding = Parameter(random(1, num_patches + 1, dim))
        self.cls_token = Parameter(random(1, 1, dim))
        self.dropout = Dropout(emb_dropout)
        self.transformer = CRATEEncoder(dim, depth, heads, dim_head, dropout)

        # define pooling, latent layer, and MLP head
        self.pool = pool
        self.to_latent = Identity()
        self.mlp_head = Sequential(LayerNorm(dim), Linear(dim, num_classes))

    # forward pass 
    def forward(self, img):
        x = self.to_patch_embedding(img)
        b, n, _ = shape(x)
        cls_tokens = repeat(self.cls_token, '1 1 d -> b 1 d', b = b)
        x = concatenate((cls_tokens, x), dim=1)
        x += self.pos_embedding[:, :(n + 1)]
        x = self.dropout(x)
        x = self.transformer(x)
        x = mean(x, dim = 1) if self.pool == 'mean' else x[:, 0]
        x = self.to_latent(x)
        return self.mlp_head(x)
    \end{lstlisting}
\end{algorithm}